\documentclass[review]{elsarticle}

\usepackage{amsmath,amsfonts,amssymb,amsthm,dsfont,stmaryrd}
\usepackage{hyperref}
\usepackage{cleveref}
\usepackage{color}

\usepackage{amsmath}
\usepackage{upgreek}
\usepackage{thmtools,thm-restate}
\usepackage{mathtools}
\usepackage[linesnumbered,titlenumbered,ruled]{algorithm2e}
\usepackage[shortlabels]{enumitem}
\usepackage{blkarray}
\usepackage{physics}
\usepackage{subcaption}
\allowdisplaybreaks
\usepackage{amssymb}
\usepackage{amsthm}
\usepackage{booktabs}
\usepackage{hyperref}
\usepackage{multirow}
\usepackage{makecell}

\usepackage{rotating}
\usepackage{array}
\usepackage{float}
\usepackage{epstopdf}
\usepackage{multicol}
\allowdisplaybreaks

\usepackage{graphicx}

\def\N{\mathbb{N}}

\def\R{\mathbb{R}}

\def\m{\mathrm{m}}

\def\F{\mathcal{F}}

\def\C{\mathcal{C}}

\def\R{\mathbb{R}}

\newcommand{\przemek}[1]{ #1  }

\def\x{\mathrm{x}}

\def\1{\mathds{1}}

\def\for{\mbox{  for }}

\def\det{\mathrm{det}}
\def\tr{\mathrm{tr}}

\numberwithin{equation}{section}

\usepackage{epstopdf}

\journal{Journal of \LaTeX\ Templates}

\begin{document}


\begin{frontmatter}

\title{Gaussian model for closed curves}

\author{Krzysztof Byrski }
\ead{krzysztof.j.byrski@gmail.com}
\address{Faculty of Mathematics and Computer Science, Jagiellonian University, Krak\'ow, Poland}

\author{Przemys\l{}aw Spurek}
\ead{przemyslaw.spurek@uj.edu.pl}
\address{Faculty of Mathematics and Computer Science, Jagiellonian University, Krak\'ow, Poland}

\author{Jacek Tabor}
\ead{jacek.tabor@uj.edu.pl}
\address{Faculty of Mathematics and Computer Science, Jagiellonian University, Krak\'ow, Poland}







\begin{abstract}
Gaussian Mixture Models (GMM) do not adapt well to curved and strongly nonlinear data. However, we can use Gaussians in the curvilinear coordinate systems to solve this problem. Moreover, such a solution allows for the adaptation of clusters to the complicated shapes defined by the family of functions. But still, it is challenging to model clusters as closed curves (e.g., circles, ellipses, etc.).

In this work, we propose a density representation of the closed curve, which can be used to detect the complicated templates in the data. For this purpose, we define a new probability distribution to model closed curves. Then we construct a mixture of such distributions and show that it can be effectively trained in the case of the one-dimensional closed curves in $\R^n$.
\end{abstract}


\end{frontmatter}


\section{Introduction}

Gaussian Mixture Models~\cite{reynolds2009gaussian} (GMM) have many applications in density estimation and data clustering. Unfortunately, the models do not adapt well to curved and strongly nonlinear data due to the intrinsic linearity of the Gaussian model. To solve this problem, we can use general Gaussian distribution in the curvilinear coordinate systems~\cite{spurek2017active,zhang2005active}.
Such methods allow for data clustering on the sub-manifolds of $\R^d$. The afCEC~\cite{spurek2017active} method uses so-called f-adapted Gaussian distributions, i.e. Gaussians which are curved over functions $f \in \C(\R^{d-1}, \R)$,  and gives very good results but still it is not able to fit the closed-curve type clusters. In practical applications data often lies on circles, ellipses, and more complicated closed curves.

The ellipse fitting is a challenging problem that arises in several fields. Examples of applications, among others, include: segmentation of cells \cite{bai2009splitting} study of galaxies \cite{soh2009new}, medical diagnostics \cite{mahdavi20093d}, camera calibration and face detection \cite{ding2010features,maio2000real}. There are as many algorithms proposing solutions to this problem as the very applications of the ellipse fitting \cite{wong2012survey}. 
But ellipses cannot describe various shapes of objects, see Fig.~\ref{teser_new}. 

\begin{figure}[t]
	\centering
	\includegraphics[width=0.95\textwidth]{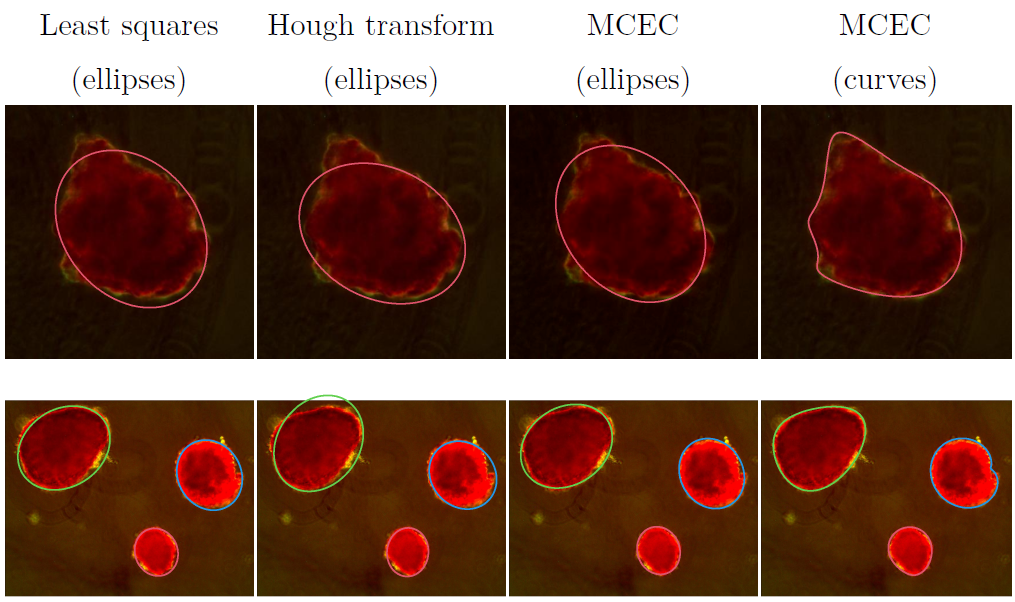}
\caption{
    \przemek{In the figure, we present the results of the Least squares (ellipses) \cite{halir1998numerically}, Hough transform (ellipses) \cite{goneid1997method}, and our algorithm MCEC for fitting the ellipses and the general closed curves on the images: Lh05-04 and Lh10-03 from the paper \cite{borovec2017supervised}. We can localize objects by using ellipses but cannot fit the shape. Using general closed curves (i.e. closed curves given by the Fourier series of order greater than 1), we can model both the object localization and the shape.
    } 
}\label{teser_new}
\end{figure}

In this paper, we propose a density representation of the closed curve, which can be used to detect the complicated templates in the data. In particular, we can approximate ellipses in 2D space, see Fig.~\ref{teser_new}. But in the more general version, we can fit to the general closed curves in $\R^2$ or $\R^n$.
The theoretical section shows the general formula for the bounded $d$-manifold embedded in $\mathbb{R}^n$, where~$d \le n$. 

The above density representation can be used for clustering. More precisely, we can construct a mixture of such distribution and show that it can be effectively trained in the case of the one-dimensional closed curve in $\R^n$. One of the most crucial benefits is the ability to reduce unnecessary clusters. 

The paper is arranged as follows. In the next section, we present the related works. In Section~\ref{Section3}, we present the CEC~\cite{spurek2014cross} algorithm, which constitutes the building block, that we use later to construct the MCEC algorithm. Then the theoretical background of the density model is presented in Section~\ref{Section4}. We start by illustrating the concept on three special cases: 2D ellipse (Subsection~\ref{Subsection4.1}), 2D closed curve (Subsection~\ref{Subsection4.2}), and the one-dimensional closed-curve in $\R^n$ (Subsection~\ref{Subsection4.3}). Finally, we show the general theoretical formula for the bounded $d$-manifold embedded in $\mathbb{R}^n$, where $d \le n$ (Subsection~\ref{Subsection4.4}).  
The last Section~\ref{Section5} is devoted to the numerical experiments showing the ability of MCEC to partition data into closed-curve type clusters.

\przemek{
Our contributions can be summarized as follows:
\begin{itemize}
\item We propose a density representation of the closed curve, which can be used to detect the complicated templates in the data.
\item We use a mixture of the above-mentioned densities in the clustering algorithm and show that it can be effectively trained in the context of one-dimensional closed curves.
\item We present the theoretical results of fitting the above-mentioned models in arbitrary spaces.
\end{itemize}
}

\section{Related works}

The key problem in the task of fitting to a shape model is the need for correspondence between the observation and the model itself. 
\przemek{Most of the papers concentrates on fitting to the ellipses. One of the most popular approach is the elliptical Hough transform \cite{goneid1997method}. The elliptical Hough transform achieves high-precision fitting by using a voting mechanism in a five-dimensional parameter space but is numerically heavy. To solve such a problem, one can use more computationally efficient methods based on the least
squares approach \cite{ahn2001least,he2011comparative,fitzgibbon1999direct,halir1998numerically,gander1994least,rosin1993note,barwick2008very,maini2006enhanced,rong2015ellipse}. These methods typically work well for one ellipse. For many ellipses fitting, one can use  robust ellipse fitting methods  \cite{fischler1981random,liang2015robust,hu2021robust,zhao2021robust,wang2022robust}. For instance, (RANSAC)~\cite{fischler1981random} proposed a random sample consensus method. Ellipses can be fit by estimating a mathematical model from collecting a random subset of the entire data. 
Alternatively, one can use the maximum correntropy criterion (MCC), which uses Gaussian and Laplacian kernels \cite{liang2015robust,hu2021robust,wang2022robust}.} 
Unfortunately, it is not trivial to generalize such solutions to more complicated curves.

There exist a few methods for fitting to the general closed curves. 
In \cite{szekely1996segmentation}, authors use Fourier descriptors. However, those approaches are sensitive to outliers, and in practice, manual labeling is often required to set the initial correspondences.
Alternatively, the Mean Shift algorithm\cite{arellano2012shape2,arellano2012shape1} has been presented as a method for arbitrary fixed shapes detection. The observed data and the shape model are represented as mixtures of Gaussians. Using a Bayesian framework, authors leverage the distance between the Gaussian mixture density model and data represented by kernel density estimation. The method is also helpful in the context of ellipses detection \cite{arellano2016robust,martorell2021multiscale,wang2022robust}. 

\przemek{Our solution also uses a mixture of densities but we describe a shape by a closed curve parameterized by trigonometric functions. Such a solution allows us to approximate the complicated shapes in the effective way. We use our model in the cross-entropy clustering framework. Consequently, we obtain a clustering method that finds closed-curve-type clusters.}

\przemek{The problem of curve fitting in 2D and 3D has many applications in biology. In \cite{arbabi2023statistical} authors propose a statistical shape model of the talus bone morphology. In \cite{shi2023new} authors use the 2D parametric curve to describe birds' egg shape. Individual tree species classification is strategically important for forest monitoring, analysis and management critical for sustainable forestry development. In the paper \cite{qian2023tree}, authors propose an algorithm for classification that uses individual tree segmentation and shape fitting. In \cite{zhang2023fast}, authors use shape fitting to estimate irregular asteroids' gravitational acceleration. 
In the paper \cite{toth2023minimal}, the authors propose a minimal solver for sphere fitting via its 2D central projection, i.e., a particular ellipse. In \cite{abeyrathna2022morphological}, authors used ellipse fitting to high-resolution images of bacterial cells. 
}

\section{Mixture of distributions}
\label{Section3}

In our model, we use a variant of Gaussian mixture model (GMM) called CEC. The main benefit of CEC over standard GMM is that it can automatically reduce the number of clusters. 
In Fig.~\ref{fig:reduce}, we present the process of removing some components.

\begin{figure}[H]
	\centering
	\begin{tabular}{c@{}c@{}c@{}c@{}c@{}}
		\includegraphics[width=0.25\textwidth]{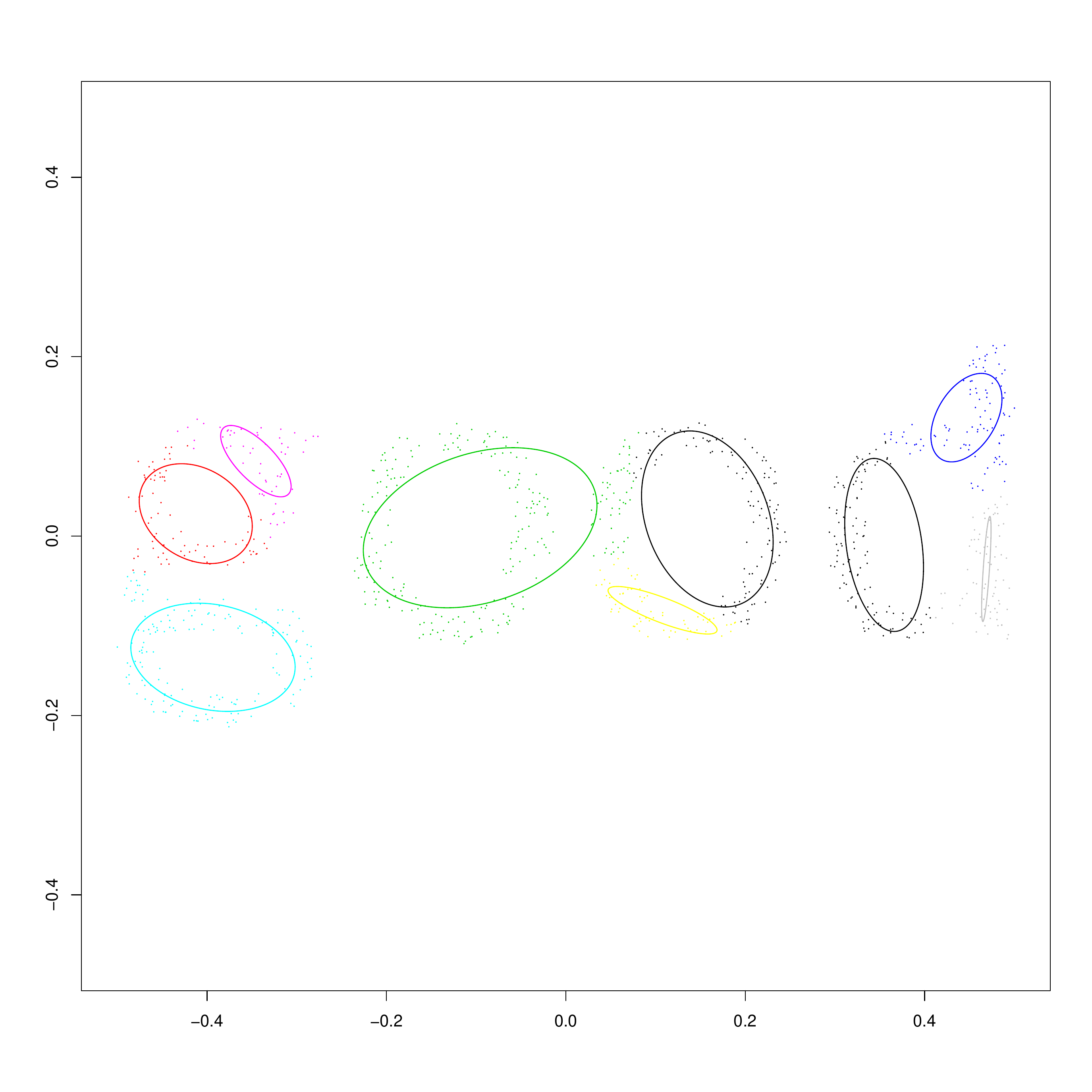} &
		\includegraphics[width=0.25\textwidth]{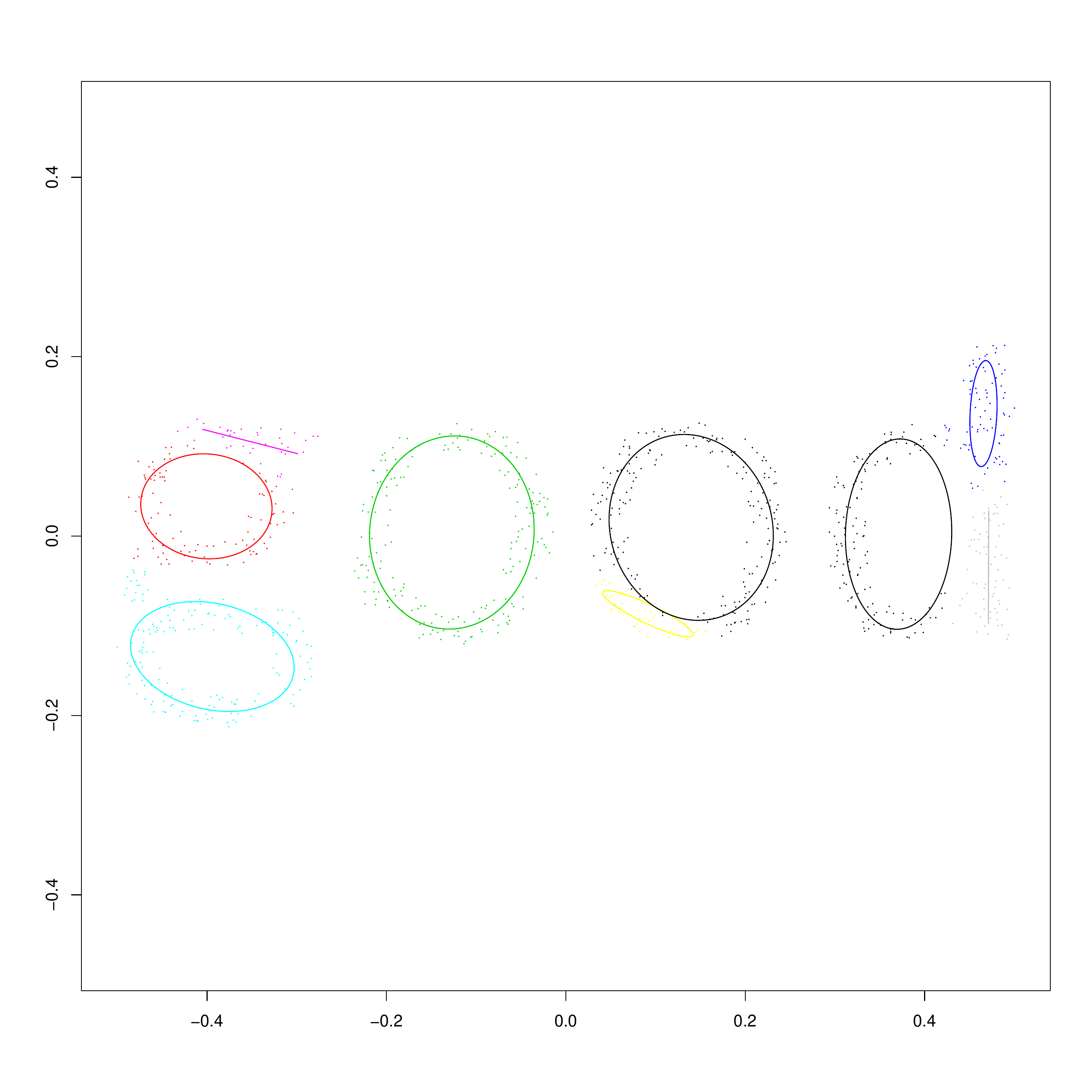} &
		\includegraphics[width=0.25\textwidth]{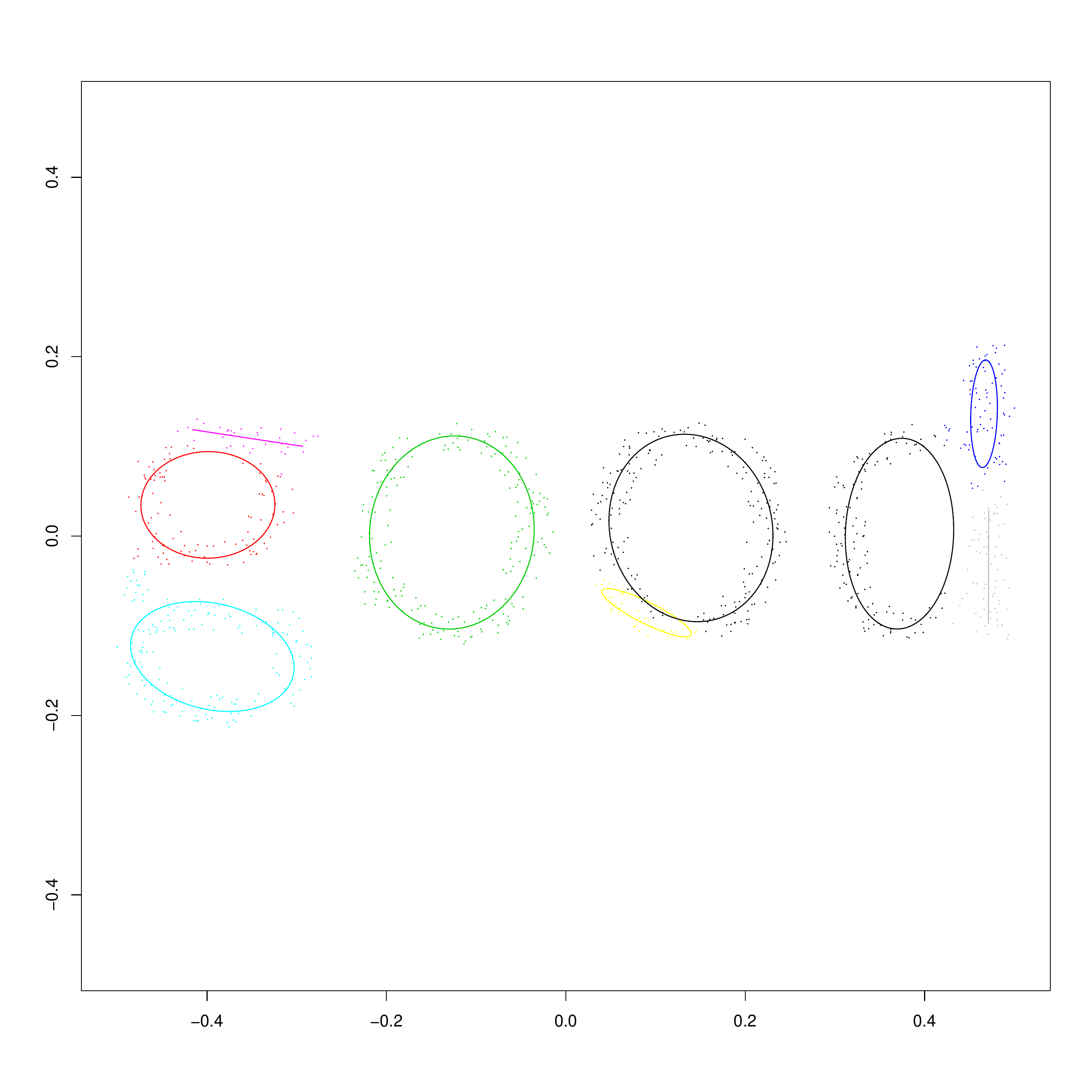} &
		\includegraphics[width=0.25\textwidth]{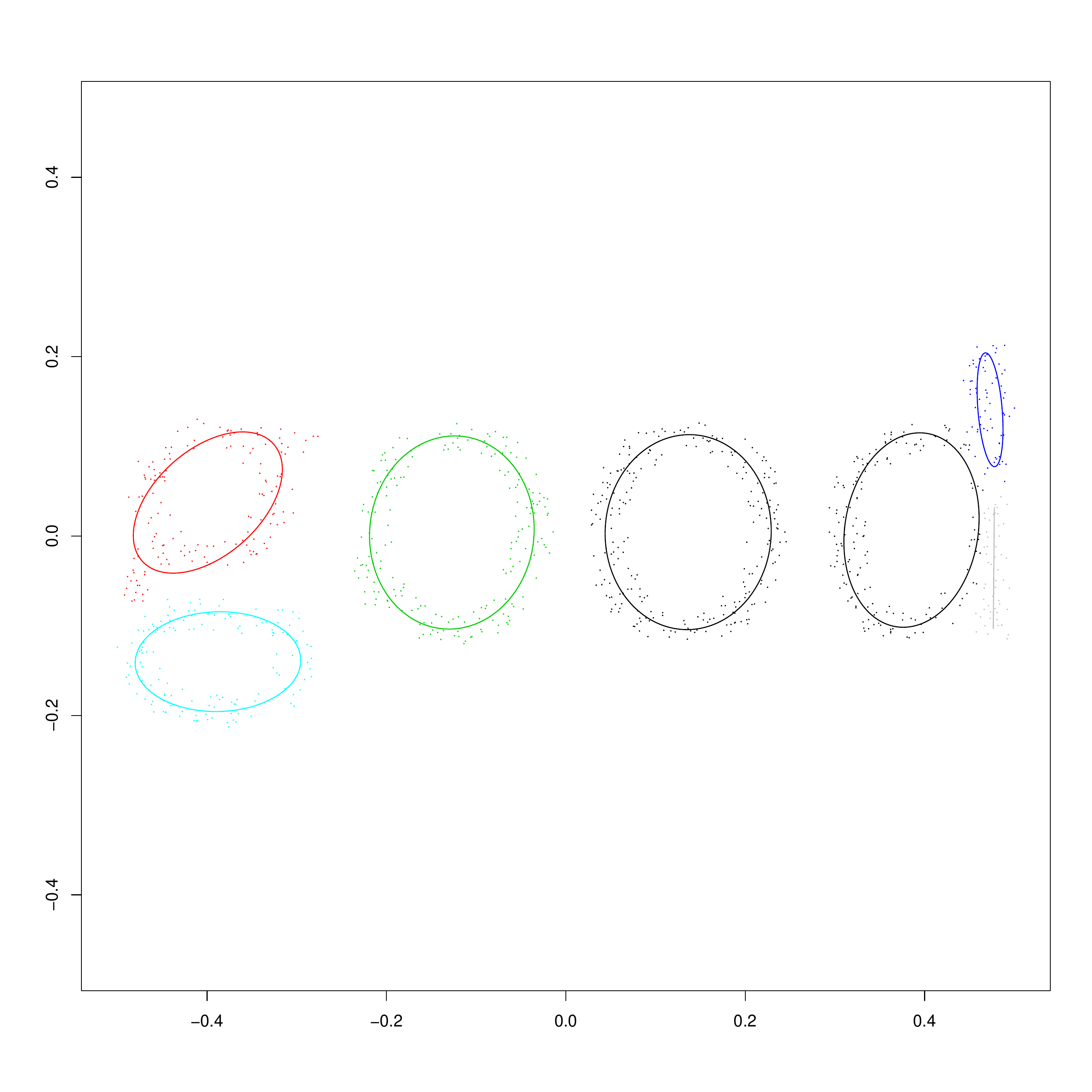} \\
		5'th iteration & 6'th iteration & 7'th iteration & 8'th iteration \\
	\end{tabular}
	\caption{Convergence of MCEC, approximating the text "good" with seven closed-curve Gaussian components gradually reducing from 9 initial clusters. In practice, pink and yellow clusters were reduced.  }
	\label{fig:reduce}
\end{figure}

In general, EM aims at finding
\begin{equation}
p_1, \ldots, p_k \geq 0 \mbox{, such that } \sum_{i=1}^k p_i =1
\end{equation}
and $f_1,\ldots,f_k \in \F$, where $\F$ is a fixed (usually Gaussian) family of densities such that the convex combination
\begin{equation}
    f = p_1f_1 + \ldots + p_kf_k,    
\end{equation}
optimally approximates the data.
The optimization is taken with respect to an MLE-based cost function:
\begin{equation}\label{eq:gmm}
EM(f,X) = - \frac{1}{|X|} \sum_{j=1}^{n} \ln(f(x_j)) = - \frac{1}{|X|} \sum_{j=1}^{n} \ln( p_1{f_1} + \ldots + p_k{f_k} ),
\end{equation}
where $|X|$  denotes the cardinality of a set $X$. The optimization in EM consists of the Expectation and Maximization steps. While the Expectation step is relatively simple, the Maximization usually (except for the simplest case when the family $\F$ denotes all Gaussian densities) needs a complicated numerical optimization.

The goal of CEC (cross-entropy clustering)
\cite{spurek2014cross} is similar, i.e., it aims at minimizing the cost function 
\begin{equation}\label{eq:cec}
CEC(f,X) = - \frac{1}{|X|} \sum_{j=1}^{n} \ln\left(  \max( p_1{f_1}, \ldots, p_k{f_k}) \right),
\end{equation}
where $p_1, \ldots, p_k \geq 0$ and $\sum_{i=1}^k p_i =1$.
Contrary to EM, minimizing the CEC cost function can often lead to reducing the number of clusters (components). In other words, after training some probabilities $p_i$ will equal zero, which implies that the clusters have disappeared. Thus $k$, contrary to the case of EM, does not denote the final number of clusters but is only its upper bound. Instead of focusing on the density estimation, CEC concentrates on the clustering, where, in the fashion similar to EM, the point $x$ is assigned to the cluster maximizing $p_if_i(x)$. The seemingly small difference in the cost function between GMM and CEC has profound further consequences, which follow from the fact that the densities in CEC do not “cooperate” to build the final approximation of $f$.
In the paper, we use CEC instead of GMM.

We provide a detailed description of the MCEC algorithm in Appendix B~\ref{AppendixB}. Since, unlike the CEC and afCEC algorithms (see Fig.~\ref{fig:teser}), which use the mixture of the classical multivariate Gaussian and the f-adapted multivariate Gaussian distributions respectively, in general, there are no closed-form formulas for the estimated values of the parameters of the probability distribution model used by the MCEC, in our algorithm we use the BFGS gradient optimizer to obtain the estimated values of the variance and the Fourier series coefficients of the curve, that defines the Gaussian distribution (See Subsection ~\ref{Subsection4.3}), which minimizes the value of the cross-entropy $H^{x}(X\|f_{\mathcal{N}(M_{\upvarphi},\sigma^2)})$ for some cluster $X$.

\begin{figure}[H]
	\centering
	\begin{tabular}{@{ }c@{ }c@{ }c@{ }c@{ }}
		AfCEC  & MCEC & AfCEC  & MCEC  \\		
\includegraphics[width=0.23\textwidth, trim=0 0 0 0, clip]{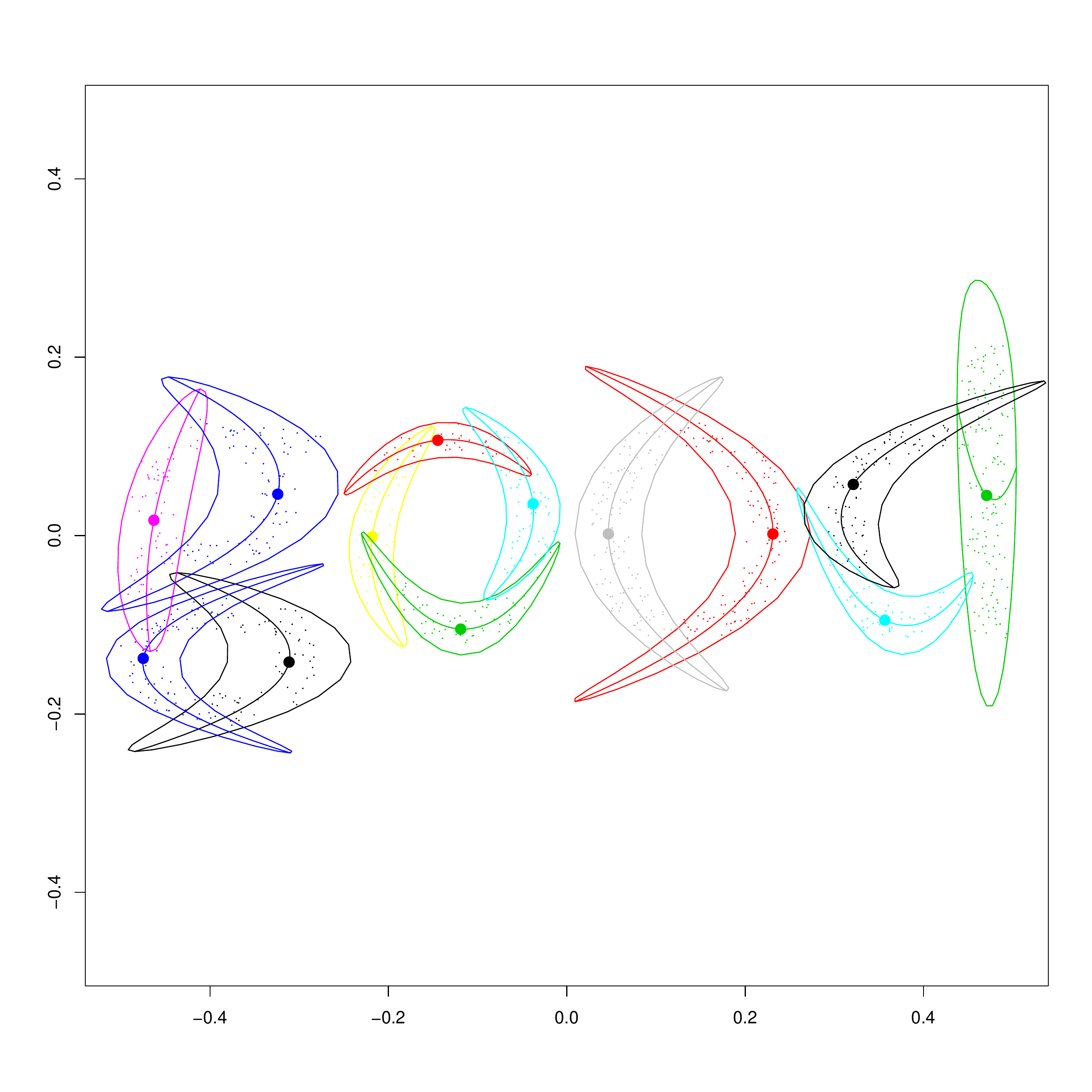} &
\includegraphics[width=0.23\textwidth, trim=0 0 0 0, clip]{16.pdf}  	  &
\includegraphics[width=0.23\textwidth, trim=0 0 0 0, clip]{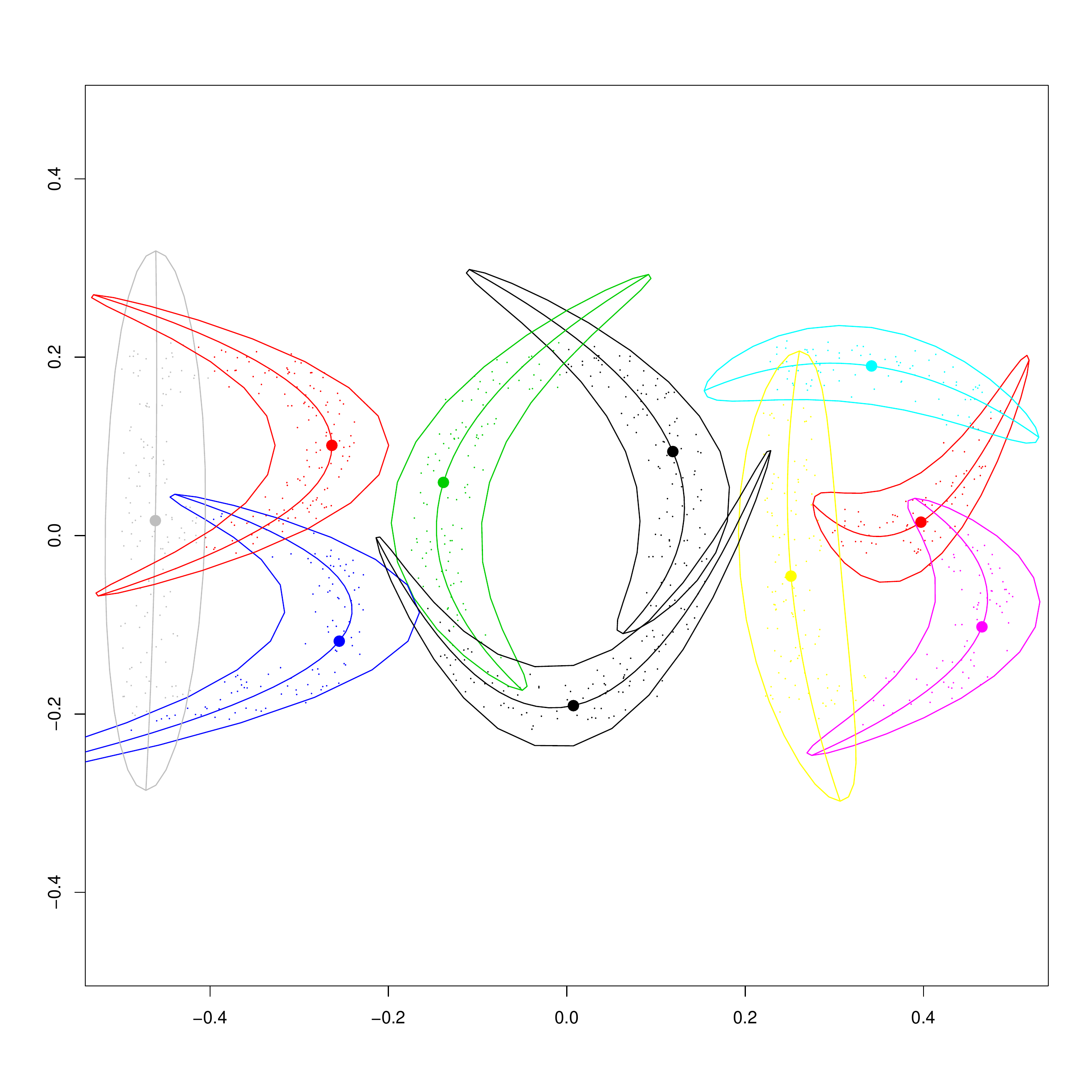}  &
  \includegraphics[width=0.23\textwidth, trim=0 0 0 0, clip]{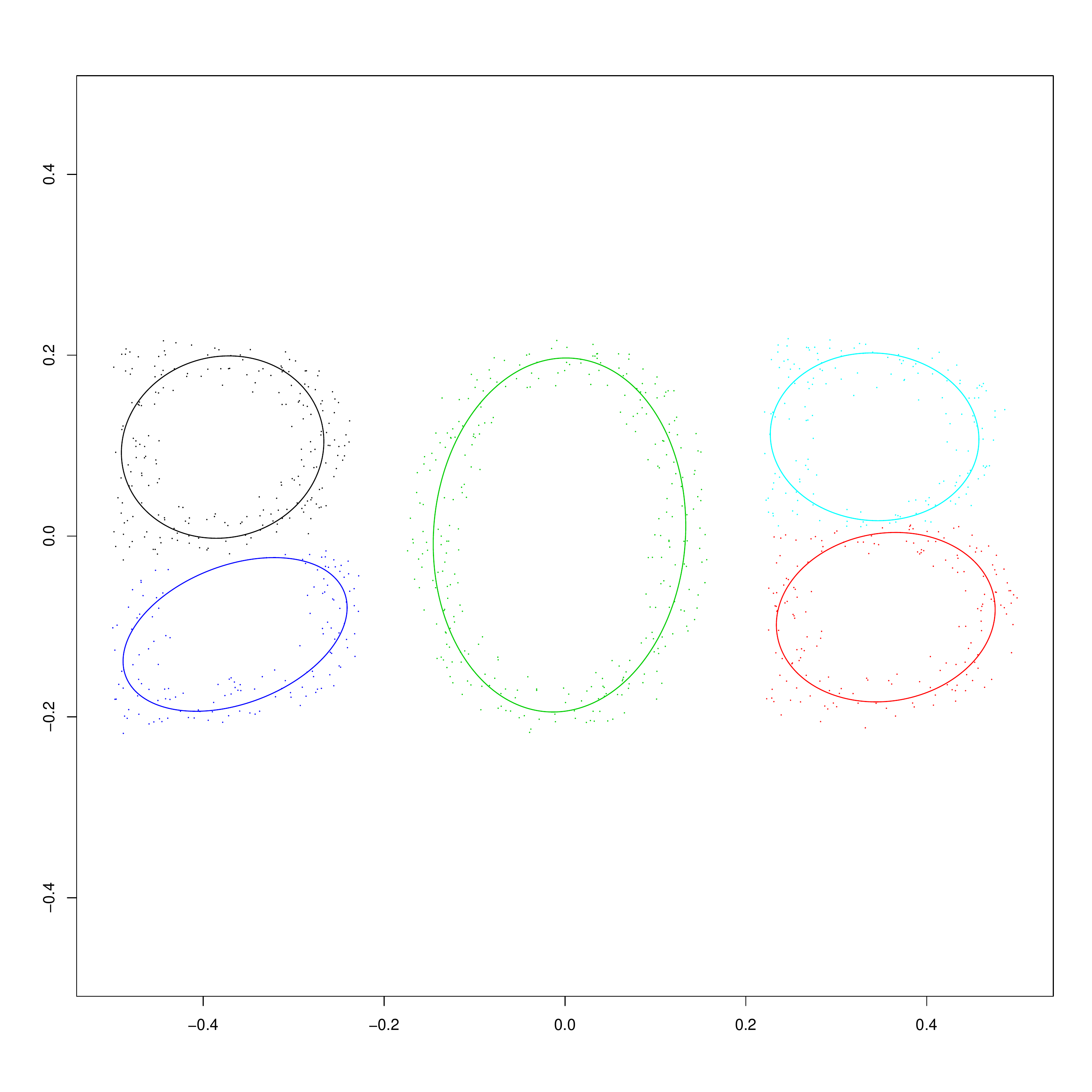}
\\
	\end{tabular}
	\caption{Comparison of the method describing objects using the f-adapted Gaussians (afCEC) with our approach (MCEC), which uses closed curves. Let's note that we can use fewer parameters to approximate the data using the closed-curve type clusters.}
	\label{fig:teser}
\end{figure}

\section{Density model}
\label{Section4}

This section presents a density model that we use in the CEC algorithm.
We focus on the Gaussian distributions on closed curves. This approach aims at transforming a normal distribution (which assumes the intrinsic linearity of the model) into the case of manifolds given by the closed curves. The above model is used in the MCEC method as a representation of each cluster.

\subsection{Toy example: ellipses in $\R^2$}
\label{Subsection4.1}

We begin with an illustration of our idea in the two-dimensional case of the simple ellipse shape.
We need two main components -- Gaussian distribution and parametrization of ellipse shape. 
Let us recall that a two-dimensional Gaussian density with mean $\m$
and covariance matrix $\Sigma = \sigma^2 \mathrm{I}$ is given by the formula~(\ref{norfmal_den}):

\begin{equation}\label{norfmal_den}
f_{\mathcal{N}(\m,\sigma^2 \mathrm{I})}(\x)=\frac{1}{\sqrt{\det(2\pi\sigma^2 \mathrm{I})}} e^{-\frac{1}{2} (\x-\m)^T(\sigma^2 \mathrm{I})^{-1}(\x-\m) } \for \x \in \R^2,
\end{equation}
Ellipse can be parametrized by the trigonometric functions
\begin{equation}\label{elipses_parametri}
\upvarphi:[0,2\pi) \ni s \to ( a_0 +a_1 \cos(s), b_0 +b_1 \cos(s)) \in \R^2.
\end{equation}

Now we can define the probability density function of the Multivariate Normal Distribution with mean $\m=\upvarphi(s)$ and the covariance matrix $\Sigma=\sigma^2 \mathrm{I}$:


\begin{flalign}\label{elipse_dens}
f_{\mathcal{N}(\upvarphi(s),\Sigma)}(\x) &= \frac{1}{\sqrt{ \det( 2 \pi \sigma^2 \mathrm{I} )}} e^{ -\frac{1}{2} (\mathrm{x}-\upvarphi(s))^{\top} (\sigma^2 \mathrm{I})^{-1} (\mathrm{x}-\upvarphi(s)) } =&\\
&= \frac{1}{\sqrt{(2 \pi)^2 \sigma^{4}}} e^{ -\frac{1}{2 \sigma^2} {\Vert \mathrm{x}-\upvarphi(s) \Vert}^2 } \notag
\end{flalign}
\newline
Density given by equation (\ref{elipse_dens}) can be understood as joint distribution.
We aim at constructing the probability density distribution on the ellipse parameterized by $\upvarphi$ with variance $\sigma^2$. We can calculate it by taking the marginal distribution of  $f_{\mathcal{N}(\upvarphi(s),\Sigma)}$ given by equation (\ref{elipse_dens}).

\begin{restatable}{definition}{definition1}
\label{definition11}
By Gaussian distribution on the  ellipse parameterized by $\upvarphi$ with variance $\sigma^2$ we understand the marginal distribution
\begin{equation}\label{eq:margin}
f_{\mathcal{N}(M_{\upvarphi},\sigma^2)}(\x) =  \int \limits_{[0,1]} \frac{1}{\sqrt{(2 \pi)^2 \sigma^{4}}} e^{ -\frac{1}{2 \sigma^2} {\Vert \x-\upvarphi(2 \pi s) \Vert}^2 } \dd{s}
\end{equation}
\end{restatable}

\begin{figure}[h]
    \begin{center}
        \includegraphics[width=5cm, trim=70 350 50 330, clip]{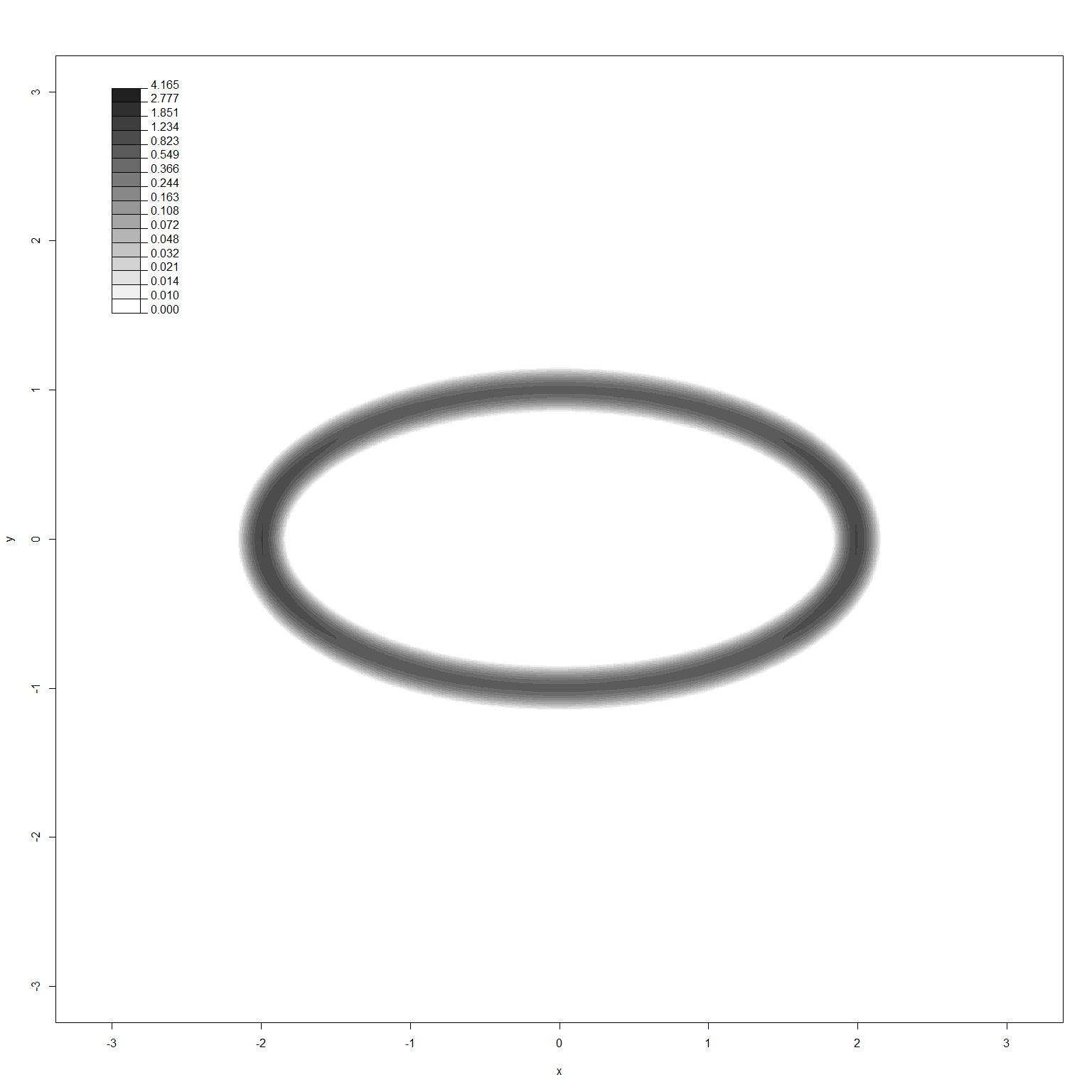}
    	\includegraphics[width=5cm, trim=70 350 50 330, clip]{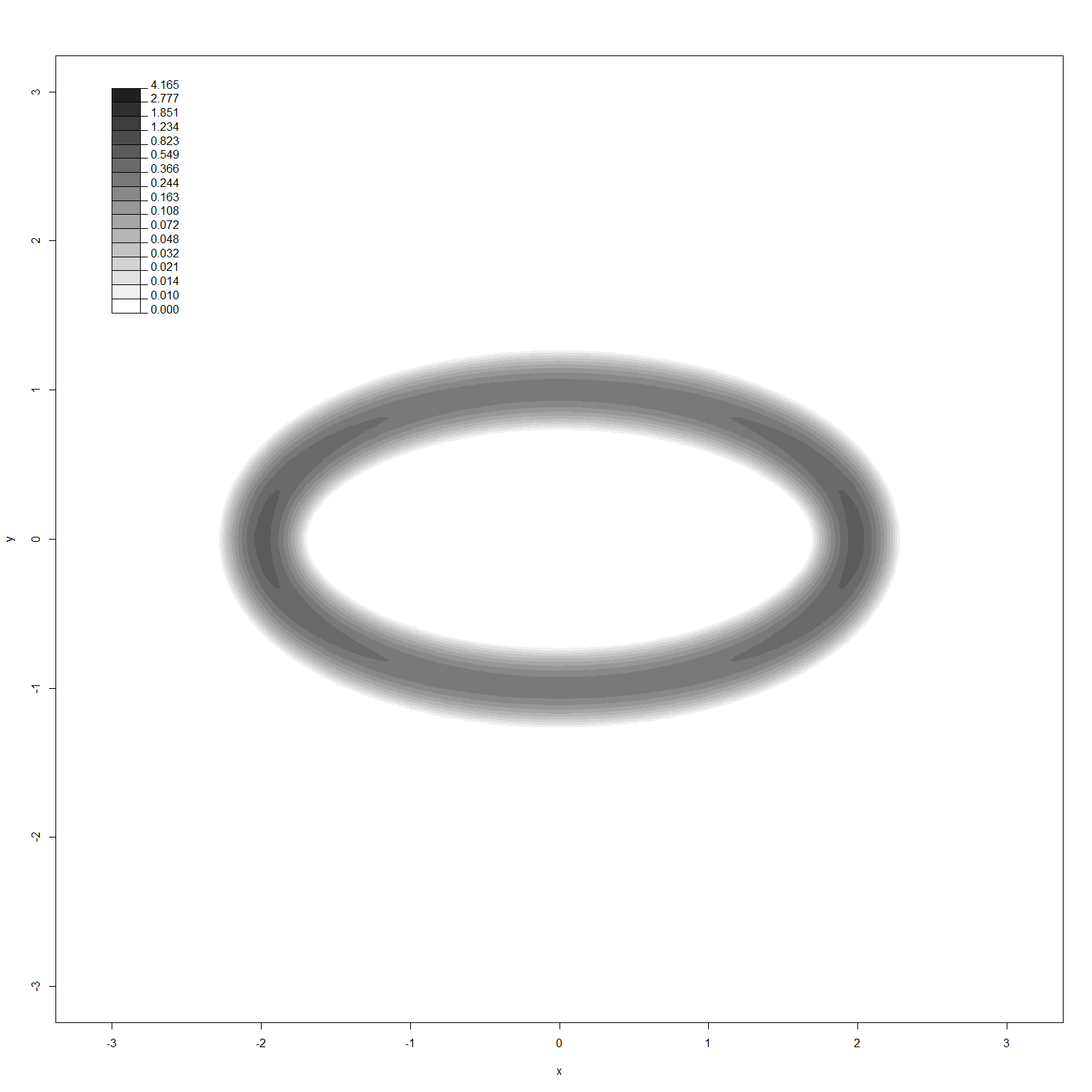}
	\end{center}
	
	\caption{\przemek{In the figure, we present two-dimensional level sets of the probability distribution function given by Definition~\ref{definition11}. We model 2D ellipses with parameters $\sigma=0.05$ (left) and $\sigma=0.1$ (right). The parameter sigma allows modeling ellipses with different levels of noise.
	The integral was solved numerically.}
	}
	\label{fig:den_1}
\end{figure}

The probability distribution given by Definition~\ref{definition11} is in fact the convolution of the ellipse and the Gaussian distribution. In practice, in simple words, in each point of the ellipse, we put the Gaussian noise, see Fig.~\ref{fig:den_1}.     

Unfortunately, the above integral is difficult to calculate. 
More precisely, it cannot be expressed in terms of the elementary functions\footnote{ Only in the case of a circle given by the equation:
$\upvarphi(\mathrm{s})=(r \cos(2 \pi \mathrm{s}_1),r \sin(2 \pi \mathrm{s}_1))$, where $r \in \mathbb{R}_{+}$
it can be expressed by Modified Bessel Function of the First Kind.}.
Instead of calculating the exact formula, we can use the approximation.
We divide an ellipse into $k \in \N_{+}$ equal parts and approximate the integral by sum, see equation (\ref{eq:margin_app1}):

\begin{equation}\label{eq:margin_app1}
f_{\mathcal{N}(M_{\upvarphi},\sigma^2)}(\mathrm{x})  \approx  \frac{1}{k} \sum_{i=0}^{k-1}   \frac{1}{\sqrt{(2 \pi)^2 \sigma^{4}}} e^{ -\frac{1}{2 \sigma^2} {\Vert \x-\upvarphi( 2 \pi i/k ) \Vert}^2 }
\end{equation}

\begin{figure}[h]
    \begin{center}
        \includegraphics[width=5cm, trim=70 350 50 330, clip]{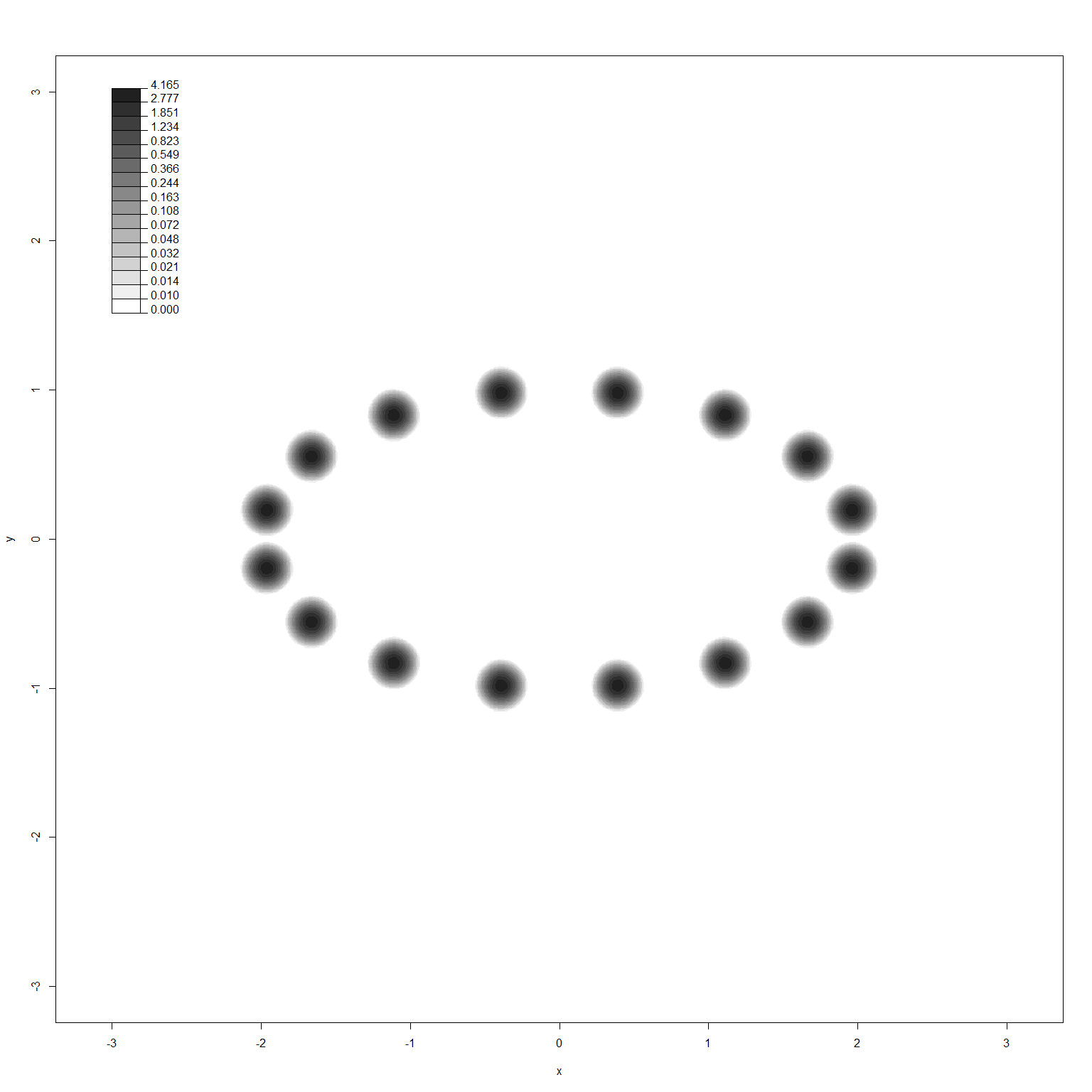}
    	\includegraphics[width=5cm, trim=70 350 50 330, clip]{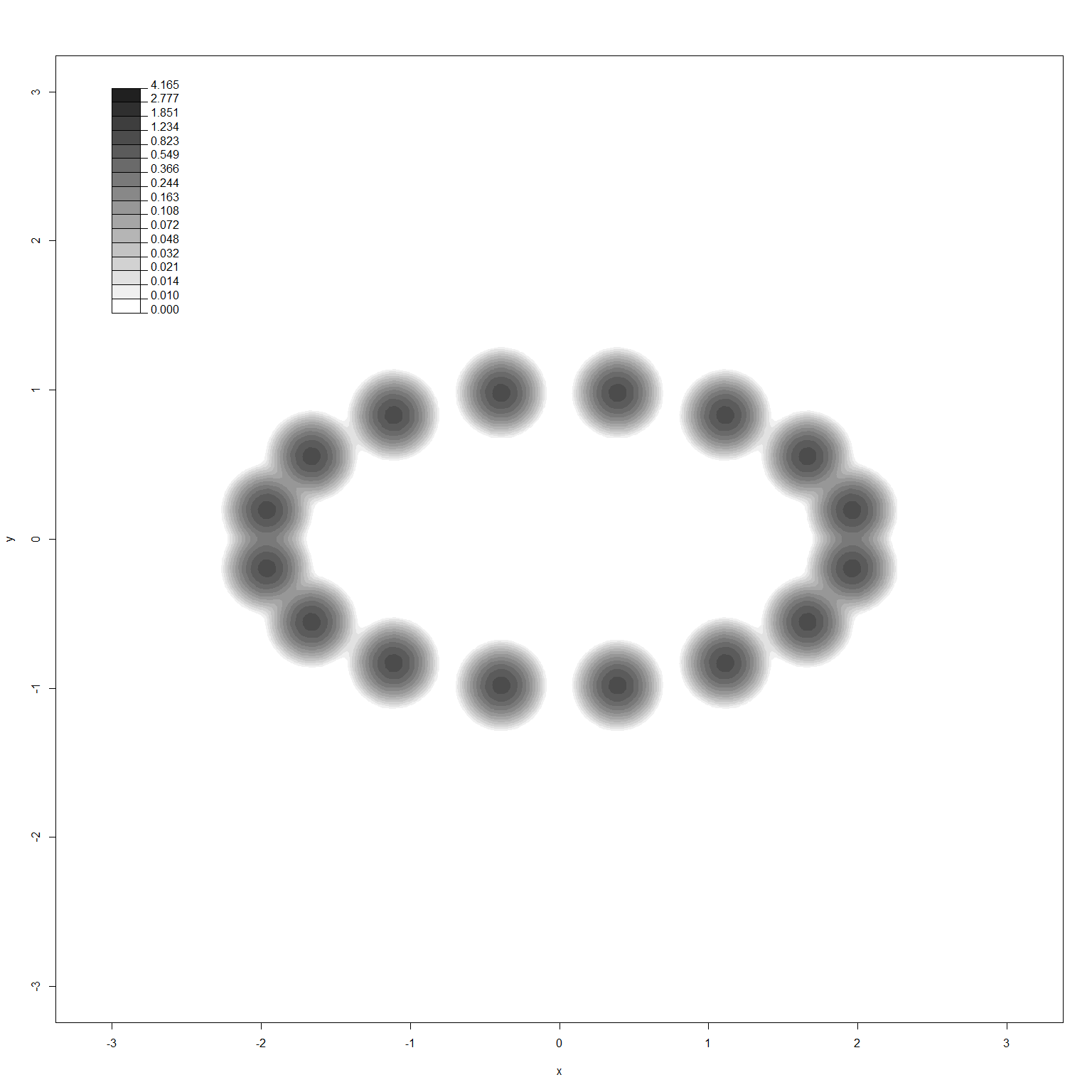}
	\end{center}
	
	\caption{\przemek{In the figure, we present an approximation of ellipses distribution by  formula (\ref{eq:margin_app1}) with $k=16$ and $\sigma=0.05$ (left) and $\sigma=0.01$ (right). Unfortunately, such a solution describes the shape of the ellipses only locally and tends to produce missing parts in a probability distribution.}
	}
	\label{fig:den_2}
\end{figure}

Unfortunately, such a solution describes the shape of the ellipses only locally and tends to over-fit the points. Consequently, 
informally speaking, there are the "holes" in our probability distribution function, see Fig.~\ref{fig:den_2}. 

In order to solve this problem, we can divide our curve into a collection of equal size segments and construct a chain of the Gaussian distributions. It means that we effectively approximate each interval by the two-dimensional Gaussian distribution $N(m_i,\Sigma_i)$ for $i=0,\ldots, {k-1}$

\begin{equation}\label{eq:margin_app2}
f_{\mathcal{N}(M_{\upvarphi},\sigma^2)}(\x)  \approx  
\frac{1}{k} \sum_{i=0}^{k-1} f_{N(\m_i,\Sigma_i)}(\x)  
\end{equation}

where
\newline
$$
\m_i = k\int_{i/k}^{(i+1)/k}\upvarphi(2 \pi s) ds
$$
$$
\Sigma_i = {\sigma^2}\mathrm{I} + k\int_{i/k}^{(i+1)/k} (\upvarphi(2 \pi s)-\m_i)(\upvarphi(2 \pi s)-\m_i) ds
$$
\newline
Parameters $m_i$ and $\sigma_i$ for $i=0,\ldots, {k-1}$  optimally approximate the part of the ellipse.

Since the end of one interval is the beginning of the next, we never get empty spaces between components, see Fig.~\ref{fig:den_3}. Consequently, we do not have an overfitting effect like in the case of the previous model.

\begin{figure}[h]
    \begin{center}
        \includegraphics[width=5cm, trim=70 350 50 330, clip]{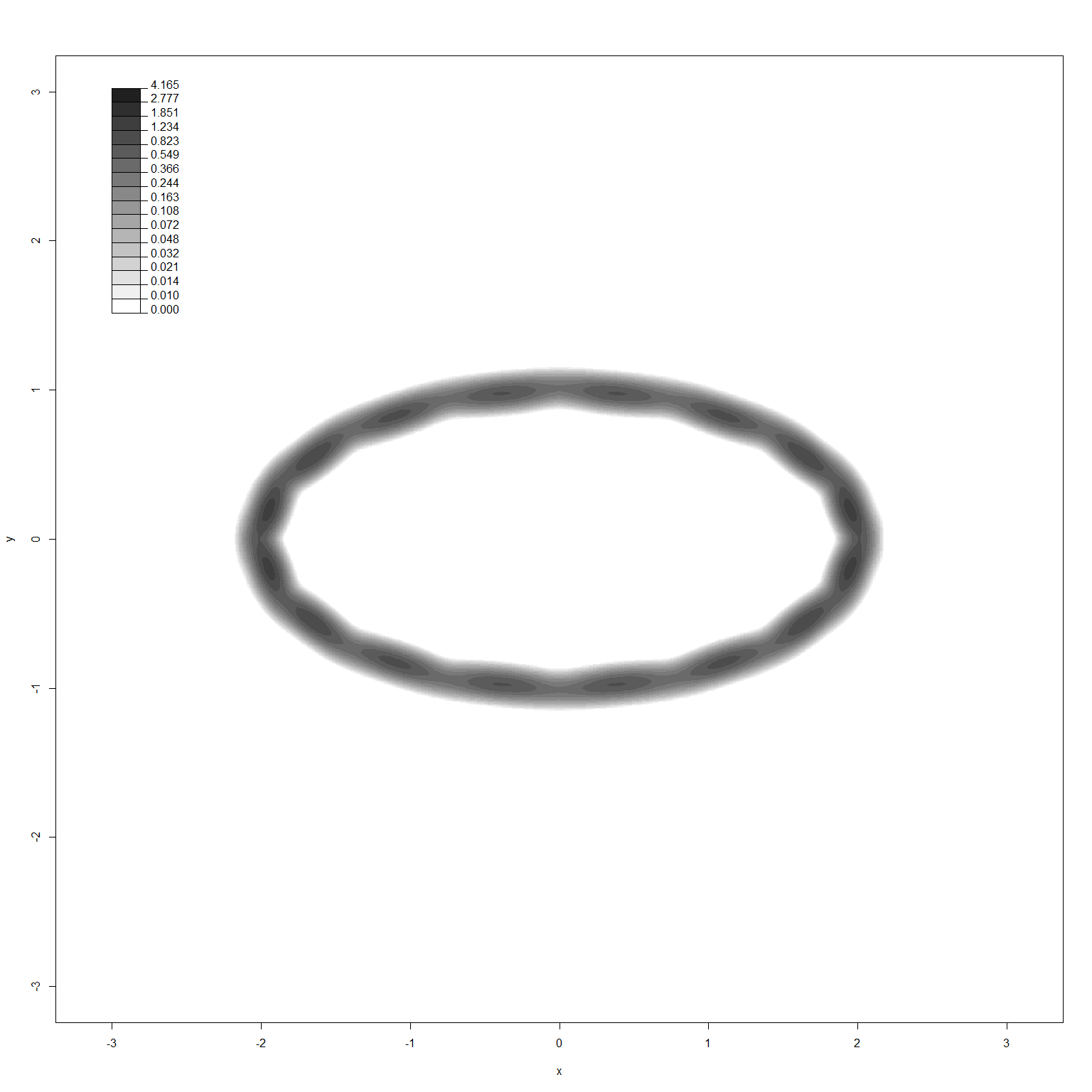}
    	\includegraphics[width=5cm, trim=70 350 50 330, clip]{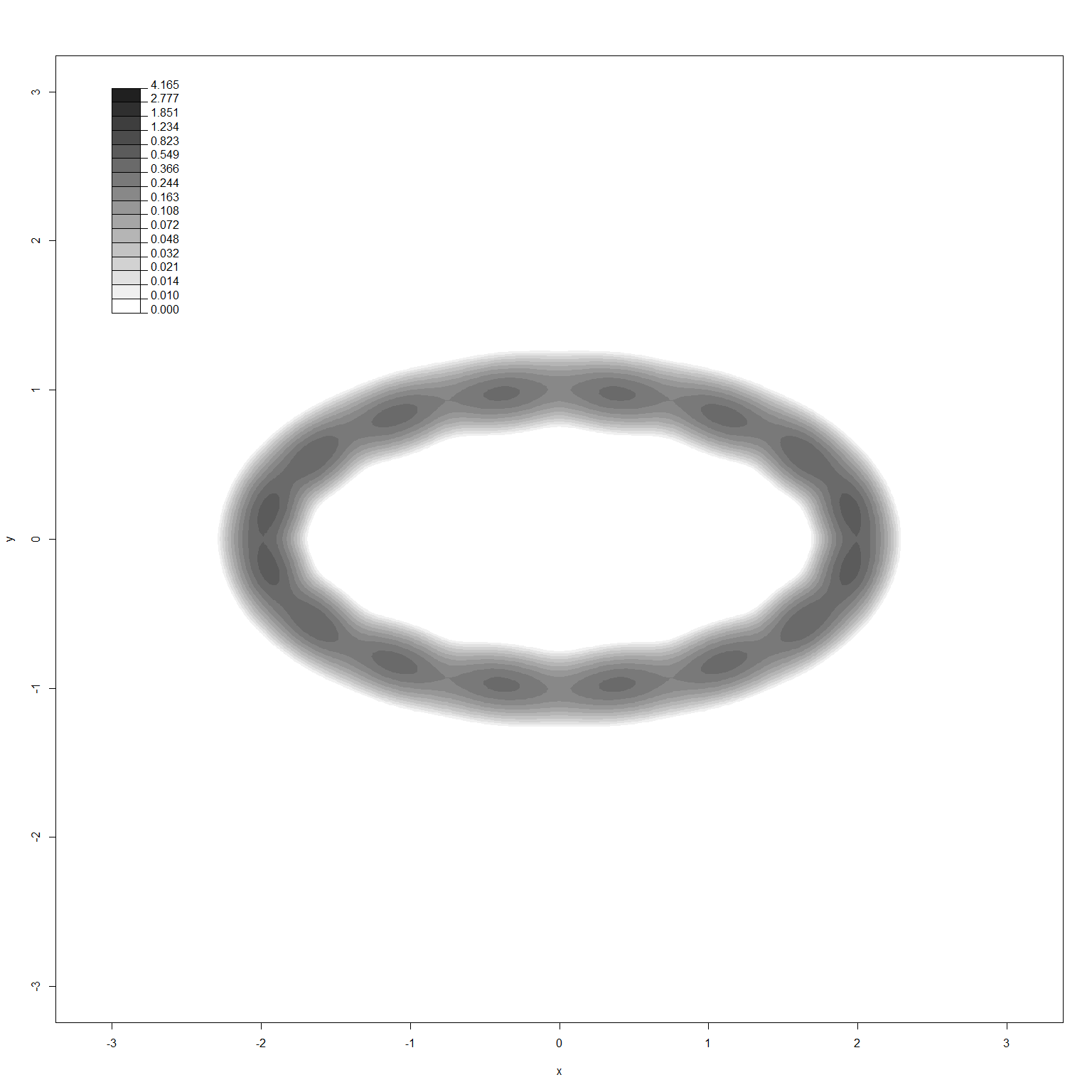}
	\end{center}
	
	\caption{\przemek{In the figure, we present an approximation of ellipses distribution by formula~(\ref{eq:margin_app2}) for $k=16$ with $\sigma=0.05$ (left) and $\sigma=0.01$ (right). As we can see, such a solution adjusts the ellipse with high accuracy.}
	}
	\label{fig:den_3}
\end{figure}

\subsection{Arbitrary  shape in $\R^2$}
\label{Subsection4.2}

We want to model different shapes defined by closed curves. In practice, we consider curves given by the equation (\ref{2d_parametri}):

\begin{equation}\label{2d_parametri}
    \upvarphi(s):[0,1] \mapsto   \left( \upvarphi^{(1)}(s), \upvarphi^{(2)}(s) \right) \in \mathbb{R}^n.
\end{equation}

We aim to construct a probability distribution on the curve parameterized by $\upvarphi$ with variance $\sigma^2$. By analogy with Definition~\ref{definition11} we obtain the Gaussian distribution on the curve, which probability density function is given by the formula~(\ref{eq:margin}).

\begin{restatable}{definition}{definition2}
\label{definition2}
Gaussian distribution on the curve parameterized by $\upvarphi$ with variance $\sigma^2$ is defined by the marginal distribution:
\begin{equation}\label{eq:margin}
f_{\mathcal{N}(M_{\upvarphi},\sigma^2)}(\x) =  \int \limits_{[0,1]} \frac{1}{\sqrt{(2 \pi)^2 \sigma^{4}}} e^{ -\frac{1}{2 \sigma^2} {\Vert \x-\upvarphi(s) \Vert}^2 } \dd{s}
\end{equation}
\end{restatable}

Analogically to the case of the ellipses from the previous section, we can approximate such a formula by the Gaussian distribution $N(m_i,\Sigma_i)$ for $i=0,\ldots, k-1$

\begin{equation}\label{eq:margin_app3}
f_{\mathcal{N}(M_{\upvarphi},\sigma^2)}(\x)  \approx  
\frac{1}{k} \sum_{i=0}^{k-1} f_{N(\m_i,\Sigma_i)}(\x)  
\end{equation}

where 
$$
\m_i = k\int_{i/k}^{(i+1)/k}\upvarphi(s) ds
$$
$$
\Sigma_i = {\sigma^2}\mathrm{I} + k\int_{i/k}^{(i+1)/k} (\upvarphi(s)-\m_i)(\upvarphi(s)-\m_i) ds
$$
\newline
Parameters $m_i$ and $\sigma_i$ for $i=0,\ldots, k-1$ optimally approximate the part of the curve.

To visualize such approximation let us consider 2D rabbit shape, see Fig.~\ref{fig:rabbit1}, given by the curve $\upvarphi(s) = \left( \upvarphi^{(1)}(s), \upvarphi^{(2)}(s) \right)$, where $\upvarphi^{(1)}(s)$ and $\upvarphi^{(2)}(s)$ are given by the following formulas:

\begin{flalign}
x = \phi^{(1)}(\mathrm{s}_1) &= 1 \cdot \cos(2 \pi \mathrm{s}_1) + 0.5 \cdot \sin(2 \pi \mathrm{s}_1) +& \\
&+ 0.5 \cdot \cos(4 \pi \mathrm{s}_1) + 0.25 \cdot \sin(4 \pi \mathrm{s}_1) +& \notag\\
&+ 0 \cdot \cos(6 \pi \mathrm{s}_1) + 0 \cdot \sin(6 \pi \mathrm{s}_1) -& \notag\\
&- 0.125 \cdot \cos(8 \pi \mathrm{s}_1) + 0.25 \cdot \sin(8 \pi \mathrm{s}_1) +& \notag\\
&+ 0.125 \cdot \cos(10 \pi \mathrm{s}_1) - 0.125 \cdot \sin(10 \pi \mathrm{s}_1) \notag
\end{flalign}

\noindent
and

\begin{flalign}
y = \phi^{(2)}(\mathrm{s}_1) &= 0.25 \cdot \cos(2 \pi \mathrm{s}_1) + 1 \cdot \sin(2 \pi \mathrm{s}_1) +&\\
&+ 0 \cdot \cos(4 \pi \mathrm{s}_1) + 0.5 \cdot \sin(4 \pi \mathrm{s}_1) -&\notag\\
&- 0.125 \cdot \cos(6 \pi \mathrm{s}_1) + 0.25 \cdot \sin(6 \pi \mathrm{s}_1) +&\notag\\
&+ 0 \cdot \cos(8 \pi \mathrm{s}_1) + 0 \cdot \sin(8 \pi \mathrm{s}_1) +&\notag\\
&+ 0.125 \cdot \cos(10 \pi \mathrm{s}_1) + 0.125 \cdot \sin(10 \pi \mathrm{s}_1)\notag
\end{flalign}


\begin{figure}[h]
	\centering
	\begin{subfigure}[t]{0.32\textwidth}
		\includegraphics[width=\textwidth, trim=200 300 50 100, clip]{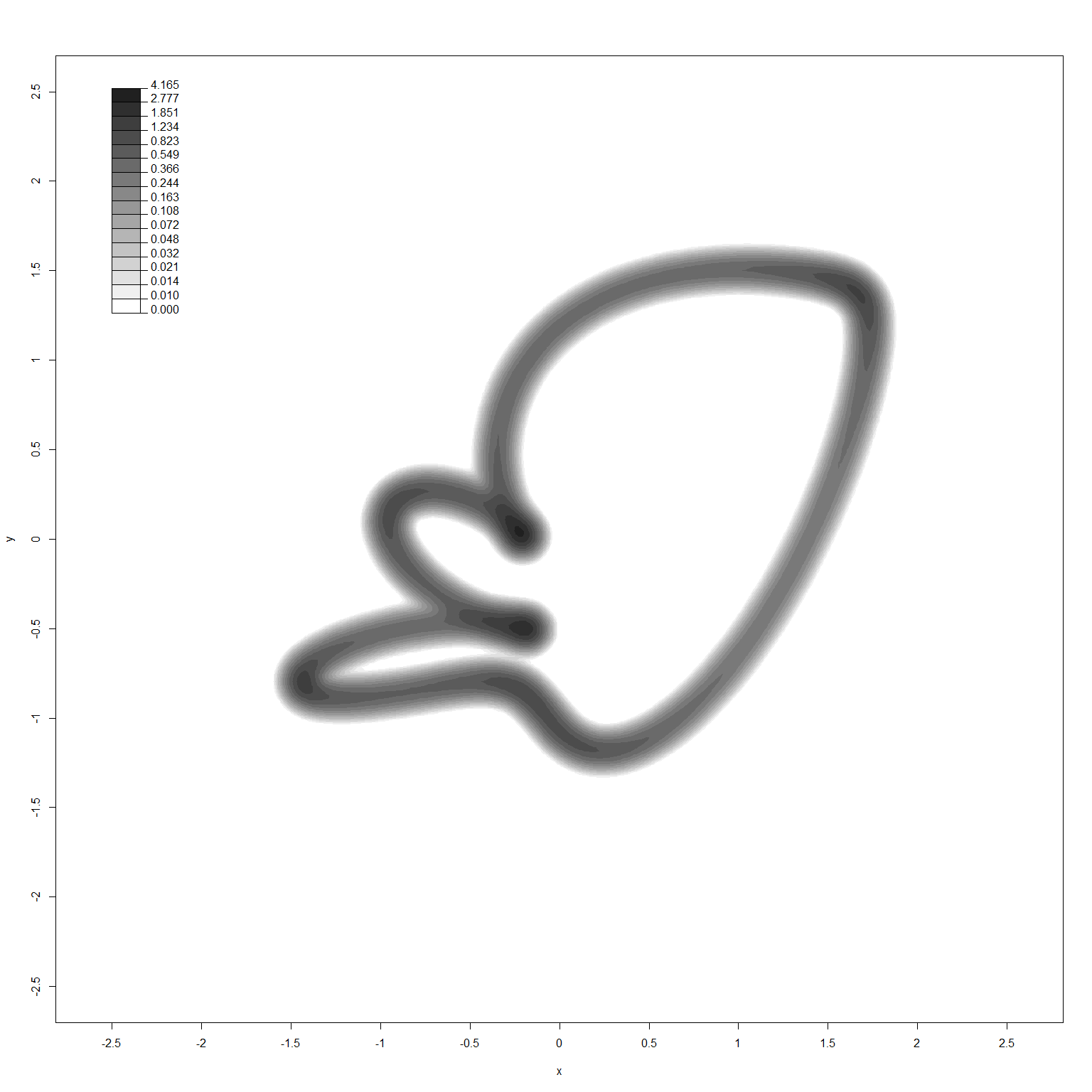}
		\caption{Original distribution $f_{\mathcal{N}(M_{\upvarphi},\sigma^2)}$ calculated numerically for $\sigma=0.05$.}
	\end{subfigure}
	\begin{subfigure}[t]{0.32\textwidth}
		\includegraphics[width=\textwidth, trim=200 300 50 100, clip]{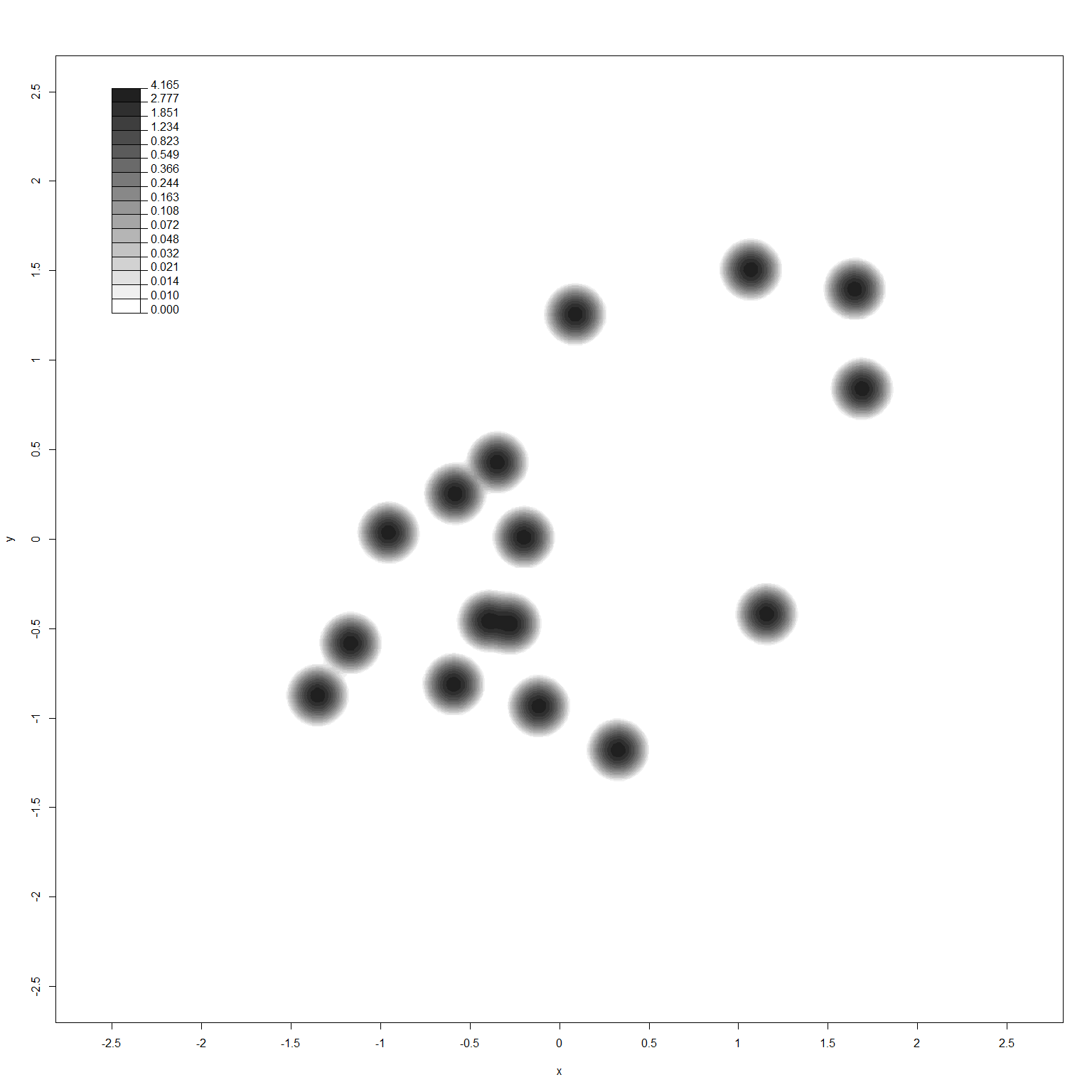}
		\caption{Probability density function approximated by formula~(\ref{eq:margin_app1}) for $\sigma=0.05$ and $k=16$ partitions.}
	\end{subfigure}
	\begin{subfigure}[t]{0.32\textwidth}
		\includegraphics[width=\textwidth, trim=200 300 50 100, clip]{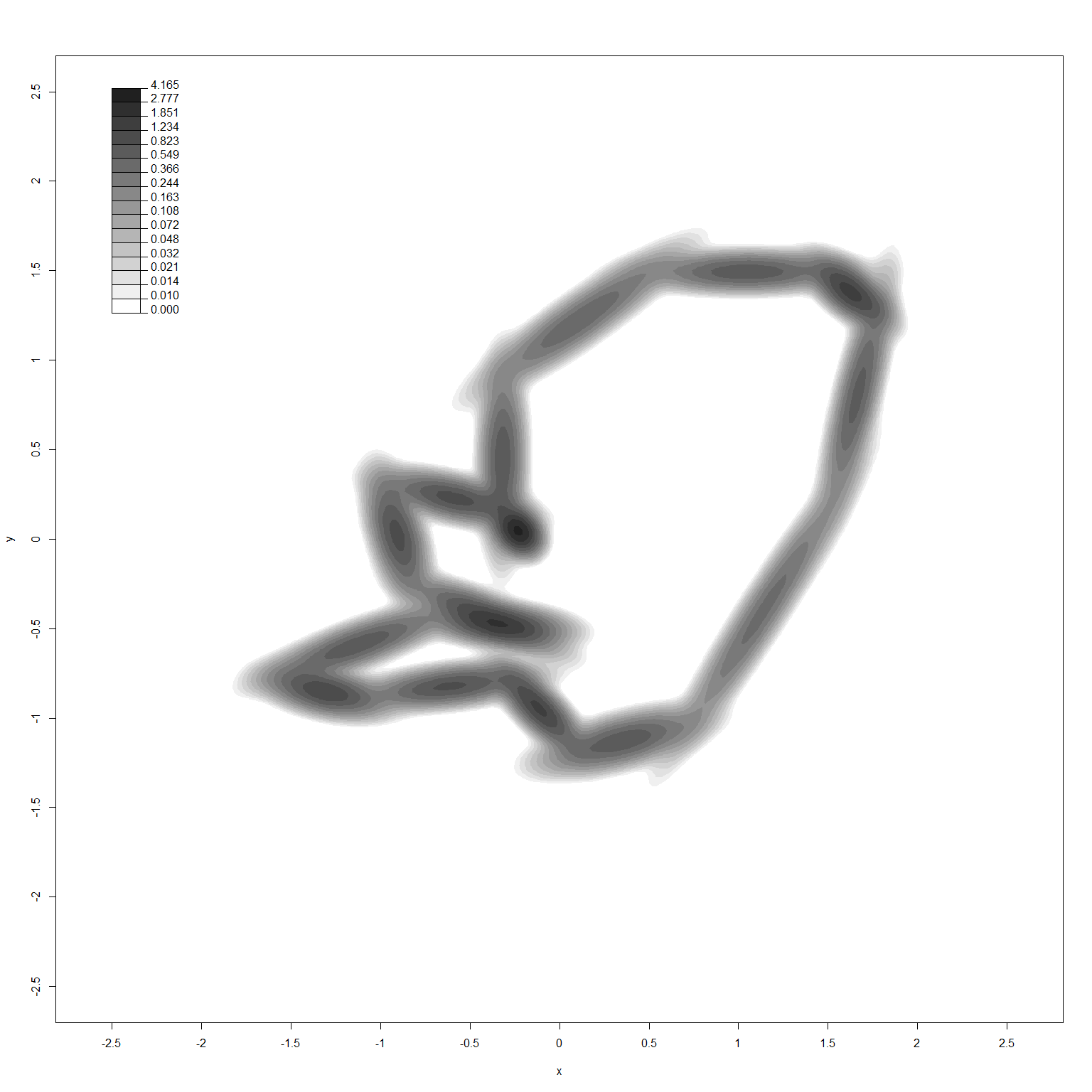}
		\caption{Probability density function approximated by formula~(\ref{eq:margin_app2}) for $\sigma=0.05$ and $k=16$ partitions.}
	\end{subfigure}
	\\
	\begin{subfigure}[t]{0.32\textwidth}
		\includegraphics[width=\textwidth, trim=200 300 50 100, clip]{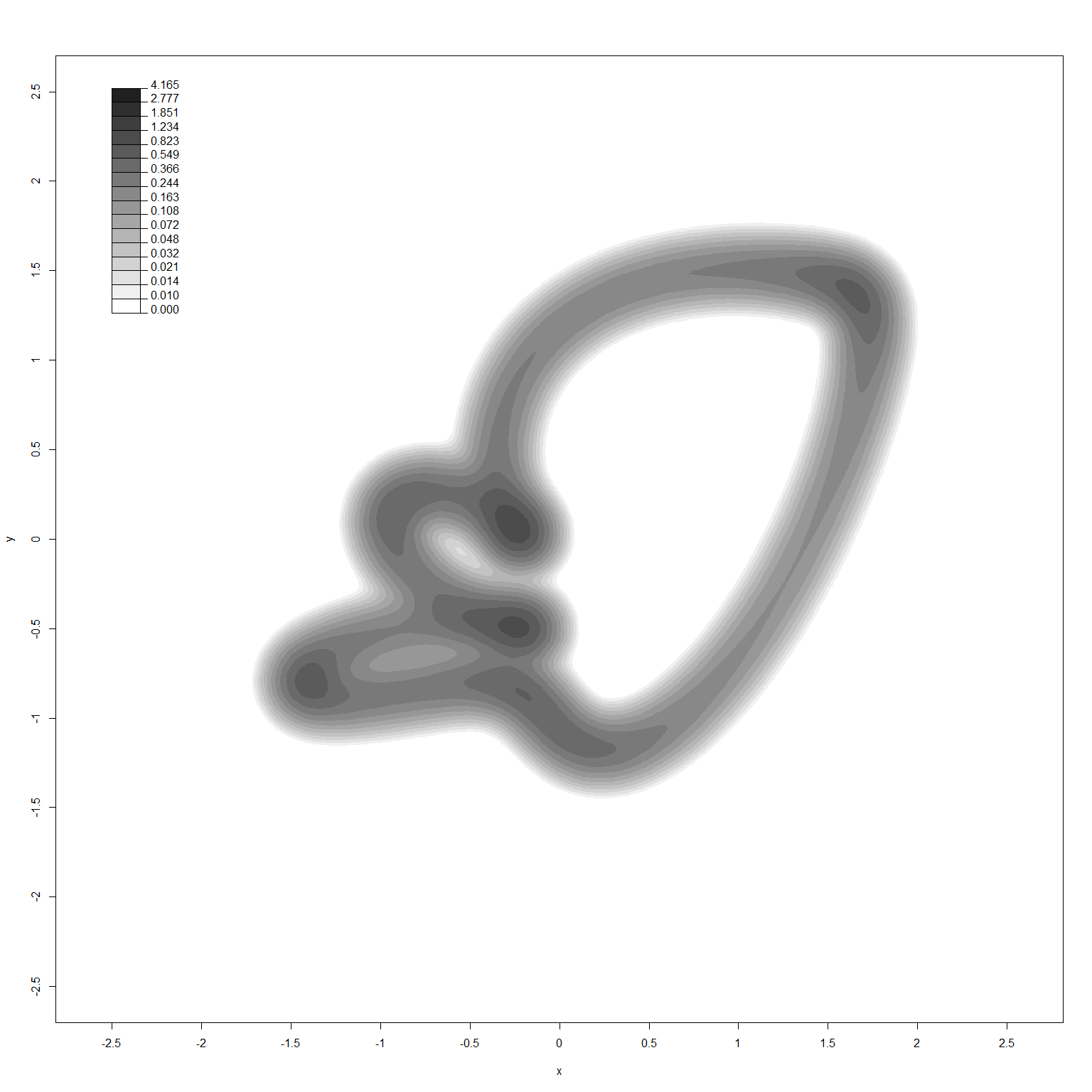}
		\caption{Original distribution $f_{\mathcal{N}(M_{\upvarphi},\sigma^2)}$ calculated numerically for $\sigma=0.1$.}
	\end{subfigure}
	\begin{subfigure}[t]{0.32\textwidth}
		\includegraphics[width=\textwidth, trim=200 300 50 100, clip]{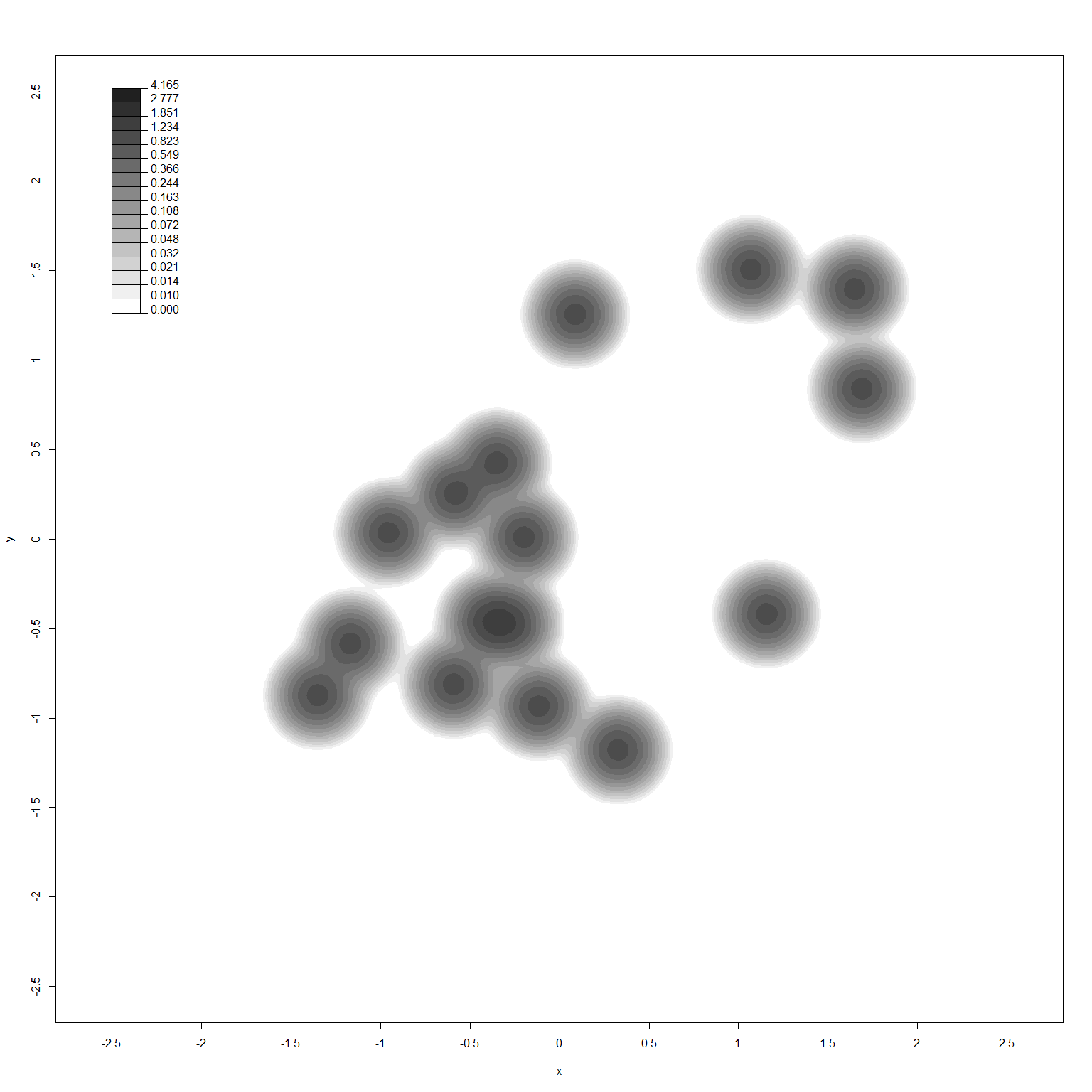}
		\caption{Probability density function approximated by formula~(\ref{eq:margin_app1}) for $\sigma=0.1$ and $k=16$ partitions.}
	\end{subfigure}
	\begin{subfigure}[t]{0.32\textwidth}
		\includegraphics[width=\textwidth, trim=200 300 50 100, clip]{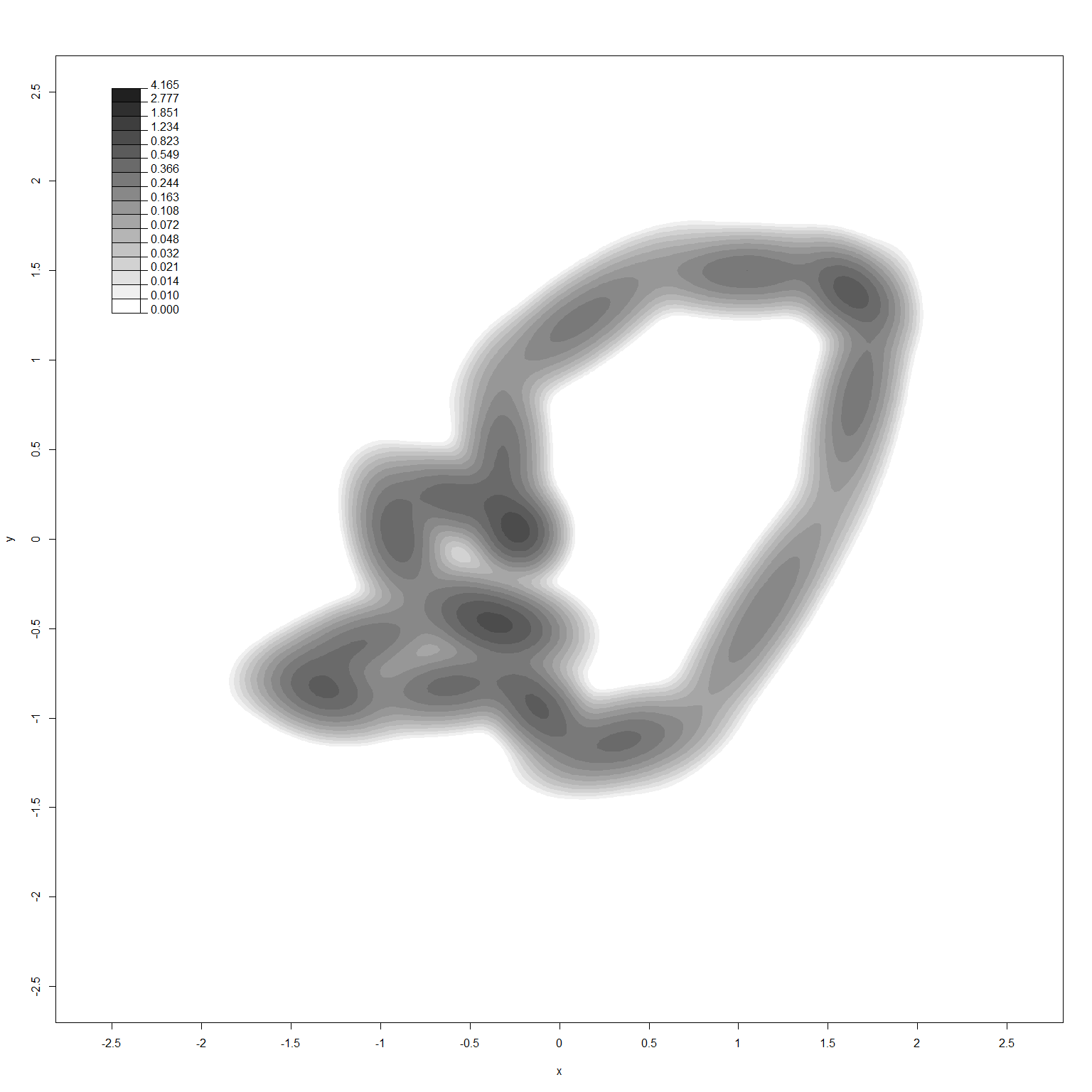}
		\caption{Probability density function approximated by formula~(\ref{eq:margin_app2}) for $\sigma=0.1$ and $k=16$ partitions.}
	\end{subfigure}
	\caption{Level sets of the probability distribution function induced by the rabbit-shaped curve from the Subsection \ref{Subsection4.2} for various values of $\sigma^2$. }
	\label{fig:rabbit1}
\end{figure}

\subsection{One dimensional closed curve embedded in $\R^n$ }
\label{Subsection4.3}

The above definitions can be generalized to the arbitrary dimension. 
Let's define the closed curve $\upvarphi$ given by the equation (\ref{general_parametri}):

\begin{equation}\label{general_parametri}
\upvarphi(s):[0,1] \mapsto   \left( \upvarphi^{(1)}(s), \ldots, \upvarphi^{(n)}(s) \right) \in \mathbb{R}^n,
\end{equation}
that is the curve satisfying: $\upvarphi(0)=\upvarphi(1)$ (note that every closed curve can be reduced to this form after the appropriate linear change of the variables). Now, let's define the random vectors: $\mathrm{X}=(\mathrm{X}_1,\ldots,\mathrm{X}_n)$ and the random variable $\mathrm{S}$ together with their joint probability distribution function given by the formula:

\begin{flalign}
f_{\mathrm{X};\mathrm{S}}(\mathrm{x};s) &= f_{\mathrm{X}_1,\ldots,\mathrm{X}_n;\mathrm{S}}(\mathrm{x}_1,\ldots,\mathrm{x}_n;s) =&\\
&= f_{\mathcal{N}(\upvarphi(s),\sigma^2 \mathrm{I})}(\mathrm{x}_1,\ldots,\mathrm{x}_n) =& \notag\\
&= \frac{1}{\sqrt{\lvert 2 \pi \sigma^2 \mathrm{I} \rvert}} e^{ -\frac{1}{2} (\mathrm{x}-\upvarphi(s))^{\top} (\sigma^2 \mathrm{I})^{-1} (\mathrm{x}-\upvarphi(s)) } =& \notag\\
&= \frac{1}{\sqrt{(2 \pi)^n \sigma^{2 n}}} e^{ -\frac{1}{2 \sigma^2} {\Vert \mathrm{x}-\upvarphi(s) \Vert}^2 } \notag
\end{flalign}
\newline
\noindent
where $\sigma^2 \in \mathbb{R}_{+}$ and where by $f_{\mathcal{N}(\upvarphi(s),\sigma^2 \mathrm{I})}$ we denote the probability distribution function of the Multivariate Normal Distribution with mean $\upmu=\upvarphi(s)$ and the covariance matrix $\Sigma=\sigma^2 \mathrm{I}$.

\begin{restatable}{definition}{definition1}
\label{definition1}
By Gaussian distribution on the closed curve $\upvarphi$ embedded in $\mathbb{R}^n$ with variance $\sigma^2$ we understand the marginal distribution with respect to the random vector $\mathrm{X}$ of the above-mentioned probability distribution with the joint probability distribution function given by $f_{\mathrm{X};\mathrm{S}}(\mathrm{x};s)$. We denote the corresponding marginal probability distribution function by:
\begin{flalign*}
f_{\mathcal{N}(M_{\upvarphi},\sigma^2)}(\mathrm{x}) &= f_{\mathrm{X}_1,\ldots,\mathrm{X}_n}(\mathrm{x}_1,\ldots,\mathrm{x}_n) =&\\
&= {\displaystyle \int \limits_{[0,1]} f_{\mathrm{X}_1,\ldots,\mathrm{X}_n;\mathrm{S}}(\mathrm{x}_1,\ldots,\mathrm{x}_n;s) \dd{s}} =&\\
&= {\displaystyle \int \limits_{[0,1]} \frac{1}{\sqrt{(2 \pi)^n \sigma^{2 n}}} e^{ -\frac{1}{2 \sigma^2} {\Vert \mathrm{x}-\upvarphi(s) \Vert}^2 } \dd{s}}
\end{flalign*}
\end{restatable}

\begin{figure}[H]
	\centering
	    \begin{tabular}{c@{}c@{}c@{}c@{}}
    		\includegraphics[width=0.32\textwidth]{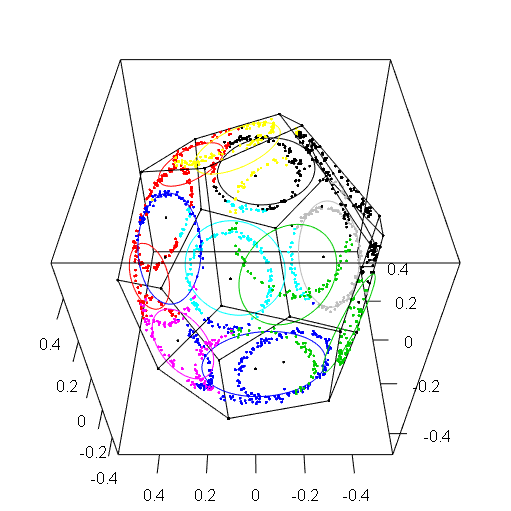} &
    		\includegraphics[width=0.32\textwidth]{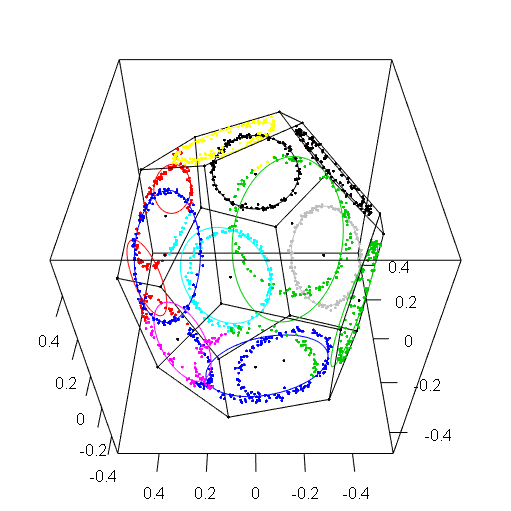} &
    		\includegraphics[width=0.32\textwidth]{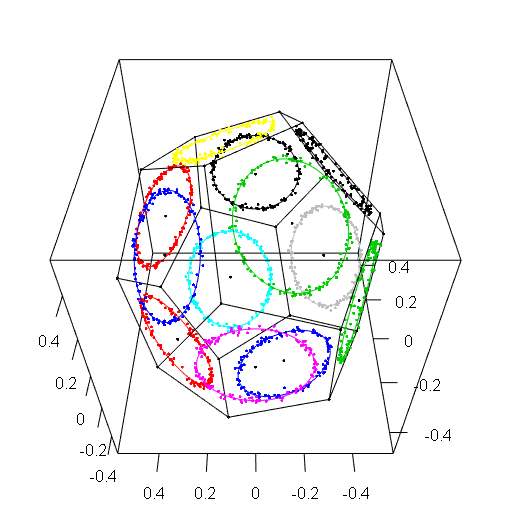} \\
    	\end{tabular}
	\caption{Let us consider data containing points concentrated around the circles lie on the plane of the faces of the dodecahedron.
    In the image, we present the convergence of MCEC in the subsequent 27 iterations (Iteration 1 (left), Iteration 6 (middle), and Iteration 27 (right)).}
	\label{fig:teser3}
\end{figure}

\subsection{Generalization to the manifold $M$ homeomorphic to the $d$-dimensional torus}
\label{Subsection4.4}

This subsection presents the general model for $d$-dimensional torus $M$ in  $\mathbb{R}^n$. Since, by assumption, $M$ is hyper-torus, it can be parametrized by the map:

\begin{flalign}
\upvarphi(\mathrm{s}):[0,1]^d \mapsto  & \upvarphi(\mathrm{s}_1,\ldots,\mathrm{s}_d) =&\\
&= \left( \upvarphi^{(1)}(\mathrm{s}_1,\ldots,\mathrm{s}_d), \ldots, \upvarphi^{(n)}(\mathrm{s}_1,\ldots,\mathrm{s}_d) \right) \in \mathbb{R}^n \notag.
\end{flalign}

\noindent
\newline
Now, let's define the random vectors: $\mathrm{X}=(\mathrm{X}_1,\ldots,\mathrm{X}_n)$ and $\mathrm{S}=(\mathrm{S}_1,\ldots,\mathrm{S}_d)$ together with their joint probability distribution function given by the formula:

\begin{flalign}
f_{\mathrm{X};\mathrm{S}}(\mathrm{x};\mathrm{s}) &= f_{\mathrm{X}_1,\ldots,\mathrm{X}_n;\mathrm{S}_1,\ldots,\mathrm{S}_d}(\mathrm{x}_1,\ldots,\mathrm{x}_n;\mathrm{s}_1,\ldots,\mathrm{s}_d) =&\\
&= f_{\mathcal{N}(\upvarphi(\mathrm{s}),\sigma^2 \mathrm{I})}(\mathrm{x}_1,\ldots,\mathrm{x}_n) =& \notag\\
&= \frac{1}{\sqrt{\lvert 2 \pi \sigma^2 \mathrm{I} \rvert}} e^{ -\frac{1}{2} (\mathrm{x}-\upvarphi(\mathrm{s}))^{\top} (\sigma^2 \mathrm{I})^{-1} (\mathrm{x}-\upvarphi(\mathrm{s})) } =& \notag\\
&= \frac{1}{\sqrt{(2 \pi)^n \sigma^{2 n}}} e^{ -\frac{1}{2 \sigma^2} {\Vert \mathrm{x}-\upvarphi(\mathrm{s}) \Vert}^2 } \notag
\end{flalign}

\noindent
where $\sigma>0$.

\begin{restatable}{definition}{definition2}
\label{definition2}
By Gaussian distribution on the manifold $M$ homeomorphic to the $d$-dimensional torus with variance $\sigma^2$, we understand the marginal distribution with respect to the random vector $\mathrm{X}$ of the above-mentioned probability distribution with the joint probability distribution function given by $f_{\mathrm{X};\mathrm{S}}(\mathrm{x};\mathrm{s})$. We denote the corresponding marginal probability distribution function by:
\begin{flalign}
f_{\mathcal{N}(M_{\upvarphi},\sigma^2)}(\mathrm{x}) &= f_{\mathrm{X}_1,\ldots,\mathrm{X}_n}(\mathrm{x}_1,\ldots,\mathrm{x}_n) =&\\
&= {\displaystyle \idotsint \limits_{[0,1]^d} f_{\mathrm{X}_1,\ldots,\mathrm{X}_n;\mathrm{S}_1,\ldots,\mathrm{S}_d}(\mathrm{x}_1,\ldots,\mathrm{x}_n;\mathrm{s}_1,\ldots,\mathrm{s}_d) \dd{\mathrm{s}_1} \ldots \dd{\mathrm{s}_d}} =&\notag\\
&= {\displaystyle \idotsint \limits_{[0,1]^d} \frac{1}{\sqrt{(2 \pi)^n \sigma^{2 n}}} e^{ -\frac{1}{2 \sigma^2} {\Vert \mathrm{x}-\upvarphi(\mathrm{s}_1,\ldots,\mathrm{s}_d) \Vert}^2 } \dd{\mathrm{s}_1} \ldots \dd{\mathrm{s}_d}}\notag
\end{flalign}
\end{restatable}


Now we are ready to compute the value of the probability distribution function $f_{\mathcal{N}(M_{\upvarphi},\sigma^2)}(\mathrm{x})$. First of all, let's note that under some mild additional assumptions, the map $\upvarphi(\mathrm{s}):[0,1]^d \mapsto \mathbb{R}^n$ can be approximated by the $d$-dimensional Fourier Series if the manifold $M$ is homeomorphic to the $d$-dimensional torus $\mathbb{T}^d=\underbrace{S^1 \times S^1}_{n \textrm{ times}}$ (exactly as in the special case of the closed curve embedded in $\mathbb{R}^n$ from the Subsection \ref{Subsection4.3}). The $d$-dimensional Fourier Series is uniformly converged to the map $\upvarphi(\mathrm{s})$. Otherwise, the convergence is point-wise, and the manifold $M$ is approximated by the degenerated torus $\mathbb{T}^d$. Since it would be infeasible for computational purposes to consider the infinite $d$-dimensional Fourier Series, we will restrain ourselves to the $d$-dimensional Fourier Series of some finite order $k \in \mathbb{N}$. Then:

\begin{flalign}
\upvarphi^{(i)}(\mathrm{s}) & \approx {\displaystyle \sum \limits_{\mathrm{l} \in \{ -k,\ldots,k \}^d}} \mathrm{a}_{\mathrm{l}}^{(i)} {\displaystyle \prod \limits_{m=1}^{d}}
\left \lbrace
\begin{array}{ll}
\cos(-2 \pi \mathrm{l}_m \mathrm{s}_m) & \textrm{, $\mathrm{l}_m < 0$} \\
1 & \textrm{, $\mathrm{l}_m = 0$} \\
\sin(2 \pi \mathrm{l}_m \mathrm{s}_m) & \textrm{, $\mathrm{l}_m > 0$}
\end{array}
\right. &
\end{flalign}

\noindent
\newline
where $\mathrm{a}_{\mathrm{l}}^{(i)} \in \mathbb{R}$ are the coefficients indexed by $i=1,\ldots,n$ and the multiindex $l \in \{ -k,\ldots,k \}^d$.
\newline
\newline
\null \qquad
Since, in general, the integral
\begin{flalign}
\label{equation9}
f_{\mathcal{N}(M_{\upvarphi},\sigma^2)}(\mathrm{x}) &=
{\displaystyle \idotsint \limits_{[0,1]^d} \frac{1}{\sqrt{(2 \pi)^n \sigma^{2 n}}} e^{ -\frac{1}{2 \sigma^2} {\Vert \mathrm{x}-\upvarphi(\mathrm{s}_1,\ldots,\mathrm{s}_d) \Vert}^2 } \dd{\mathrm{s}_1} \ldots \dd{\mathrm{s}_d}} &
\end{flalign}
in the Defintion \ref{definition2} cannot be expressed in terms of the elementary functions\footnote{Even for the simplest case $d=1$ it can be computed only by using Modified Bessel Function of the First Kind.}, to approximate its value we use some properly tailored quadrature. The idea is to partition the $d$-dimensional unit hypercube domain $[0,1]^d$ into $K^d$ equally sized $d$-dimensional subintervals, where $K \in \mathbb{N}_{+}$ and approximate the value of the probability distribution function for associated pieces of manifold $M$ using the probability distribution function of $n$-dimensional Multivariate Normal Distribution.
\newline
\newline
\null \qquad
For the multiindex $\mathrm{j} \in \{0,\ldots,K-1\}^d$ identifying the $d$-dimensional subinterval of the map $\upvarphi(\mathrm{s})$ domain, let's define:

{\scriptsize
\begin{flalign}
\upmu_{\mathrm{j}}^{(i)} &= K^d { \displaystyle \int \limits_{\mathrm{j}_d / K}^{(\mathrm{j}_d+1)/K}} \ldots { \displaystyle \int \limits_{\mathrm{j}_1 / K}^{(\mathrm{j}_1+1)/K}} \upvarphi^{(i)}(\mathrm{s}) \dd{\mathrm{s}_1} \ldots \dd{\mathrm{s}_d} &
\end{flalign}
}%

{\scriptsize
\begin{flalign}
\Sigma_{\mathrm{j}}^{(i_1)(i_2)} &= {\sigma^2}\mathrm{I} + K^d { \displaystyle \int \limits_{\mathrm{j}_d / K}^{(\mathrm{j}_d+1)/K}} \ldots { \displaystyle \int \limits_{\mathrm{j}_1 / K}^{(\mathrm{j}_1+1)/K}}
(\upvarphi^{(i_1)}(\mathrm{s})-\upmu_{\mathrm{j}}^{(i_1)})(\upvarphi^{(i_2)}(\mathrm{s})-\upmu_{\mathrm{j}}^{(i_2)}) 
\dd{\mathrm{s}_1} \ldots \dd{\mathrm{s}_d} &\notag
\end{flalign}
}%

\noindent
where $i, i_1, i_2 \in \{1,\ldots,n\}$. Then, it can be proven that the integral \ref{equation9} can be approximated using the following formula:

\begin{flalign}\label{equation13}
f_{\mathcal{N}(M_{\upvarphi},\sigma^2)}(\mathrm{x}) & \approx \frac{1}{K^d} {\displaystyle \sum \limits_{\mathrm{l} \in \{0,\ldots,K-1\}^d}} f_{\mathcal{N}(\upmu_{\mathrm{l}},\Sigma_{\mathrm{l}})}(\mathrm{x}) =&\\
&= \frac{1}{K^d} {\displaystyle \sum \limits_{\mathrm{l} \in \{0,\ldots,K-1\}^d}} \frac{1}{\sqrt{(2 \pi)^n \lvert \Sigma_{\mathrm{l}} \rvert}} e^{ -\frac{1}{2} (\mathrm{x}-\upmu_{\mathrm{l}})^{\top} \Sigma_{\mathrm{l}}^{-1} (\mathrm{x}-\upmu_{\mathrm{l}}) }\notag
\end{flalign}

To benefit from equation (\ref{equation13}) we need to show, that $\upmu_{\mathrm{j}}$ and $\Sigma_{\mathrm{j}}$ can be easily computed. This is the result given by the following theorem:
\newline

\begin{restatable}{theorem}{theoremone}
\label{theoremone}
Under the above notions:
\newline
\newline
a) The mean vectors $\upmu_{\mathrm{j}}$ are given by the formula:
\newline
\newline
\newline
$
\upmu_{\mathrm{j}}^{(i)} = K^d {\displaystyle \sum \limits_{\mathrm{l} \in \{ -k,\ldots,k \}^d}} \mathrm{a}_{\mathrm{l}}^{(i)} {\displaystyle \prod \limits_{m=1}^{d}}
\left \lbrace
\begin{array}{ll}
\left. -\frac{1}{2 \pi \mathrm{l}_m} \sin(-2 \pi \mathrm{l}_m \mathrm{s}_m) \right \rvert_{\mathrm{s}_m=\mathrm{j}_m / K}^{\mathrm{s}_m=(\mathrm{j}_m+1) / K} & \textrm{, $\mathrm{l}_m < 0$} \\
\\
\left. \mathrm{s}_m \right \rvert_{\mathrm{s}_m=\mathrm{j}_m / K}^{\mathrm{s}_m=(\mathrm{j}_m+1) / K} & \textrm{, $\mathrm{l}_m = 0$} \\
\\
\left. -\frac{1}{2 \pi \mathrm{l}_m} \cos(2 \pi \mathrm{l}_m \mathrm{s}_m) \right \rvert_{\mathrm{s}_m=\mathrm{j}_m / K}^{\mathrm{s}_m=(\mathrm{j}_m+1) / K} & \textrm{, $\mathrm{l}_m > 0$}
\end{array}
\right.
$
\newline
\newline
\newline
b) The covariance matrices $\Sigma_{\mathrm{j}}$ are given by the formula:
\newline
\newline
\newline
$
\Sigma_{\mathrm{j}}^{(i_1)(i_2)} = {\sigma^2}\mathrm{I} +
$
\newline
\newline
\resizebox{\linewidth}{!}{
$
+ K^d {\displaystyle \sum \limits_{\substack{\mathrm{l}^1 \in \{ -k,\ldots,k \}^d \\ \mathrm{l}^2 \in \{ -k,\ldots,k \}^d}}} \mathrm{a}_{\mathrm{l}^1}^{(i_1)} \mathrm{a}_{\mathrm{l}^2}^{(i_2)} {\displaystyle \prod \limits_{m=1}^{d}}
\left \lbrace
\begin{array}{ll}
\left \lbrace
\begin{array}{ll}
\left. -\frac{ \sin(-2 \pi \mathrm{s}_m ( \mathrm{l}_m^1 + \mathrm{l}_m^2 )) }{ 4 \pi ( \mathrm{l}_m^1 + \mathrm{l}_m^2 ) } -
\frac{ \sin(-2 \pi \mathrm{s}_m ( \mathrm{l}_m^1 - \mathrm{l}_m^2 )) }{ 4 \pi ( \mathrm{l}_m^1 - \mathrm{l}_m^2 ) } \right \rvert_{\mathrm{s}_m=\mathrm{j}_m / K}^{\mathrm{s}_m=(\mathrm{j}_m+1) / K}
& \textrm{, $\mathrm{l}_m^1 \ne \mathrm{l}_m^2$} \\
\\
\left. -\frac{ \sin(-2 \pi \mathrm{s}_m ( \mathrm{l}_m^1 + \mathrm{l}_m^2 )) }{ 4 \pi ( \mathrm{l}_m^1 + \mathrm{l}_m^2 ) } + \frac{\mathrm{s}_m}{2} \right \rvert_{\mathrm{s}_m=\mathrm{j}_m / K}^{\mathrm{s}_m=(\mathrm{j}_m+1) / K}
& \textrm{, $\mathrm{l}_m^1 = \mathrm{l}_m^2$} \\
\end{array}
\right.
& \textrm{, $\mathrm{l}_m^1 < 0$, $\mathrm{l}_m^2 < 0$} \\
\\
\left. -\frac{1}{2 \pi \mathrm{l}_m^1} \sin(-2 \pi \mathrm{l}_m^1 \mathrm{s}_m) \right \rvert_{\mathrm{s}_m=\mathrm{j}_m / K}^{\mathrm{s}_m=(\mathrm{j}_m+1) / K} & \textrm{, $\mathrm{l}_m^1 < 0$, $\mathrm{l}_m^2 = 0$} \\
\\
\left \lbrace
\begin{array}{ll}
\left. \frac{ \cos(-2 \pi \mathrm{s}_m ( \mathrm{l}_m^1 - \mathrm{l}_m^2 )) } { 4 \pi ( \mathrm{l}_m^1 - \mathrm{l}_m^2 ) } -
\frac{ \cos(-2 \pi \mathrm{s}_m ( \mathrm{l}_m^1 + \mathrm{l}_m^2 )) }{ 4 \pi ( \mathrm{l}_m^1 + \mathrm{l}_m^2 ) } \right \rvert_{\mathrm{s}_m=\mathrm{j}_m / K}^{\mathrm{s}_m=(\mathrm{j}_m+1) / K}
& \textrm{, $\mathrm{l}_m^1 \ne - \mathrm{l}_m^2$} \\
\\
\left. \frac{ \cos(-2 \pi \mathrm{s}_m ( \mathrm{l}_m^1 - \mathrm{l}_m^2 )) } { 4 \pi ( \mathrm{l}_m^1 - \mathrm{l}_m^2 ) } \right \rvert_{\mathrm{s}_m=\mathrm{j}_m / K}^{\mathrm{s}_m=(\mathrm{j}_m+1) / K}
& \textrm{, $\mathrm{l}_m^1 = - \mathrm{l}_m^2$} \\
\end{array}
\right.
& \textrm{, $\mathrm{l}_m^1 < 0$, $\mathrm{l}_m^2 > 0$} \\
\\
\left. -\frac{1}{2 \pi \mathrm{l}_m^2} \sin(-2 \pi \mathrm{l}_m^2 \mathrm{s}_m) \right \rvert_{\mathrm{s}_m=\mathrm{j}_m / K}^{\mathrm{s}_m=(\mathrm{j}_m+1) / K}  & \textrm{, $\mathrm{l}_m^1 = 0$, $\mathrm{l}_m^2 < 0$} \\
\\
\left. \mathrm{s}_m \right \rvert_{\mathrm{s}_m=\mathrm{j}_m / K}^{\mathrm{s}_m=(\mathrm{j}_m+1) / K} & \textrm{, $\mathrm{l}_m^1 = 0$, $\mathrm{l}_m^2 = 0$} \\
\\
\left. -\frac{1}{2 \pi \mathrm{l}_m^2} \cos(2 \pi \mathrm{l}_m^2 \mathrm{s}_m) \right \rvert_{\mathrm{s}_m=\mathrm{j}_m / K}^{\mathrm{s}_m=(\mathrm{j}_m+1) / K} & \textrm{, $\mathrm{l}_m^1 = 0$, $\mathrm{l}_m^2 > 0$} \\
\\
\left \lbrace
\begin{array}{ll}
\left. -\frac{ \cos(2 \pi \mathrm{s}_m ( \mathrm{l}_m^1 - \mathrm{l}_m^2 )) } { 4 \pi ( \mathrm{l}_m^1 - \mathrm{l}_m^2 ) } -
\frac{ \cos(2 \pi \mathrm{s}_m ( \mathrm{l}_m^1 + \mathrm{l}_m^2 )) }{ 4 \pi ( \mathrm{l}_m^1 + \mathrm{l}_m^2 ) } \right \rvert_{\mathrm{s}_m=\mathrm{j}_m / K}^{\mathrm{s}_m=(\mathrm{j}_m+1) / K}
& \textrm{, $\mathrm{l}_m^1 \ne - \mathrm{l}_m^2$} \\
\\
\left. -\frac{ \cos(2 \pi \mathrm{s}_m ( \mathrm{l}_m^1 - \mathrm{l}_m^2 )) } { 4 \pi ( \mathrm{l}_m^1 - \mathrm{l}_m^2 ) } \right \rvert_{\mathrm{s}_m=\mathrm{j}_m / K}^{\mathrm{s}_m=(\mathrm{j}_m+1) / K}
& \textrm{, $\mathrm{l}_m^1 = - \mathrm{l}_m^2$} \\
\end{array}
\right.
& \textrm{, $\mathrm{l}_m^1 > 0$, $\mathrm{l}_m^2 < 0$} \\
\\
\left. -\frac{1}{2 \pi \mathrm{l}_m^1} \cos(2 \pi \mathrm{l}_m^1 \mathrm{s}_m) \right \rvert_{\mathrm{s}_m=\mathrm{j}_m / K}^{\mathrm{s}_m=(\mathrm{j}_m+1) / K} & \textrm{, $\mathrm{l}_m^1 > 0$, $\mathrm{l}_m^2 = 0$} \\
\\
\left \lbrace
\begin{array}{ll}
\left. \frac{ \sin(2 \pi \mathrm{s}_m ( \mathrm{l}_m^1 - \mathrm{l}_m^2 )) }{ 4 \pi ( \mathrm{l}_m^1 - \mathrm{l}_m^2 ) } -
\frac{ \sin(2 \pi \mathrm{s}_m ( \mathrm{l}_m^1 + \mathrm{l}_m^2 )) }{ 4 \pi ( \mathrm{l}_m^1 + \mathrm{l}_m^2 ) } \right \rvert_{\mathrm{s}_m=\mathrm{j}_m / K}^{\mathrm{s}_m=(\mathrm{j}_m+1) / K}
& \textrm{, $\mathrm{l}_m^1 \ne \mathrm{l}_m^2$} \\
\\
\left. \frac{\mathrm{s}_m}{2} - \frac{ \sin(2 \pi \mathrm{s}_m ( \mathrm{l}_m^1 + \mathrm{l}_m^2 )) }{ 4 \pi ( \mathrm{l}_m^1 + \mathrm{l}_m^2 ) } \right \rvert_{\mathrm{s}_m=\mathrm{j}_m / K}^{\mathrm{s}_m=(\mathrm{j}_m+1) / K}
& \textrm{, $\mathrm{l}_m^1 = \mathrm{l}_m^2$} \\
\end{array}
\right.
& \textrm{, $\mathrm{l}_m^1 > 0$, $\mathrm{l}_m^2 > 0$}
\end{array}
\right. -
$
}
\newline
\newline
$
- \upmu_{\mathrm{j}}^{(i_1)} \upmu_{\mathrm{j}}^{(i_2)}
$
\newline
\end{restatable}
\begin{proof}
\renewcommand{\qedsymbol}{}
See ~\ref{theoremone-proof}.
\end{proof}


\przemek{As we explicitly stated at the end of Section ~\ref{Section3}, in general, there are no closed-form formulas for the estimated values of the parameters of the probability distribution model used by the MCEC. Therefore in our algorithm, we leverage the BFGS gradient optimizer in order to obtain the appropriate estimated values. Since the BFGS optimizer we use requires us to provide the computed exact value of the gradient of the cross-entropy with respect to the values of the estimated parameters, in the following Theorem~\ref{theoremtwo}, we provide the neat closed-form formula that yields the value of the above-mentioned gradient.
}
\begin{restatable}{theorem}{theoremtwo}
\label{theoremtwo}
 Let's define
\newline
\newline
$
g \left( \mathrm{j}, \mathrm{l} \right) \coloneqq {\displaystyle \prod \limits_{m=1}^{d}}
\left \lbrace
\begin{array}{ll}
\left. -\frac{1}{2 \pi \mathrm{l}_m} \sin(-2 \pi \mathrm{l}_m \mathrm{s}_m) \right \rvert_{\mathrm{s}_m=\mathrm{j}_m / K}^{\mathrm{s}_m=(\mathrm{j}_m+1) / K} & \textrm{, $\mathrm{l}_m < 0$} \\
\\
\left. \mathrm{s}_m \right \rvert_{\mathrm{s}_m=\mathrm{j}_m / K}^{\mathrm{s}_m=(\mathrm{j}_m+1) / K} & \textrm{, $\mathrm{l}_m = 0$} \\
\\
\left. -\frac{1}{2 \pi \mathrm{l}_m} \cos(2 \pi \mathrm{l}_m \mathrm{s}_m) \right \rvert_{\mathrm{s}_m=\mathrm{j}_m / K}^{\mathrm{s}_m=(\mathrm{j}_m+1) / K} & \textrm{, $\mathrm{l}_m > 0$}
\end{array}
\right.
$
\newline
\newline
\newline
\resizebox{\linewidth}{!}{
$
g \left( \mathrm{j}, \mathrm{l}^1, \mathrm{l}^2 \right) \coloneqq {\displaystyle \prod \limits_{m=1}^{d}}
\left \lbrace
\begin{array}{ll}
\left \lbrace
\begin{array}{ll}
\left. -\frac{ \sin(-2 \pi \mathrm{s}_m ( \mathrm{l}_m^1 + \mathrm{l}_m^2 )) }{ 4 \pi ( \mathrm{l}_m^1 + \mathrm{l}_m^2 ) } -
\frac{ \sin(-2 \pi \mathrm{s}_m ( \mathrm{l}_m^1 - \mathrm{l}_m^2 )) }{ 4 \pi ( \mathrm{l}_m^1 - \mathrm{l}_m^2 ) } \right \rvert_{\mathrm{s}_m=\mathrm{j}_m / K}^{\mathrm{s}_m=(\mathrm{j}_m+1) / K}
& \textrm{, $\mathrm{l}_m^1 \ne \mathrm{l}_m^2$} \\
\\
\left. -\frac{ \sin(-2 \pi \mathrm{s}_m ( \mathrm{l}_m^1 + \mathrm{l}_m^2 )) }{ 4 \pi ( \mathrm{l}_m^1 + \mathrm{l}_m^2 ) } + \frac{\mathrm{s}_m}{2} \right \rvert_{\mathrm{s}_m=\mathrm{j}_m / K}^{\mathrm{s}_m=(\mathrm{j}_m+1) / K}
& \textrm{, $\mathrm{l}_m^1 = \mathrm{l}_m^2$} \\
\end{array}
\right.
& \textrm{, $\mathrm{l}_m^1 < 0$, $\mathrm{l}_m^2 < 0$} \\
\\
\left. -\frac{1}{2 \pi \mathrm{l}_m^1} \sin(-2 \pi \mathrm{l}_m^1 \mathrm{s}_m) \right \rvert_{\mathrm{s}_m=\mathrm{j}_m / K}^{\mathrm{s}_m=(\mathrm{j}_m+1) / K} & \textrm{, $\mathrm{l}_m^1 < 0$, $\mathrm{l}_m^2 = 0$} \\
\\
\left \lbrace
\begin{array}{ll}
\left. \frac{ \cos(-2 \pi \mathrm{s}_m ( \mathrm{l}_m^1 - \mathrm{l}_m^2 )) } { 4 \pi ( \mathrm{l}_m^1 - \mathrm{l}_m^2 ) } -
\frac{ \cos(-2 \pi \mathrm{s}_m ( \mathrm{l}_m^1 + \mathrm{l}_m^2 )) }{ 4 \pi ( \mathrm{l}_m^1 + \mathrm{l}_m^2 ) } \right \rvert_{\mathrm{s}_m=\mathrm{j}_m / K}^{\mathrm{s}_m=(\mathrm{j}_m+1) / K}
& \textrm{, $\mathrm{l}_m^1 \ne - \mathrm{l}_m^2$} \\
\\
\left. \frac{ \cos(-2 \pi \mathrm{s}_m ( \mathrm{l}_m^1 - \mathrm{l}_m^2 )) } { 4 \pi ( \mathrm{l}_m^1 - \mathrm{l}_m^2 ) } \right \rvert_{\mathrm{s}_m=\mathrm{j}_m / K}^{\mathrm{s}_m=(\mathrm{j}_m+1) / K}
& \textrm{, $\mathrm{l}_m^1 = - \mathrm{l}_m^2$} \\
\end{array}
\right.
& \textrm{, $\mathrm{l}_m^1 < 0$, $\mathrm{l}_m^2 > 0$} \\
\\
\left. -\frac{1}{2 \pi \mathrm{l}_m^2} \sin(-2 \pi \mathrm{l}_m^2 \mathrm{s}_m) \right \rvert_{\mathrm{s}_m=\mathrm{j}_m / K}^{\mathrm{s}_m=(\mathrm{j}_m+1) / K}  & \textrm{, $\mathrm{l}_m^1 = 0$, $\mathrm{l}_m^2 < 0$} \\
\\
\left. \mathrm{s}_m \right \rvert_{\mathrm{s}_m=\mathrm{j}_m / K}^{\mathrm{s}_m=(\mathrm{j}_m+1) / K} & \textrm{, $\mathrm{l}_m^1 = 0$, $\mathrm{l}_m^2 = 0$} \\
\\
\left. -\frac{1}{2 \pi \mathrm{l}_m^2} \cos(2 \pi \mathrm{l}_m^2 \mathrm{s}_m) \right \rvert_{\mathrm{s}_m=\mathrm{j}_m / K}^{\mathrm{s}_m=(\mathrm{j}_m+1) / K} & \textrm{, $\mathrm{l}_m^1 = 0$, $\mathrm{l}_m^2 > 0$} \\
\\
\left \lbrace
\begin{array}{ll}
\left. -\frac{ \cos(2 \pi \mathrm{s}_m ( \mathrm{l}_m^1 - \mathrm{l}_m^2 )) } { 4 \pi ( \mathrm{l}_m^1 - \mathrm{l}_m^2 ) } -
\frac{ \cos(2 \pi \mathrm{s}_m ( \mathrm{l}_m^1 + \mathrm{l}_m^2 )) }{ 4 \pi ( \mathrm{l}_m^1 + \mathrm{l}_m^2 ) } \right \rvert_{\mathrm{s}_m=\mathrm{j}_m / K}^{\mathrm{s}_m=(\mathrm{j}_m+1) / K}
& \textrm{, $\mathrm{l}_m^1 \ne - \mathrm{l}_m^2$} \\
\\
\left. -\frac{ \cos(2 \pi \mathrm{s}_m ( \mathrm{l}_m^1 - \mathrm{l}_m^2 )) } { 4 \pi ( \mathrm{l}_m^1 - \mathrm{l}_m^2 ) } \right \rvert_{\mathrm{s}_m=\mathrm{j}_m / K}^{\mathrm{s}_m=(\mathrm{j}_m+1) / K}
& \textrm{, $\mathrm{l}_m^1 = - \mathrm{l}_m^2$} \\
\end{array}
\right.
& \textrm{, $\mathrm{l}_m^1 > 0$, $\mathrm{l}_m^2 < 0$} \\
\\
\left. -\frac{1}{2 \pi \mathrm{l}_m^1} \cos(2 \pi \mathrm{l}_m^1 \mathrm{s}_m) \right \rvert_{\mathrm{s}_m=\mathrm{j}_m / K}^{\mathrm{s}_m=(\mathrm{j}_m+1) / K} & \textrm{, $\mathrm{l}_m^1 > 0$, $\mathrm{l}_m^2 = 0$} \\
\\
\left \lbrace
\begin{array}{ll}
\left. \frac{ \sin(2 \pi \mathrm{s}_m ( \mathrm{l}_m^1 - \mathrm{l}_m^2 )) }{ 4 \pi ( \mathrm{l}_m^1 - \mathrm{l}_m^2 ) } -
\frac{ \sin(2 \pi \mathrm{s}_m ( \mathrm{l}_m^1 + \mathrm{l}_m^2 )) }{ 4 \pi ( \mathrm{l}_m^1 + \mathrm{l}_m^2 ) } \right \rvert_{\mathrm{s}_m=\mathrm{j}_m / K}^{\mathrm{s}_m=(\mathrm{j}_m+1) / K}
& \textrm{, $\mathrm{l}_m^1 \ne \mathrm{l}_m^2$} \\
\\
\left. \frac{\mathrm{s}_m}{2} - \frac{ \sin(2 \pi \mathrm{s}_m ( \mathrm{l}_m^1 + \mathrm{l}_m^2 )) }{ 4 \pi ( \mathrm{l}_m^1 + \mathrm{l}_m^2 ) } \right \rvert_{\mathrm{s}_m=\mathrm{j}_m / K}^{\mathrm{s}_m=(\mathrm{j}_m+1) / K}
& \textrm{, $\mathrm{l}_m^1 = \mathrm{l}_m^2$} \\
\end{array}
\right.
& \textrm{, $\mathrm{l}_m^1 > 0$, $\mathrm{l}_m^2 > 0$}
\end{array}
\right.
$
}
\newline
\newline
\newline
Then, under the above notions:
\newline
\newline
$
{\pdv{}{ [ \mathrm{a}_{\mathrm{j}}^{(i)} | \sigma ] }} \left( {\displaystyle \sum \limits_{\mathrm{x} \in \mathrm{x}}} \ln(f_{\mathcal{N}(\upvarphi(\mathrm{s}),\sigma^2)}(\mathrm{x})) \right) =
$
\newline
\newline
\newline
\resizebox{\linewidth}{!}{
$
= \frac{1}{K^d} {\displaystyle \sum \limits_{\mathrm{x} \in \mathrm{x}}} P {\displaystyle \sum \limits_{ \mathrm{l} \in \{ 0,\ldots,K-1 \}^d }}
- \frac{1}{2} \frac{1}{ \sqrt{ (2 \pi)^n \lvert \Sigma_{\mathrm{l}} \rvert }} \tr( \Sigma_{\mathrm{l}}^{-1} {\pdv{}{ [ \mathrm{a}_{\mathrm{j}}^{(i)} | \sigma ] }}
\left( \Sigma_{\mathrm{l}}( \mathrm{a}_{\mathrm{j}}^{(i)}, \sigma ) \right) ) e^{ -\frac{1}{2} ( \mathrm{x} - \upmu_{\mathrm{l}} )^{\top} \Sigma_{\mathrm{l}}^{-1} ( \mathrm{x} - \upmu_{\mathrm{l}} ) } -\\
$
}
\resizebox{\linewidth}{!}{
$
\phantom{=} - \frac{1}{2} \frac{1}{ \sqrt{ (2 \pi)^n \lvert \Sigma_{\mathrm{l}} \rvert }} e^{ -\frac{1}{2} ( \mathrm{x} - \upmu_{\mathrm{l}} )^{\top} \Sigma_{\mathrm{l}}^{-1} ( \mathrm{x} - \upmu_{\mathrm{l}} ) }
\left( \left( \mathrm{x} - {\pdv{}{ [ \mathrm{a}_{\mathrm{j}}^{(i)} | \sigma ] }} \left( \upmu_{\mathrm{l}}( \mathrm{a}_{\mathrm{j}}^{(i)}, \sigma ) \right) \right)^{\top}
\Sigma_{\mathrm{l}}^{-1} ( \mathrm{x} - \upmu_{\mathrm{l}} ) \right. +
$
}
\resizebox{\linewidth}{!}{
$
\phantom{=} + ( \mathrm{x} - \upmu_{\mathrm{l}} )^{\top}
\left( - \Sigma_{\mathrm{l}}^{-1} \left( {\pdv{}{ [ \mathrm{a}_{\mathrm{j}}^{(i)} | \sigma ] }} \left( \Sigma_{\mathrm{l}}( \mathrm{a}_{\mathrm{j}}^{(i)}, \sigma ) \right) \right) \Sigma_{\mathrm{l}}^{-1} \right) ( \mathrm{x} - \upmu_{\mathrm{l}} ) +
( \mathrm{x} - \upmu_{\mathrm{l}} )^{\top} \Sigma_{\mathrm{l}}^{-1} \left( \mathrm{x} - {\pdv{}{ [ \mathrm{a}_{\mathrm{j}}^{(i)} | \sigma ] }} \left( \upmu_{\mathrm{l}}( \mathrm{a}_{\mathrm{j}}^{(i)} ) \right) \right)
\left. \vphantom{ \left( \mathrm{x} - {\pdv{}{\mathrm{a}_{\mathrm{j}}^{(i)}}} \left( \upmu_{\mathrm{l}}( \mathrm{a}_{\mathrm{j}}^{(i)} ) \right) \right)^{\top} } \right),
$
}
\newline
\newline
where:
\newline
\newline
$P = \frac{1}{ \frac{1}{K^d} {\displaystyle \sum \limits_{ \mathrm{l} \in \{ 0,\ldots,K-1 \}^d }}
\frac{1}{ \sqrt{ (2 \pi)^n \lvert \Sigma_{\mathrm{l}} \rvert }}
e^{ -\frac{1}{2} ( \mathrm{x} - \upmu_{\mathrm{l}} )^{\top} \Sigma_{\mathrm{l}}^{-1} ( \mathrm{x} - \upmu_{\mathrm{l}} ) } }
$
\newline
\newline
\begin{flalign*}
& {\pdv{}{ \mathrm{a}_{\mathrm{j}}^{(i)} }} \left( \upmu_{\mathrm{l}}( \mathrm{a}_{\mathrm{j}}^{(i)} ) \right) = \left( K^d \delta_{{i}{\bar{i}}} g \left( \mathrm{l},\mathrm{j} \right) \right)_{\bar{i}=1}^{n} &
\end{flalign*}
\begin{flalign*}
& {\pdv{}{ \sigma }} \upmu_{\mathrm{l}} = \vec{0} &
\end{flalign*}
\newline
\begin{flalign*}
& {\pdv{}{ \mathrm{a}_{\mathrm{j}}^{(i)} }} \left( \Sigma_{\mathrm{l}}( \mathrm{a}_{\mathrm{j}}^{(i)} ) \right) =&\\
&= \left(
K^d \left( 
\delta_{{i_2}{i}} \left( {\displaystyle \sum \limits_{\mathrm{l}^1 \in \{ -k,\ldots,k \}^d}} \mathrm{a}_{\mathrm{l}^1}^{(i_1)} g \left( \mathrm{l}, \mathrm{l}^1, \mathrm{j} \right) \right) +
\delta_{{i_1}{i}} \left( {\displaystyle \sum \limits_{\mathrm{l}^2 \in \{ -k,\ldots,k \}^d}} \mathrm{a}_{\mathrm{l}^2}^{(i_1)} g \left( \mathrm{l}, \mathrm{j}, \mathrm{l}^2 \right) \right)
\right) \right. -&\\
& \phantom{=} - \left.
\vphantom{
K^d \left( {\displaystyle \sum \limits_{\substack{\mathrm{l}^1 \in \{ -k,\ldots,k \}^d \\ \mathrm{l}^2 \in \{ -k,\ldots,k \}^d}}} {\pdv{}{ \mathrm{a}_{\mathrm{j}}^{(i)} }} \left(
\mathrm{a}_{\mathrm{l}^1}^{(i_1)} \mathrm{a}_{\mathrm{l}^2}^{(i_2)} g \left( \mathrm{l}, \mathrm{l}^1, \mathrm{l}^2 \right)
\right) \right)
}
{\pdv{}{ \mathrm{a}_{\mathrm{j}}^{(i)} }} \left( \upmu_{\mathrm{j}}^{(i_1)} \left( \mathrm{a}_{\mathrm{j}}^{(i)} \right) \right) \upmu_{\mathrm{j}}^{(i_2)} -
\upmu_{\mathrm{j}}^{(i_1)} {\pdv{}{ \mathrm{a}_{\mathrm{j}}^{(i)} }} \left( \upmu_{\mathrm{j}}^{(i_2)} \left( \mathrm{a}_{\mathrm{j}}^{(i)} \right) \right)
\right)_{i_1, i_2 = 1}^{n}
\end{flalign*}
\newline
\begin{flalign*}
& {\pdv{}{ \sigma }} \left( \Sigma_{\mathrm{l}}( \sigma ) \right) = 2 \sigma \mathrm{I} &
\end{flalign*}
\end{restatable}
\begin{proof}
\renewcommand{\qedsymbol}{}
See ~\ref{theoremtwo-proof}.
\end{proof}

\setcounter{section}{4}

\newcolumntype{C}[1]{>{\centering\arraybackslash}m{#1}}

\section{Experiments}
\label{Section5}


In this section, we conduct a case study of our algorithm on four classes of synthetic data sets (data clustering) and 2D images (ellipse fitting). Each synthetic data consists of points sampled on the smoothly varying number of curves given by the Fourier series of orders ranging from $1$~to~$4$ (each order for one class of the tests). For simplicity, we restrained ourselves only to the two-dimensional data, i.e., the one-dimensional curves embedded in the $\mathbb{R}^2$ space. 


As for the clustering task, we compare our algorithm MCEC with three other state-of-the-art methods: afCEC (using the Hartigan heuristics), CEC, and the GMM by computing three performance metrics: MLE, BIC, and AIC. To make our comparison fair, since MCEC operates on the "whole" curves, we assign the bigger number of clusters to the three remaining methods: $2$x for the classes: $1$  and $2$ and $4$x for the classes: $3$ and $4$. Additionally, for the MCEC, we compute the Rand Index and the Jaccard Index since we have access to the original clustering of the data sets.

In the last subsection devoted to the ellipse fitting task, we present the comparison with two state-of-art methods: the Least squares and the Hough transform.



\subsection{Closed curves given by the Fourier series of order $1$}

The performance metrics provided in Table~\ref{experiment1-1} indicate that MCEC outperforms the other three methods in terms of every performance metric in all 8 test cases. Expectedly, the advantage of our algorithm over the reference methods tends to be the smallest in the case of the afCEC algorithm and the biggest in the case of the GMM. It's worth mentioning that in most of the test cases, the advantage of our algorithm becomes visible just from the first start.

\begin{table}[]
\caption{Experiment No. 1: Values of the performance metrics computed for the final number of starts being the power of two.}
\centering
\resizebox*{!}{0.65\textheight}{%
\begin{tabular}{clcccc}
\toprule
& & MCEC & afCEC Hartigan & CEC & GMM \\
\midrule
\multicolumn{1}{c}{\multirow{5}{*}{\includegraphics[scale=0.07]{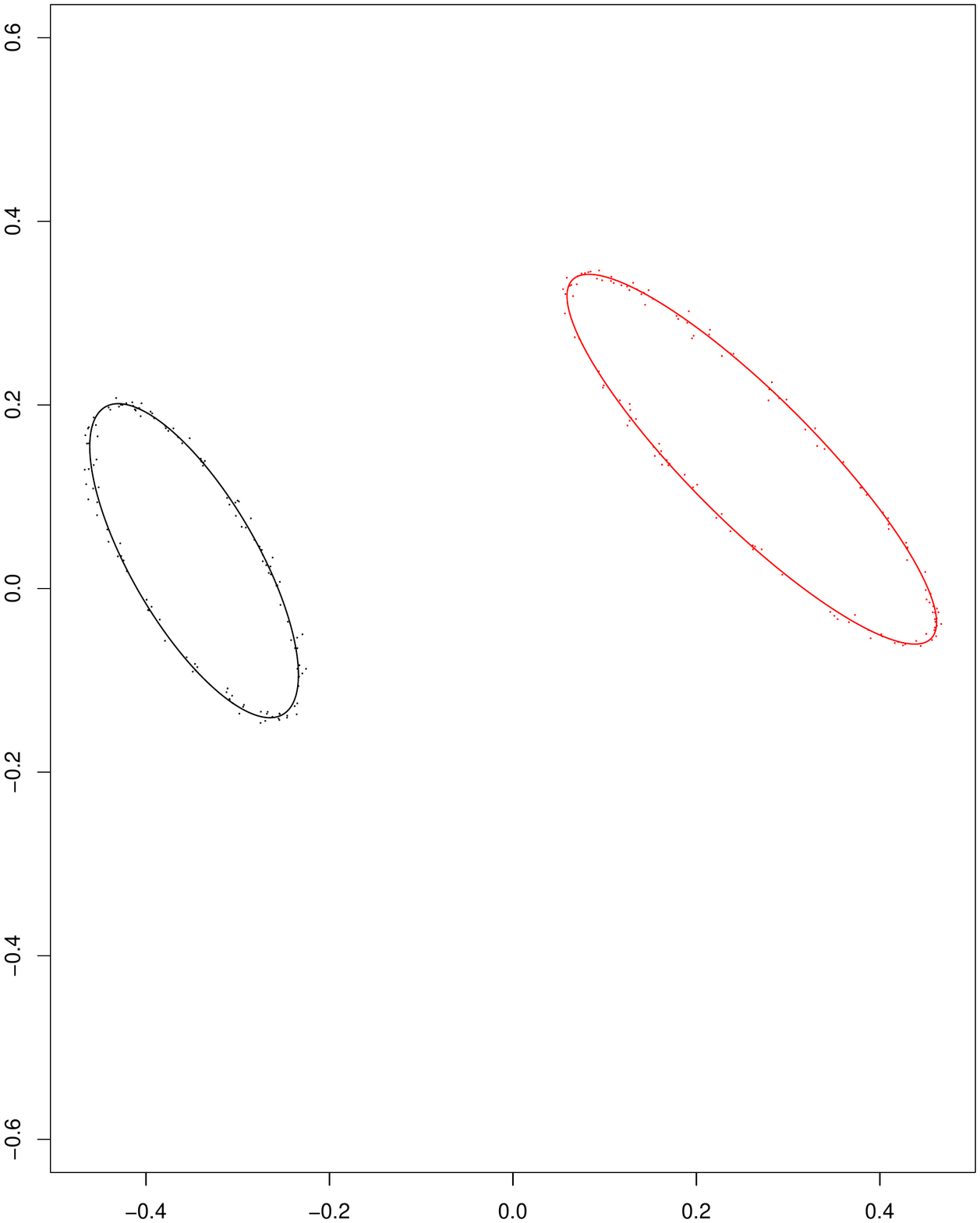}}} & MLE & \textbf{880.46} & 579.97 & 453.68 & 396.61 \\
\multicolumn{1}{c}{} & BIC & \textbf{-1677.74} & -988.04 & -779.81 & -665.68 \\
\multicolumn{1}{c}{} & AIC & \textbf{-1730.92} & -1097.94 & -861.35 & -747.22 \\
\multicolumn{1}{c}{} & Rand index & 1.00 & - & - & - \\
\multicolumn{1}{c}{} & Jaccard index & 1.00 & - & - & - \\
\midrule
\midrule
\multicolumn{1}{c}{\multirow{5}{*}{\includegraphics[scale=0.07]{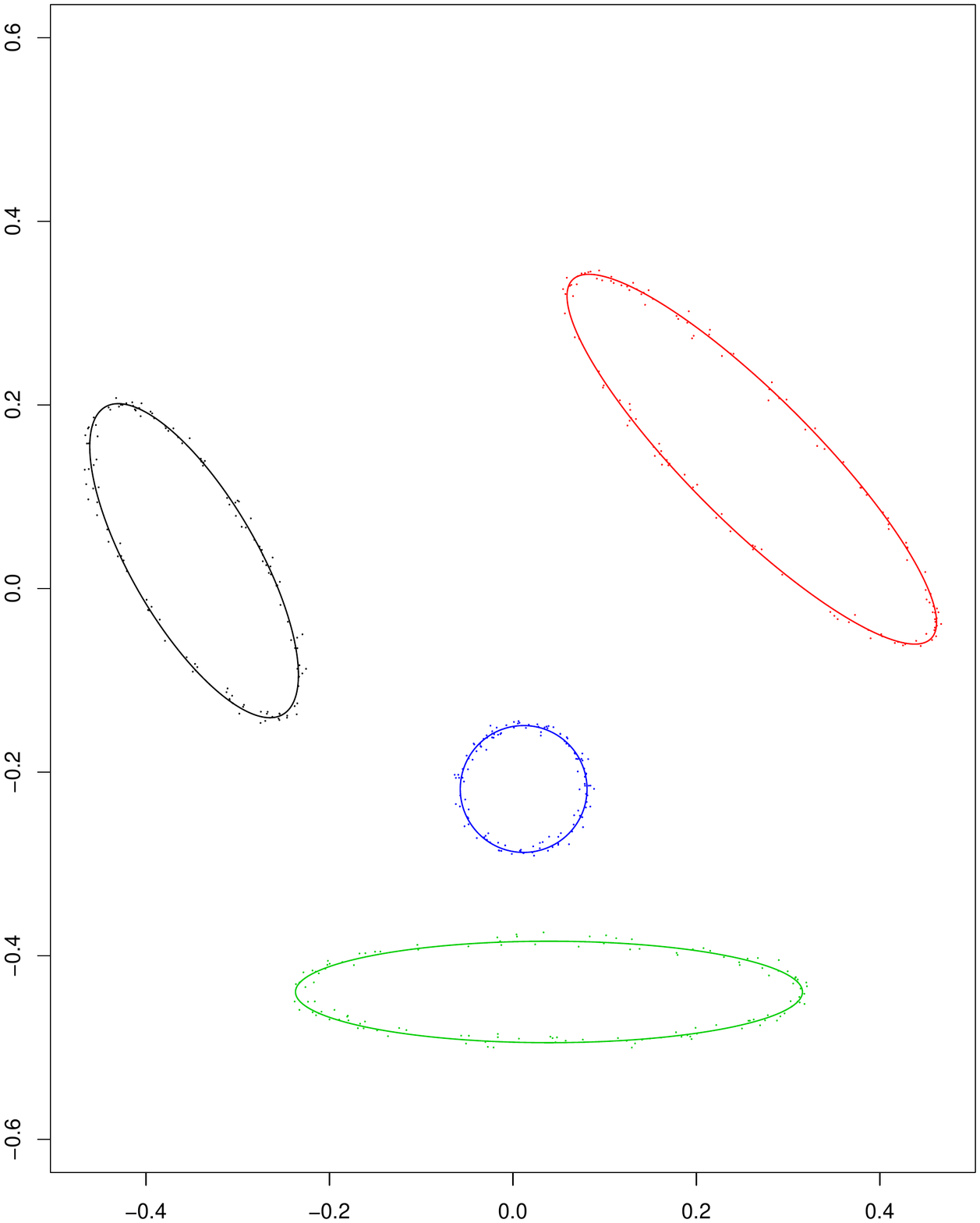}}} & MLE & \textbf{1481.77} & 1062.49 & 631.59 & 568.79 \\
\multicolumn{1}{c}{} & BIC & \textbf{-2770.15} & -1731.98 & -969.98 & -844.38 \\
\multicolumn{1}{c}{} & AIC & \textbf{-2901.53} & -1998.99 & -1169.18 & -1043.58 \\
\multicolumn{1}{c}{} & Rand index & 1.00 & - & - & - \\
\multicolumn{1}{c}{} & Jaccard index & 1.00 & - & - & - \\
\midrule
\midrule
\multicolumn{1}{c}{\multirow{5}{*}{\includegraphics[scale=0.07]{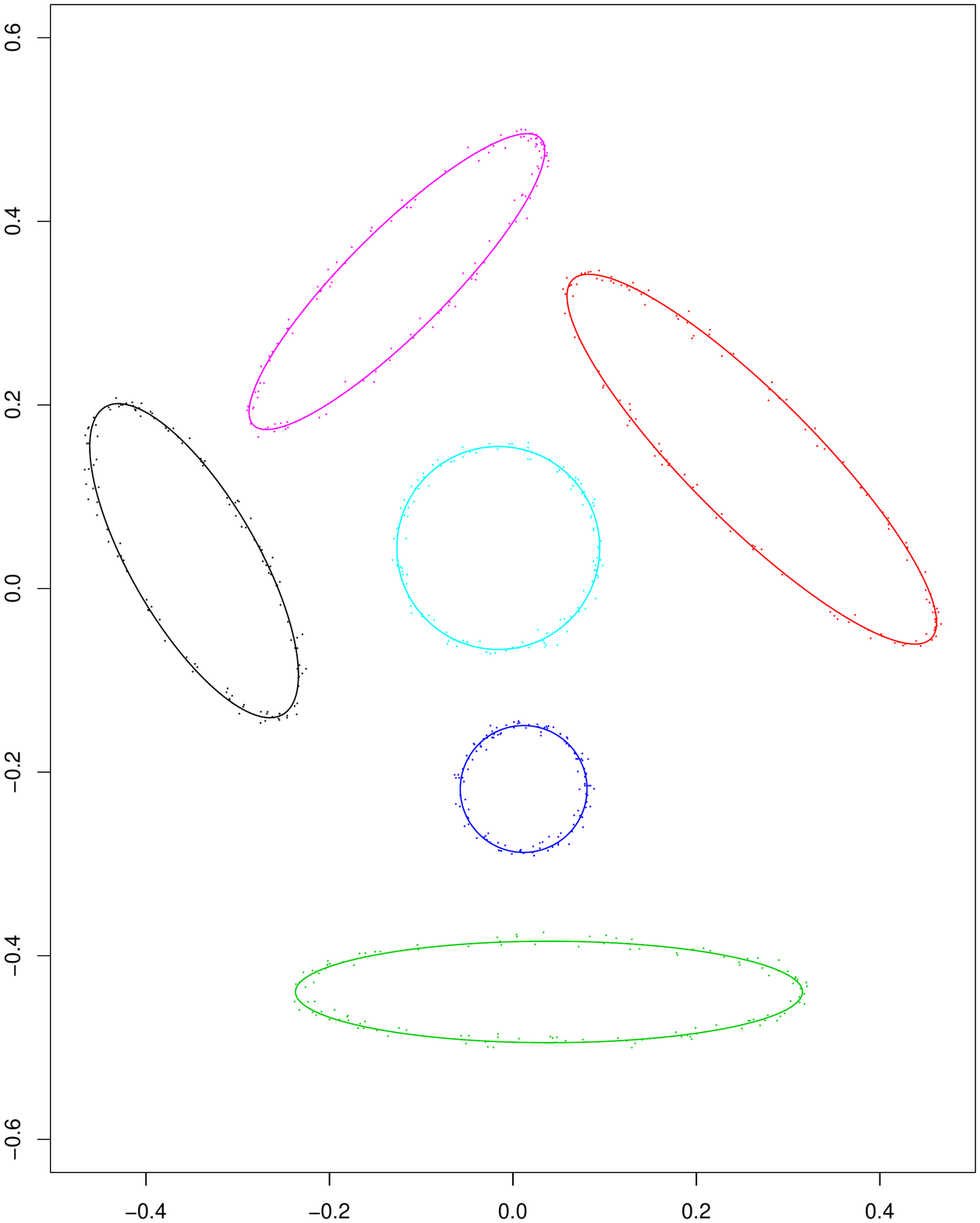}}} & MLE & \textbf{1898.91} & 1312.19 & 715.22 & 602.87 \\
\multicolumn{1}{c}{} & BIC & \textbf{-3485.57} & -1993.21 & -958.73 & -734.03 \\
\multicolumn{1}{c}{} & AIC & \textbf{-3703.83} & -2434.37 & -1288.44 & -1063.74 \\
\multicolumn{1}{c}{} & Rand index & 1.00 & - & - & - \\
\multicolumn{1}{c}{} & Jaccard index & 1.00 & - & - & - \\
\midrule
\multicolumn{1}{c}{\multirow{5}{*}{\includegraphics[scale=0.07]{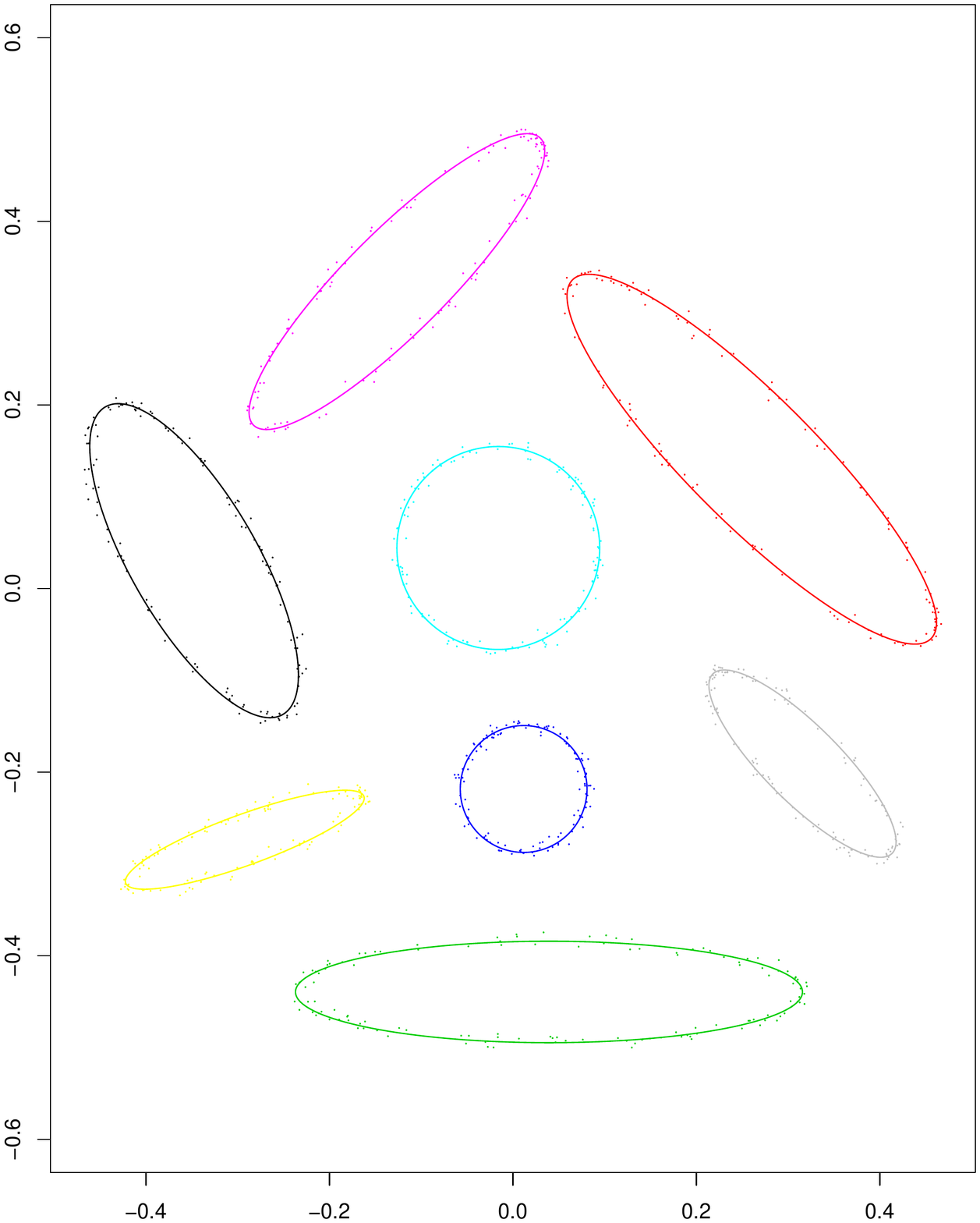}}} & MLE & \textbf{2340.88} & 1569.73 & 817.11 & 683.47 \\
\multicolumn{1}{c}{} & BIC & \textbf{-4245.08} & -2259.17 & -975.74 & -708.45 \\
\multicolumn{1}{c}{} & AIC & \textbf{-4555.76} & -2885.47 & -1444.23 & -1176.94 \\
\multicolumn{1}{c}{} & Rand index & 1.00 & - & - & - \\
\multicolumn{1}{c}{} & Jaccard index & 1.00 & - & - & - \\
\bottomrule
\end{tabular}
}
\label{experiment1-1}
\end{table}


\subsection{Closed curves given by the Fourier series of order $2$}

The performance metrics provided in Table~\ref{experiment2-1} indicate that MCEC significantly outperforms the other three methods in terms of every performance metric in all 4 test cases. Expectedly, similarly as in the case of Experiment No. 1, the advantage of our algorithm over the reference methods tends to be the smallest in the case of the afCEC algorithm and the biggest in the case of the GMM. It's worth mentioning that in most of the test cases, the advantage of our algorithm becomes visible just from the first start.

\begin{table}[]
\caption{Experiment No. 2: Values of the performance metrics computed for the final number of starts being the power of two.}
\centering
\footnotesize
\begin{tabular}{clcccc}
\toprule
& & MCEC & afCEC Hartigan & CEC & GMM \\
\midrule
\multicolumn{1}{c}{\multirow{5}{*}{\includegraphics[scale=0.07]{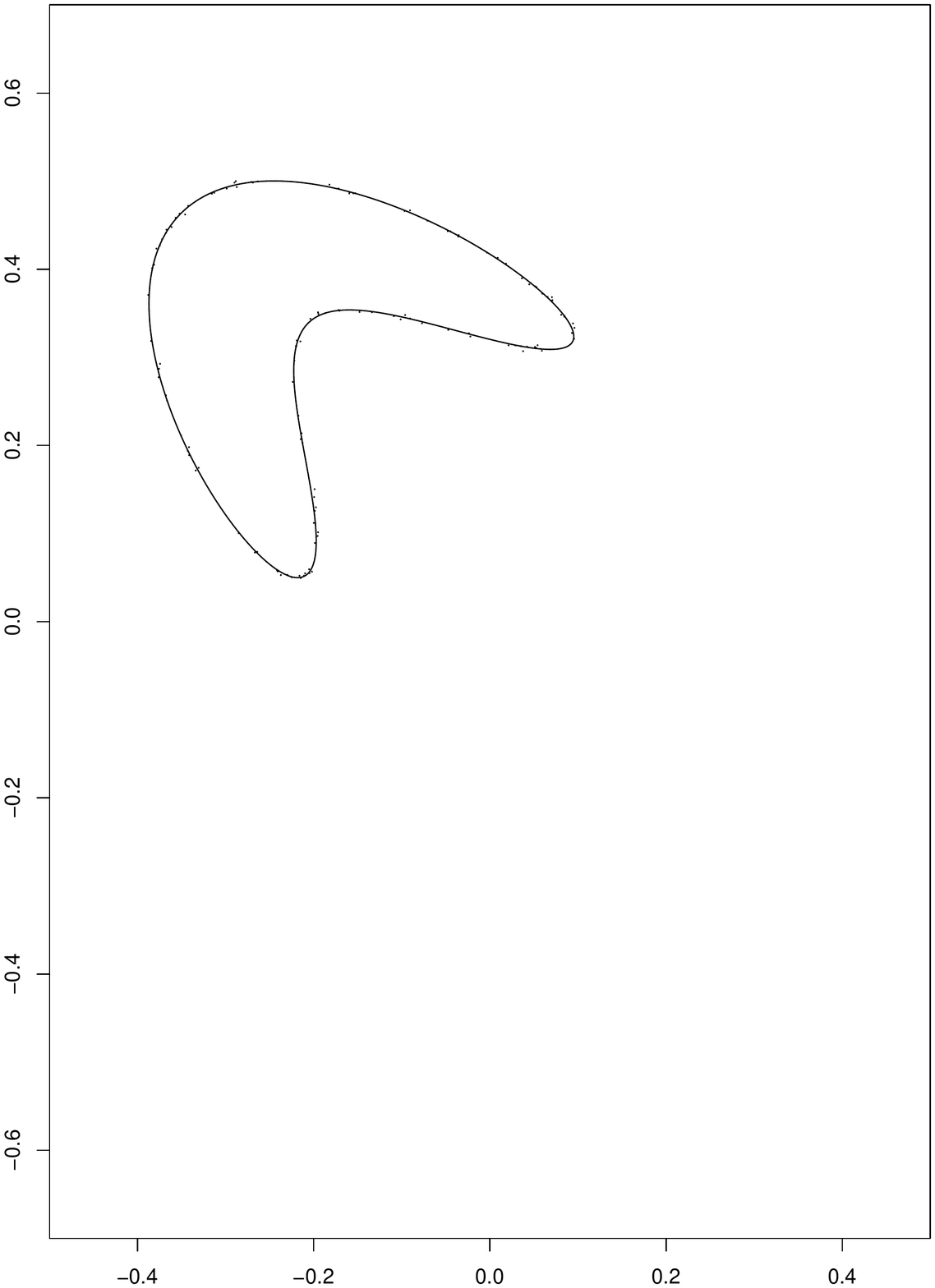}}} & MLE & \textbf{542.49} & 229.02 & 195.71 & 185.35 \\
\multicolumn{1}{c}{} & BIC & \textbf{-1031.62} & -385.26 & -338.05 & -317.33 \\
\multicolumn{1}{c}{} & AIC & \textbf{-1062.99} & -428.04 & -369.42 & -348.70 \\
\multicolumn{1}{c}{} & Rand index & 1.00 & - & - & - \\
\multicolumn{1}{c}{} & Jaccard index & 1.00 & - & - & - \\
\midrule
\multicolumn{1}{c}{\multirow{5}{*}{\includegraphics[scale=0.07]{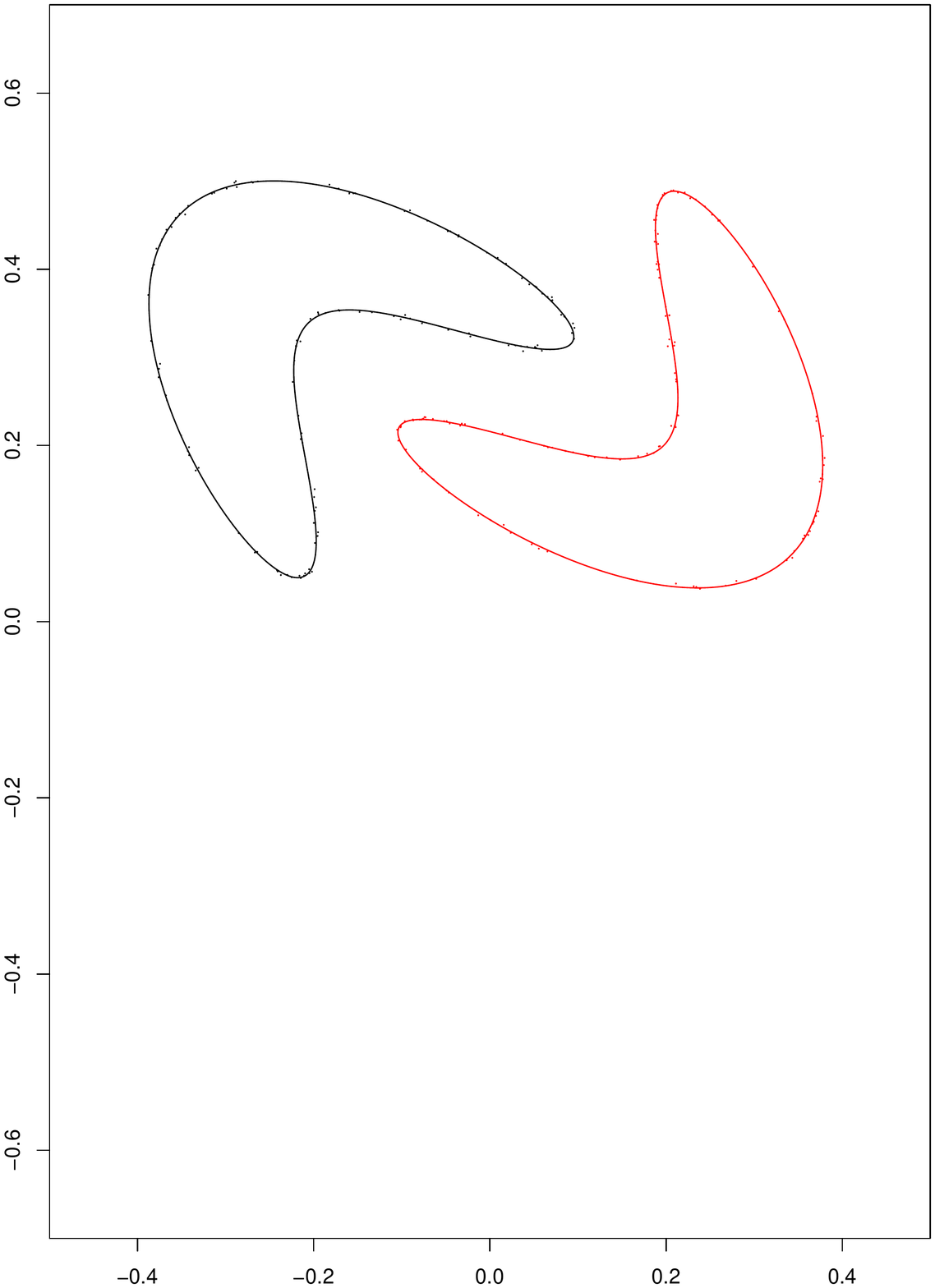}}} & MLE & \textbf{955.79} & 381.74 & 307.84 & 268.87 \\
\multicolumn{1}{c}{} & BIC & \textbf{-1784.03} & -591.58 & -488.14 & -410.20 \\
\multicolumn{1}{c}{} & AIC & \textbf{-1865.57} & -701.48 & -569.68 & -491.74 \\
\multicolumn{1}{c}{} & Rand index & 1.00 & - & - & - \\
\multicolumn{1}{c}{} & Jaccard index & 1.00 & - & - & - \\
\midrule
\multicolumn{1}{c}{\multirow{5}{*}{\includegraphics[scale=0.07]{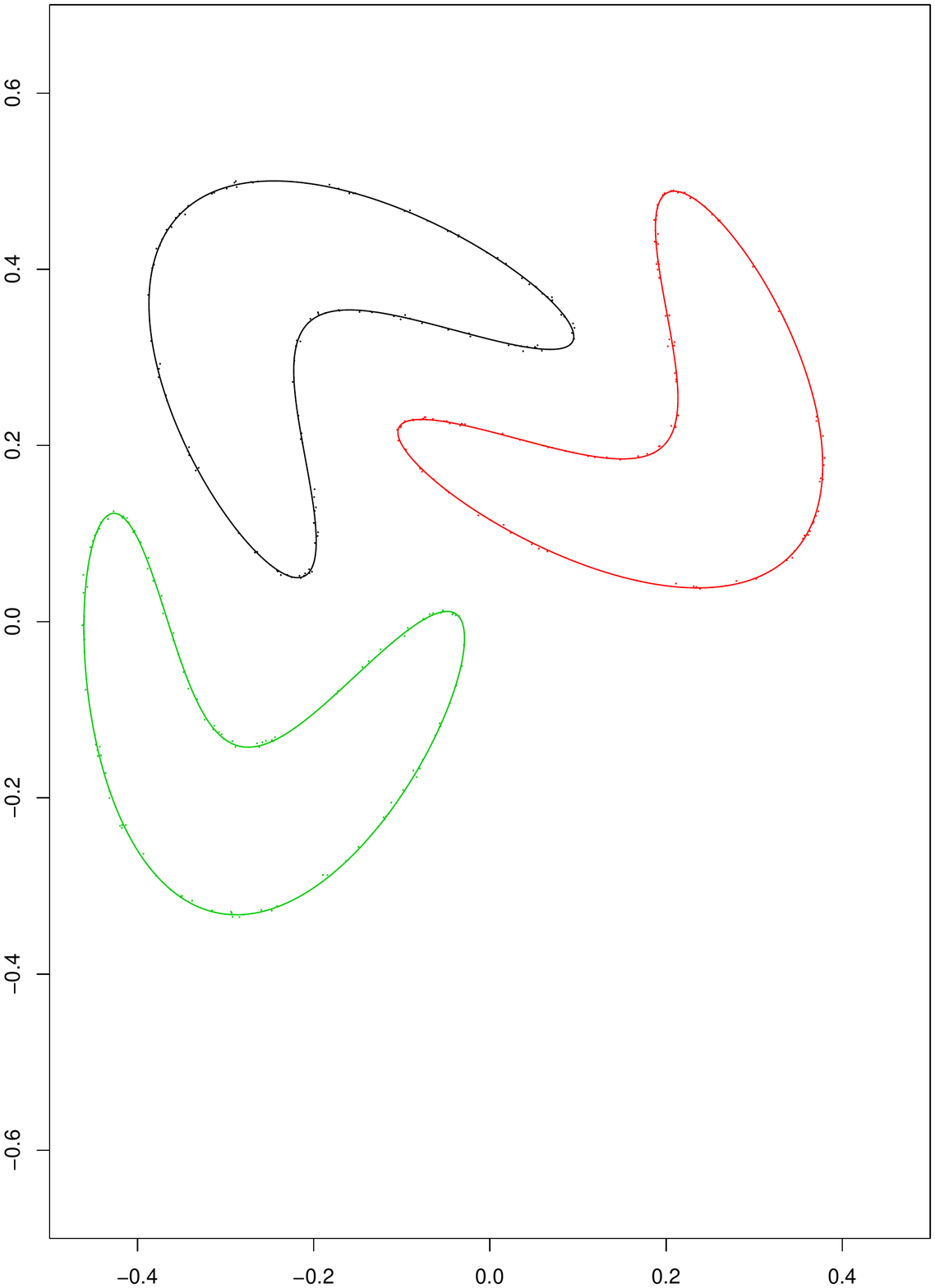}}} & MLE & \textbf{1261.75} & 451.79 & 261.29 & 233.12 \\
\multicolumn{1}{c}{} & BIC & \textbf{-2315.23} & -623.90 & -314.32 & -257.96 \\
\multicolumn{1}{c}{} & AIC & \textbf{-2453.50} & -809.59 & -452.59 & -396.23 \\
\multicolumn{1}{c}{} & Rand index & 1.00 & - & - & - \\
\multicolumn{1}{c}{} & Jaccard index & 1.00 & - & - & - \\
\midrule
\multicolumn{1}{c}{\multirow{5}{*}{\includegraphics[scale=0.07]{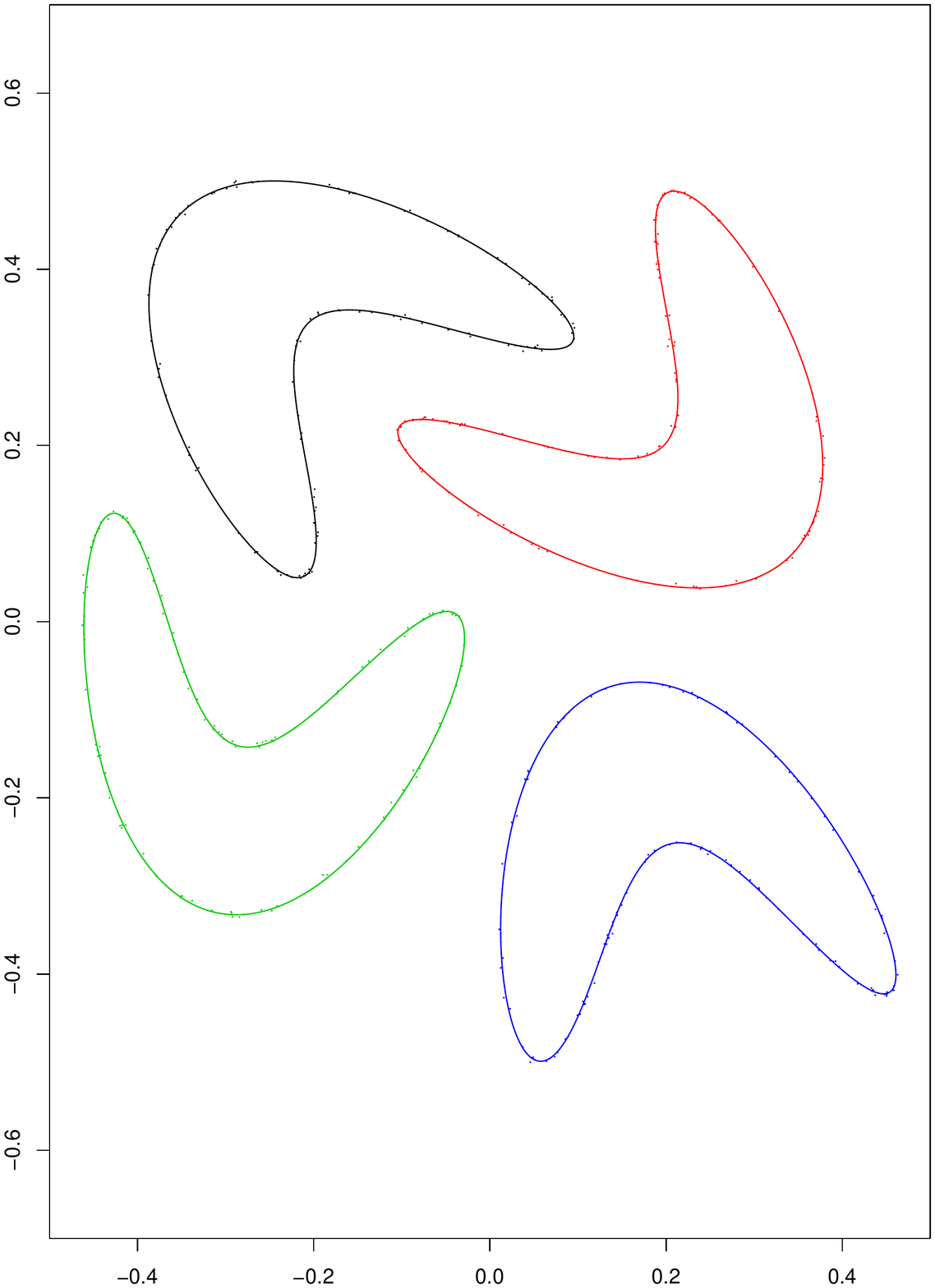}}} & MLE & \textbf{1551.54} & 432.14 & 245.16 & 164.01 \\
\multicolumn{1}{c}{} & BIC & \textbf{-2809.89} & -471.27 & -197.12 & -34.82 \\
\multicolumn{1}{c}{} & AIC & \textbf{-3009.09} & -738.28 & -396.32 & -234.02 \\
\multicolumn{1}{c}{} & Rand index & 1.00 & - & - & - \\
\multicolumn{1}{c}{} & Jaccard index & 1.00 & - & - & - \\
\bottomrule
\end{tabular}
\label{experiment2-1}
\end{table}


\begin{figure}[H]
	\centering
	\includegraphics[width=0.95\textwidth]{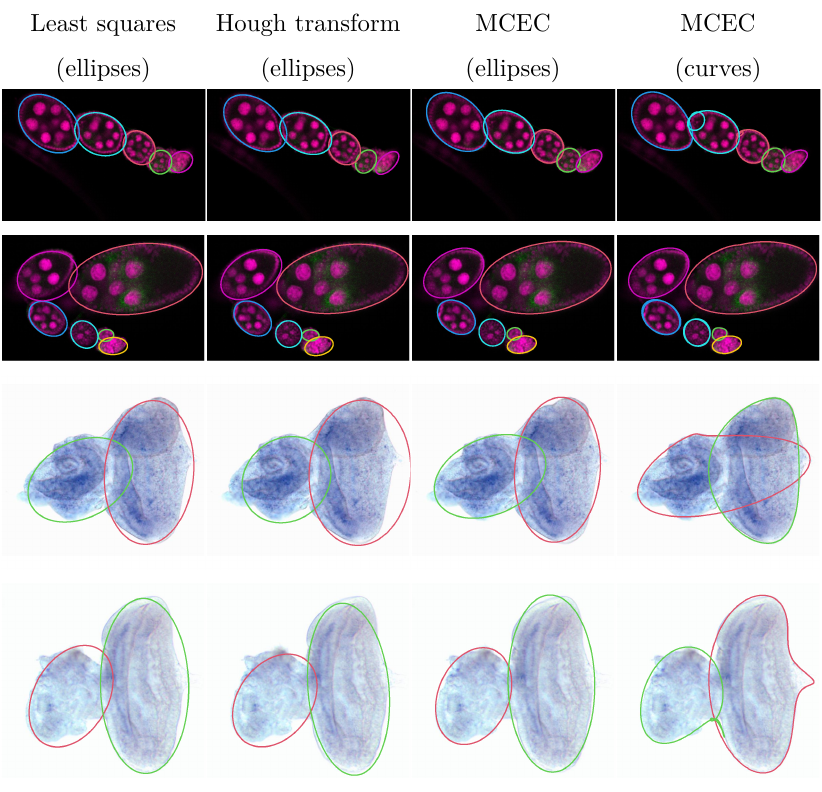}
	\caption{
 In the figure, we present the comparison between our MCEC and classical methods of ellipses fitting: Least squares and Hough transform on the images from the paper \cite{borovec2017supervised}: insitu7544, insitu7545, img\_5 and img\_14. As we can see, algorithms for the ellipses fitting give similar results. Thanks to the curved model, we can describe the positions and shapes of the objects.
 }
 \label{fig_img}
\end{figure}

\subsection{Closed curves given by the Fourier series of order $3$}

The performance metrics provided in Table~\ref{experiment3-1} indicate that MCEC significantly outperforms the other three methods in terms of every performance metric in all 2 test cases. Expectedly, similarly as in the case of the previous experiments, the advantage of our algorithm over the reference methods tends to be smallest in the case of the afCEC algorithm, where both of the algorithms perform comparably to each other, and biggest in the case of the GMM. Unlike the previous two experiments, our algorithm's advantage becomes visible from the first start only for the CEC and GMM algorithms.

\begin{table}[]
\caption{Experiment No. 3: Values of the performance metrics computed for the final number of starts being the power of two}
\centering
\footnotesize
\begin{tabular}{clcccc}
\toprule
& & MCEC & afCEC Hartigan & CEC & GMM \\
\midrule
\multicolumn{1}{c}{\multirow{5}{*}{\includegraphics[scale=0.07]{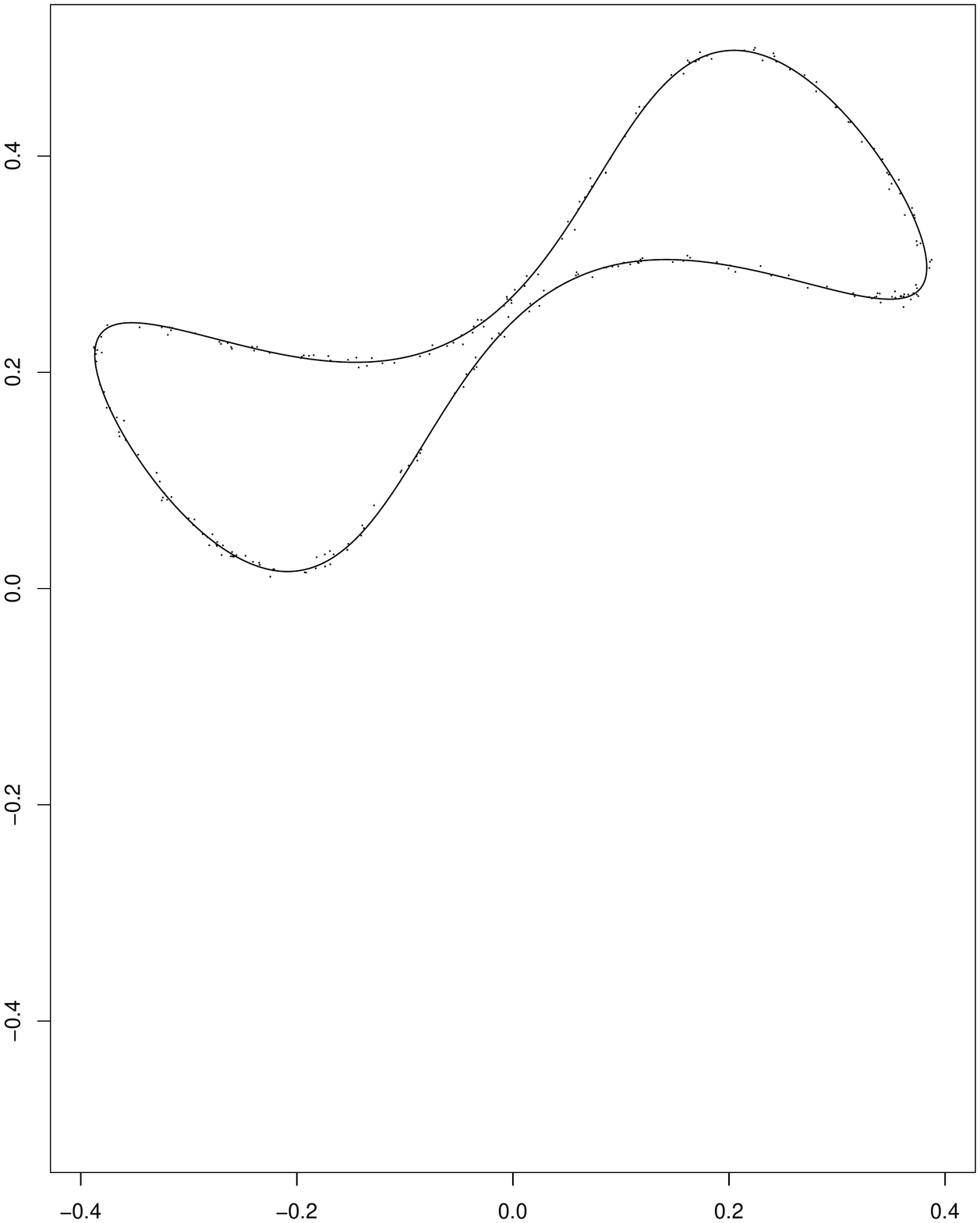}}} & MLE & \textbf{821.85} & 742.68 & 507.09 & 390.28 \\
\multicolumn{1}{c}{} & BIC & \textbf{-1560.51} & -1313.47 & -886.63 & -653.01 \\
\multicolumn{1}{c}{} & AIC & \textbf{-1613.69} & -1423.37 & -968.17 & -734.55 \\
\multicolumn{1}{c}{} & Rand index & 1.00 & - & - & - \\
\multicolumn{1}{c}{} & Jaccard index & 1.00 & - & - & - \\
\midrule
\multicolumn{1}{c}{\multirow{5}{*}{\includegraphics[scale=0.07]{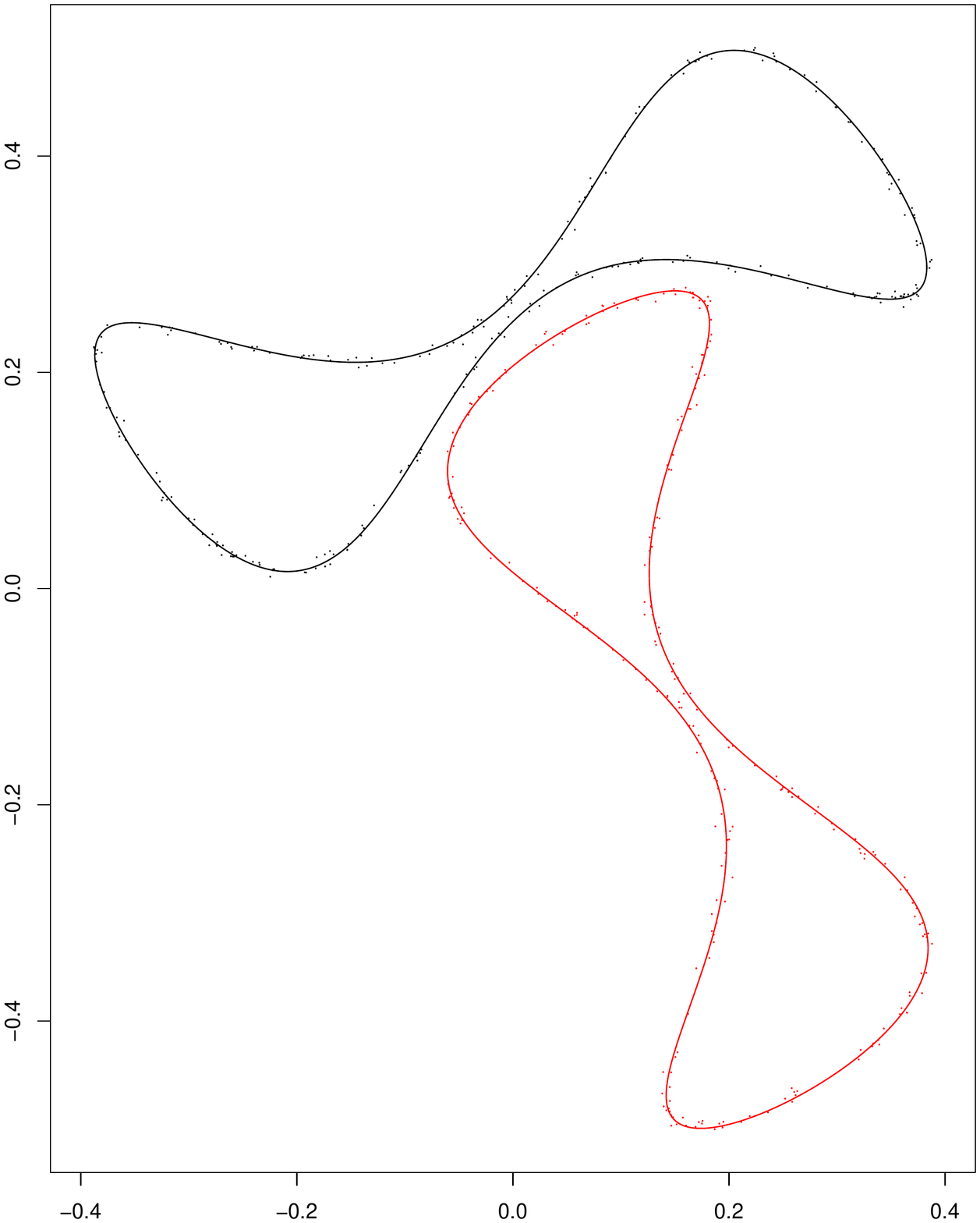}}} & MLE & \textbf{1175.35} & 1112.05 & 803.08 & 503.88 \\
\multicolumn{1}{c}{} & BIC & \textbf{-2157.31} & -1831.09 & -1312.95 & -714.56 \\
\multicolumn{1}{c}{} & AIC & \textbf{-2288.69} & -2098.11 & -1512.15 & -913.76 \\
\multicolumn{1}{c}{} & Rand index & 0.97 & - & - & - \\
\multicolumn{1}{c}{} & Jaccard index & 0.93 & - & - & - \\
\bottomrule
\end{tabular}

\label{experiment3-1}
\end{table}


\subsection{Closed curves given by the Fourier series of order $4$}

The performance metrics provided in Table~\ref{experiment4-1} indicate that MCEC significantly outperforms the other three methods in terms of every performance metric in all 4 test cases. Expectedly, similarly as in the case of the previous experiments, the advantage of our algorithm over the reference methods tends to be the smallest in the case of the afCEC algorithm and the biggest in the case of the GMM.

\begin{table}[]
\caption{Experiment No. 4: Values of the performance metrics computed for the final number of starts being the power of two.}
\centering
\footnotesize
\begin{tabular}{clcccc}
\toprule
& & MCEC & afCEC Hartigan & CEC & GMM \\
\midrule
\multicolumn{1}{c}{\multirow{5}{*}{\includegraphics[scale=0.07]{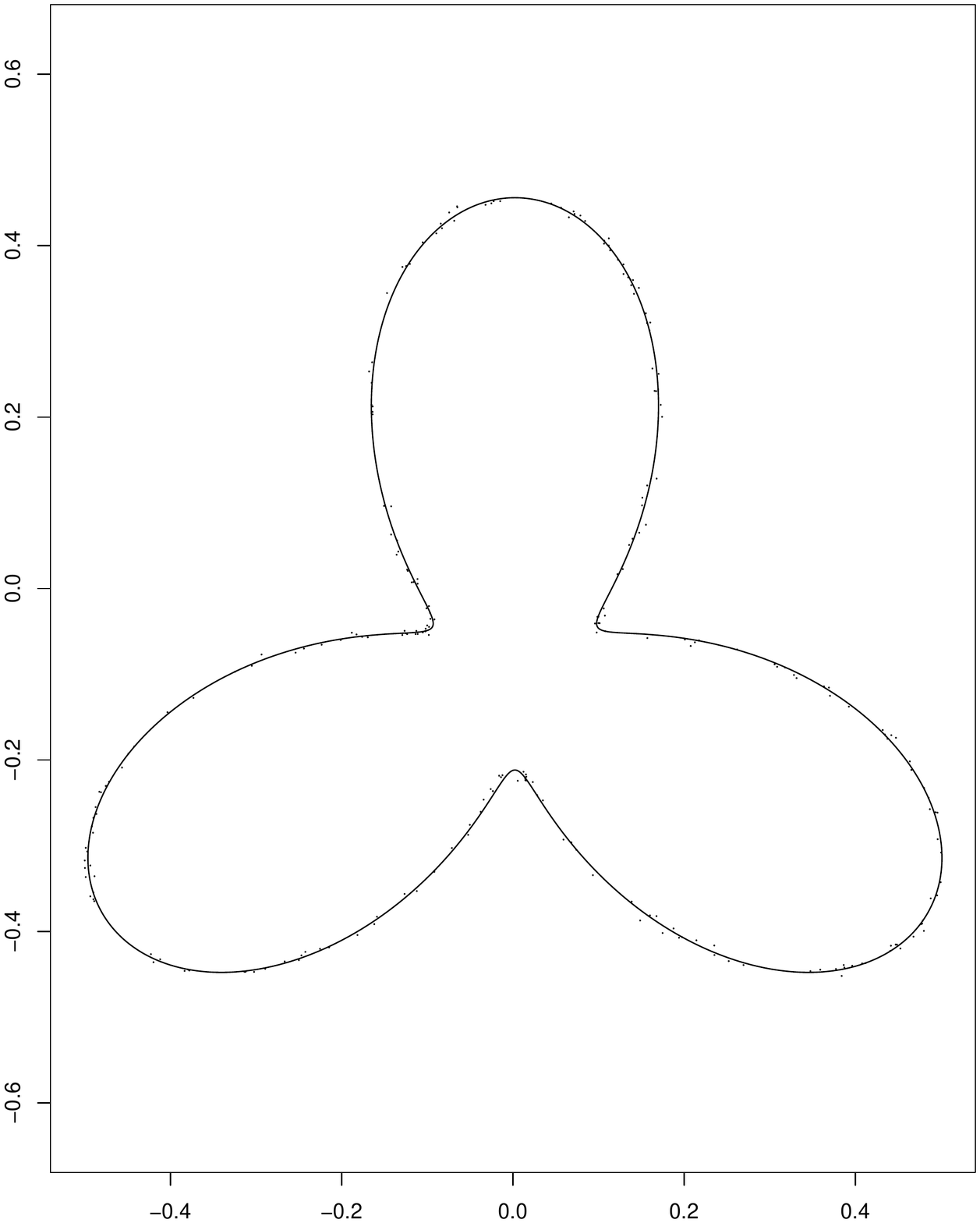}}} & MLE & \textbf{758.31} & 387.79 & 123.44 & 73.48 \\
\multicolumn{1}{c}{} & BIC & \textbf{-1411.26} & -603.69 & -119.34 & -19.43 \\
\multicolumn{1}{c}{} & AIC & \textbf{-1478.62} & -713.59 & -200.88 & -100.97 \\
\multicolumn{1}{c}{} & Rand index & 1.00 & - & - & - \\
\multicolumn{1}{c}{} & Jaccard index & 1.00 & - & - & - \\
\bottomrule
\end{tabular}
\label{experiment4-1}
\end{table}

\subsection{2D images}

\przemek{This subsection compares our algorithm with the classical approaches to ellipses fitting on 2D images: the least squares method and the Hough transform. As we can see, the methods of ellipses fitting work similarly to our model. Thanks to the curve fitting capability, we can get not only positions but also the shapes of objects, see Fig.~\ref{fig_img}}



\section{Conclusions}
\label{Section6}

In this work, we propose a new approach to the density representation of the closed curve, which was used for data clustering. Our approach automatically detects the number of components and is able to approximate the complicated shapes. We show in experiments that such an approach gives noticeably better results than the classical mixture models.

\newpage

\newpage

\newpage

\newpage

\newpage

\section{Appendix A}
\label{AppendixA}

\theoremone*
\begin{proof}
\label{theoremone-proof}
a) By definition of the mean vectors $\upmu_{\mathrm{j}}$, we have:
\newline
\newline
\newline
$\upmu_{\mathrm{j}}^{(i)} = K^d { \displaystyle \int \limits_{\mathrm{j}_d / K}^{(\mathrm{j}_d+1)/K}} \ldots { \displaystyle \int \limits_{\mathrm{j}_1 / K}^{(\mathrm{j}_1+1)/K}} \upvarphi^{(i)}(\mathrm{s}) \dd{\mathrm{s}_1} \ldots \dd{\mathrm{s}_d}$
\newline
\newline
\newline
where $i=1, \ldots, n$ and $j \in \{ 0,\ldots,K-1 \}^d$. Since, by assumption, we can approximate our closed $d$-dimensional surface embedded in $\mathbb{R}^n$ by the multidimensional Fourier series of order $k$, we obtain:

{
\footnotesize
\begin{flalign*}
\upmu_{\mathrm{j}}^{(i)} &=
K^d { \displaystyle \int \limits_{\mathrm{j}_d / K}^{(\mathrm{j}_d+1)/K}} \ldots { \displaystyle \int \limits_{\mathrm{j}_1 / K}^{(\mathrm{j}_1+1)/K}}
{\displaystyle \sum \limits_{\mathrm{l} \in \{ -k,\ldots,k \}^d}} \mathrm{a}_{\mathrm{l}}^{(i)} {\displaystyle \prod \limits_{m=1}^{d}}
\left \lbrace
\begin{array}{ll}
\cos(-2 \pi \mathrm{l}_m \mathrm{s}_m) & \textrm{, $\mathrm{l}_m < 0$} \\
1 & \textrm{, $\mathrm{l}_m = 0$} \\
\sin(2 \pi \mathrm{l}_m \mathrm{s}_m) & \textrm{, $\mathrm{l}_m > 0$}
\end{array}
\right.
\dd{\mathrm{s}_1} \ldots \dd{\mathrm{s}_d} =&\\
& &\\
&= K^d {\displaystyle \sum \limits_{\mathrm{l} \in \{ -k,\ldots,k \}^d}} \mathrm{a}_{\mathrm{l}}^{(i)} { \displaystyle \int \limits_{\mathrm{j}_d / K}^{(\mathrm{j}_d+1)/K}} \ldots { \displaystyle \int \limits_{\mathrm{j}_1 / K}^{(\mathrm{j}_1+1)/K}}
{\displaystyle \prod \limits_{m=1}^{d}}
\left \lbrace
\begin{array}{ll}
\cos(-2 \pi \mathrm{l}_m \mathrm{s}_m) & \textrm{, $\mathrm{l}_m < 0$} \\
1 & \textrm{, $\mathrm{l}_m = 0$} \\
\sin(2 \pi \mathrm{l}_m \mathrm{s}_m) & \textrm{, $\mathrm{l}_m > 0$}
\end{array}
\right.
\dd{\mathrm{s}_1} \ldots \dd{\mathrm{s}_d} =&\\
& &\\
&= K^d {\displaystyle \sum \limits_{\mathrm{l} \in \{ -k,\ldots,k \}^d}} \mathrm{a}_{\mathrm{l}}^{(i)} {\displaystyle \prod \limits_{m=1}^{d}}
\left(
{ \displaystyle \int \limits_{\mathrm{j}_m / K}^{(\mathrm{j}_m+1)/K}}
\left \lbrace
\begin{array}{ll}
\cos(-2 \pi \mathrm{l}_m \mathrm{s}_m) & \textrm{, $\mathrm{l}_m < 0$} \\
1 & \textrm{, $\mathrm{l}_m = 0$} \\
\sin(2 \pi \mathrm{l}_m \mathrm{s}_m) & \textrm{, $\mathrm{l}_m > 0$}
\end{array}
\right.
\dd{\mathrm{s}_m} \right) =&\\
& &\\
&= K^d {\displaystyle \sum \limits_{\mathrm{l} \in \{ -k,\ldots,k \}^d}} \mathrm{a}_{\mathrm{l}}^{(i)} {\displaystyle \prod \limits_{m=1}^{d}}
\left \lbrace
\begin{array}{ll}
\left. -\frac{1}{2 \pi \mathrm{l}_m} \sin(-2 \pi \mathrm{l}_m \mathrm{s}_m) \right \rvert_{\mathrm{s}_m=\mathrm{j}_m / K}^{\mathrm{s}_m=(\mathrm{j}_m+1) / K} & \textrm{, $\mathrm{l}_m < 0$} \\
\\
\left. \mathrm{s}_m \right \rvert_{\mathrm{s}_m=\mathrm{j}_m / K}^{\mathrm{s}_m=(\mathrm{j}_m+1) / K} & \textrm{, $\mathrm{l}_m = 0$} \\
\\
\left. -\frac{1}{2 \pi \mathrm{l}_m} \cos(2 \pi \mathrm{l}_m \mathrm{s}_m) \right \rvert_{\mathrm{s}_m=\mathrm{j}_m / K}^{\mathrm{s}_m=(\mathrm{j}_m+1) / K} & \textrm{, $\mathrm{l}_m > 0$}
\end{array}
\right.
\end{flalign*}
}

\noindent
\newline
\newline
b) By definition of the covariance matrices $\Sigma_{\mathrm{j}}$, we have:
\newline
\newline
\newline
{
\footnotesize
$
\Sigma_{\mathrm{j}}^{(i_1)(i_2)} = {\sigma^2}\mathrm{I}_{{i_1}{i_2}} + K^d { \displaystyle \int \limits_{\mathrm{j}_d / K}^{(\mathrm{j}_d+1)/K}} \ldots { \displaystyle \int \limits_{\mathrm{j}_1 / K}^{(\mathrm{j}_1+1)/K}}
(\upvarphi^{(i_1)}(\mathrm{s})-\upmu_{\mathrm{j}}^{(i_1)})(\upvarphi^{(i_2)}(\mathrm{s})-\upmu_{\mathrm{j}}^{(i_2)}) 
\dd{\mathrm{s}_1} \ldots \dd{\mathrm{s}_d}
$
}
\newline
\newline
\newline
First, let's note that:
{
\footnotesize
\begin{flalign*}
\Sigma_{\mathrm{j}}^{(i_1)(i_2)} &=
{\sigma^2}\mathrm{I}_{{i_1}{i_2}} + K^d { \displaystyle \int \limits_{\mathrm{j}_d / K}^{(\mathrm{j}_d+1)/K}} \ldots { \displaystyle \int \limits_{\mathrm{j}_1 / K}^{(\mathrm{j}_1+1)/K}}
(\upvarphi^{(i_1)}(\mathrm{s})-\upmu_{\mathrm{j}}^{(i_1)})(\upvarphi^{(i_2)}(\mathrm{s})-\upmu_{\mathrm{j}}^{(i_2)}) 
\dd{\mathrm{s}_1} \ldots \dd{\mathrm{s}_d} =&\\
& &\\
&= {\sigma^2}\mathrm{I}_{{i_1}{i_2}} +&\\
& \phantom{=} + K^d { \displaystyle \int \limits_{\mathrm{j}_d / K}^{(\mathrm{j}_d+1)/K}} \ldots { \displaystyle \int \limits_{\mathrm{j}_1 / K}^{(\mathrm{j}_1+1)/K}}
\upvarphi^{(i_1)}(\mathrm{s}) \upvarphi^{(i_2)}(\mathrm{s})
\dd{\mathrm{s}_1} \ldots \dd{\mathrm{s}_d} -&\\
& \phantom{=} - \upmu_{\mathrm{j}}^{(i_2)} \left(
K^d { \displaystyle \int \limits_{\mathrm{j}_d / K}^{(\mathrm{j}_d+1)/K}} \ldots { \displaystyle \int \limits_{\mathrm{j}_1 / K}^{(\mathrm{j}_1+1)/K}} \upvarphi^{(i_1)}(\mathrm{s}) \dd{\mathrm{s}_1} \ldots \dd{\mathrm{s}_d}
\right) -&\\
& \phantom{=} - \upmu_{\mathrm{j}}^{(i_1)} \left(
K^d { \displaystyle \int \limits_{\mathrm{j}_d / K}^{(\mathrm{j}_d+1)/K}} \ldots { \displaystyle \int \limits_{\mathrm{j}_1 / K}^{(\mathrm{j}_1+1)/K}} \upvarphi^{(i_2)}(\mathrm{s}) \dd{\mathrm{s}_1} \ldots \dd{\mathrm{s}_d}
\right) +&\\
& \phantom{=} + \upmu_{\mathrm{j}}^{(i_1)} \upmu_{\mathrm{j}}^{(i_2)} \left(
K^d { \displaystyle \int \limits_{\mathrm{j}_d / K}^{(\mathrm{j}_d+1)/K}} \ldots { \displaystyle \int \limits_{\mathrm{j}_1 / K}^{(\mathrm{j}_1+1)/K}} 1 \dd{\mathrm{s}_1} \ldots \dd{\mathrm{s}_d}
\right) =&\\
& &\\
&= {\sigma^2}\mathrm{I}_{{i_1}{i_2}} +
\left( K^d { \displaystyle \int \limits_{\mathrm{j}_d / K}^{(\mathrm{j}_d+1)/K}} \ldots { \displaystyle \int \limits_{\mathrm{j}_1 / K}^{(\mathrm{j}_1+1)/K}}
\upvarphi^{(i_1)}(\mathrm{s}) \upvarphi^{(i_2)}(\mathrm{s})
\dd{\mathrm{s}_1} \ldots \dd{\mathrm{s}_d}\right)
- \upmu_{\mathrm{j}}^{(i_1)} \upmu_{\mathrm{j}}^{(i_2)}
\end{flalign*}
}

\noindent
\newline
Now, it remains to show how to compute the value of the integral
\newline
\newline
\newline
$
K^d { \displaystyle \int \limits_{\mathrm{j}_d / K}^{(\mathrm{j}_d+1)/K}} \ldots { \displaystyle \int \limits_{\mathrm{j}_1 / K}^{(\mathrm{j}_1+1)/K}}
\upvarphi^{(i_1)}(\mathrm{s}) \upvarphi^{(i_2)}(\mathrm{s})
\dd{\mathrm{s}_1} \ldots \dd{\mathrm{s}_d}
$
\newline
\newline
\newline
Using the same argumentation as in point a), we obtain:
\newline
{
\tiny
\begin{flalign*}
& K^d { \displaystyle \int \limits_{\mathrm{j}_d / K}^{(\mathrm{j}_d+1)/K}} \ldots { \displaystyle \int \limits_{\mathrm{j}_1 / K}^{(\mathrm{j}_1+1)/K}}
\upvarphi^{(i_1)}(\mathrm{s}) \upvarphi^{(i_2)}(\mathrm{s})
\dd{\mathrm{s}_1} \ldots \dd{\mathrm{s}_d} =&\\
& &\\
&= K^d {\displaystyle \sum \limits_{\substack{\mathrm{l}^1 \in \{ -k,\ldots,k \}^d \\ \mathrm{l}^2 \in \{ -k,\ldots,k \}^d}}} \mathrm{a}_{\mathrm{l}^1}^{(i_1)} \mathrm{a}_{\mathrm{l}^2}^{(i_2)} {\displaystyle \prod \limits_{m=1}^{d}}
\left(
{ \displaystyle \int \limits_{\mathrm{j}_m / K}^{(\mathrm{j}_m+1)/K}}
\left \lbrace
\begin{array}{ll}
\cos(-2 \pi \mathrm{l}_m^1 \mathrm{s}_m) \cos(-2 \pi \mathrm{l}_m^2 \mathrm{s}_m) & \textrm{, $\mathrm{l}_m^1 < 0$, $\mathrm{l}_m^2 < 0$} \\
\cos(-2 \pi \mathrm{l}_m^1 \mathrm{s}_m) & \textrm{, $\mathrm{l}_m^1 < 0$, $\mathrm{l}_m^2 = 0$} \\
\cos(-2 \pi \mathrm{l}_m^1 \mathrm{s}_m) \sin(2 \pi \mathrm{l}_m^2 \mathrm{s}_m) & \textrm{, $\mathrm{l}_m^1 < 0$, $\mathrm{l}_m^2 > 0$} \\
\cos(-2 \pi \mathrm{l}_m^2 \mathrm{s}_m)  & \textrm{, $\mathrm{l}_m^1 = 0$, $\mathrm{l}_m^2 < 0$} \\
1 & \textrm{, $\mathrm{l}_m^1 = 0$, $\mathrm{l}_m^2 = 0$} \\
\sin(2 \pi \mathrm{l}_m^2 \mathrm{s}_m) & \textrm{, $\mathrm{l}_m^1 = 0$, $\mathrm{l}_m^2 > 0$} \\
\sin(2 \pi \mathrm{l}_m^1 \mathrm{s}_m) \cos(-2 \pi \mathrm{l}_m^2 \mathrm{s}_m) & \textrm{, $\mathrm{l}_m^1 > 0$, $\mathrm{l}_m^2 < 0$} \\
\sin(2 \pi \mathrm{l}_m^1 \mathrm{s}_m) & \textrm{, $\mathrm{l}_m^1 > 0$, $\mathrm{l}_m^2 = 0$} \\
\sin(2 \pi \mathrm{l}_m^1 \mathrm{s}_m) \sin(2 \pi \mathrm{l}_m^2 \mathrm{s}_m) & \textrm{, $\mathrm{l}_m^1 > 0$, $\mathrm{l}_m^2 > 0$}
\end{array}
\right.
\dd{\mathrm{s}_m} \right) &\\
\end{flalign*}
}

\noindent
Applying some trigonometric identities involving products of the sine and cosine function, we get:
\newline

$
K^d { \displaystyle \int \limits_{\mathrm{j}_d / K}^{(\mathrm{j}_d+1)/K}} \ldots { \displaystyle \int \limits_{\mathrm{j}_1 / K}^{(\mathrm{j}_1+1)/K}}
\upvarphi^{(i_1)}(\mathrm{s}) \upvarphi^{(i_2)}(\mathrm{s})
\dd{\mathrm{s}_1} \ldots \dd{\mathrm{s}_d}=
$
\newline

\resizebox{\linewidth}{!}{
$
=K^d {\displaystyle \sum \limits_{\substack{\mathrm{l}^1 \in \{ -k,\ldots,k \}^d \\ \mathrm{l}^2 \in \{ -k,\ldots,k \}^d}}} \mathrm{a}_{\mathrm{l}^1}^{(i_1)} \mathrm{a}_{\mathrm{l}^2}^{(i_2)} {\displaystyle \prod \limits_{m=1}^{d}}
\left(
{ \displaystyle \int \limits_{\mathrm{j}_m / K}^{(\mathrm{j}_m+1)/K}}
\left \lbrace
\begin{array}{ll}
\frac{ \cos(-2 \pi \mathrm{s}_m ( \mathrm{l}_m^1 + \mathrm{l}_m^2 )) + \cos(-2 \pi \mathrm{s}_m ( \mathrm{l}_m^1 - \mathrm{l}_m^2 )) } {2} & \textrm{, $\mathrm{l}_m^1 < 0$, $\mathrm{l}_m^2 < 0$} \\
\\
\cos(-2 \pi \mathrm{l}_m^1 \mathrm{s}_m) & \textrm{, $\mathrm{l}_m^1 < 0$, $\mathrm{l}_m^2 = 0$} \\
\\
\frac{ \sin(-2 \pi \mathrm{s}_m ( \mathrm{l}_m^1 - \mathrm{l}_m^2 )) - \sin(-2 \pi \mathrm{s}_m ( \mathrm{l}_m^1 + \mathrm{l}_m^2 )) } {2} & \textrm{, $\mathrm{l}_m^1 < 0$, $\mathrm{l}_m^2 > 0$} \\
\\
\cos(-2 \pi \mathrm{l}_m^2 \mathrm{s}_m)  & \textrm{, $\mathrm{l}_m^1 = 0$, $\mathrm{l}_m^2 < 0$} \\
\\
1 & \textrm{, $\mathrm{l}_m^1 = 0$, $\mathrm{l}_m^2 = 0$} \\
\\
\sin(2 \pi \mathrm{l}_m^2 \mathrm{s}_m) & \textrm{, $\mathrm{l}_m^1 = 0$, $\mathrm{l}_m^2 > 0$} \\
\\
\frac{ \sin(2 \pi \mathrm{s}_m ( \mathrm{l}_m^1 - \mathrm{l}_m^2 )) + \sin(2 \pi \mathrm{s}_m ( \mathrm{l}_m^1 + \mathrm{l}_m^2 )) } {2} & \textrm{, $\mathrm{l}_m^1 > 0$, $\mathrm{l}_m^2 < 0$} \\
\\
\sin(2 \pi \mathrm{l}_m^1 \mathrm{s}_m) & \textrm{, $\mathrm{l}_m^1 > 0$, $\mathrm{l}_m^2 = 0$} \\
\\
\frac{ \cos(2 \pi \mathrm{s}_m ( \mathrm{l}_m^1 - \mathrm{l}_m^2 )) - \cos(2 \pi \mathrm{s}_m ( \mathrm{l}_m^1 + \mathrm{l}_m^2 )) } {2} & \textrm{, $\mathrm{l}_m^1 > 0$, $\mathrm{l}_m^2 > 0$}
\end{array}
\right.
\dd{\mathrm{s}_m} \right)=
$
}


\resizebox{\linewidth}{!}{
$
= K^d {\displaystyle \sum \limits_{\substack{\mathrm{l}^1 \in \{ -k,\ldots,k \}^d \\ \mathrm{l}^2 \in \{ -k,\ldots,k \}^d}}} \mathrm{a}_{\mathrm{l}^1}^{(i_1)} \mathrm{a}_{\mathrm{l}^2}^{(i_2)} {\displaystyle \prod \limits_{m=1}^{d}}
\left \lbrace
\begin{array}{ll}
\left \lbrace
\begin{array}{ll}
\left. -\frac{ \sin(-2 \pi \mathrm{s}_m ( \mathrm{l}_m^1 + \mathrm{l}_m^2 )) }{ 4 \pi ( \mathrm{l}_m^1 + \mathrm{l}_m^2 ) } -
\frac{ \sin(-2 \pi \mathrm{s}_m ( \mathrm{l}_m^1 - \mathrm{l}_m^2 )) }{ 4 \pi ( \mathrm{l}_m^1 - \mathrm{l}_m^2 ) } \right \rvert_{\mathrm{s}_m=\mathrm{j}_m / K}^{\mathrm{s}_m=(\mathrm{j}_m+1) / K}
& \textrm{, $\mathrm{l}_m^1 \ne \mathrm{l}_m^2$} \\
\\
\left. -\frac{ \sin(-2 \pi \mathrm{s}_m ( \mathrm{l}_m^1 + \mathrm{l}_m^2 )) }{ 4 \pi ( \mathrm{l}_m^1 + \mathrm{l}_m^2 ) } + \frac{\mathrm{s}_m}{2} \right \rvert_{\mathrm{s}_m=\mathrm{j}_m / K}^{\mathrm{s}_m=(\mathrm{j}_m+1) / K}
& \textrm{, $\mathrm{l}_m^1 = \mathrm{l}_m^2$} \\
\end{array}
\right.
& \textrm{, $\mathrm{l}_m^1 < 0$, $\mathrm{l}_m^2 < 0$} \\
\\
\left. -\frac{1}{2 \pi \mathrm{l}_m^1} \sin(-2 \pi \mathrm{l}_m^1 \mathrm{s}_m) \right \rvert_{\mathrm{s}_m=\mathrm{j}_m / K}^{\mathrm{s}_m=(\mathrm{j}_m+1) / K} & \textrm{, $\mathrm{l}_m^1 < 0$, $\mathrm{l}_m^2 = 0$} \\
\\
\left \lbrace
\begin{array}{ll}
\left. \frac{ \cos(-2 \pi \mathrm{s}_m ( \mathrm{l}_m^1 - \mathrm{l}_m^2 )) } { 4 \pi ( \mathrm{l}_m^1 - \mathrm{l}_m^2 ) } -
\frac{ \cos(-2 \pi \mathrm{s}_m ( \mathrm{l}_m^1 + \mathrm{l}_m^2 )) }{ 4 \pi ( \mathrm{l}_m^1 + \mathrm{l}_m^2 ) } \right \rvert_{\mathrm{s}_m=\mathrm{j}_m / K}^{\mathrm{s}_m=(\mathrm{j}_m+1) / K}
& \textrm{, $\mathrm{l}_m^1 \ne - \mathrm{l}_m^2$} \\
\\
\left. \frac{ \cos(-2 \pi \mathrm{s}_m ( \mathrm{l}_m^1 - \mathrm{l}_m^2 )) } { 4 \pi ( \mathrm{l}_m^1 - \mathrm{l}_m^2 ) } \right \rvert_{\mathrm{s}_m=\mathrm{j}_m / K}^{\mathrm{s}_m=(\mathrm{j}_m+1) / K}
& \textrm{, $\mathrm{l}_m^1 = - \mathrm{l}_m^2$} \\
\end{array}
\right.
& \textrm{, $\mathrm{l}_m^1 < 0$, $\mathrm{l}_m^2 > 0$} \\
\\
\left. -\frac{1}{2 \pi \mathrm{l}_m^2} \sin(-2 \pi \mathrm{l}_m^2 \mathrm{s}_m) \right \rvert_{\mathrm{s}_m=\mathrm{j}_m / K}^{\mathrm{s}_m=(\mathrm{j}_m+1) / K}  & \textrm{, $\mathrm{l}_m^1 = 0$, $\mathrm{l}_m^2 < 0$} \\
\\
\left. \mathrm{s}_m \right \rvert_{\mathrm{s}_m=\mathrm{j}_m / K}^{\mathrm{s}_m=(\mathrm{j}_m+1) / K} & \textrm{, $\mathrm{l}_m^1 = 0$, $\mathrm{l}_m^2 = 0$} \\
\\
\left. -\frac{1}{2 \pi \mathrm{l}_m^2} \cos(2 \pi \mathrm{l}_m^2 \mathrm{s}_m) \right \rvert_{\mathrm{s}_m=\mathrm{j}_m / K}^{\mathrm{s}_m=(\mathrm{j}_m+1) / K} & \textrm{, $\mathrm{l}_m^1 = 0$, $\mathrm{l}_m^2 > 0$} \\
\\
\left \lbrace
\begin{array}{ll}
\left. -\frac{ \cos(2 \pi \mathrm{s}_m ( \mathrm{l}_m^1 - \mathrm{l}_m^2 )) } { 4 \pi ( \mathrm{l}_m^1 - \mathrm{l}_m^2 ) } -
\frac{ \cos(2 \pi \mathrm{s}_m ( \mathrm{l}_m^1 + \mathrm{l}_m^2 )) }{ 4 \pi ( \mathrm{l}_m^1 + \mathrm{l}_m^2 ) } \right \rvert_{\mathrm{s}_m=\mathrm{j}_m / K}^{\mathrm{s}_m=(\mathrm{j}_m+1) / K}
& \textrm{, $\mathrm{l}_m^1 \ne - \mathrm{l}_m^2$} \\
\\
\left. -\frac{ \cos(2 \pi \mathrm{s}_m ( \mathrm{l}_m^1 - \mathrm{l}_m^2 )) } { 4 \pi ( \mathrm{l}_m^1 - \mathrm{l}_m^2 ) } \right \rvert_{\mathrm{s}_m=\mathrm{j}_m / K}^{\mathrm{s}_m=(\mathrm{j}_m+1) / K}
& \textrm{, $\mathrm{l}_m^1 = - \mathrm{l}_m^2$} \\
\end{array}
\right.
& \textrm{, $\mathrm{l}_m^1 > 0$, $\mathrm{l}_m^2 < 0$} \\
\\
\left. -\frac{1}{2 \pi \mathrm{l}_m^1} \cos(2 \pi \mathrm{l}_m^1 \mathrm{s}_m) \right \rvert_{\mathrm{s}_m=\mathrm{j}_m / K}^{\mathrm{s}_m=(\mathrm{j}_m+1) / K} & \textrm{, $\mathrm{l}_m^1 > 0$, $\mathrm{l}_m^2 = 0$} \\
\\
\left \lbrace
\begin{array}{ll}
\left. \frac{ \sin(2 \pi \mathrm{s}_m ( \mathrm{l}_m^1 - \mathrm{l}_m^2 )) }{ 4 \pi ( \mathrm{l}_m^1 - \mathrm{l}_m^2 ) } -
\frac{ \sin(2 \pi \mathrm{s}_m ( \mathrm{l}_m^1 + \mathrm{l}_m^2 )) }{ 4 \pi ( \mathrm{l}_m^1 + \mathrm{l}_m^2 ) } \right \rvert_{\mathrm{s}_m=\mathrm{j}_m / K}^{\mathrm{s}_m=(\mathrm{j}_m+1) / K}
& \textrm{, $\mathrm{l}_m^1 \ne \mathrm{l}_m^2$} \\
\\
\left. \frac{\mathrm{s}_m}{2} - \frac{ \sin(2 \pi \mathrm{s}_m ( \mathrm{l}_m^1 + \mathrm{l}_m^2 )) }{ 4 \pi ( \mathrm{l}_m^1 + \mathrm{l}_m^2 ) } \right \rvert_{\mathrm{s}_m=\mathrm{j}_m / K}^{\mathrm{s}_m=(\mathrm{j}_m+1) / K}
& \textrm{, $\mathrm{l}_m^1 = \mathrm{l}_m^2$} \\
\end{array}
\right.
& \textrm{, $\mathrm{l}_m^1 > 0$, $\mathrm{l}_m^2 > 0$}
\end{array}
\right. \\
$
}
\newline
\newline
\newline
Substituting the above expression into the formula for $\Sigma_{\mathrm{j}}^{(i_1)(i_2)}$ yields thesis.

\end{proof}


\theoremtwo*
\begin{proof}
\label{theoremtwo-proof}
We will write: $\upmu_{\mathrm{l}}( \mathrm{a}_{\mathrm{j}}^{(i)} )$ and $\Sigma_{\mathrm{l}}( \mathrm{a}_{\mathrm{j}}^{(i)}, \sigma )$ under the differentiation sign to emphasise that $\upmu_{\mathrm{l}}$ depends on $\mathrm{a}_{\mathrm{j}}^{(i)}$ and $\Sigma_{\mathrm{l}}$ depends on $\mathrm{a}_{\mathrm{j}}^{(i)}$ and $ \sigma $. By $ \pdv{}{ [ \mathrm{a}_{\mathrm{j}}^{(i)} | \sigma] } $ using the Backus-Naur-like notation we will denote the differentiation with respect \textbf{either} to the variable $ \mathrm{a}_{\mathrm{j}}^{(i)} $ \textbf{or} to the variable $\sigma$. Then:
\newline
\newline
$
{\pdv{}{ [ \mathrm{a}_{\mathrm{j}}^{(i)} | \sigma ] }} \left( {\displaystyle \sum \limits_{\mathrm{x} \in X}} \ln(f_{\mathcal{N}(\upvarphi(\mathrm{s}),\sigma^2)}(\mathrm{x})) \right) =
$
\newline
\newline
\newline
$
= {\pdv{}{ [ \mathrm{a}_{\mathrm{j}}^{(i)} | \sigma ] }} \left( {\displaystyle \sum \limits_{\mathrm{x} \in X}} \ln( \frac{1}{K^d} {\displaystyle \sum \limits_{ \mathrm{l} \in \{ 0,\ldots,K-1 \}^d }}
f_{ \mathcal{N} \left( \upmu_{\mathrm{l}}( \mathrm{a}_{\mathrm{j}}^{(i)} ), \Sigma_{\mathrm{l}}( \mathrm{a}_{\mathrm{j}}^{(i)}, \sigma ) \right) }(\mathrm{x}) ) \right) =
$
\newline
\newline
\newline
\resizebox{\linewidth}{!}{
$
= {\displaystyle \sum \limits_{\mathrm{x} \in X}} {\pdv{}{ [ \mathrm{a}_{\mathrm{j}}^{(i)} | \sigma ] }} \left( \ln( \frac{1}{K^d} {\displaystyle \sum \limits_{ \mathrm{l} \in \{ 0,\ldots,K-1 \}^d }}
\frac{1}{ \sqrt{ \left \lvert 2 \pi \Sigma_{\mathrm{l}}( \mathrm{a}_{\mathrm{j}}^{(i)}, \sigma ) \right \rvert }}
e^{ -\frac{1}{2} \left( \mathrm{x} - \upmu_{\mathrm{l}}( \mathrm{a}_{\mathrm{j}}^{(i)} ) \right)^{\top} \Sigma_{\mathrm{l}}^{-1}( \mathrm{a}_{\mathrm{j}}^{(i)}, \sigma ) \left( \mathrm{x} - \upmu_{\mathrm{l}}( \mathrm{a}_{\mathrm{j}}^{(i)} ) \right ) } ) \right)
$
}
\newline
\newline
\newline
Let's define:
$P \coloneqq \frac{1}{ \frac{1}{K^d} {\displaystyle \sum \limits_{ \mathrm{l} \in \{ 0,\ldots,K-1 \}^d }}
\frac{1}{ \sqrt{ (2 \pi)^n \lvert \Sigma_{\mathrm{l}} \rvert }}
e^{ -\frac{1}{2} ( \mathrm{x} - \upmu_{\mathrm{l}} )^{\top} \Sigma_{\mathrm{l}}^{-1} ( \mathrm{x} - \upmu_{\mathrm{l}} ) } }
$.
Then:
\newline
\newline
\newline
$
{\pdv{}{ [ \mathrm{a}_{\mathrm{j}}^{(i)} | \sigma ] }} \left( {\displaystyle \sum \limits_{\mathrm{x} \in X}} \ln(f_{\mathcal{N}(\upvarphi(\mathrm{s}),\sigma^2)}(\mathrm{x})) \right) =
$
\newline
\resizebox{\linewidth}{!}{
$
= {\displaystyle \sum \limits_{\mathrm{x} \in X}} P {\pdv{}{ [ \mathrm{a}_{\mathrm{j}}^{(i)} | \sigma ] }} \left( \frac{1}{K^d} {\displaystyle \sum \limits_{ \mathrm{l} \in \{ 0,\ldots,K-1 \}^d }}
\frac{1}{ \sqrt{ (2 \pi)^n \left \lvert \Sigma_{\mathrm{l}}( \mathrm{a}_{\mathrm{j}}^{(i)}, \sigma ) \right \rvert }}
e^{ -\frac{1}{2} \left( \mathrm{x} - \upmu_{\mathrm{l}}( \mathrm{a}_{\mathrm{j}}^{(i)} ) \right)^{\top} \Sigma_{\mathrm{l}}^{-1}( \mathrm{a}_{\mathrm{j}}^{(i)}, \sigma ) \left( \mathrm{x} - \upmu_{\mathrm{l}}( \mathrm{a}_{\mathrm{j}}^{(i)} ) \right) } \right) =
$
}
\newline
\resizebox{\linewidth}{!}{
$
= \frac{1}{K^d} {\displaystyle \sum \limits_{\mathrm{x} \in X}} P {\displaystyle \sum \limits_{ \mathrm{l} \in \{ 0,\ldots,K-1 \}^d }}
{\pdv{}{ [ \mathrm{a}_{\mathrm{j}}^{(i)} | \sigma ] }} \left( \frac{1}{ \sqrt{ (2 \pi)^n \left \lvert \Sigma_{\mathrm{l}}( \mathrm{a}_{\mathrm{j}}^{(i)}, \sigma ) \right \rvert }}
e^{ -\frac{1}{2} \left( \mathrm{x} - \upmu_{\mathrm{l}}( \mathrm{a}_{\mathrm{j}}^{(i)} ) \right)^{\top} \Sigma_{\mathrm{l}}^{-1}( \mathrm{a}_{\mathrm{j}}^{(i)}, \sigma ) \left( \mathrm{x} - \upmu_{\mathrm{l}}( \mathrm{a}_{\mathrm{j}}^{(i)} ) \right) } \right) =
$
}
\newline
\newline
\newline
$
=\frac{1}{K^d} {\displaystyle \sum \limits_{\mathrm{x} \in \mathrm{x}}} P {\displaystyle \sum \limits_{ \mathrm{l} \in \{ 0,\ldots,K-1 \}^d }}
 \left( {\pdv{}{ [ \mathrm{a}_{\mathrm{j}}^{(i)} | \sigma ] }} \left( \frac{1}{ \sqrt{ (2 \pi)^n \left \lvert \Sigma_{\mathrm{l}}( \mathrm{a}_{\mathrm{j}}^{(i)}, \sigma ) \right \rvert }} \right) \right)
e^{ -\frac{1}{2} ( \mathrm{x} - \upmu_{\mathrm{l}} )^{\top} \Sigma_{\mathrm{l}}^{-1} ( \mathrm{x} - \upmu_{\mathrm{l}} ) } +
$
\newline
\newline
$
\phantom{=} + \frac{1}{ \sqrt{ (2 \pi)^n \lvert \Sigma_{\mathrm{l}} \rvert }}
\left( {\pdv{}{ [ \mathrm{a}_{\mathrm{j}}^{(i)} | \sigma ] }} \left(
e^{ -\frac{1}{2} \left( \mathrm{x} - \upmu_{\mathrm{l}}( \mathrm{a}_{\mathrm{j}}^{(i)} ) \right)^{\top} \Sigma_{\mathrm{l}}^{-1}( \mathrm{a}_{\mathrm{j}}^{(i)}, \sigma ) \left( \mathrm{x} - \upmu_{\mathrm{l}}( \mathrm{a}_{\mathrm{j}}^{(i)} ) \right) }
\right) \right) =
$
\newline
\newline
\newline
\resizebox{\linewidth}{!}{
$
= \frac{1}{K^d} {\displaystyle \sum \limits_{\mathrm{x} \in \mathrm{x}}} P {\displaystyle \sum \limits_{ \mathrm{l} \in \{ 0,\ldots,K-1 \}^d }}
-\frac{1}{2} ( (2 \pi)^n \lvert \Sigma_{\mathrm{l}} \rvert )^{-\frac{3}{2}} ( 2 \pi )^n \left( {\pdv{}{ [ \mathrm{a}_{\mathrm{j}}^{(i)} | \sigma ] }} \left( \left \lvert \Sigma_{\mathrm{l}}( \mathrm{a}_{\mathrm{j}}^{(i)}, \sigma ) \right \rvert \right) \right)
e^{ -\frac{1}{2} ( \mathrm{x} - \upmu_{\mathrm{l}} )^{\top} \Sigma_{\mathrm{l}}^{-1} ( \mathrm{x} - \upmu_{\mathrm{l}} ) } +
$
}
\newline
\resizebox{\linewidth}{!}{
$
\phantom{=} + \frac{1}{ \sqrt{ (2 \pi)^n \lvert \Sigma_{\mathrm{l}} \rvert }} e^{ -\frac{1}{2} ( \mathrm{x} - \upmu_{\mathrm{l}} )^{\top} \Sigma_{\mathrm{l}}^{-1} ( \mathrm{x} - \upmu_{\mathrm{l}} ) }
\left( {\pdv{}{ [ \mathrm{a}_{\mathrm{j}}^{(i)} | \sigma ] }} \left(
-\frac{1}{2} \left( \mathrm{x} - \upmu_{\mathrm{l}}( \mathrm{a}_{\mathrm{j}}^{(i)} ) \right)^{\top} \Sigma_{\mathrm{l}}^{-1}( \mathrm{a}_{\mathrm{j}}^{(i)}, \sigma ) \left( \mathrm{x} - \upmu_{\mathrm{l}}( \mathrm{a}_{\mathrm{j}}^{(i)} ) \right)
\right) \right)
$
}
\newline
\newline
\newline
By Jacobi's formula, we have:
\newline
\begin{flalign*}
{\pdv{}{ [ \mathrm{a}_{\mathrm{j}}^{(i)} | \sigma ] }} \left( \left \lvert \Sigma_{\mathrm{l}}( \mathrm{a}_{\mathrm{j}}^{(i)}, \sigma ) \right \rvert \right) &=
\tr( \Sigma_{\mathrm{l}}^{\mathrm{D}} {\pdv{}{ [ \mathrm{a}_{\mathrm{j}}^{(i)} | \sigma ] }} \left( \Sigma_{\mathrm{l}}( \mathrm{a}_{\mathrm{j}}^{(i)}, \sigma ) \right) ) =&\\
&= \tr( \lvert \Sigma_{\mathrm{l}} \rvert \frac{1}{ \lvert \Sigma_{\mathrm{l}} \rvert } \Sigma_{\mathrm{l}}^{\mathrm{D}} {\pdv{}{ [ \mathrm{a}_{\mathrm{j}}^{(i)} | \sigma ] }} \left( \Sigma_{\mathrm{l}}( \mathrm{a}_{\mathrm{j}}^{(i)}, \sigma ) \right) ) =&\\
&= \lvert \Sigma_{\mathrm{l}} \rvert \tr( \Sigma_{\mathrm{l}}^{-1} {\pdv{}{ [ \mathrm{a}_{\mathrm{j}}^{(i)} | \sigma ] }} \left( \Sigma_{\mathrm{l}}( \mathrm{a}_{\mathrm{j}}^{(i)}, \sigma ) \right) )
\end{flalign*}
\newline
On the other hand, the chain rule for the matrix derivatives yields:
\newline
\newline
\newline
$
{\pdv{}{ [ \mathrm{a}_{\mathrm{j}}^{(i)} | \sigma ] }} \left( \left( \mathrm{x} - \upmu_{\mathrm{l}}( \mathrm{a}_{\mathrm{j}}^{(i)} ) \right)^{\top} \Sigma_{\mathrm{l}}^{-1}( \mathrm{a}_{\mathrm{j}}^{(i)}, \sigma )
\left( \mathrm{x} - \upmu_{\mathrm{l}}( \mathrm{a}_{\mathrm{j}}^{(i)} ) \right) \right) =
$
\newline
\newline
\newline
\resizebox{\linewidth}{!}{
$
= \left( {\pdv{}{ [ \mathrm{a}_{\mathrm{j}}^{(i)} | \sigma ] }} \left( \left( \mathrm{x} - \upmu_{\mathrm{l}}( \mathrm{a}_{\mathrm{j}}^{(i)} ) \right)^{\top} \right) \right) \Sigma_{\mathrm{l}}^{-1} ( \mathrm{x} - \upmu_{\mathrm{l}} ) +
( \mathrm{x} - \upmu_{\mathrm{l}} )^{\top} \left( {\pdv{}{ [ \mathrm{a}_{\mathrm{j}}^{(i)} | \sigma ] }} \left( \Sigma_{\mathrm{l}}^{-1}( \mathrm{a}_{\mathrm{j}}^{(i)}, \sigma ) \right) \right) ( \mathrm{x} - \upmu_{\mathrm{l}} ) +
$
}
\newline
$\phantom{=} + ( \mathrm{x} - \upmu_{\mathrm{l}} )^{\top} \Sigma_{\mathrm{l}}^{-1} \left( {\pdv{}{ [ \mathrm{a}_{\mathrm{j}}^{(i)} | \sigma ] }} \left( \mathrm{x} - \upmu_{\mathrm{l}}( \mathrm{a}_{\mathrm{j}}^{(i)} ) \right) \right) =
$
\newline
\newline
\newline
\resizebox{\linewidth}{!}{
$
= \left( \mathrm{x} - {\pdv{}{ [ \mathrm{a}_{\mathrm{j}}^{(i)} | \sigma ] }} \left( \upmu_{\mathrm{l}}( \mathrm{a}_{\mathrm{j}}^{(i)} ) \right) \right)^{\top} \Sigma_{\mathrm{l}}^{-1} ( \mathrm{x} - \upmu_{\mathrm{l}} ) +
( \mathrm{x} - \upmu_{\mathrm{l}} )^{\top} \left( - \Sigma_{\mathrm{l}}^{-1} \left( {\pdv{}{ [ \mathrm{a}_{\mathrm{j}}^{(i)} | \sigma ] }} \left( \Sigma_{\mathrm{l}}( \mathrm{a}_{\mathrm{j}}^{(i)}, \sigma ) \right) \right)
\Sigma_{\mathrm{l}}^{-1} \right) ( \mathrm{x} - \upmu_{\mathrm{l}} ) +
$
}
\newline
$
\phantom{=} + ( \mathrm{x} - \upmu_{\mathrm{l}} )^{\top} \Sigma_{\mathrm{l}}^{-1} \left( \mathrm{x} - {\pdv{}{ [ \mathrm{a}_{\mathrm{j}}^{(i)} | \sigma ] }} \left( \upmu_{\mathrm{l}}( \mathrm{a}_{\mathrm{j}}^{(i)} ) \right) \right)
$
\newline
\newline
\newline
where in the last equality, we used the remarkable identity:
\newline
\newline
\newline
$
{\pdv{}{ [ \mathrm{a}_{\mathrm{j}}^{(i)} | \sigma ] }} \left( \Sigma_{\mathrm{l}}^{-1}( \mathrm{a}_{\mathrm{j}}^{(i)}, \sigma ) \right) =
- \Sigma_{\mathrm{l}}^{-1} \left( {\pdv{}{ [ \mathrm{a}_{\mathrm{j}}^{(i)} | \sigma ] }} \left( \Sigma_{\mathrm{l}}( \mathrm{a}_{\mathrm{j}}^{(i)}, \sigma ) \right) \right) \Sigma_{\mathrm{l}}^{-1}
$
\newline
\newline
\newline
Combining these results, we obtain:
\newline
\newline
\newline
$
{\pdv{}{ [ \mathrm{a}_{\mathrm{j}}^{(i)} | \sigma ] }} \left( {\displaystyle \sum \limits_{\mathrm{x} \in \mathrm{x}}} \ln(f_{\mathcal{N}(\upvarphi(\mathrm{s}),\sigma^2)}(\mathrm{x})) \right) =
$
\newline
\newline
\newline
\resizebox{\linewidth}{!}{
$
= \frac{1}{K^d} {\displaystyle \sum \limits_{\mathrm{x} \in \mathrm{x}}} P {\displaystyle \sum \limits_{ \mathrm{l} \in \{ 0,\ldots,K-1 \}^d }}
-\frac{1}{2} ( (2 \pi)^n \lvert \Sigma_{\mathrm{l}} \rvert )^{-\frac{3}{2}} ( ( 2 \pi )^n \lvert \Sigma_{\mathrm{l}} \rvert ) \tr( \Sigma_{\mathrm{l}}^{-1} {\pdv{}{ [ \mathrm{a}_{\mathrm{j}}^{(i)} | \sigma ] }}
\left( \Sigma_{\mathrm{l}}( \mathrm{a}_{\mathrm{j}}^{(i)}, \sigma ) \right) ) e^{ -\frac{1}{2} ( \mathrm{x} - \upmu_{\mathrm{l}} )^{\top} \Sigma_{\mathrm{l}}^{-1} ( \mathrm{x} - \upmu_{\mathrm{l}} ) } -
$
}
\newline
\resizebox{\linewidth}{!}{
$
\phantom{=} - \frac{1}{2} \frac{1}{ \sqrt{ (2 \pi)^n \lvert \Sigma_{\mathrm{l}} \rvert }} e^{ -\frac{1}{2} ( \mathrm{x} - \upmu_{\mathrm{l}} )^{\top} \Sigma_{\mathrm{l}}^{-1} ( \mathrm{x} - \upmu_{\mathrm{l}} ) }
\left( \left( \mathrm{x} - {\pdv{}{ [ \mathrm{a}_{\mathrm{j}}^{(i)} | \sigma ] }} \left( \upmu_{\mathrm{l}}( \mathrm{a}_{\mathrm{j}}^{(i)}, \sigma ) \right) \right)^{\top}
\Sigma_{\mathrm{l}}^{-1} ( \mathrm{x} - \upmu_{\mathrm{l}} ) \right. +
$
}
\newline
\resizebox{\linewidth}{!}{
$
\phantom{=} + ( \mathrm{x} - \upmu_{\mathrm{l}} )^{\top}
\left( - \Sigma_{\mathrm{l}}^{-1} \left( {\pdv{}{ [ \mathrm{a}_{\mathrm{j}}^{(i)} | \sigma ] }} \left( \Sigma_{\mathrm{l}}( \mathrm{a}_{\mathrm{j}}^{(i)}, \sigma ) \right) \right) \Sigma_{\mathrm{l}}^{-1} \right) ( \mathrm{x} - \upmu_{\mathrm{l}} ) +
( \mathrm{x} - \upmu_{\mathrm{l}} )^{\top} \Sigma_{\mathrm{l}}^{-1} \left( \mathrm{x} - {\pdv{}{ [ \mathrm{a}_{\mathrm{j}}^{(i)} | \sigma ] }} \left( \upmu_{\mathrm{l}}( \mathrm{a}_{\mathrm{j}}^{(i)} ) \right) \right)
\left. \vphantom{ \left( \mathrm{x} - {\pdv{}{\mathrm{a}_{\mathrm{j}}^{(i)}}} \left( \upmu_{\mathrm{l}}( \mathrm{a}_{\mathrm{j}}^{(i)} ) \right) \right)^{\top} } \right) =
$
}
\newline
\newline
\newline
\resizebox{\linewidth}{!}{
$
= \frac{1}{K^d} {\displaystyle \sum \limits_{\mathrm{x} \in \mathrm{x}}} P {\displaystyle \sum \limits_{ \mathrm{l} \in \{ 0,\ldots,K-1 \}^d }}
- \frac{1}{2} \frac{1}{ \sqrt{ (2 \pi)^n \lvert \Sigma_{\mathrm{l}} \rvert }} \tr( \Sigma_{\mathrm{l}}^{-1} {\pdv{}{ [ \mathrm{a}_{\mathrm{j}}^{(i)} | \sigma ] }}
\left( \Sigma_{\mathrm{l}}( \mathrm{a}_{\mathrm{j}}^{(i)}, \sigma ) \right) ) e^{ -\frac{1}{2} ( \mathrm{x} - \upmu_{\mathrm{l}} )^{\top} \Sigma_{\mathrm{l}}^{-1} ( \mathrm{x} - \upmu_{\mathrm{l}} ) } -
$
}
\newline
\resizebox{\linewidth}{!}{
$
\phantom{=} - \frac{1}{2} \frac{1}{ \sqrt{ (2 \pi)^n \lvert \Sigma_{\mathrm{l}} \rvert }} e^{ -\frac{1}{2} ( \mathrm{x} - \upmu_{\mathrm{l}} )^{\top} \Sigma_{\mathrm{l}}^{-1} ( \mathrm{x} - \upmu_{\mathrm{l}} ) }
\left( \left( \mathrm{x} - {\pdv{}{ [ \mathrm{a}_{\mathrm{j}}^{(i)} | \sigma ] }} \left( \upmu_{\mathrm{l}}( \mathrm{a}_{\mathrm{j}}^{(i)}, \sigma ) \right) \right)^{\top}
\Sigma_{\mathrm{l}}^{-1} ( \mathrm{x} - \upmu_{\mathrm{l}} ) \right. +
$
}
\newline
\resizebox{\linewidth}{!}{
$
\phantom{=} + ( \mathrm{x} - \upmu_{\mathrm{l}} )^{\top}
\left( - \Sigma_{\mathrm{l}}^{-1} \left( {\pdv{}{ [ \mathrm{a}_{\mathrm{j}}^{(i)} | \sigma ] }} \left( \Sigma_{\mathrm{l}}( \mathrm{a}_{\mathrm{j}}^{(i)}, \sigma ) \right) \right) \Sigma_{\mathrm{l}}^{-1} \right) ( \mathrm{x} - \upmu_{\mathrm{l}} ) +
( \mathrm{x} - \upmu_{\mathrm{l}} )^{\top} \Sigma_{\mathrm{l}}^{-1} \left( \mathrm{x} - {\pdv{}{ [ \mathrm{a}_{\mathrm{j}}^{(i)} | \sigma ] }} \left( \upmu_{\mathrm{l}}( \mathrm{a}_{\mathrm{j}}^{(i)} ) \right) \right)
\left. \vphantom{ \left( \mathrm{x} - {\pdv{}{\mathrm{a}_{\mathrm{j}}^{(i)}}} \left( \upmu_{\mathrm{l}}( \mathrm{a}_{\mathrm{j}}^{(i)} ) \right) \right)^{\top} } \right)
$
}
\newline
\newline
\newline
Now it remains to show how to compute the partial derivatives:
${\pdv{}{ [ \mathrm{a}_{\mathrm{j}}^{(i)} | \sigma ] }} \left( \upmu_{\mathrm{l}}( \mathrm{a}_{\mathrm{j}}^{(i)} ) \right)$
and
${\pdv{}{ [ \mathrm{a}_{\mathrm{j}}^{(i)} | \sigma ] }} \left( \Sigma_{\mathrm{l}}( \mathrm{a}_{\mathrm{j}}^{(i)}, \sigma ) \right)$
. As for
${\pdv{}{ [ \mathrm{a}_{\mathrm{j}}^{(i)} | \sigma ] }} \left( \upmu_{\mathrm{l}}( \mathrm{a}_{\mathrm{j}}^{(i)} ) \right)$
, since, by Theorem ~\ref{theoremone},
\newline
\newline
$$
\upmu_{\mathrm{j}}^{(i)} = K^d {\displaystyle \sum \limits_{\mathrm{l} \in \{ -k,\ldots,k \}^d}} \mathrm{a}_{\mathrm{l}}^{(i)} g \left( \mathrm{j}, \mathrm{l} \right)
$$
\newline
\newline
we immediately obtain:
\newline
\begin{flalign*}
{\pdv{}{ \mathrm{a}_{\mathrm{j}}^{(i)} }} \left( \upmu_{\mathrm{l}}( \mathrm{a}_{\mathrm{j}}^{(i)} ) \right) &=
\left( {\pdv{}{ \mathrm{a}_{\mathrm{j}}^{(i)} }} \left( \upmu_{\mathrm{l}}^{(\bar{i})}( \mathrm{a}_{\mathrm{j}}^{(i)} ) \right) \right)_{\bar{i}=1}^{n} =&\\
&= \left( {\pdv{}{ \mathrm{a}_{\mathrm{j}}^{(i)} }} \left( K^d {\displaystyle \sum \limits_{\bar{\mathrm{l}} \in \{ -k,\ldots,k \}^d}} \mathrm{a}_{\bar{\mathrm{l}}}^{(\bar{i})} g \left( \mathrm{l}, \bar{\mathrm{l}} \right) \right) \right)_{\bar{i}=1}^{n} =&\\
&= \left( K^d {\displaystyle \sum \limits_{\bar{\mathrm{l}} \in \{ -k,\ldots,k \}^d}} {\pdv{}{ \mathrm{a}_{\mathrm{j}}^{(i)} }} \left( \mathrm{a}_{\bar{\mathrm{l}}}^{(\bar{i})} g \left( \mathrm{l}, \bar{\mathrm{l}} \right) \right) \right)_{\bar{i}=1}^{n} =&\\
&= \left( K^d {\displaystyle \sum \limits_{\bar{\mathrm{l}} \in \{ -k,\ldots,k \}^d}} \delta_{{i}{\bar{i}}} \delta_{{\mathrm{j}}{\bar{\mathrm{l}}}} g \left( \mathrm{l}, \bar{\mathrm{l}} \right) \right)_{\bar{i}=1}^{n} =&\\
&= \left( K^d \delta_{{i}{\bar{i}}} g \left( \mathrm{l},\mathrm{j} \right) \right)_{\bar{i}=1}^{n}
\end{flalign*}
\newline
and
\newline
$$
{\pdv{}{ \sigma }} \upmu_{\mathrm{l}} = \left( {\pdv{}{ \sigma }} \upmu_{\mathrm{l}}^{(\bar{i})} \right)_{\bar{i}=1}^{n} = \vec{0}
$$
\newline
(since $\upmu_{\mathrm{l}}$ does not depend on $\sigma$). On the other hand, from Theorem~\ref{theoremone} applied to $\Sigma_{\mathrm{j}}^{(i_1)(i_2)}$ we get:
\newline
$$
\Sigma_{\mathrm{j}}^{(i_1)(i_2)} = {\sigma^2}\mathrm{I}_{{i_1}{i_2}} +
K^d \left( {\displaystyle \sum \limits_{\substack{\mathrm{l}^1 \in \{ -k,\ldots,k \}^d \\ \mathrm{l}^2 \in \{ -k,\ldots,k \}^d}}} \mathrm{a}_{\mathrm{l}^1}^{(i_1)} \mathrm{a}_{\mathrm{l}^2}^{(i_2)} g \left( \mathrm{j}, \mathrm{l}^1, \mathrm{l}^2 \right) \right) -
\upmu_{\mathrm{j}}^{(i_1)} \upmu_{\mathrm{j}}^{(i_2)}
$$
\newline
, which yields:
\newline
\newline
\newline
$
{\pdv{}{ \mathrm{a}_{\mathrm{j}}^{(i)} }} \left( \Sigma_{\mathrm{l}}( \mathrm{a}_{\mathrm{j}}^{(i)} ) \right) =
\left( {\pdv{}{ \mathrm{a}_{\mathrm{j}}^{(i)} }} \left( \Sigma_{\mathrm{l}}^{(i_1)(i_2)}( \mathrm{a}_{\mathrm{j}}^{(i)} ) \right) \right)_{i_1, i_2 = 1}^{n} =
$
\newline
\newline
\newline
\resizebox{\linewidth}{!}{
$
= \left( {\pdv{}{ \mathrm{a}_{\mathrm{j}}^{(i)} }} \left(
{\sigma^2}\mathrm{I}_{{i_1}{i_2}} +
K^d \left( {\displaystyle \sum \limits_{\substack{\mathrm{l}^1 \in \{ -k,\ldots,k \}^d \\ \mathrm{l}^2 \in \{ -k,\ldots,k \}^d}}} \mathrm{a}_{\mathrm{l}^1}^{(i_1)} \mathrm{a}_{\mathrm{l}^2}^{(i_2)} g \left( \mathrm{l}, \mathrm{l}^1, \mathrm{l}^2 \right) \right) -
\upmu_{\mathrm{j}}^{(i_1)} \left( \mathrm{a}_{\mathrm{j}}^{(i)} \right) \upmu_{\mathrm{j}}^{(i_2)} \left( \mathrm{a}_{\mathrm{j}}^{(i)} \right)
\right) \right)_{i_1, i_2 = 1}^{n} =
$
}
\newline
\newline
\newline
\resizebox{\linewidth}{!}{
$
= \left(
K^d \left( {\displaystyle \sum \limits_{\substack{\mathrm{l}^1 \in \{ -k,\ldots,k \}^d \\ \mathrm{l}^2 \in \{ -k,\ldots,k \}^d}}} {\pdv{}{ \mathrm{a}_{\mathrm{j}}^{(i)} }} \left(
\mathrm{a}_{\mathrm{l}^1}^{(i_1)} \mathrm{a}_{\mathrm{l}^2}^{(i_2)} g \left( \mathrm{l}, \mathrm{l}^1, \mathrm{l}^2 \right)
\right) \right) -
{\pdv{}{ \mathrm{a}_{\mathrm{j}}^{(i)} }} \left( \upmu_{\mathrm{j}}^{(i_1)} \left( \mathrm{a}_{\mathrm{j}}^{(i)} \right) \right) \upmu_{\mathrm{j}}^{(i_2)} -
\upmu_{\mathrm{j}}^{(i_1)} {\pdv{}{ \mathrm{a}_{\mathrm{j}}^{(i)} }} \left( \upmu_{\mathrm{j}}^{(i_2)} \left( \mathrm{a}_{\mathrm{j}}^{(i)} \right) \right)
\right)_{i_1, i_2 = 1}^{n} =
$
}
\newline
\newline
\newline
\newline
$
= \left(
K^d \left( {\displaystyle \sum \limits_{\substack{\mathrm{l}^1 \in \{ -k,\ldots,k \}^d \\ \mathrm{l}^2 \in \{ -k,\ldots,k \}^d}}}
\delta_{{i_2}{i}} \delta_{{\mathrm{l}^2}{\mathrm{j}}} \mathrm{a}_{\mathrm{l}^1}^{(i_1)} g \left( \mathrm{l}, \mathrm{l}^1, \mathrm{l}^2 \right) +
\delta_{{i_1}{i}} \delta_{{\mathrm{l}^1}{\mathrm{j}}} \mathrm{a}_{\mathrm{l}^2}^{(i_2)} g \left( \mathrm{l}, \mathrm{l}^1, \mathrm{l}^2 \right)
\right) \right. -
$
\newline
$
\phantom{=} - \left.
\vphantom{
K^d \left( {\displaystyle \sum \limits_{\substack{\mathrm{l}^1 \in \{ -k,\ldots,k \}^d \\ \mathrm{l}^2 \in \{ -k,\ldots,k \}^d}}} {\pdv{}{ \mathrm{a}_{\mathrm{j}}^{(i)} }} \left(
\mathrm{a}_{\mathrm{l}^1}^{(i_1)} \mathrm{a}_{\mathrm{l}^2}^{(i_2)} g \left( \mathrm{l}, \mathrm{l}^1, \mathrm{l}^2 \right)
\right) \right)
}
{\pdv{}{ \mathrm{a}_{\mathrm{j}}^{(i)} }} \left( \upmu_{\mathrm{j}}^{(i_1)} \left( \mathrm{a}_{\mathrm{j}}^{(i)} \right) \right) \upmu_{\mathrm{j}}^{(i_2)} -
\upmu_{\mathrm{j}}^{(i_1)} {\pdv{}{ \mathrm{a}_{\mathrm{j}}^{(i)} }} \left( \upmu_{\mathrm{j}}^{(i_2)} \left( \mathrm{a}_{\mathrm{j}}^{(i)} \right) \right)
\right)_{i_1, i_2 = 1}^{n} =
$
\newline
\newline
\newline
\newline
$
= \left(
K^d \left( 
\delta_{{i_2}{i}} \left( {\displaystyle \sum \limits_{\mathrm{l}^1 \in \{ -k,\ldots,k \}^d}} \mathrm{a}_{\mathrm{l}^1}^{(i_1)} g \left( \mathrm{l}, \mathrm{l}^1, \mathrm{j} \right) \right) +
\delta_{{i_1}{i}} \left( {\displaystyle \sum \limits_{\mathrm{l}^2 \in \{ -k,\ldots,k \}^d}} \mathrm{a}_{\mathrm{l}^2}^{(i_1)} g \left( \mathrm{l}, \mathrm{j}, \mathrm{l}^2 \right) \right)
\right) \right. -
$
\newline
$
\phantom{=} - \left.
\vphantom{
K^d \left( {\displaystyle \sum \limits_{\substack{\mathrm{l}^1 \in \{ -k,\ldots,k \}^d \\ \mathrm{l}^2 \in \{ -k,\ldots,k \}^d}}} {\pdv{}{ \mathrm{a}_{\mathrm{j}}^{(i)} }} \left(
\mathrm{a}_{\mathrm{l}^1}^{(i_1)} \mathrm{a}_{\mathrm{l}^2}^{(i_2)} g \left( \mathrm{l}, \mathrm{l}^1, \mathrm{l}^2 \right)
\right) \right)
}
{\pdv{}{ \mathrm{a}_{\mathrm{j}}^{(i)} }} \left( \upmu_{\mathrm{j}}^{(i_1)} \left( \mathrm{a}_{\mathrm{j}}^{(i)} \right) \right) \upmu_{\mathrm{j}}^{(i_2)} -
\upmu_{\mathrm{j}}^{(i_1)} {\pdv{}{ \mathrm{a}_{\mathrm{j}}^{(i)} }} \left( \upmu_{\mathrm{j}}^{(i_2)} \left( \mathrm{a}_{\mathrm{j}}^{(i)} \right) \right)
\right)_{i_1, i_2 = 1}^{n}
$
\newline
\newline
\newline
and
\newline
\newline
\begin{flalign*}
& {\pdv{}{ \sigma }} \left( \Sigma_{\mathrm{l}}( \sigma ) \right) = \left( {\pdv{}{ \sigma }} \left( \Sigma_{\mathrm{l}}^{(i_1)(i_2)}( \sigma ) \right) \right)_{i_1, i_2 = 1}^{n} =&\\
& &\\
&= \left( {\pdv{}{ \sigma }} \left(
{\sigma^2}\mathrm{I}_{{i_1}{i_2}} +
K^d \left( {\displaystyle \sum \limits_{\substack{\mathrm{l}^1 \in \{ -k,\ldots,k \}^d \\ \mathrm{l}^2 \in \{ -k,\ldots,k \}^d}}} \mathrm{a}_{\mathrm{l}^1}^{(i_1)} \mathrm{a}_{\mathrm{l}^2}^{(i_2)} g \left( \mathrm{l}, \mathrm{l}^1, \mathrm{l}^2 \right) \right) -
\upmu_{\mathrm{j}}^{(i_1)} \upmu_{\mathrm{j}}^{(i_2)} \right) \right)_{i_1, i_2 = 1}^{n} =&\\
& &\\
&= \left( {\pdv{}{ \sigma }} \left( {\sigma^2}\mathrm{I}_{{i_1}{i_2}} \right) \right)_{i_1, i_2 = 1}^{n} = \left( 2 \sigma \mathrm{I}_{{i_1}{i_2}} \right)_{i_1, i_2 = 1}^{n} = 2 \sigma \mathrm{I}
\end{flalign*}
\newline
\newline
, which in conjunction with the fact that ${\pdv{}{ \mathrm{a}_{\mathrm{j}}^{(i)} }} \upmu_{\mathrm{j}} \left( \mathrm{a}_{\mathrm{j}}^{(i)} \right)$ is known completes the proof.
\end{proof}


\setcounter{section}{7}
\section{Appendix B -- Algorithm}
\label{AppendixB}

\scalebox{0.375}{
\begin{minipage}{1.74\linewidth}
\begin{algorithm}[H]
\SetProcNameSty{textsc}
\SetProcArgSty{textsc}
\caption{ MCEC (Lloyd) }
\SetKw{End}{end \\}
\SetKwRepeat{Do}{do}{while}
\SetKwInOut{Input}{Input}
\SetKwInOut{Output}{Output}
\Input{
$k > 0$ - initial number of clusters;
\newline
$l > 0$ - Fourier series order;
\newline
$\varepsilon > 0$ - stop condition;
\newline
$\tilde{\varepsilon} > 0$ - cluster removal threshold;
\newline
$\mathrm{X}$ - data set ($n$ - dimension);
}
\Output{
$\mathrm{X}=\mathrm{X}_1 \cup \ldots \cup \mathrm{X}_k, \mathrm{X}_i \cap \mathrm{X}_j = \emptyset, \,i \ne j$ - clustering of data set $\mathrm{X}$;
\newline
\newline
$\mathrm{a}_{\mathrm{m}}^{(1)}[i], \ldots, \mathrm{a}_{\mathrm{m}}^{(n)}[i] \in \mathbb{R}$, where $m \in \{-l,\ldots,l\}$ - coefficients of the Fourier series defining the curve
$\upvarphi[i]$ fitted to the cluster $\mathrm{X}_i$, $i=1,\ldots,k$;
\newline
\newline
$\sigma^2[i] \in \mathbb{R}_{+}$ - variance of the cluster $\mathrm{X}_i$, $i=1,\ldots,k$;
\newline
$p_1,\ldots,p_k$ - probabilities (weights);
}
\Begin{
	Obtain initial clustering $\mathrm{X}_1,\ldots,\mathrm{X}_k$, initial coefficients: $\mathrm{a}_{\mathrm{m}}^{(1)}[i], \ldots, \mathrm{a}_{\mathrm{m}}^{(n)}[i] \in \mathbb{R}$, where $m \in \{-l,\ldots,l\}$, of the Fourier series defining the curve $\upvarphi[i]$ and $\sigma^2[i] \in \mathbb{R}_{+}$\;
	\For { $ i \in \{ 1, \ldots, k \} $ } {
		\uIf { $ \left \rvert \mathrm{X}_i \right \rvert \ge \tilde{\varepsilon}\% \cdot \left \rvert \mathrm{X} \right \rvert $ } {
		    Find coefficients: $\mathrm{a}_{\mathrm{m}}^{(1)}[i], \ldots, \mathrm{a}_{\mathrm{m}}^{(n)}[i] \in \mathbb{R}$, where $m \in \{-l,\ldots,l\}$, of the Fourier series defining the curve $\upvarphi[i]$
					and $\sigma^2[i] \in \mathbb{R}_{+}$, that minimize $H^{\cross}(\mathrm{X}_i \left \rVert \mathcal{N}(\upvarphi[i],\sigma^2[i]) \right.)$ using their previous values as the initial guess for the optimization algorithm\;
			Obtain 
			$p_i \leftarrow \frac{\left \rvert \mathrm{X}_i \right \rvert}{\left \rvert \mathrm{X} \right \rvert}$\;
			${\mathrm{X}_i}.\textsc{ACTIVE} \leftarrow \mathbf{true}$\;
		} \uElse {
			${\mathrm{X}_i}.\textsc{ACTIVE} \leftarrow \mathbf{false}$\;
		} \End
	}
	$\textsc{ITER} \leftarrow 0$\;
	$h_{\textsc{ITER}} \leftarrow \infty$\;
	\Do { $ h_{\textsc{ITER}} < h_{\textsc{ITER}-1}-\varepsilon $ }{
		\For { $ \mathrm{x} \in \mathrm{X} $ } {
			$h_{min} \leftarrow \infty$\;
			\For { $ i \in \{ 1, \ldots, k \} $ } {
				\If { $ {\mathrm{X}_i}.\textsc{ACTIVE} $ } {
					$h \leftarrow -\ln(p_i)-\ln(f_{\mathcal{N}(\upvarphi[i],\sigma^2[i])}(\mathrm{x}))$\;
					\If { $ h < h_{min} $ } {
						$h_{min} \leftarrow h$\;
						$\mathrm{X}_{\mathrm{cl}(\mathrm{x})} \leftarrow \mathrm{X}_{\mathrm{cl}(\mathrm{x})} \setminus \{\mathrm{x}\}$\;
						$\mathrm{X}_i \leftarrow \mathrm{X}_i \cup \{\mathrm{x}\}$\;
					}
				}
			}
		}
		Update the probabilities $p$\;
		\For { $ i \in \{ 1, \ldots, k \} $ } {
			\If { $ \left \rvert \mathrm{X}_i \right \rvert < \tilde{\varepsilon}\% \cdot \left \rvert \mathrm{X} \right \rvert $ } {
				Remove points from cluster $\mathrm{X}_i$ by assigning them to the "closest" active one with respect to the function $-\ln(p_i)-\ln(f_{\mathcal{N}(\upvarphi[i],\sigma^2[i])}(\mathrm{x}))$ and update the probabilities $p$\;
				$ {\mathrm{X}_i}.\textsc{ACTIVE} \leftarrow \textbf{false} $\;
			}
		}
		\For { $ i \in \{ 1, \ldots, k \} $ } {
			\If { $ {\mathrm{X}_i}.\textsc{ACTIVE} $ } {
					Find coefficients: $\mathrm{a}_{\mathrm{m}}^{(1)}[i], \ldots, \mathrm{a}_{\mathrm{m}}^{(n)}[i] \in \mathbb{R}$, where $m \in \{-l,\ldots,l\}$, of the Fourier series defining the curve $\upvarphi[i]$
					and $\sigma^2[i] \in \mathbb{R}_{+}$, that minimize $H^{\cross}(\mathrm{X}_i \left \rVert \mathcal{N}(\upvarphi[i],\sigma^2[i]) \right.)$ using their previous values as the initial guess for the optimization algorithm\;
			}
		}
		$\textsc{ITER} \leftarrow \textsc{ITER}+1$\;
		$h_{\textsc{ITER}} \leftarrow { \displaystyle \sum \limits_{i=0}^k p_i(-\ln(p_i)+H^{\cross}(\mathrm{X}_i \left \rVert \mathcal{N}(\upvarphi[i],\sigma^2[i]) \right.))}$\;
	}
}
\end{algorithm}
\end{minipage}
}


\section{Appendix C -- Experiments}

In the section below, we provide a more detailed insight into the performance metrics computed for each of the experiments described in Section~\ref{Section5} in terms of the dependence of the results acquired from the algorithms used in our comparison on the initial guess chosen as the starting point in the initial steps of the algorithms based on the value of the random seed. The subsequent tables show how the performance metrics vary throughout the exponentially growing number of starts of the particular algorithms for the "best" results obtained through the consecutive number of starts with respect to the optimization criterion used by each of the four algorithms: the value of the cross-entropy for MCEC, afCEC and CEC and the Bayesian Information Criterion for the GMM. Together with the tables providing the exact values we put for each performance metric, the charts show the tendency in graphic form. The maximum number of starts (apart from the test case with two curves from Experiment No. 3) was chosen in a way that ensures the correspondence of the cluster labeling acquired from the MCEC algorithm with the model labeling in terms of the Rand index and the Jaccard index matrics (I) as well as the "visual" likeness between the model curves in the test data and the curves estimated by the MCEC (II). On the other hand, as our algorithm tends to outperform the remaining three methods just from the first start in the case of the most simple test cases from the Experiment No. 1, we set the fixed number of 16 starts as the least maximum number of starts in order to illustrate the behavior of the afCEC, CEC, and GMM on that data.

\begin{table}[H]
\caption{Experiment No. 1: Values of the performance metrics computed for the consecutive numbers of starts being the powers of two - Part 1}
\centering

\setlength\tabcolsep{1.5pt}
\resizebox{\columnwidth}{!}{%
\begin{tabular}{c|>{\tiny}c|>{\tiny}c>{\tiny}c>{\tiny}c>{\tiny}c>{\tiny}c|>{\tiny}c>{\tiny}c>{\tiny}c|>{\tiny}c>{\tiny}c>{\tiny}c|>{\tiny}c>{\tiny}c>{\tiny}c|}
\cline{2-16}
& \rule{0pt}{1.5cm} & \multicolumn{5}{c|}{\makecell{MCEC}} & \multicolumn{3}{c|}{\makecell{afCEC Hartigan}} & \multicolumn{3}{c|}{CEC} & \multicolumn{3}{c|}{GMM} \\
\cline{3-16}

& \rotatebox[origin=c]{90}{\rlap{\parbox{1.5cm}{\makecell{Number of \\ starts}}}} &
\multicolumn{1}{c|}{MLE} & \multicolumn{1}{c|}{BIC} & \multicolumn{1}{c|}{AIC} & \multicolumn{1}{c|}{\makecell{Rand \\ index}} & \multicolumn{1}{c|}{\makecell{Jaccard \\ index}} &
\multicolumn{1}{c|}{MLE} & \multicolumn{1}{c|}{BIC} & \multicolumn{1}{c|}{AIC} &
\multicolumn{1}{c|}{MLE} & \multicolumn{1}{c|}{BIC} & \multicolumn{1}{c|}{AIC} &
\multicolumn{1}{c|}{MLE} & \multicolumn{1}{c|}{BIC} & \multicolumn{1}{c|}{AIC} \\

\hline
\hline
\multicolumn{1}{|c|}{\multirow{5}{*}{\includegraphics[scale=0.07]{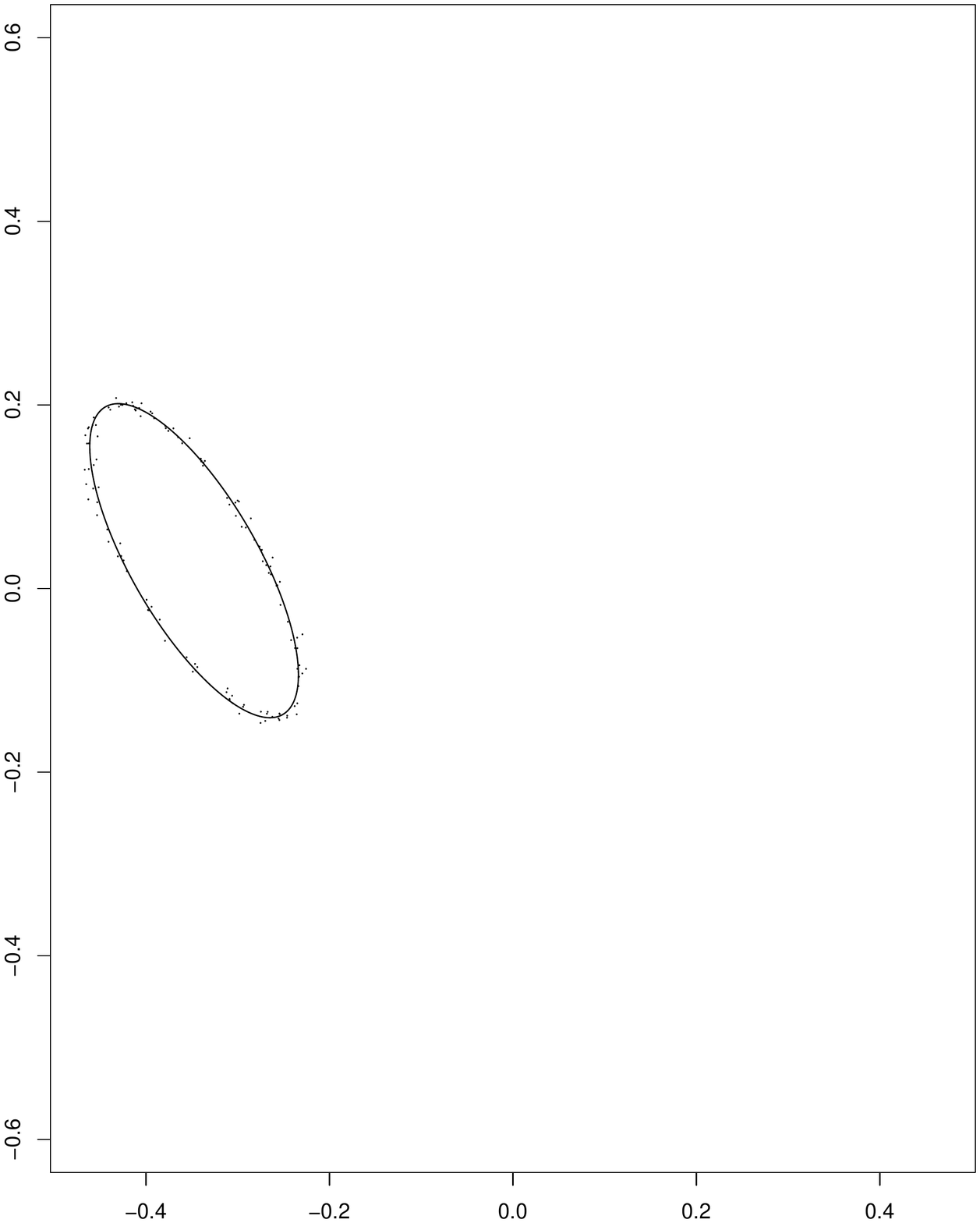}}} & 1 & \textbf{548.54} & \textbf{-1063.12} & \textbf{-1083.08} & 1.00 & 1.00 & 461.48 & -850.17 & -892.95 & 305.61 & -557.84 & -589.21 & 304.31 & -555.25 & -586.62 \\ 
\multicolumn{1}{|c|}{} & 2 & \textbf{548.54} & \textbf{-1063.12} & \textbf{-1083.08} & 1.00 & 1.00 & 461.48 & -850.17 & -892.95 & 305.61 & -557.84 & -589.21 & 304.31 & -555.25 & -586.62 \\ 
\multicolumn{1}{|c|}{} & 4 & \textbf{548.54} & \textbf{-1063.12} & \textbf{-1083.08} & 1.00 & 1.00 & 461.48 & -850.17 & -892.95 & 305.54 & -557.71 & -589.08 & 304.31 & -555.25 & -586.62 \\ 
\multicolumn{1}{|c|}{} & 8 & \textbf{548.54} & \textbf{-1063.12} & \textbf{-1083.08} & 1.00 & 1.00 & 461.48 & -850.17 & -892.95 & 317.22 & -581.07 & -612.44 & 304.31 & -555.25 & -586.62 \\ 
\multicolumn{1}{|c|}{} & 16 & \textbf{548.54} & \textbf{-1063.12} & \textbf{-1083.08} & 1.00 & 1.00 & 461.48 & -850.17 & -892.95 & 317.22 & -581.07 & -612.44 & 304.31 & -555.25 & -586.62 \\ 
\hline
\hline
\multicolumn{1}{|c|}{\multirow{5}{*}{\includegraphics[scale=0.07]{5-1-2.pdf}}} & 1 & \textbf{880.46} & \textbf{-1677.74} & \textbf{-1730.92} & 1.00 & 1.00 & 558.20 & -944.51 & -1054.41 & 398.49 & -669.44 & -750.98 & 386.15 & -644.75 & -726.29 \\ 
\multicolumn{1}{|c|}{} & 2 & \textbf{880.46} & \textbf{-1677.74} & \textbf{-1730.92} & 1.00 & 1.00 & 565.72 & -959.55 & -1069.45 & 398.49 & -669.44 & -750.98 & 386.15 & -644.75 & -726.29 \\ 
\multicolumn{1}{|c|}{} & 4 & \textbf{880.46} & \textbf{-1677.74} & \textbf{-1730.92} & 1.00 & 1.00 & 550.18 & -928.45 & -1038.35 & 398.45 & -669.36 & -750.90 & 396.61 & -665.68 & -747.22 \\ 
\multicolumn{1}{|c|}{} & 8 & \textbf{880.46} & \textbf{-1677.74} & \textbf{-1730.92} & 1.00 & 1.00 & 680.34 & -1188.78 & -1298.68 & 398.45 & -669.36 & -750.90 & 396.61 & -665.68 & -747.22 \\ 
\multicolumn{1}{|c|}{} & 16 & \textbf{880.46} & \textbf{-1677.74} & \textbf{-1730.92} & 1.00 & 1.00 & 579.97 & -988.04 & -1097.94 & 453.68 & -779.81 & -861.35 & 396.61 & -665.68 & -747.22 \\ 
\hline
\hline
\multicolumn{1}{|c|}{\multirow{5}{*}{\includegraphics[scale=0.07]{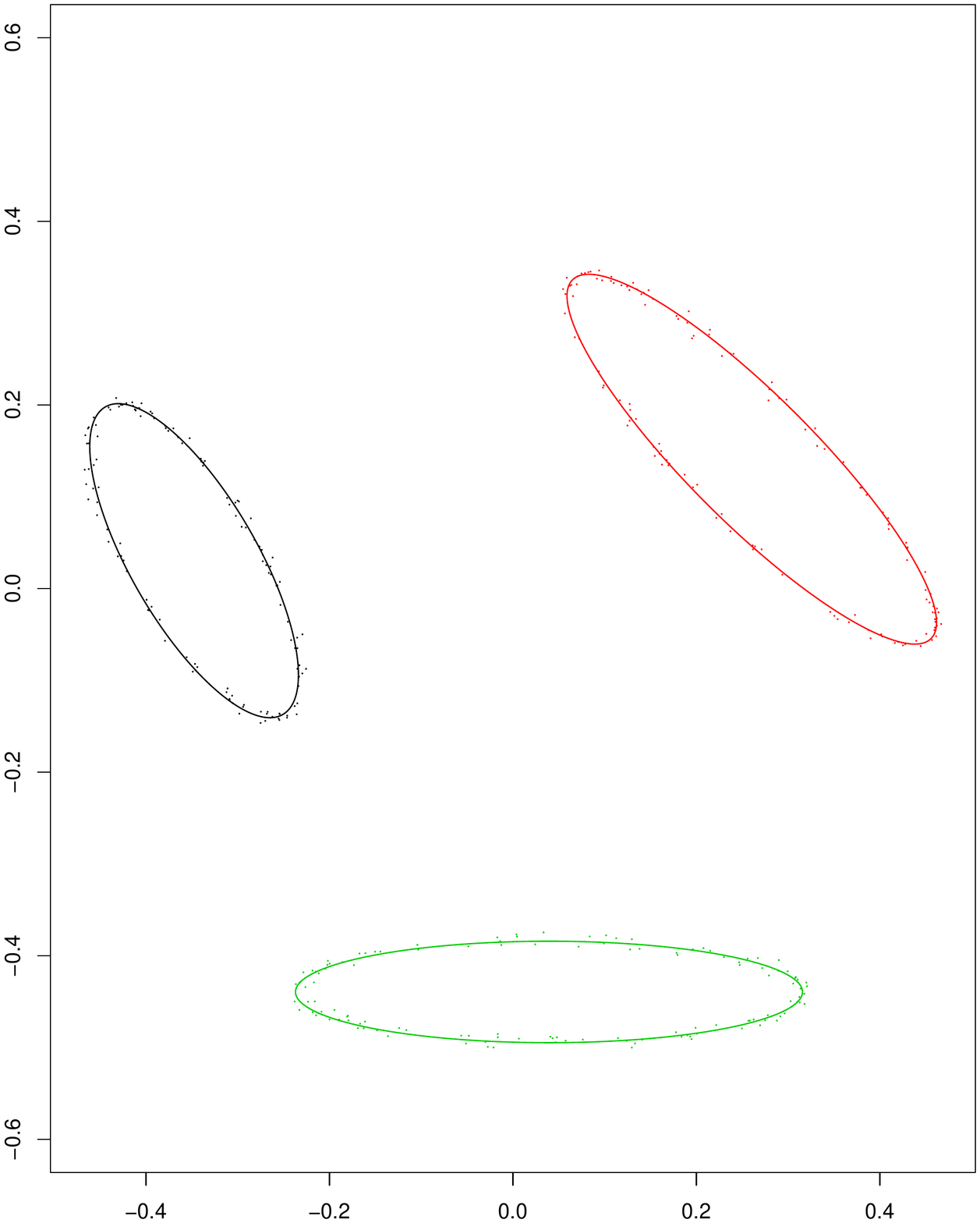}}} & 1 & \textbf{1137.77} & \textbf{-2138.68} & \textbf{-2229.54} & 1.00 & 1.00 & 619.51 & -959.35 & -1145.03 & 464.37 & -720.48 & -858.75 & 404.50 & -600.73 & -739.00 \\ 
\multicolumn{1}{|c|}{} & 2 & \textbf{1137.77} & \textbf{-2138.68} & \textbf{-2229.54} & 1.00 & 1.00 & 730.54 & -1181.40 & -1367.08 & 464.37 & -720.48 & -858.75 & 404.50 & -600.73 & -739.00 \\ 
\multicolumn{1}{|c|}{} & 4 & \textbf{1137.77} & \textbf{-2138.68} & \textbf{-2229.54} & 1.00 & 1.00 & 715.61 & -1151.53 & -1337.21 & 464.37 & -720.48 & -858.75 & 404.50 & -600.73 & -739.00 \\ 
\multicolumn{1}{|c|}{} & 8 & \textbf{1137.77} & \textbf{-2138.68} & \textbf{-2229.54} & 1.00 & 1.00 & 734.57 & -1189.46 & -1375.14 & 464.37 & -720.48 & -858.75 & 404.50 & -600.73 & -739.00 \\ 
\multicolumn{1}{|c|}{} & 16 & \textbf{1137.77} & \textbf{-2138.68} & \textbf{-2229.54} & 1.00 & 1.00 & 734.12 & -1188.56 & -1374.24 & 470.11 & -731.94 & -870.21 & 404.50 & -600.73 & -739.00 \\ 
\hline
\hline
\multicolumn{1}{|c|}{\multirow{5}{*}{\includegraphics[scale=0.07]{5-1-4.pdf}}} & 1 & 710.75 & -1228.12 & -1359.50 & 0.85 & 0.54 & \textbf{1040.37} & \textbf{-1687.73} & \textbf{-1954.75} & 598.74 & -904.28 & -1103.48 & 568.79 & -844.38 & -1043.58 \\ 
\multicolumn{1}{|c|}{} & 2 & \textbf{1481.77} & \textbf{-2770.15} & \textbf{-2901.53} & 1.00 & 1.00 & 853.55 & -1314.09 & -1581.10 & 598.74 & -904.28 & -1103.48 & 568.79 & -844.38 & -1043.58 \\ 
\multicolumn{1}{|c|}{} & 4 & \textbf{1481.77} & \textbf{-2770.15} & \textbf{-2901.53} & 1.00 & 1.00 & 940.49 & -1487.97 & -1754.98 & 598.74 & -904.28 & -1103.48 & 568.79 & -844.38 & -1043.58 \\ 
\multicolumn{1}{|c|}{} & 8 & \textbf{1481.77} & \textbf{-2770.15} & \textbf{-2901.53} & 1.00 & 1.00 & 1038.28 & -1683.55 & -1950.56 & 612.23 & -931.26 & -1130.46 & 568.79 & -844.38 & -1043.58 \\ 
\multicolumn{1}{|c|}{} & 16 & \textbf{1481.77} & \textbf{-2770.15} & \textbf{-2901.53} & 1.00 & 1.00 & 1062.49 & -1731.98 & -1998.99 & 631.59 & -969.98 & -1169.18 & 568.79 & -844.38 & -1043.58 \\ 
\hline
\hline
\multicolumn{1}{|c|}{\multirow{8}{*}{\includegraphics[scale=0.07]{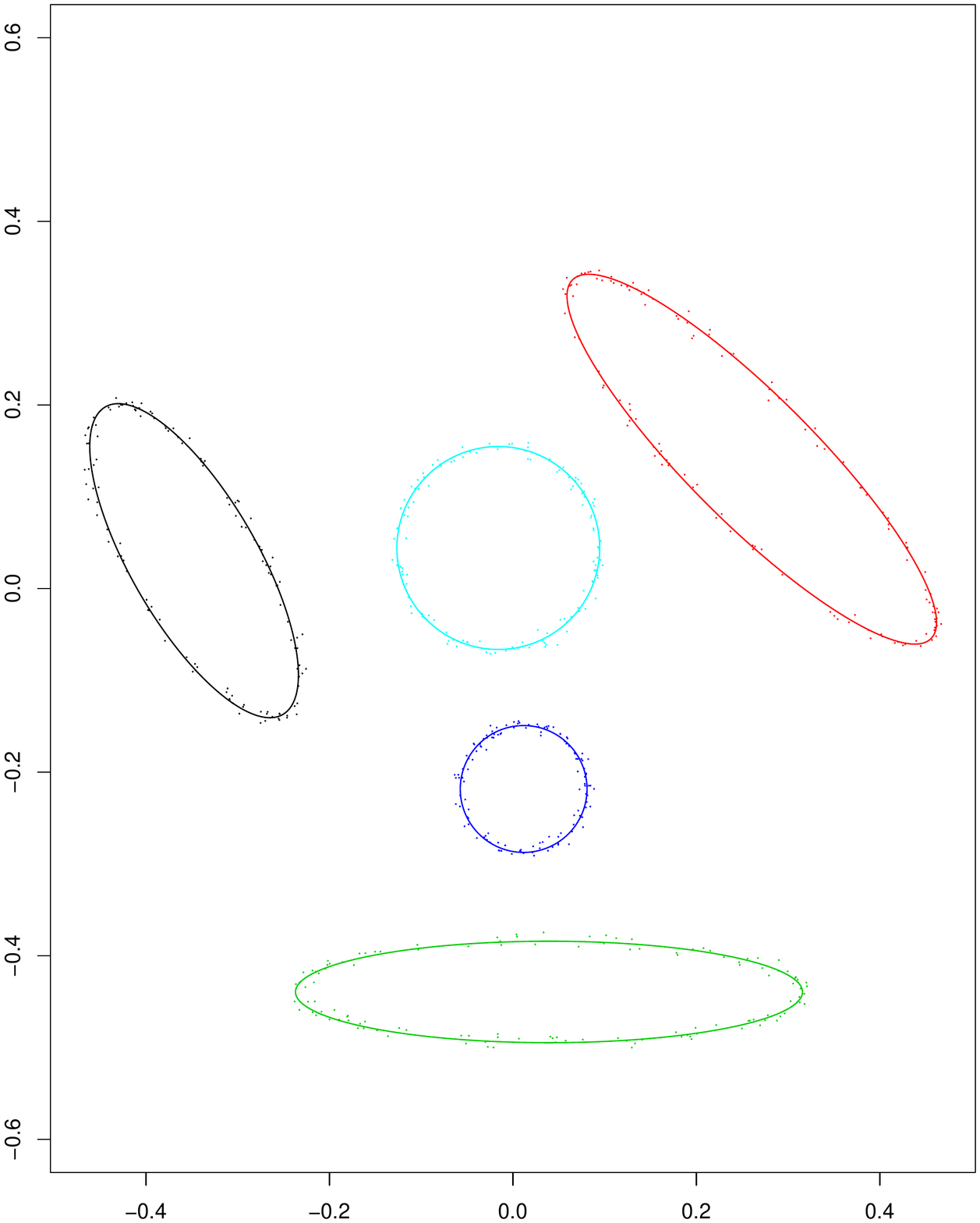}}} & 1 & \textbf{1175.28} & \textbf{-2098.57} & \textbf{-2272.57} & 0.90 & 0.64 & 1163.25 & -1816.05 & -2168.51 & 606.56 & -831.89 & -1095.11 & 563.60 & -745.97 & -1009.19 \\ 
\multicolumn{1}{|c|}{} & 2 & \textbf{1175.28} & \textbf{-2098.57} & \textbf{-2272.57} & 0.90 & 0.64 & 1070.42 & -1630.38 & -1982.84 & 666.59 & -951.95 & -1215.18 & 580.84 & -780.45 & -1043.68 \\ 
\multicolumn{1}{|c|}{} & 4 & \textbf{1237.69} & \textbf{-2275.08} & \textbf{-2413.39} & 0.92 & 0.71 & 1146.97 & -1783.48 & -2135.94 & 666.59 & -951.95 & -1215.18 & 580.84 & -780.45 & -1043.68 \\ 
\multicolumn{1}{|c|}{} & 8 & \textbf{1237.69} & \textbf{-2275.08} & \textbf{-2413.39} & 0.92 & 0.71 & 1081.50 & -1652.55 & -2005.00 & 688.28 & -995.33 & -1258.55 & 580.84 & -780.45 & -1043.68 \\ 
\multicolumn{1}{|c|}{} & 16 & \textbf{1409.54} & \textbf{-2567.08} & \textbf{-2741.07} & 0.98 & 0.89 & 1188.08 & -1865.70 & -2218.16 & 688.28 & -995.33 & -1258.55 & 580.84 & -780.45 & -1043.68 \\ 
\multicolumn{1}{|c|}{} & 32 & \textbf{1432.35} & \textbf{-2612.71} & \textbf{-2786.71} & 0.99 & 0.96 & 1144.56 & -1778.66 & -2131.12 & 688.28 & -995.33 & -1258.55 & 588.26 & -795.30 & -1058.53 \\ 
\multicolumn{1}{|c|}{} & 64 & \textbf{1432.35} & \textbf{-2612.71} & \textbf{-2786.71} & 0.99 & 0.96 & 1169.19 & -1827.92 & -2180.38 & 688.28 & -995.33 & -1258.55 & 588.26 & -795.30 & -1058.53 \\ 
\multicolumn{1}{|c|}{} & 128 & \textbf{1719.17} & \textbf{-3186.35} & \textbf{-3360.35} & 1.00 & 1.00 & 1280.34 & -2050.23 & -2402.68 & 733.81 & -1086.39 & -1349.61 & 612.66 & -844.09 & -1107.31 \\ 
\hline
\hline

\end{tabular}
}

\end{table}

\begin{table}[H]
\caption{Experiment No. 1: Values of the performance metrics computed for the consecutive numbers of starts being the powers of two - Part 2}
\centering
\setlength\tabcolsep{1.5pt}
\resizebox{\columnwidth}{!}{%
\begin{tabular}{c|>{\tiny}c|>{\tiny}c>{\tiny}c>{\tiny}c>{\tiny}c>{\tiny}c|>{\tiny}c>{\tiny}c>{\tiny}c|>{\tiny}c>{\tiny}c>{\tiny}c|>{\tiny}c>{\tiny}c>{\tiny}c|}
\cline{2-16}
& \rule{0pt}{1.5cm} & \multicolumn{5}{c|}{\makecell{MCEC}} & \multicolumn{3}{c|}{\makecell{afCEC Hartigan}} & \multicolumn{3}{c|}{CEC} & \multicolumn{3}{c|}{GMM} \\
\cline{3-16}

& \rotatebox[origin=c]{90}{\rlap{\parbox{1.5cm}{\makecell{Number of \\ starts}}}} &
\multicolumn{1}{c|}{MLE} & \multicolumn{1}{c|}{BIC} & \multicolumn{1}{c|}{AIC} & \multicolumn{1}{c|}{\makecell{Rand \\ index}} & \multicolumn{1}{c|}{\makecell{Jaccard \\ index}} &
\multicolumn{1}{c|}{MLE} & \multicolumn{1}{c|}{BIC} & \multicolumn{1}{c|}{AIC} &
\multicolumn{1}{c|}{MLE} & \multicolumn{1}{c|}{BIC} & \multicolumn{1}{c|}{AIC} &
\multicolumn{1}{c|}{MLE} & \multicolumn{1}{c|}{BIC} & \multicolumn{1}{c|}{AIC} \\

\hline
\hline
\multicolumn{1}{|c|}{\multirow{8}{*}{\includegraphics[scale=0.07]{5-1-6.pdf}}} & 1 & 292.77 & -326.44 & -507.54 & 0.51 & 0.19 & \textbf{827.79} & \textbf{-1024.42} & \textbf{-1465.58} & 617.94 & -764.17 & -1093.88 & 562.10 & -652.49 & -982.20 \\ 
\multicolumn{1}{|c|}{} & 2 & \textbf{1417.44} & \textbf{-2575.76} & \textbf{-2756.87} & 0.94 & 0.75 & 1031.08 & -1431.00 & -1872.16 & 634.18 & -796.65 & -1126.36 & 574.76 & -677.82 & -1007.53 \\ 
\multicolumn{1}{|c|}{} & 4 & \textbf{1417.44} & \textbf{-2575.76} & \textbf{-2756.87} & 0.94 & 0.75 & 1078.51 & -1525.87 & -1967.03 & 634.18 & -796.65 & -1126.36 & 574.76 & -677.82 & -1007.53 \\ 
\multicolumn{1}{|c|}{} & 8 & \textbf{1417.44} & \textbf{-2575.76} & \textbf{-2756.87} & 0.94 & 0.75 & 1028.08 & -1425.00 & -1866.16 & 634.18 & -796.65 & -1126.36 & 601.30 & -730.89 & -1060.60 \\ 
\multicolumn{1}{|c|}{} & 16 & \textbf{1417.44} & \textbf{-2575.76} & \textbf{-2756.87} & 0.94 & 0.75 & 1270.27 & -1909.39 & -2350.55 & 715.22 & -958.73 & -1288.44 & 601.30 & -730.89 & -1060.60 \\ 
\multicolumn{1}{|c|}{} & 32 & \textbf{1417.44} & \textbf{-2575.76} & \textbf{-2756.87} & 0.94 & 0.75 & 1180.76 & -1730.35 & -2171.51 & 715.22 & -958.73 & -1288.44 & 602.87 & -734.03 & -1063.74 \\ 
\multicolumn{1}{|c|}{} & 64 & \textbf{1431.28} & \textbf{-2550.30} & \textbf{-2768.56} & 0.95 & 0.75 & 1300.24 & -1969.32 & -2410.48 & 715.22 & -958.73 & -1288.44 & 602.87 & -734.03 & -1063.74 \\ 
\multicolumn{1}{|c|}{} & 128 & \textbf{1898.91} & \textbf{-3485.57} & \textbf{-3703.83} & 1.00 & 1.00 & 1312.19 & -1993.21 & -2434.37 & 715.22 & -958.73 & -1288.44 & 602.87 & -734.03 & -1063.74 \\ 
\hline
\hline
\multicolumn{1}{|c|}{\multirow{9}{*}{\includegraphics[scale=0.07]{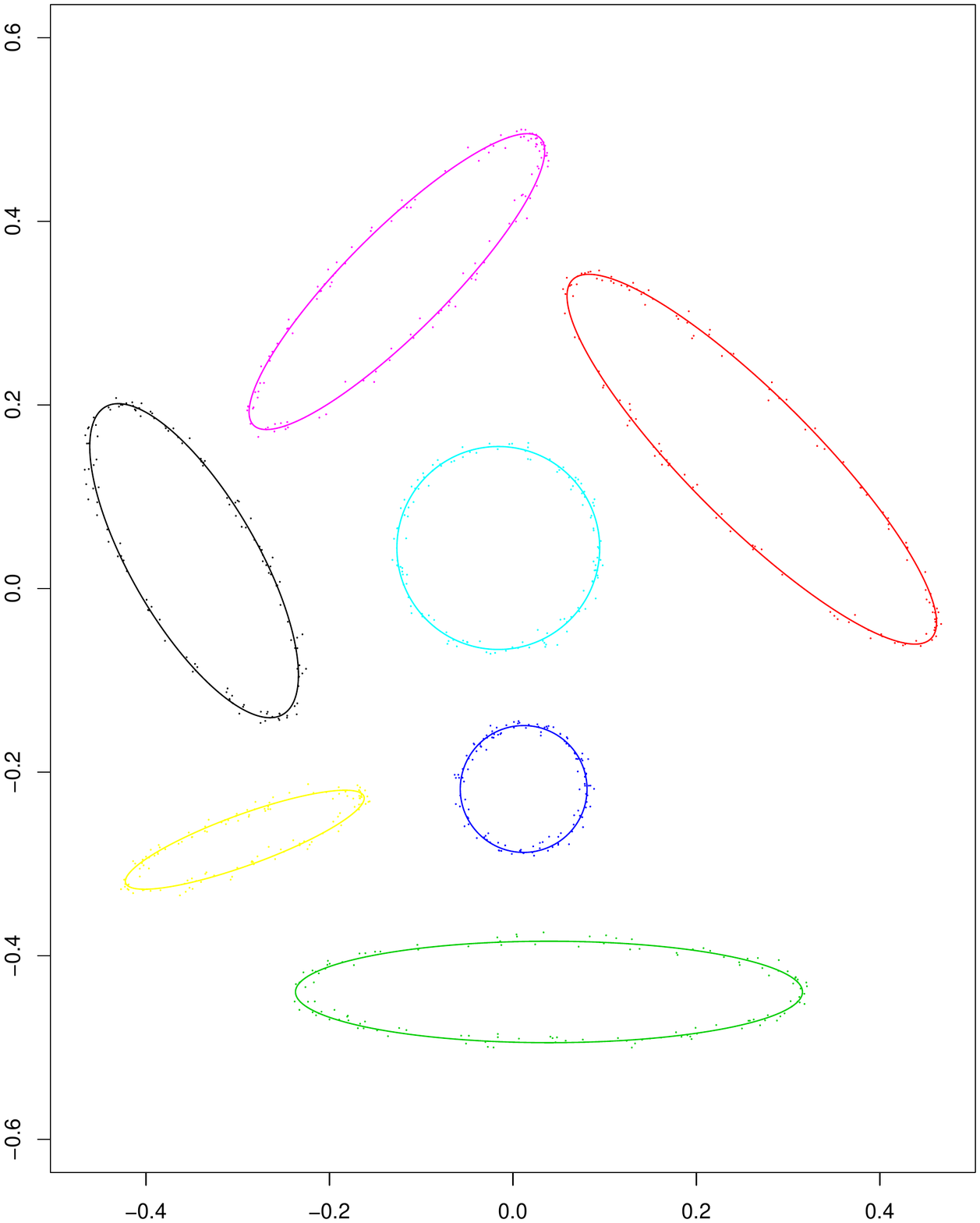}}} & 1 & 969.97 & \textbf{-1566.05} & \textbf{-1829.93} & 0.86 & 0.44 & \textbf{1009.66} & -1264.75 & -1797.32 & 706.99 & -849.76 & -1247.99 & 649.54 & -734.86 & -1133.09 \\ 
\multicolumn{1}{|c|}{} & 2 & 969.97 & -1566.05 & -1829.93 & 0.86 & 0.44 & \textbf{1173.23} & \textbf{-1591.89} & \textbf{-2124.46} & 706.99 & -849.76 & -1247.99 & 649.54 & -734.86 & -1133.09 \\ 
\multicolumn{1}{|c|}{} & 4 & 1112.25 & \textbf{-1959.38} & -2146.50 & 0.92 & 0.63 & \textbf{1272.98} & -1791.38 & \textbf{-2323.96} & 721.88 & -879.52 & -1277.75 & 650.27 & -736.31 & -1134.54 \\ 
\multicolumn{1}{|c|}{} & 8 & 1195.10 & \textbf{-2125.07} & \textbf{-2312.19} & 0.90 & 0.56 & \textbf{1212.79} & -1671.02 & -2203.59 & 779.91 & -995.60 & -1393.83 & 650.27 & -736.31 & -1134.54 \\ 
\multicolumn{1}{|c|}{} & 16 & 1195.10 & \textbf{-2125.07} & \textbf{-2312.19} & 0.90 & 0.56 & \textbf{1255.69} & -1756.82 & -2289.39 & 779.91 & -995.60 & -1393.83 & 662.02 & -759.81 & -1158.04 \\ 
\multicolumn{1}{|c|}{} & 32 & \textbf{1694.77} & \textbf{-3015.64} & \textbf{-3279.53} & 0.97 & 0.79 & 1285.07 & -1815.57 & -2348.14 & 807.84 & -1051.45 & -1449.67 & 673.93 & -783.64 & -1181.87 \\ 
\multicolumn{1}{|c|}{} & 64 & \textbf{1801.77} & \textbf{-3229.66} & \textbf{-3493.55} & 0.98 & 0.89 & 1352.85 & -1951.13 & -2483.70 & 807.84 & -1051.45 & -1449.67 & 678.26 & -792.30 & -1190.52 \\ 
\multicolumn{1}{|c|}{} & 128 & \textbf{1834.64} & \textbf{-3295.39} & \textbf{-3559.28} & 1.00 & 0.97 & 1323.68 & -1892.78 & -2425.35 & 807.84 & -1051.45 & -1449.67 & 687.53 & -810.83 & -1209.06 \\ 
\multicolumn{1}{|c|}{} & 256 & \textbf{2121.67} & \textbf{-3869.45} & \textbf{-4133.33} & 1.00 & 1.00 & 1402.26 & -2049.94 & -2582.51 & 807.84 & -1051.45 & -1449.67 & 687.53 & -810.83 & -1209.06 \\ 
\hline
\hline
\multicolumn{1}{|c|}{\multirow{8}{*}{\includegraphics[scale=0.07]{5-1-8.pdf}}} & 1 & \textbf{1346.01} & \textbf{-2255.34} & \textbf{-2566.02} & 0.92 & 0.55 & 1169.40 & -1458.50 & -2084.79 & 734.98 & -811.47 & -1279.96 & 634.84 & -611.20 & -1079.69 \\ 
\multicolumn{1}{|c|}{} & 2 & \textbf{1346.01} & \textbf{-2255.34} & \textbf{-2566.02} & 0.92 & 0.55 & 1185.49 & -1490.67 & -2116.97 & 734.98 & -811.47 & -1279.96 & 634.84 & -611.20 & -1079.69 \\ 
\multicolumn{1}{|c|}{} & 4 & \textbf{1346.01} & \textbf{-2255.34} & \textbf{-2566.02} & 0.92 & 0.55 & 1257.36 & -1634.42 & -2260.71 & 740.07 & -821.65 & -1290.14 & 634.84 & -611.20 & -1079.69 \\ 
\multicolumn{1}{|c|}{} & 8 & \textbf{1346.01} & \textbf{-2255.34} & \textbf{-2566.02} & 0.92 & 0.55 & 1300.08 & -1719.86 & -2346.16 & 794.29 & -930.09 & -1398.58 & 683.47 & -708.45 & -1176.94 \\ 
\multicolumn{1}{|c|}{} & 16 & \textbf{1346.01} & \textbf{-2255.34} & \textbf{-2566.02} & 0.92 & 0.55 & 1289.30 & -1698.30 & -2324.60 & 794.29 & -930.09 & -1398.58 & 683.47 & -708.45 & -1176.94 \\ 
\multicolumn{1}{|c|}{} & 32 & 1346.01 & \textbf{-2255.34} & -2566.02 & 0.92 & 0.55 & \textbf{1462.78} & -2045.25 & \textbf{-2671.55} & 800.12 & -941.75 & -1410.24 & 683.47 & -708.45 & -1176.94 \\ 
\multicolumn{1}{|c|}{} & 64 & \textbf{2022.93} & \textbf{-3609.17} & \textbf{-3919.85} & 0.98 & 0.85 & 1428.28 & -1976.26 & -2602.55 & 808.52 & -958.56 & -1427.05 & 683.47 & -708.45 & -1176.94 \\ 
\multicolumn{1}{|c|}{} & 128 & \textbf{2340.88} & \textbf{-4245.08} & \textbf{-4555.76} & 1.00 & 1.00 & 1569.73 & -2259.17 & -2885.47 & 817.11 & -975.74 & -1444.23 & 683.47 & -708.45 & -1176.94 \\ 
\hline
\hline

\end{tabular}
}

\end{table}


\begin{figure}[H]
	\centering
	    \begin{tabular}{c@{}c@{}c@{}c@{}}
    		\includegraphics[width=0.24\textwidth]{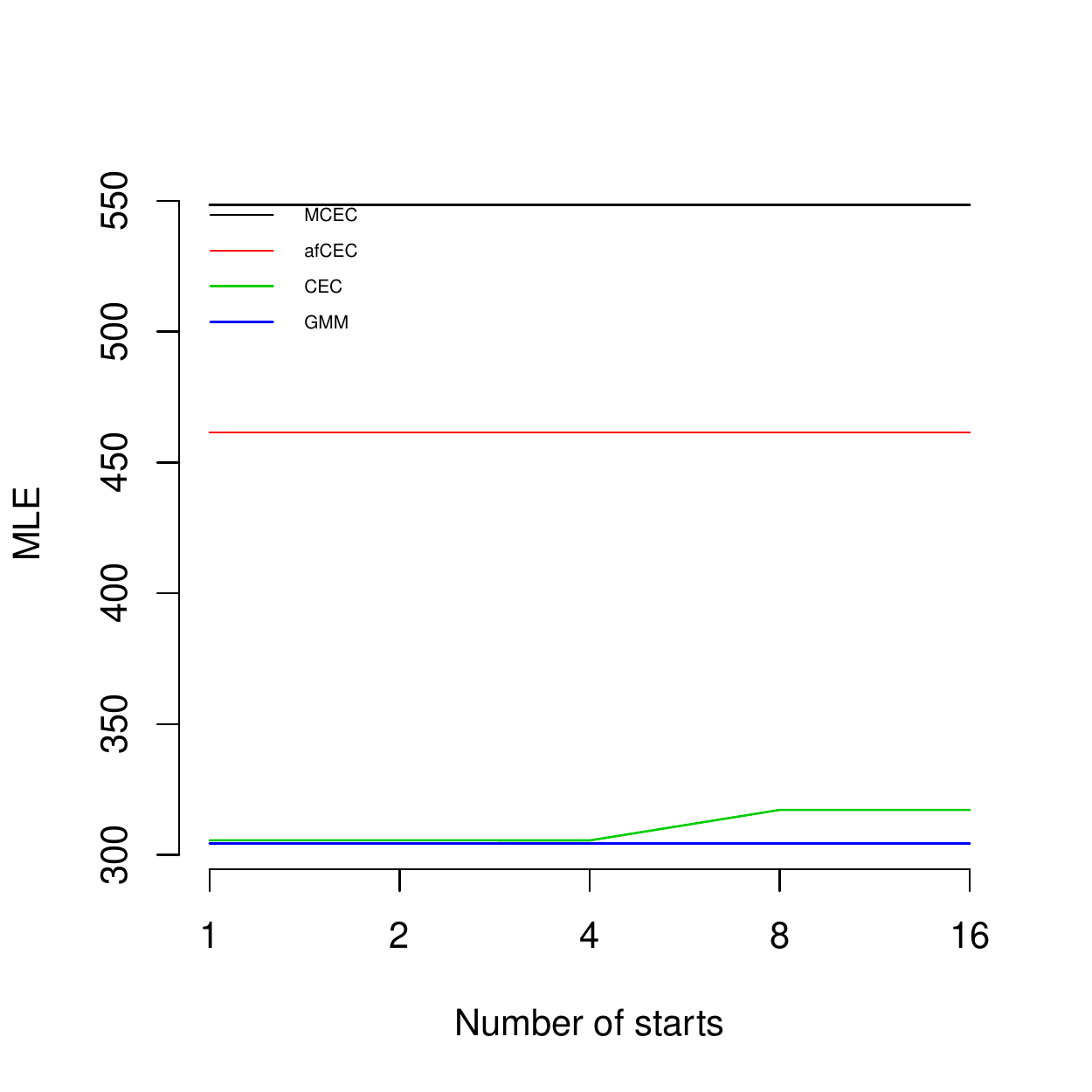} &
    		\includegraphics[width=0.24\textwidth]{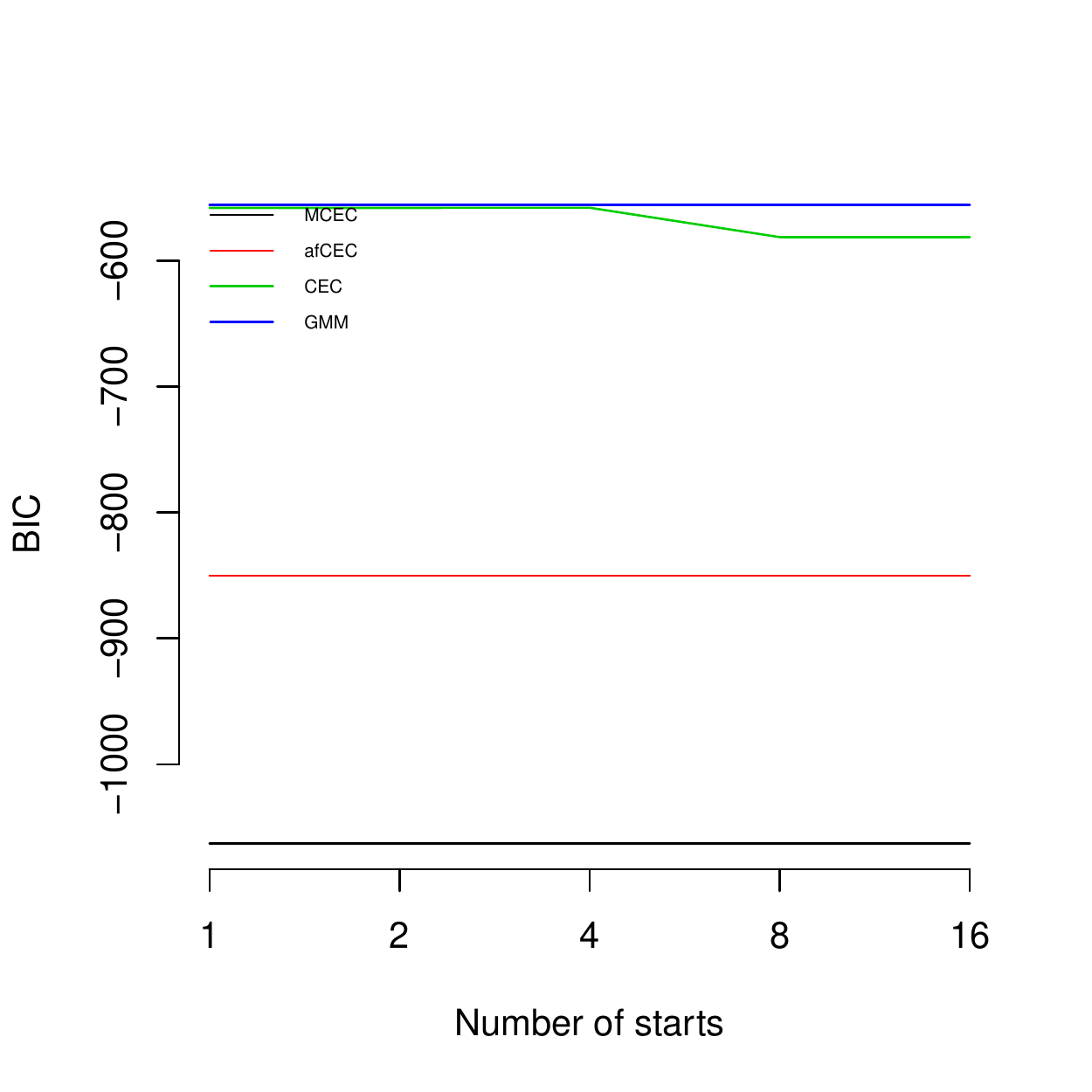} &
    		\includegraphics[width=0.24\textwidth]{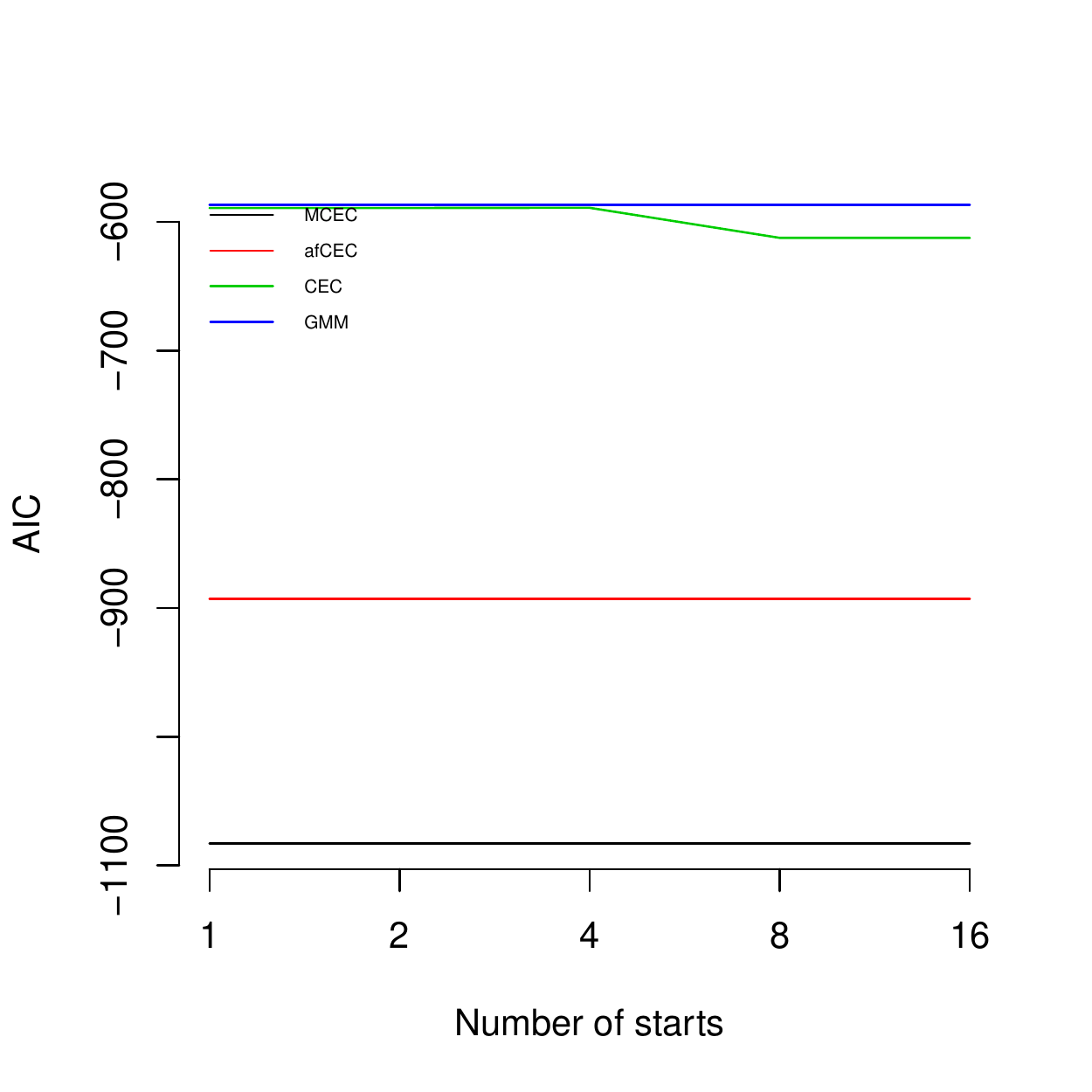} &
    		\includegraphics[width=0.24\textwidth]{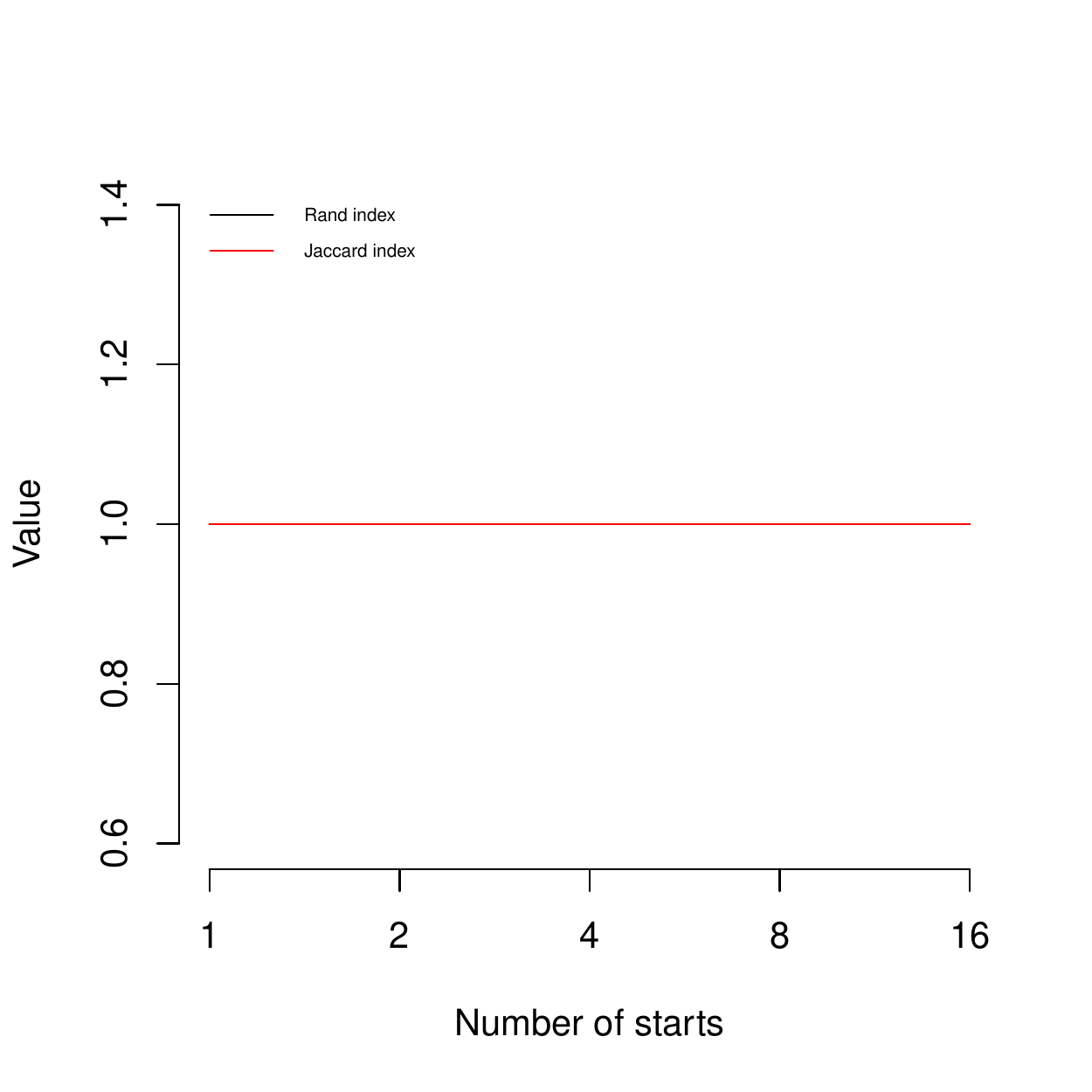} \\
    		\includegraphics[width=0.24\textwidth]{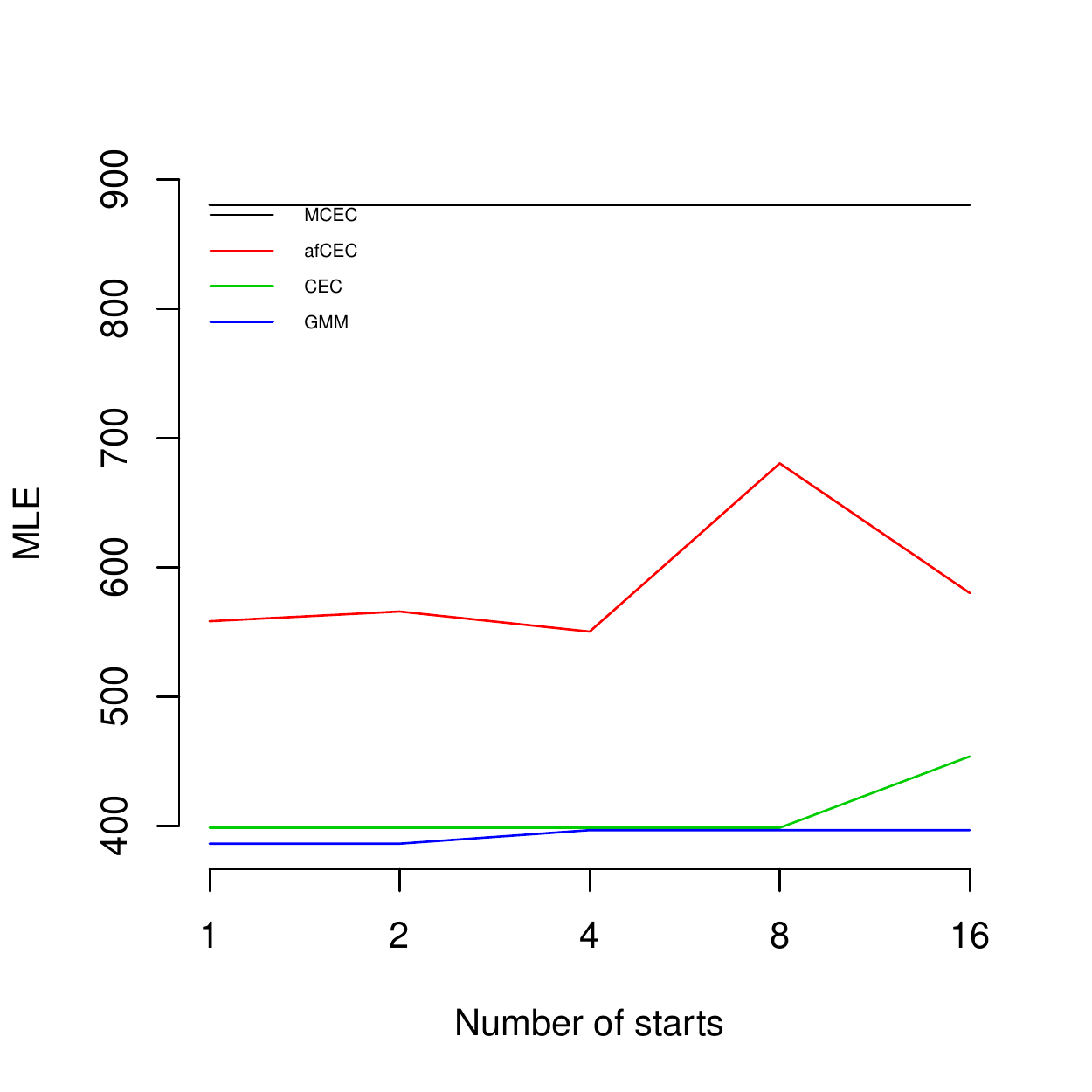} &
    		\includegraphics[width=0.24\textwidth]{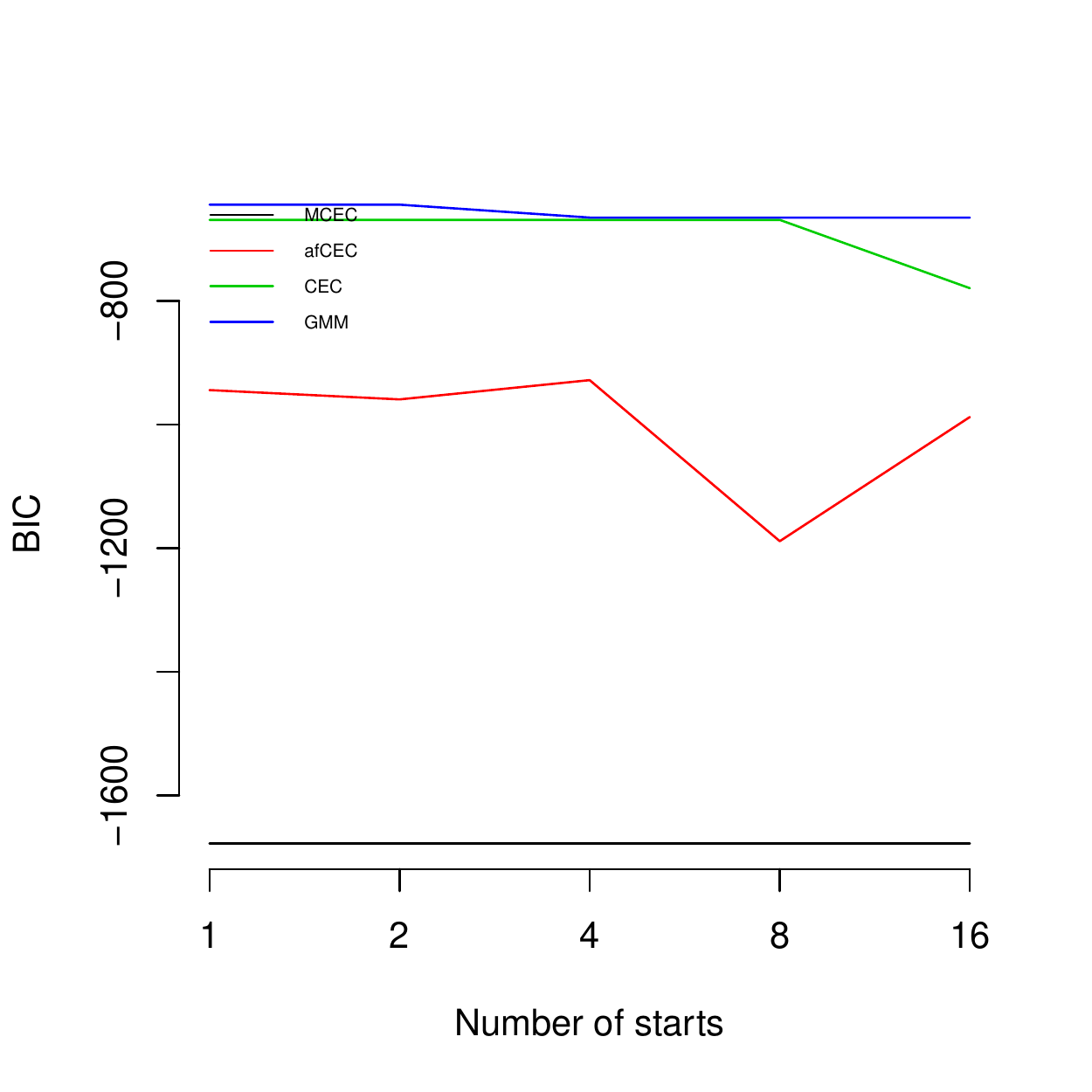} &
    		\includegraphics[width=0.24\textwidth]{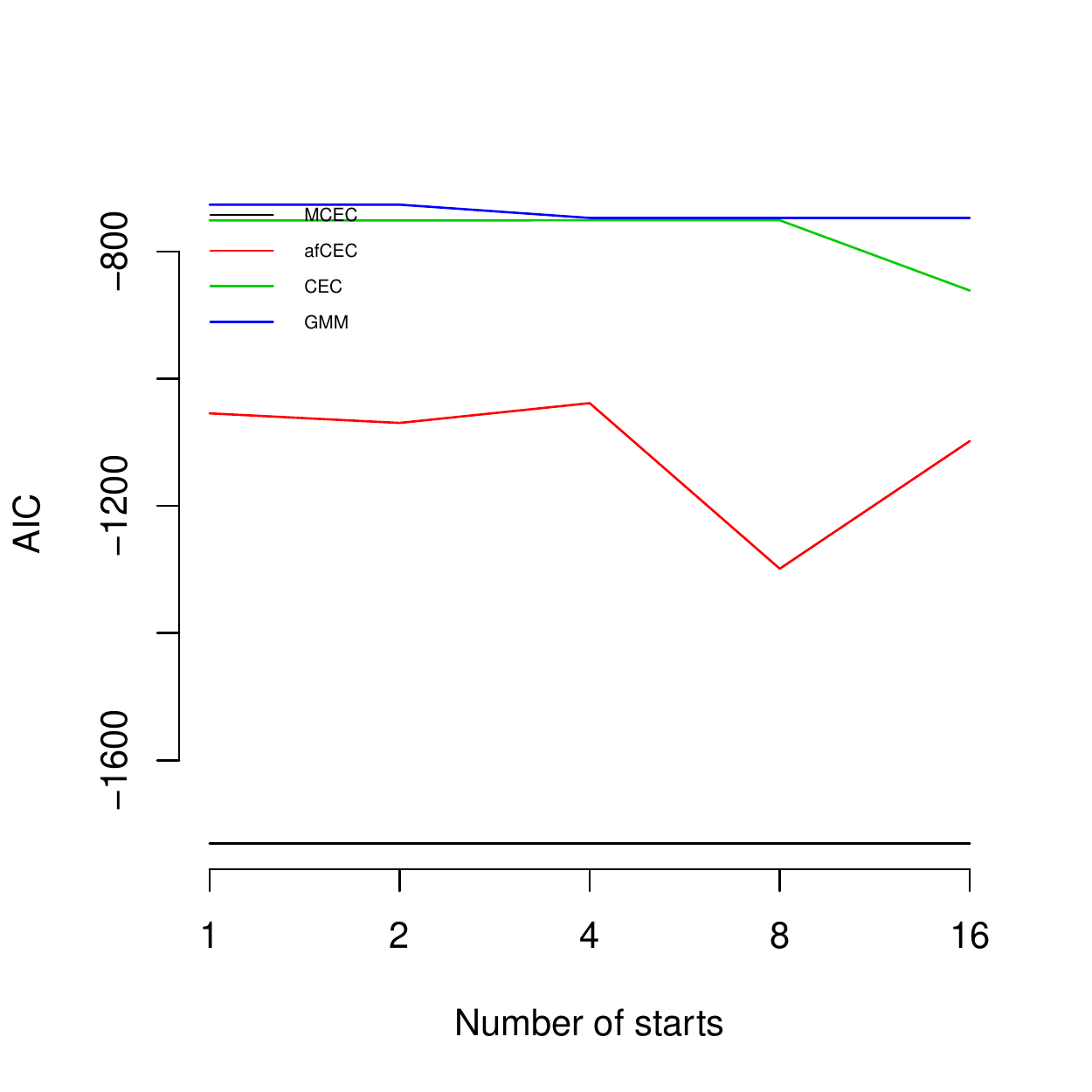} &
    		\includegraphics[width=0.24\textwidth]{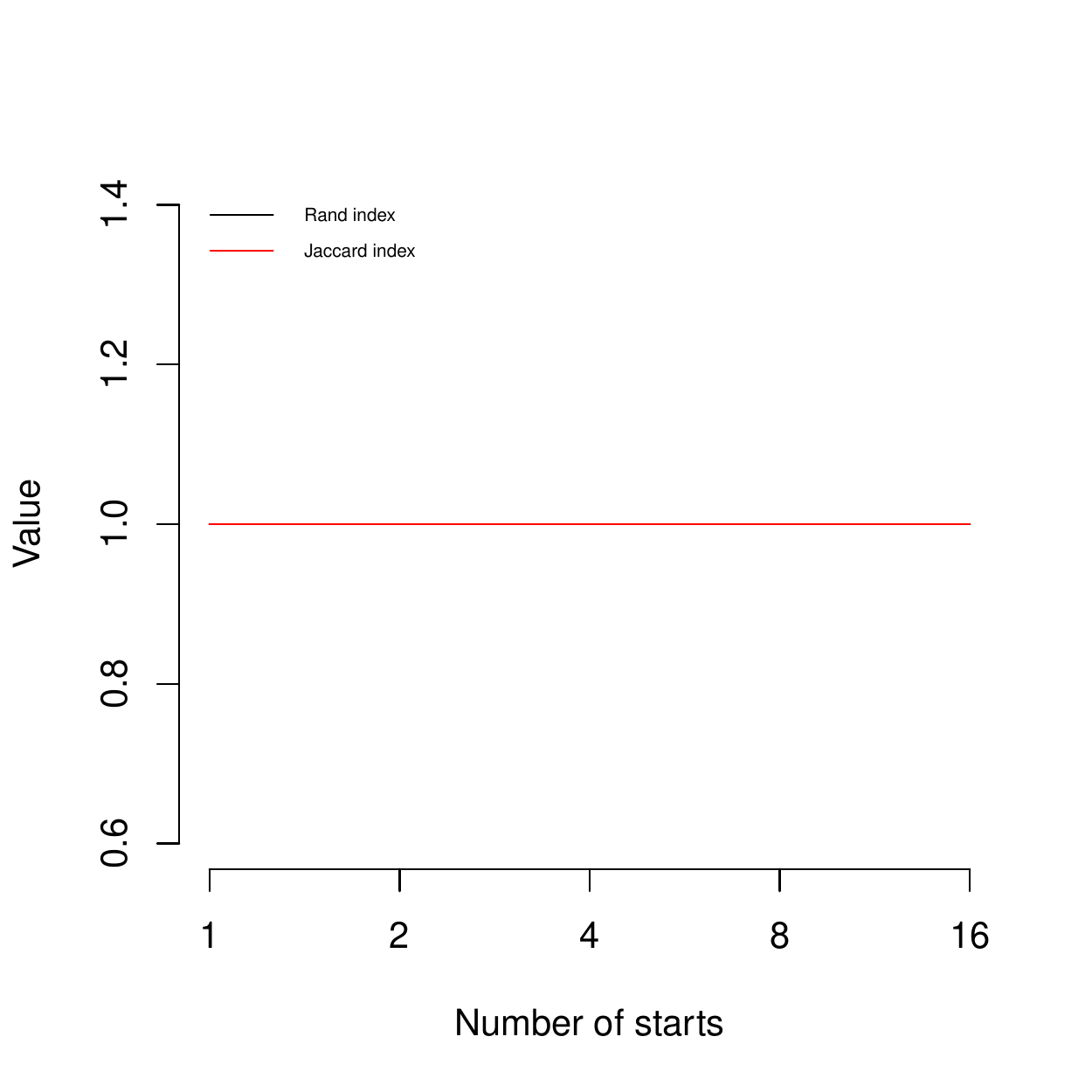} \\
    		\includegraphics[width=0.24\textwidth]{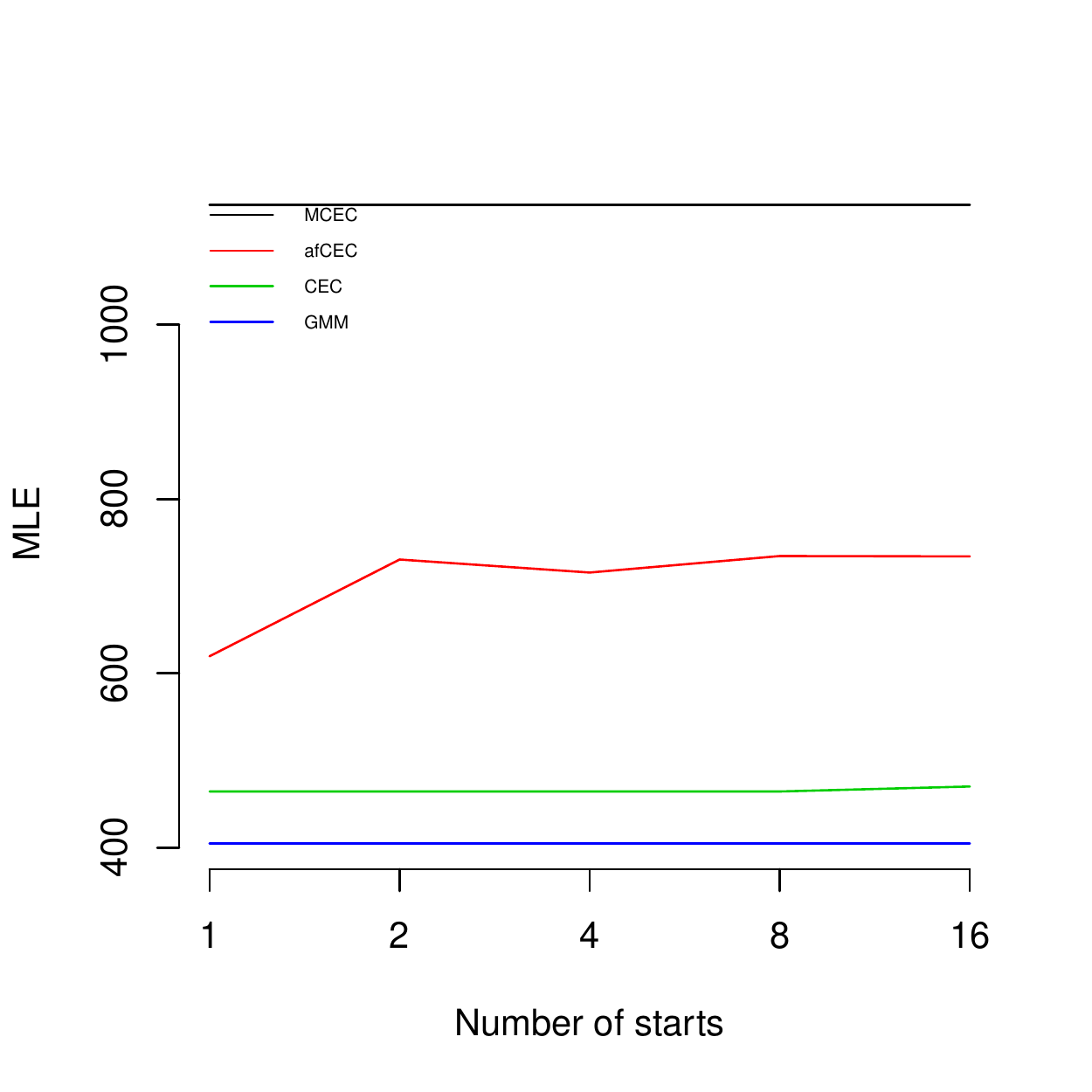} &
    		\includegraphics[width=0.24\textwidth]{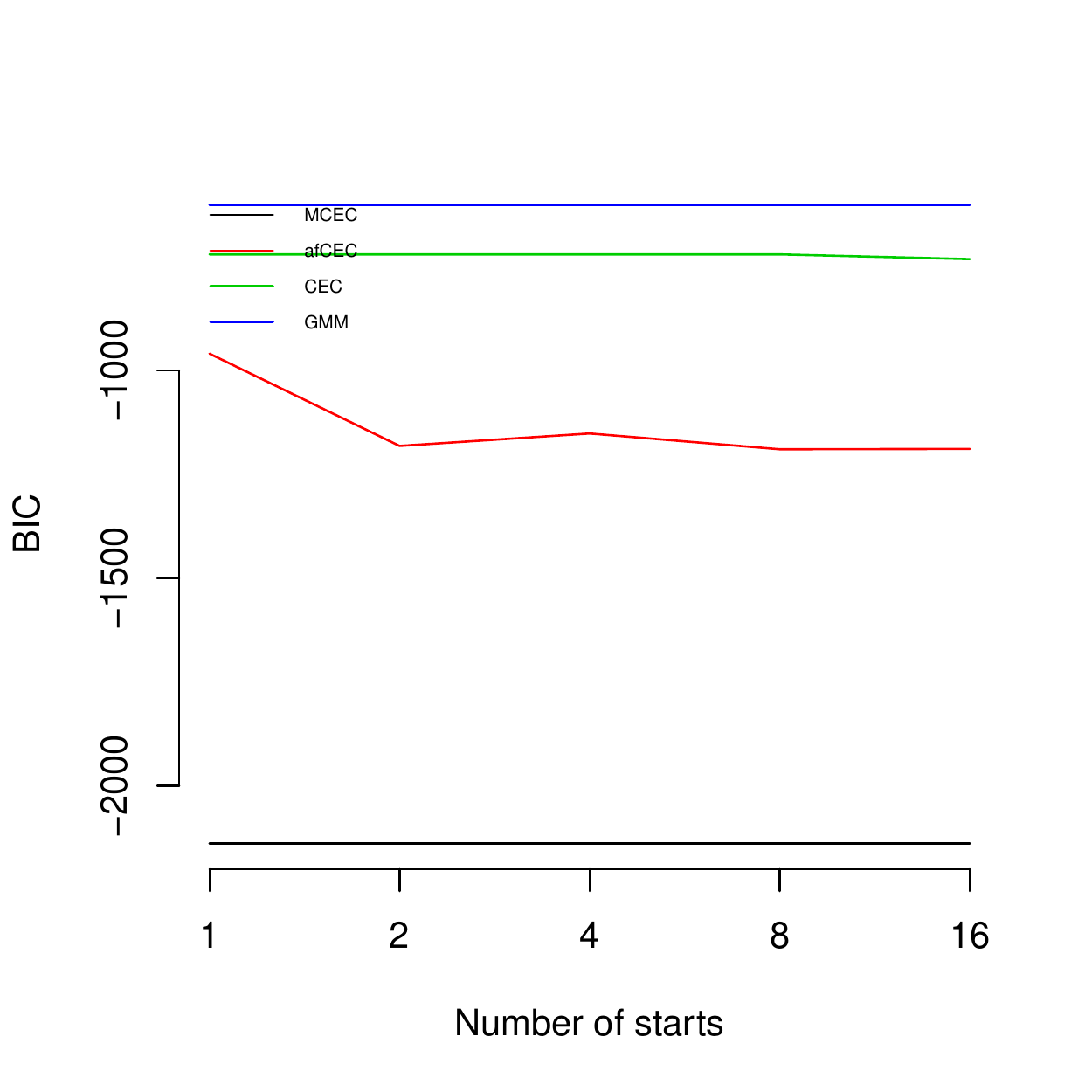} &
    		\includegraphics[width=0.24\textwidth]{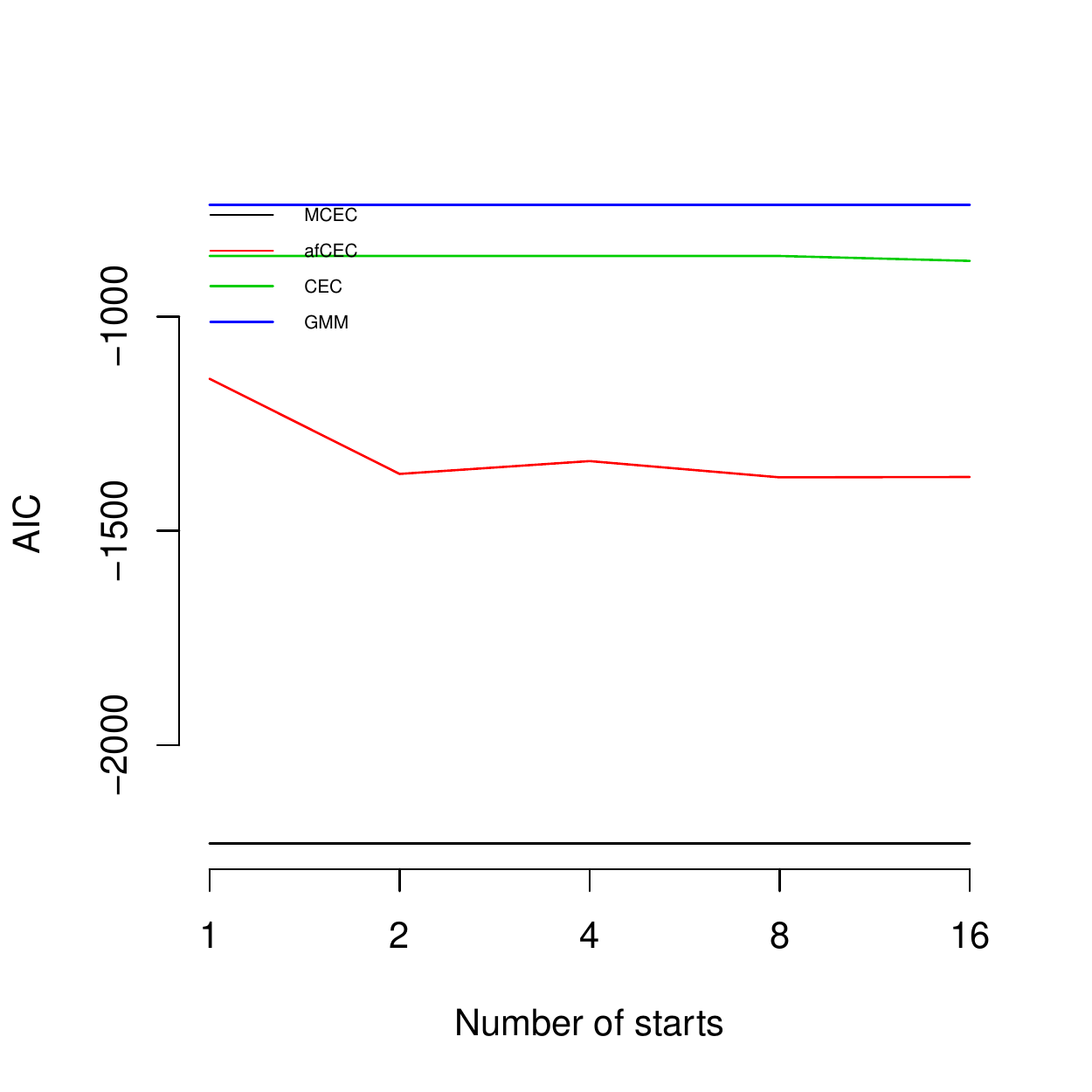} &
    		\includegraphics[width=0.24\textwidth]{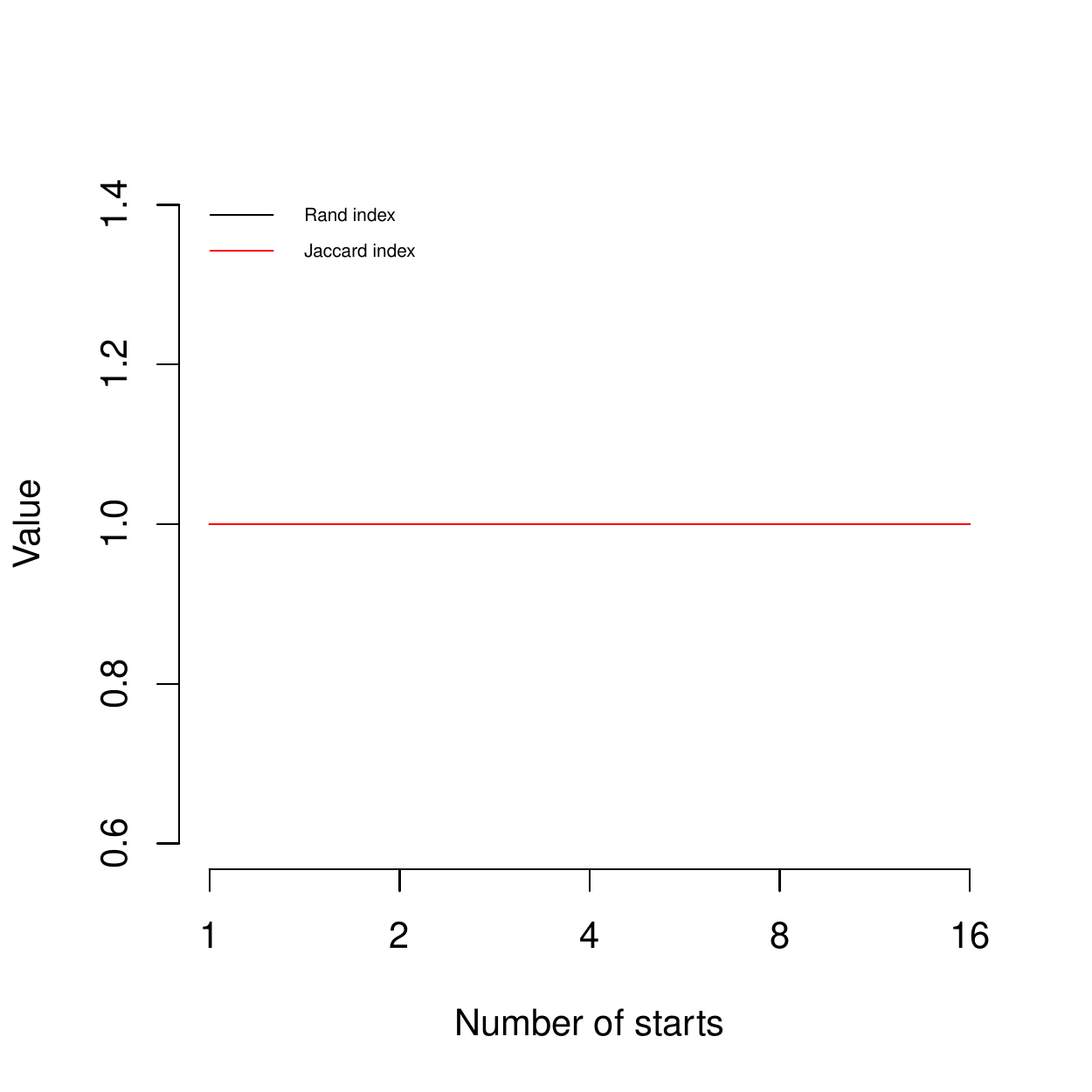} \\
    		\includegraphics[width=0.24\textwidth]{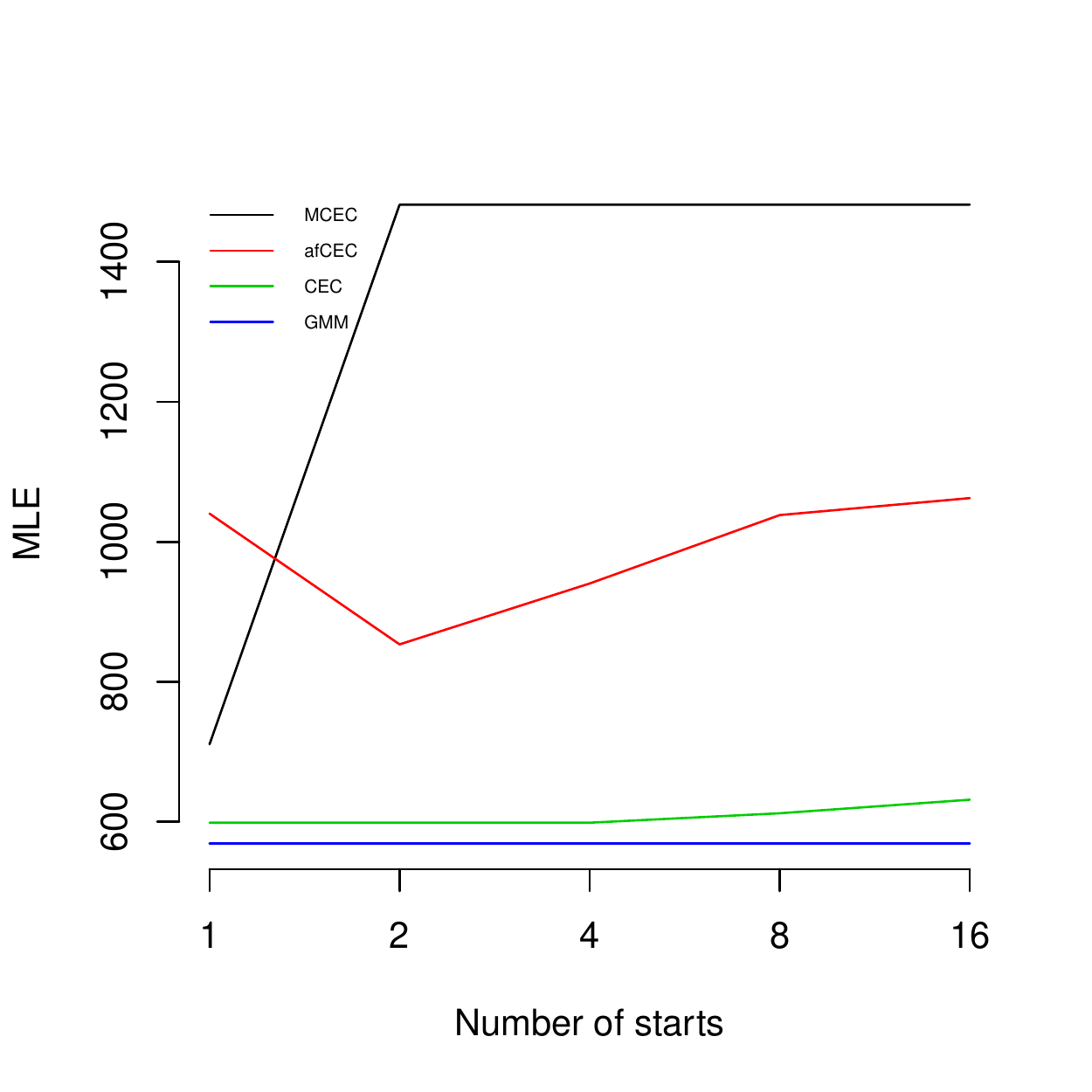} &
    		\includegraphics[width=0.24\textwidth]{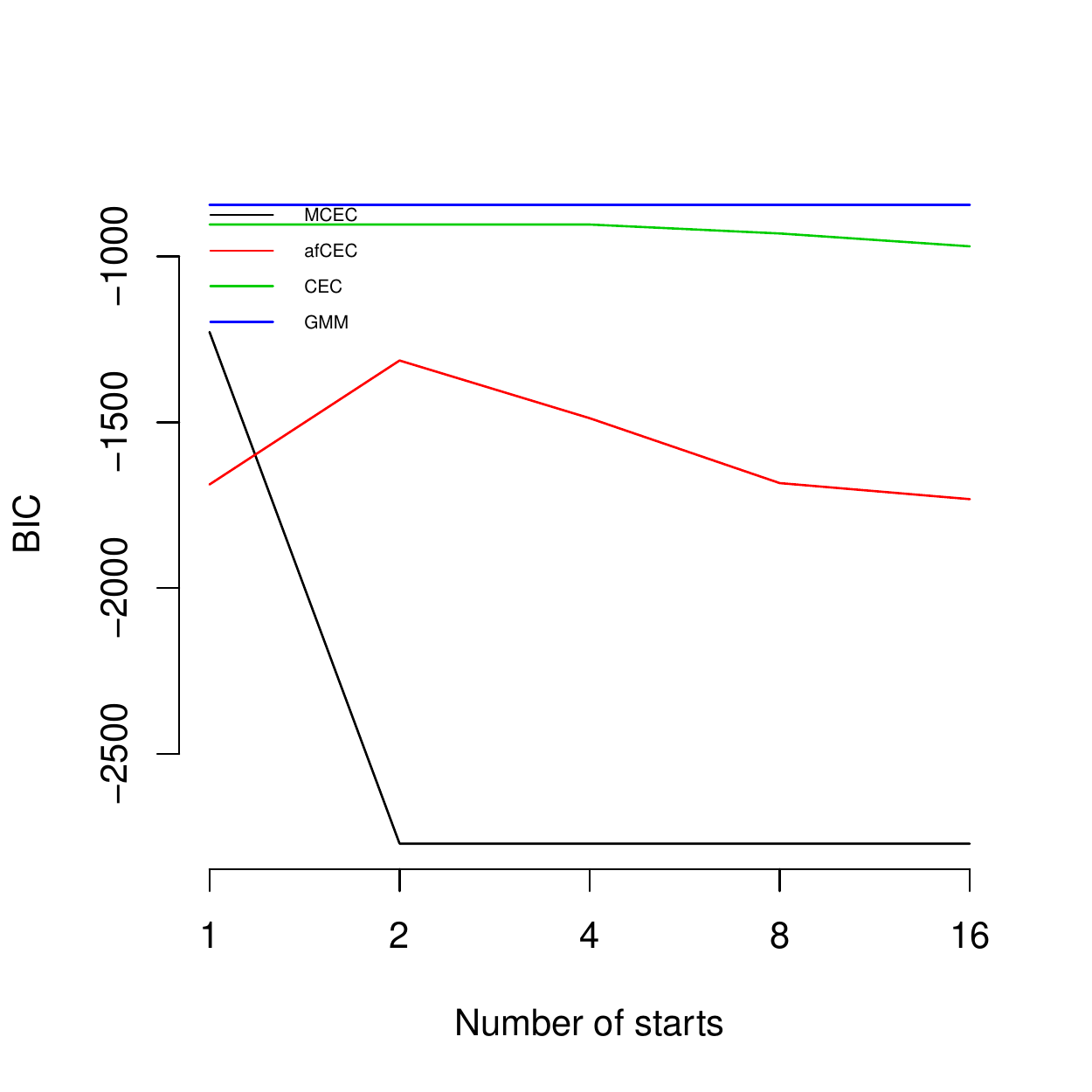} &
    		\includegraphics[width=0.24\textwidth]{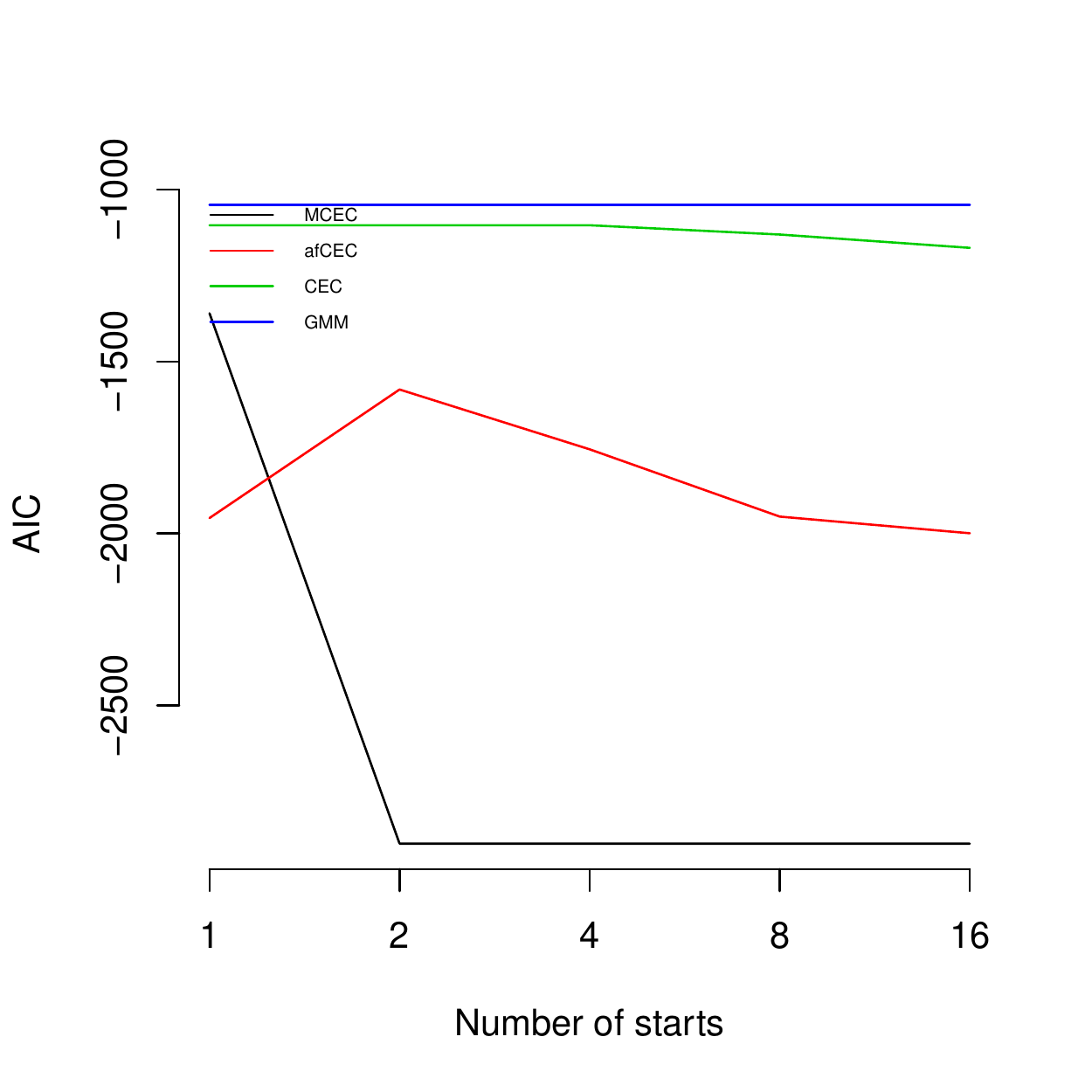} &
    		\includegraphics[width=0.24\textwidth]{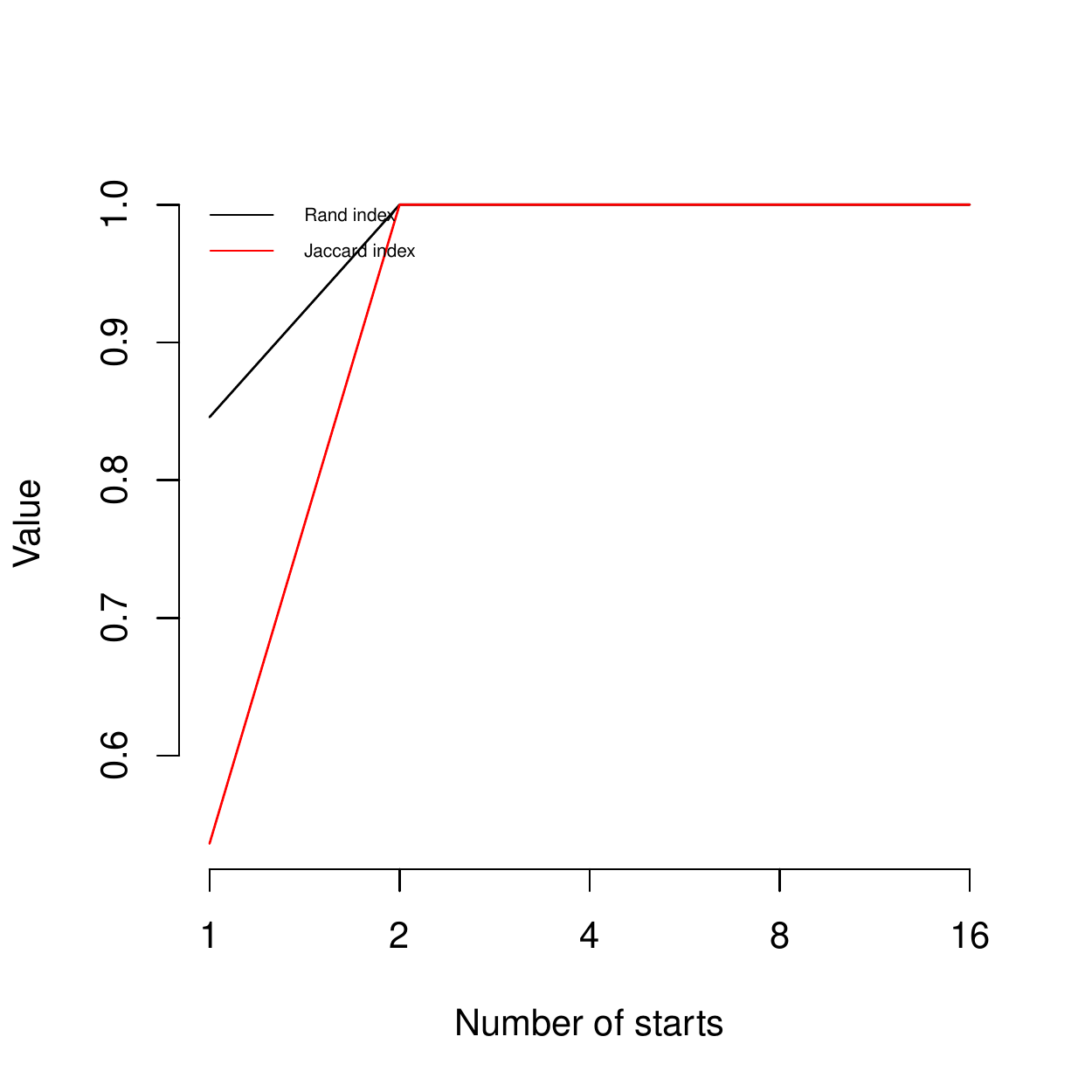} \\
    	\end{tabular}
	\caption{Experiment No. 1: Values of the performance metrics computed for the consecutive numbers of starts being the powers of two and the increasing number of clusters in the data set - Part 1. The charts in the first, second, and third columns show the computed values of MLE, BIC, and AIC, respectively, for the consecutive number of starts being the powers of two for all four clustering methods used in the comparison, ie. MCEC, afCEC, CEC, and GMM, while the last one represents the values of Rand and Jaccard indices computed similarly for the clustering obtained using the MCEC algorithm. The table's first, second, third, and fourth row correspondingly represents the test case with one, two, three, and four clusters. }
\end{figure}

\begin{figure}[H]
	\centering
	    \begin{tabular}{c@{}c@{}c@{}c@{}}
    		\includegraphics[width=0.24\textwidth]{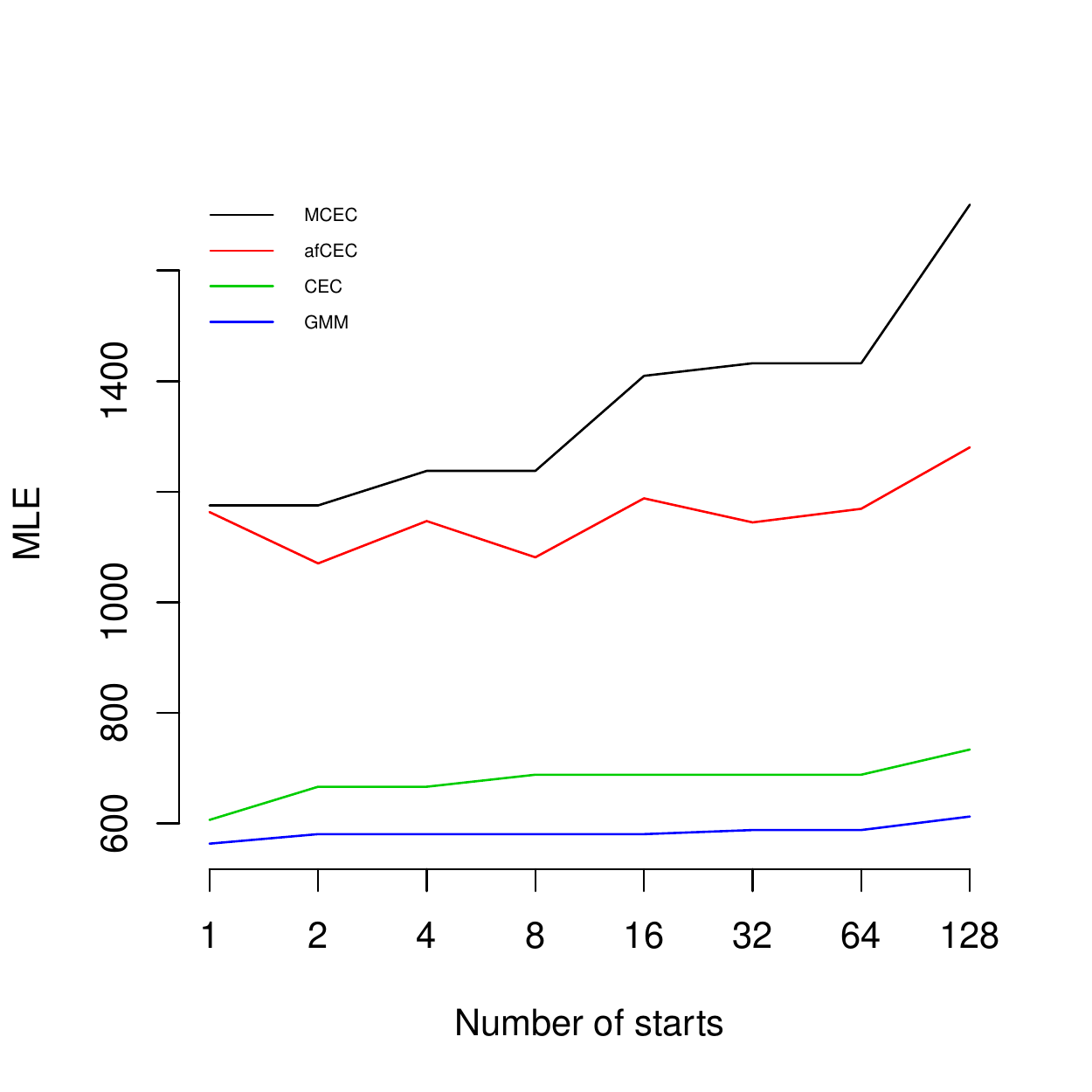} &
    		\includegraphics[width=0.24\textwidth]{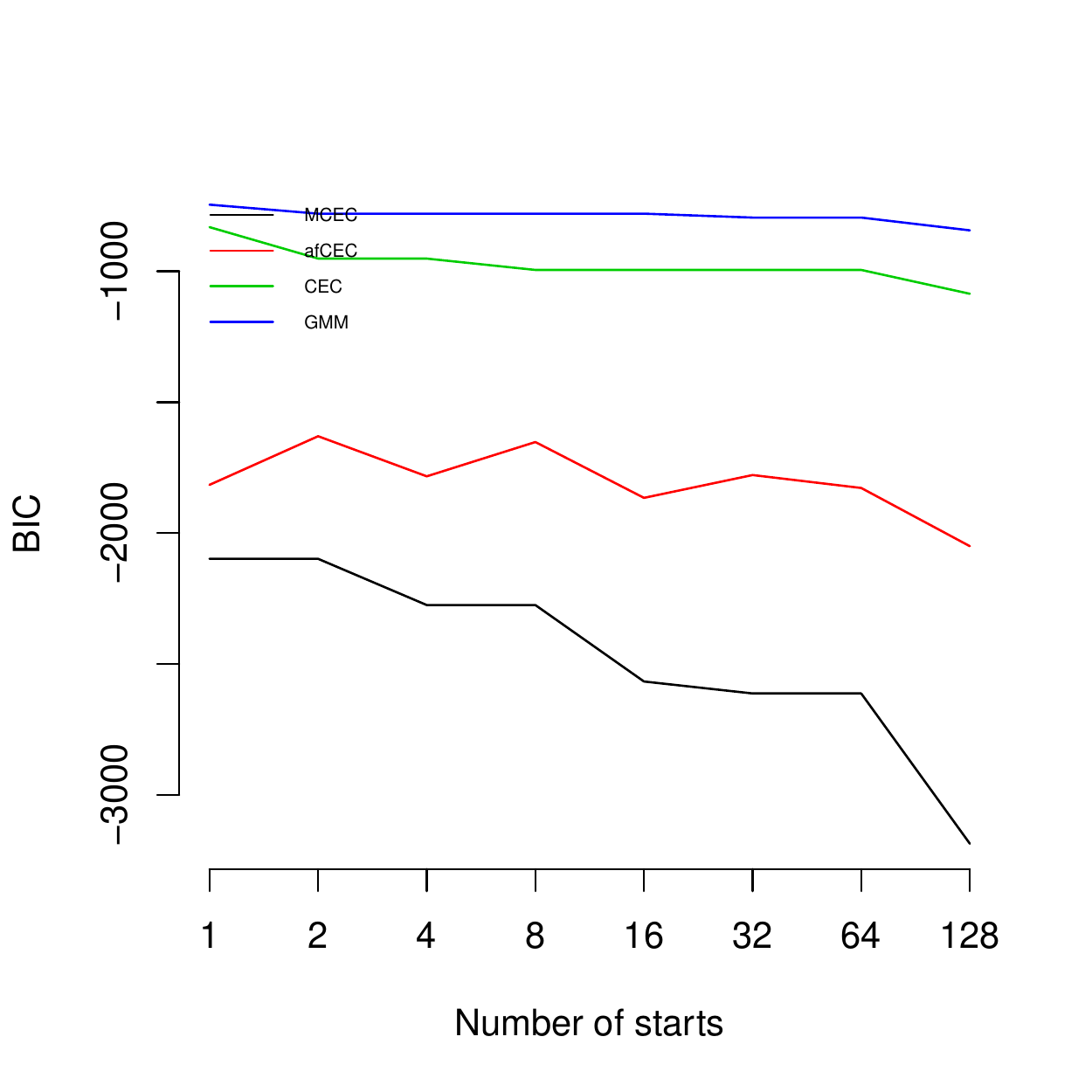} &
    		\includegraphics[width=0.24\textwidth]{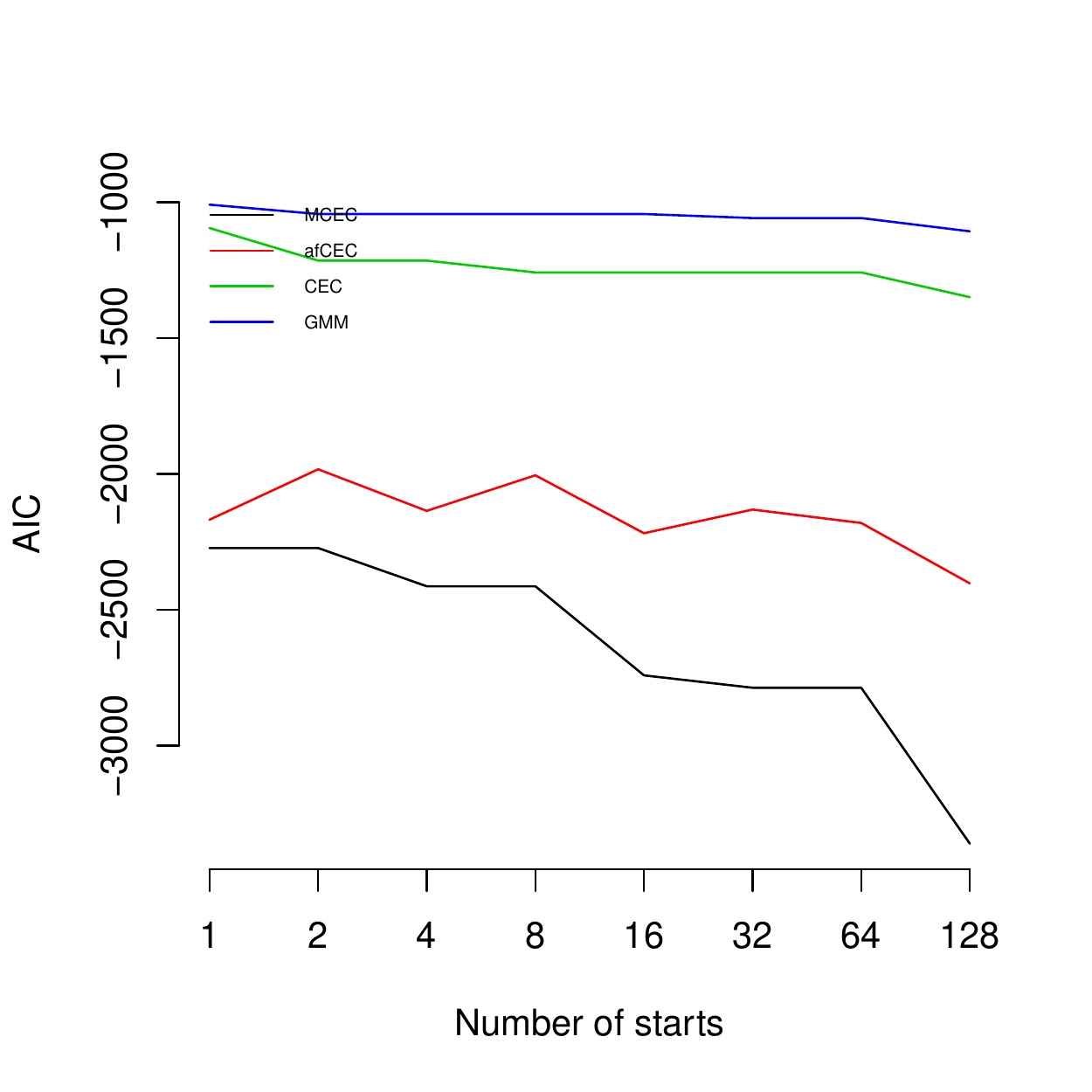} &
    		\includegraphics[width=0.24\textwidth]{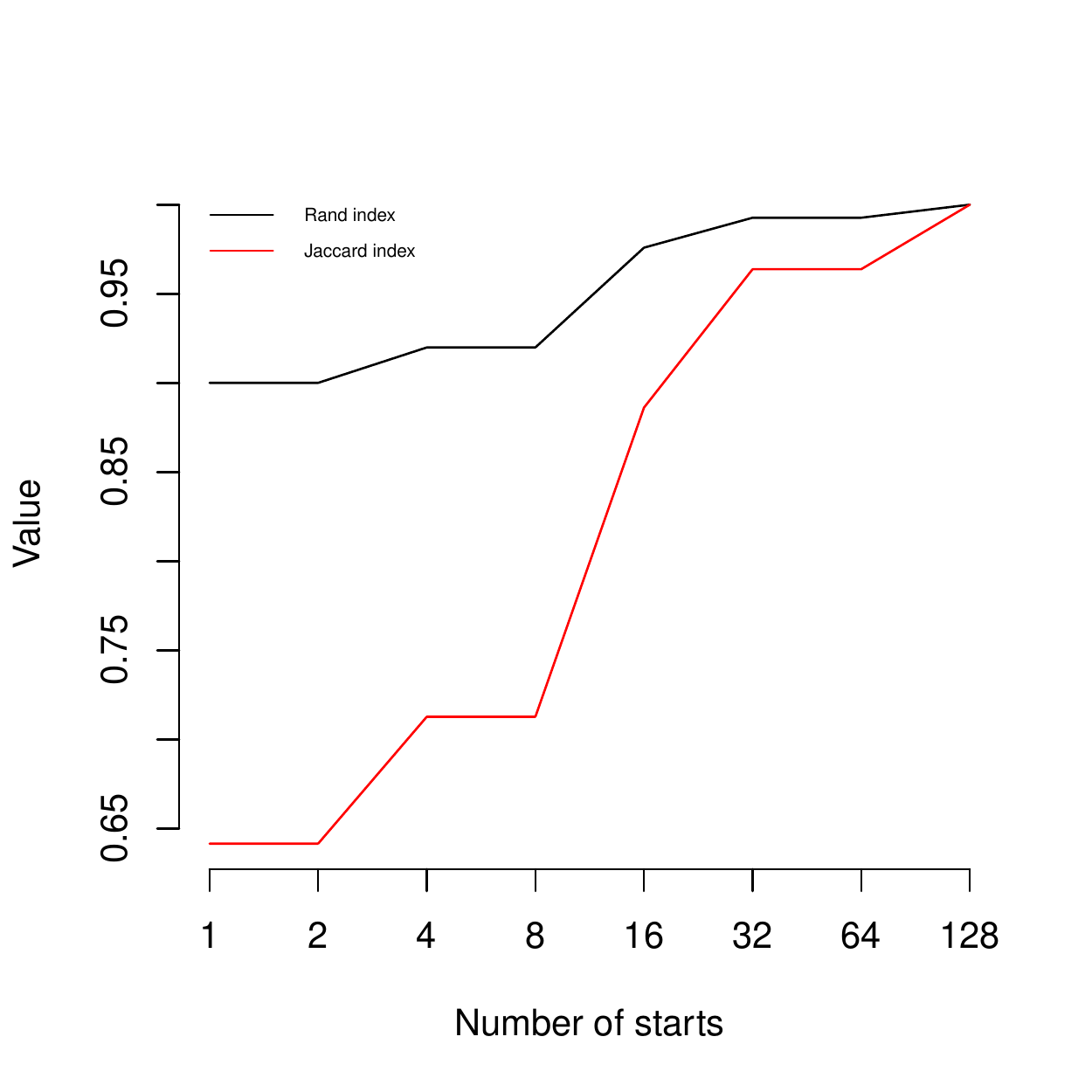} \\
    		\includegraphics[width=0.24\textwidth]{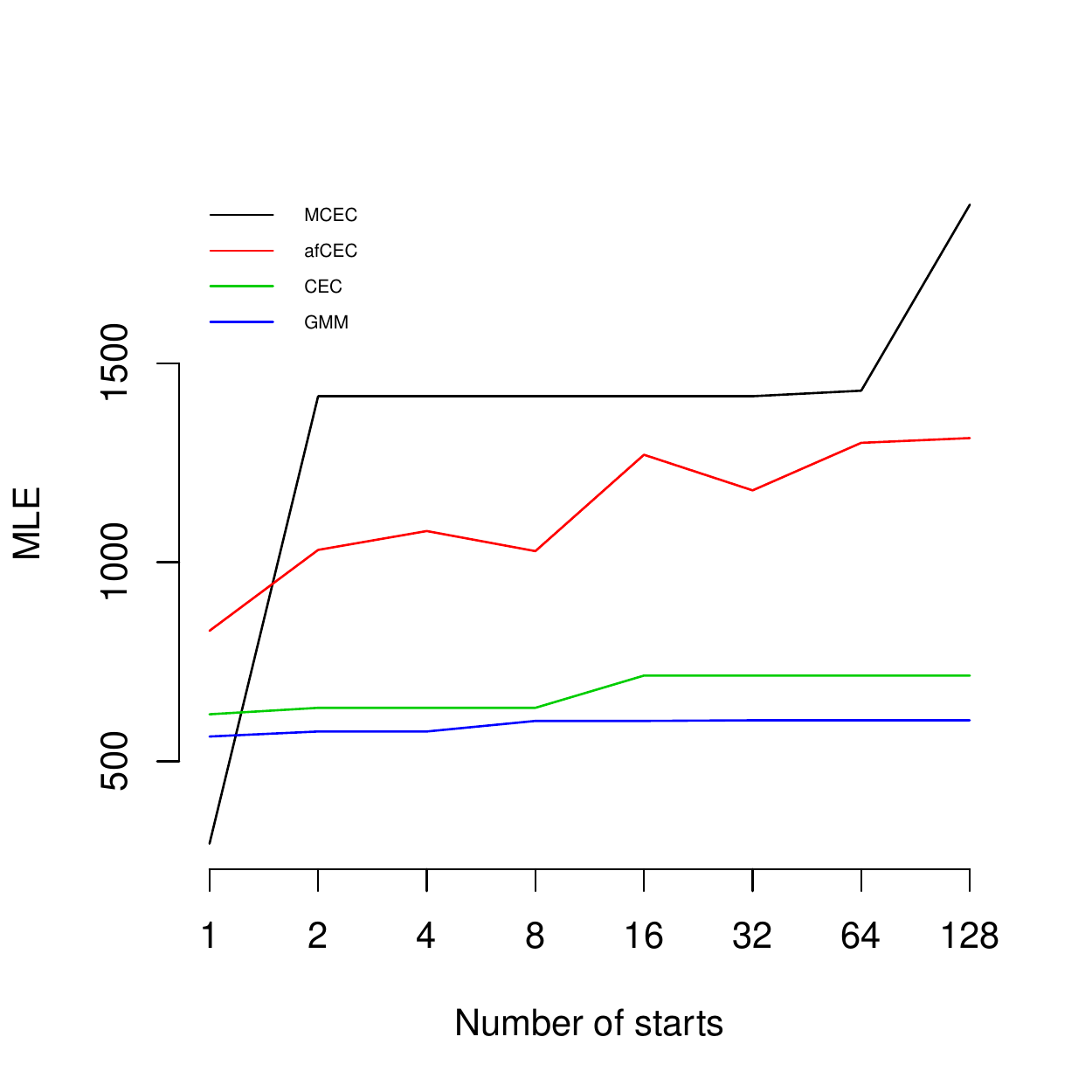} &
    		\includegraphics[width=0.24\textwidth]{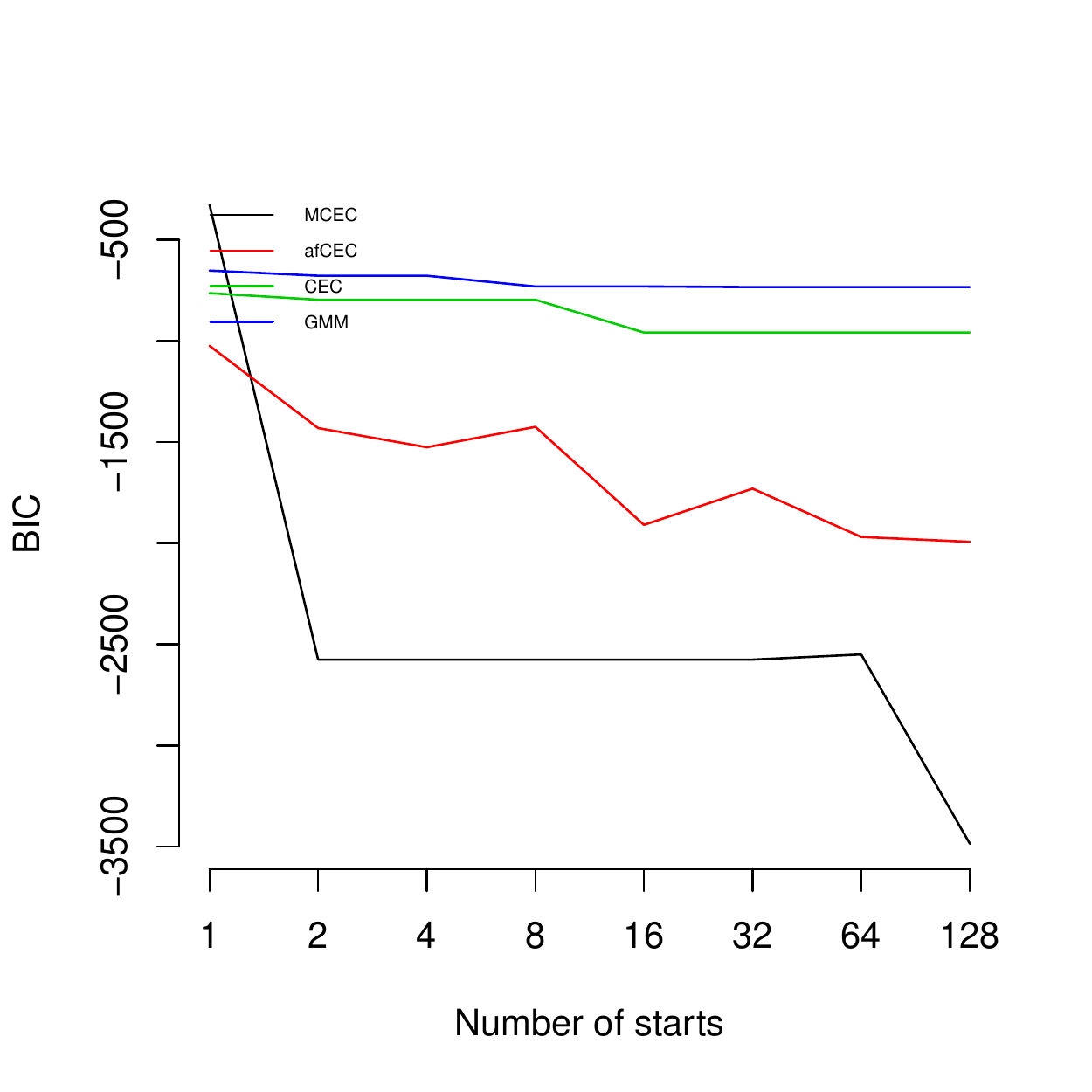} &
    		\includegraphics[width=0.24\textwidth]{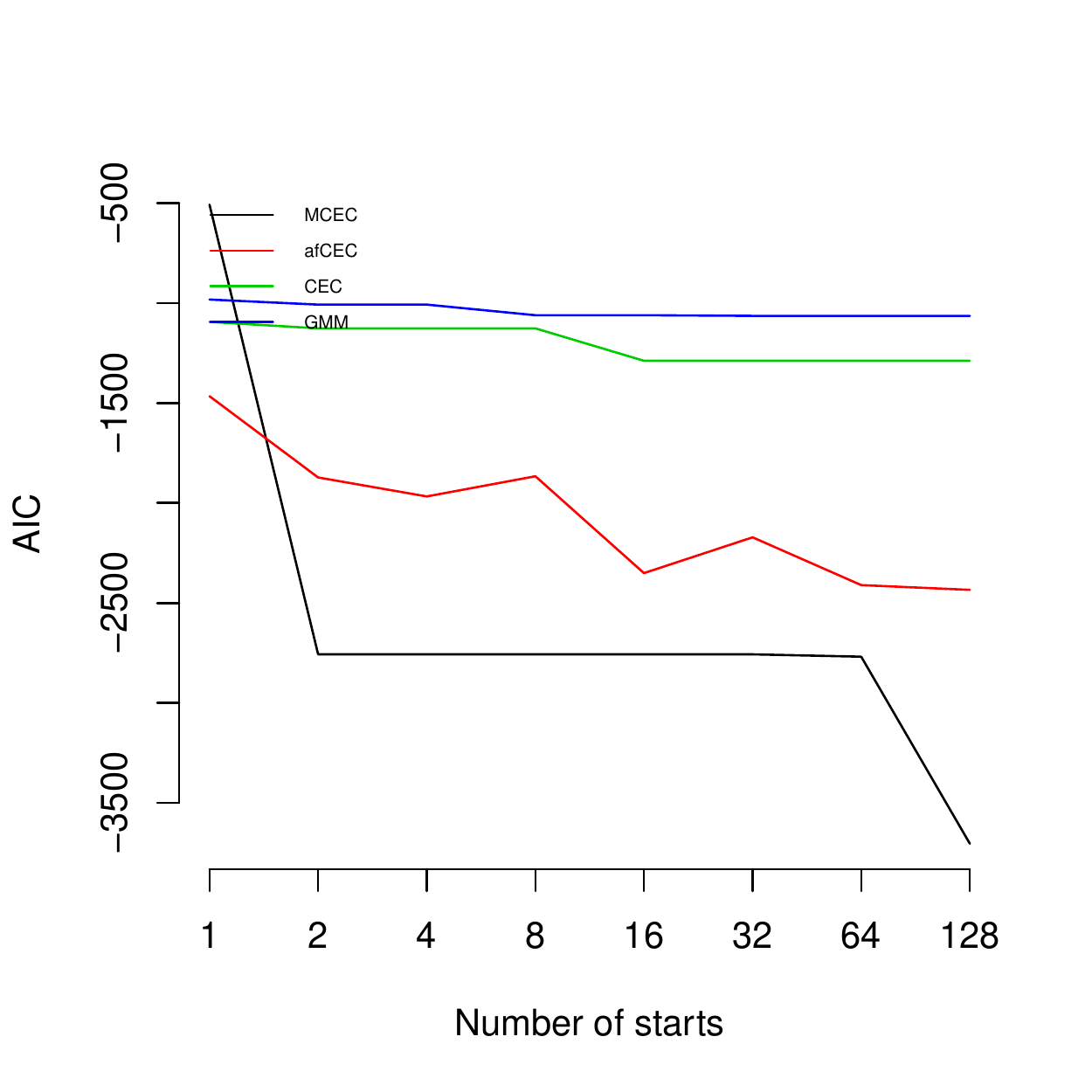} &
    		\includegraphics[width=0.24\textwidth]{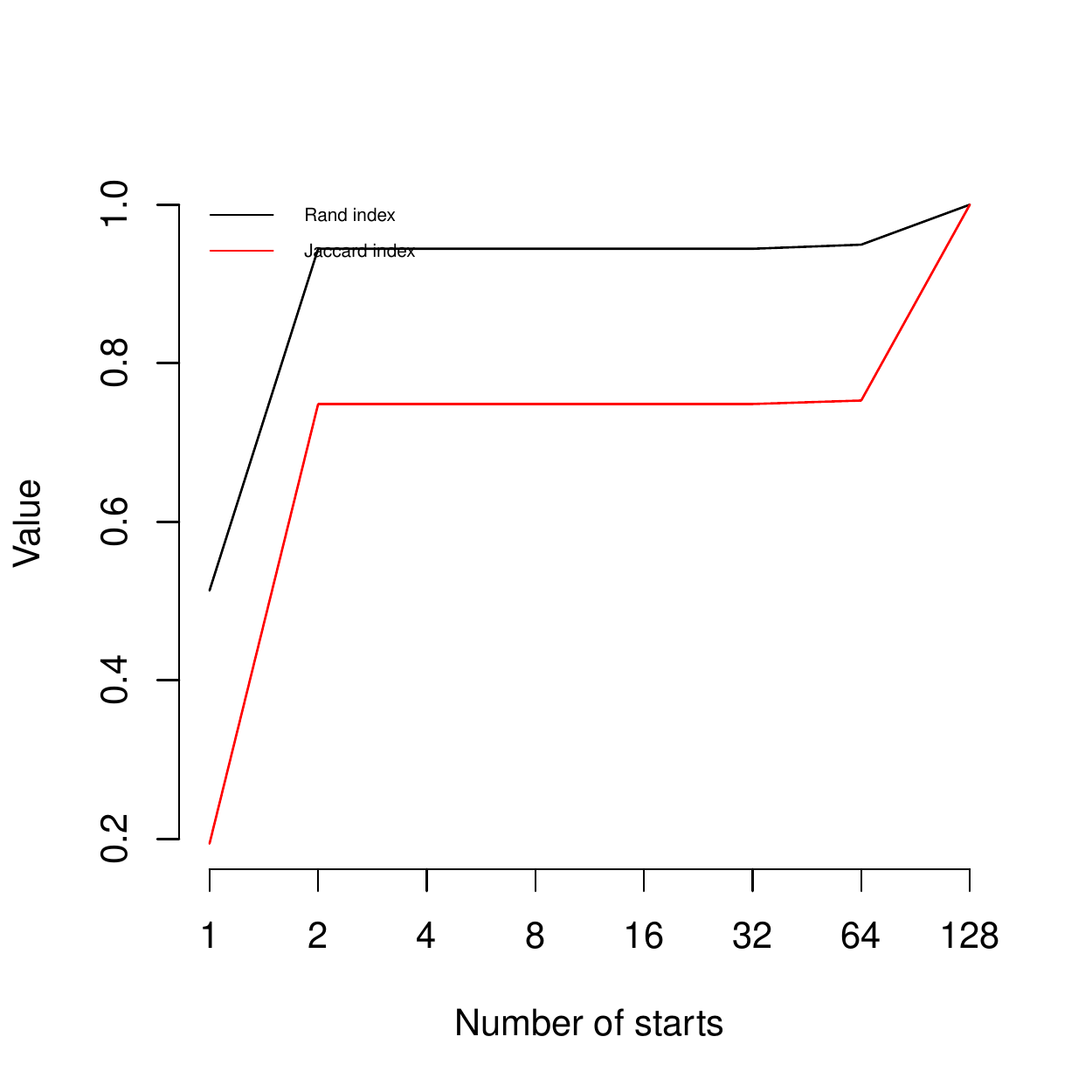} \\
    		\includegraphics[width=0.24\textwidth]{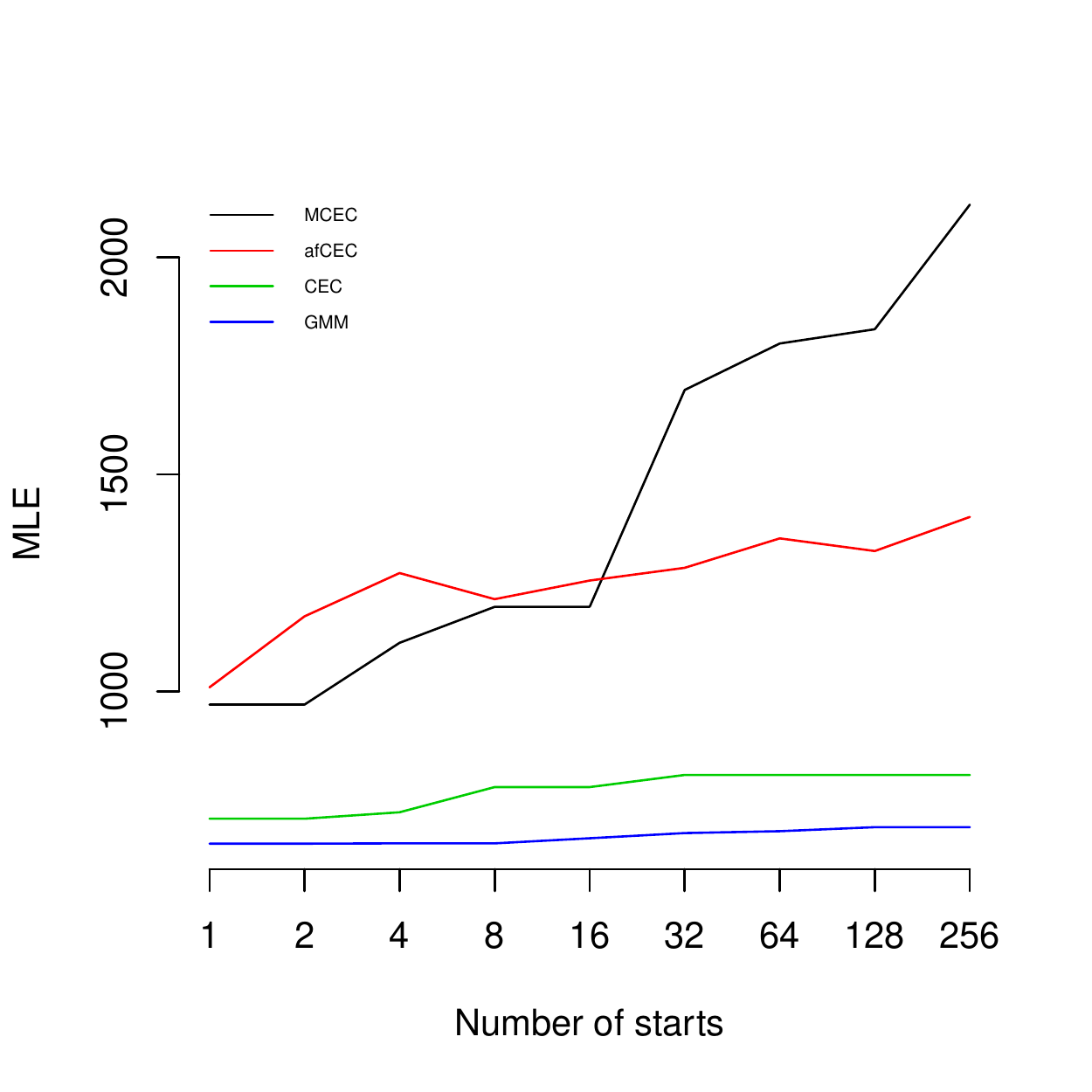} &
    		\includegraphics[width=0.24\textwidth]{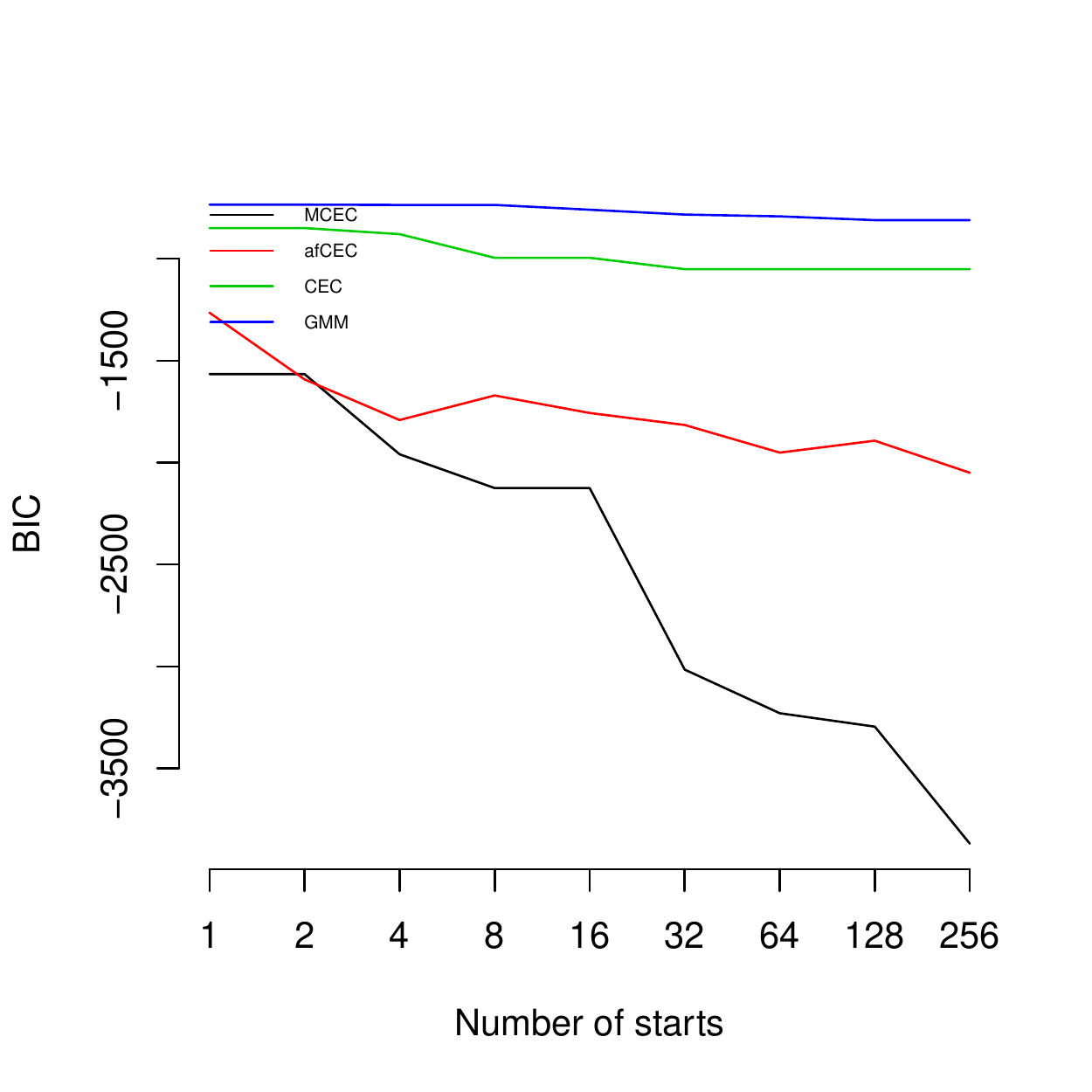} &
    		\includegraphics[width=0.24\textwidth]{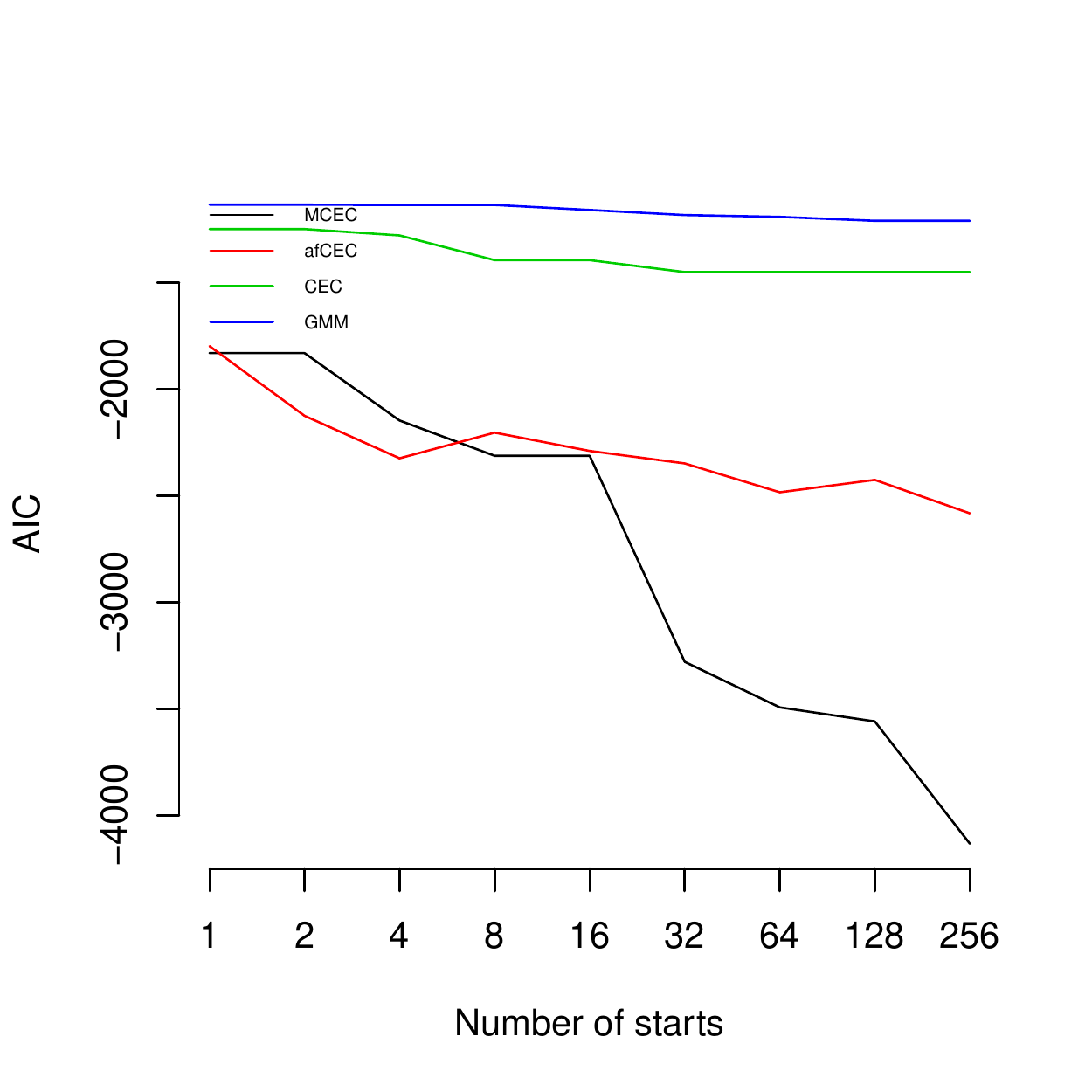} &
    		\includegraphics[width=0.24\textwidth]{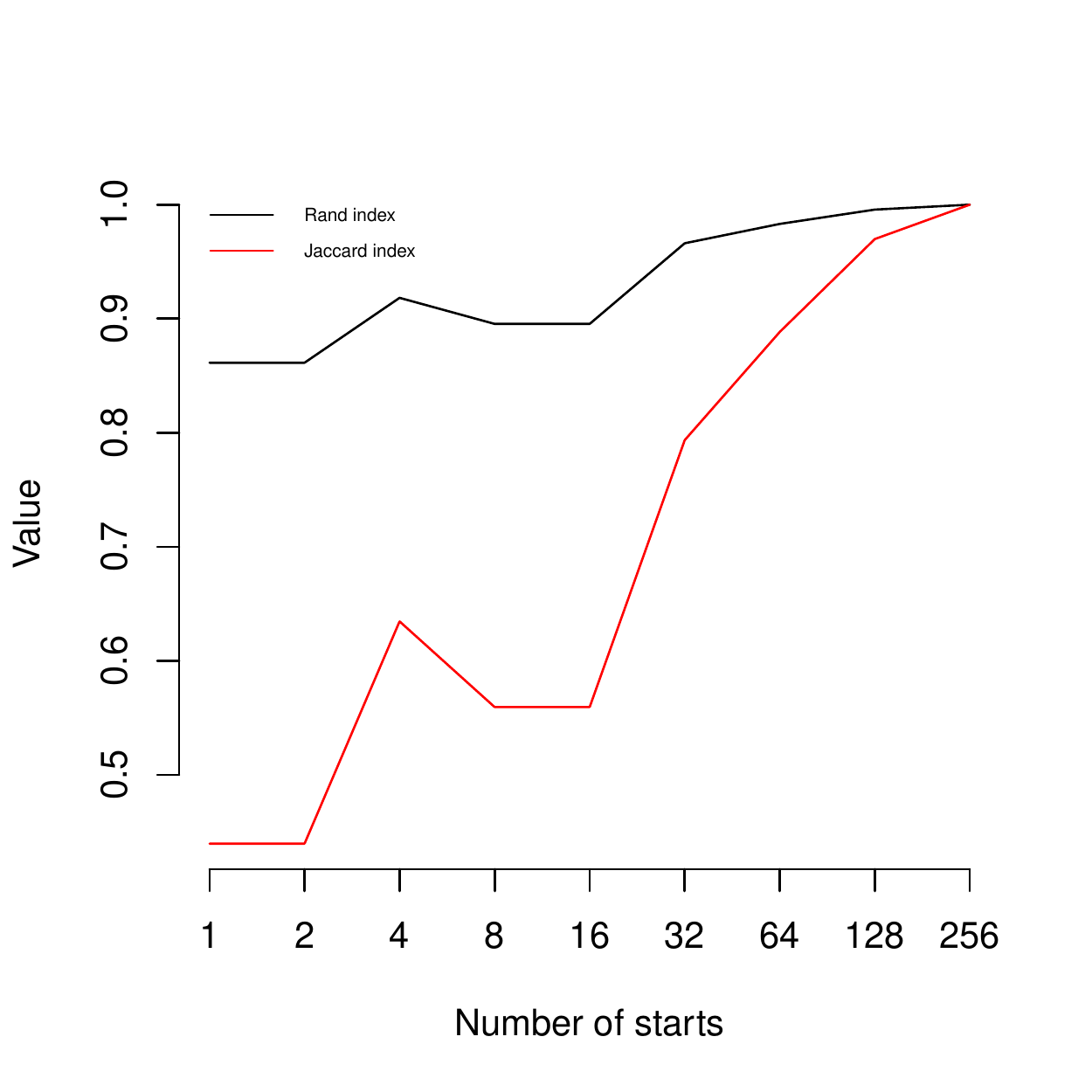} \\
    		\includegraphics[width=0.24\textwidth]{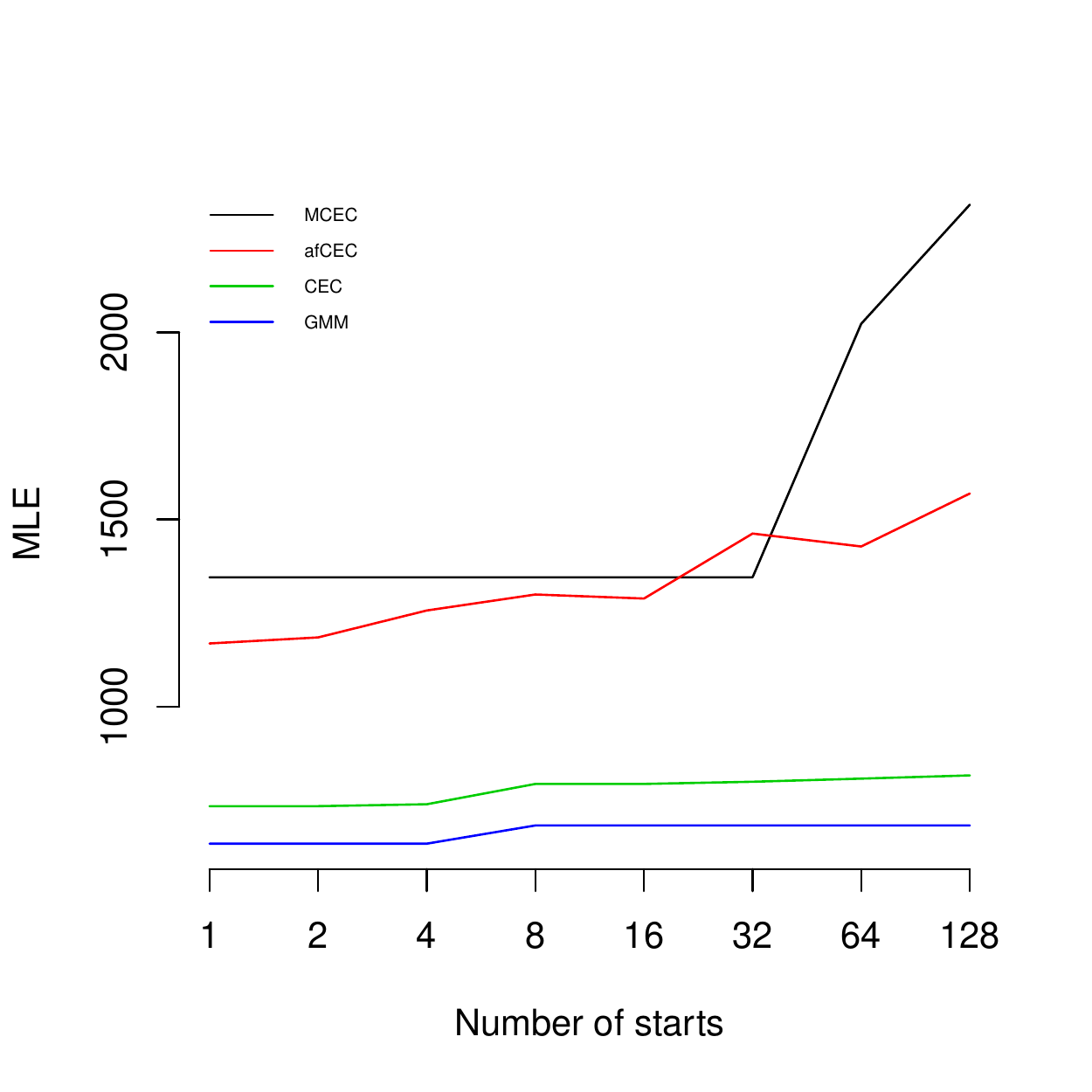} &
    		\includegraphics[width=0.24\textwidth]{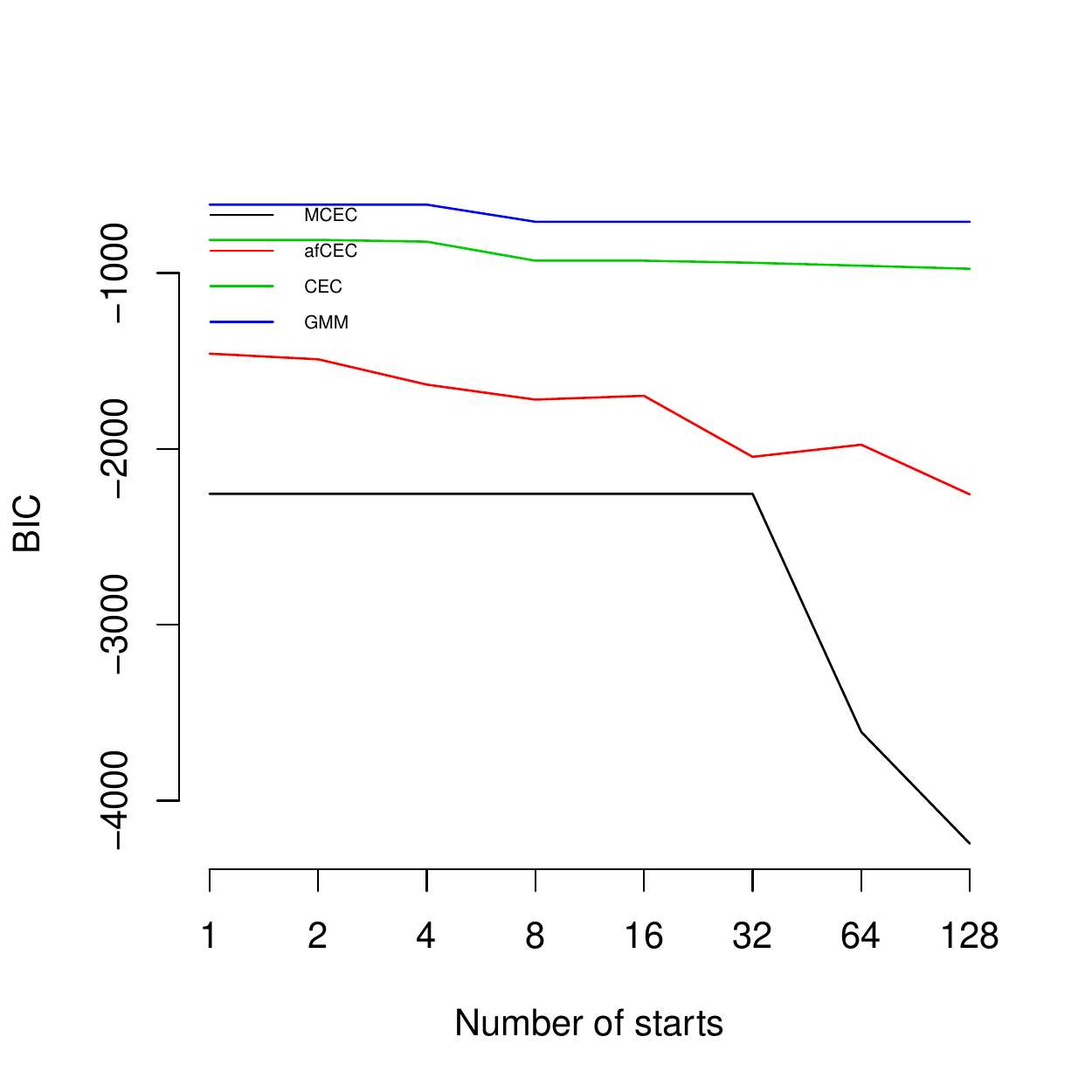} &
    		\includegraphics[width=0.24\textwidth]{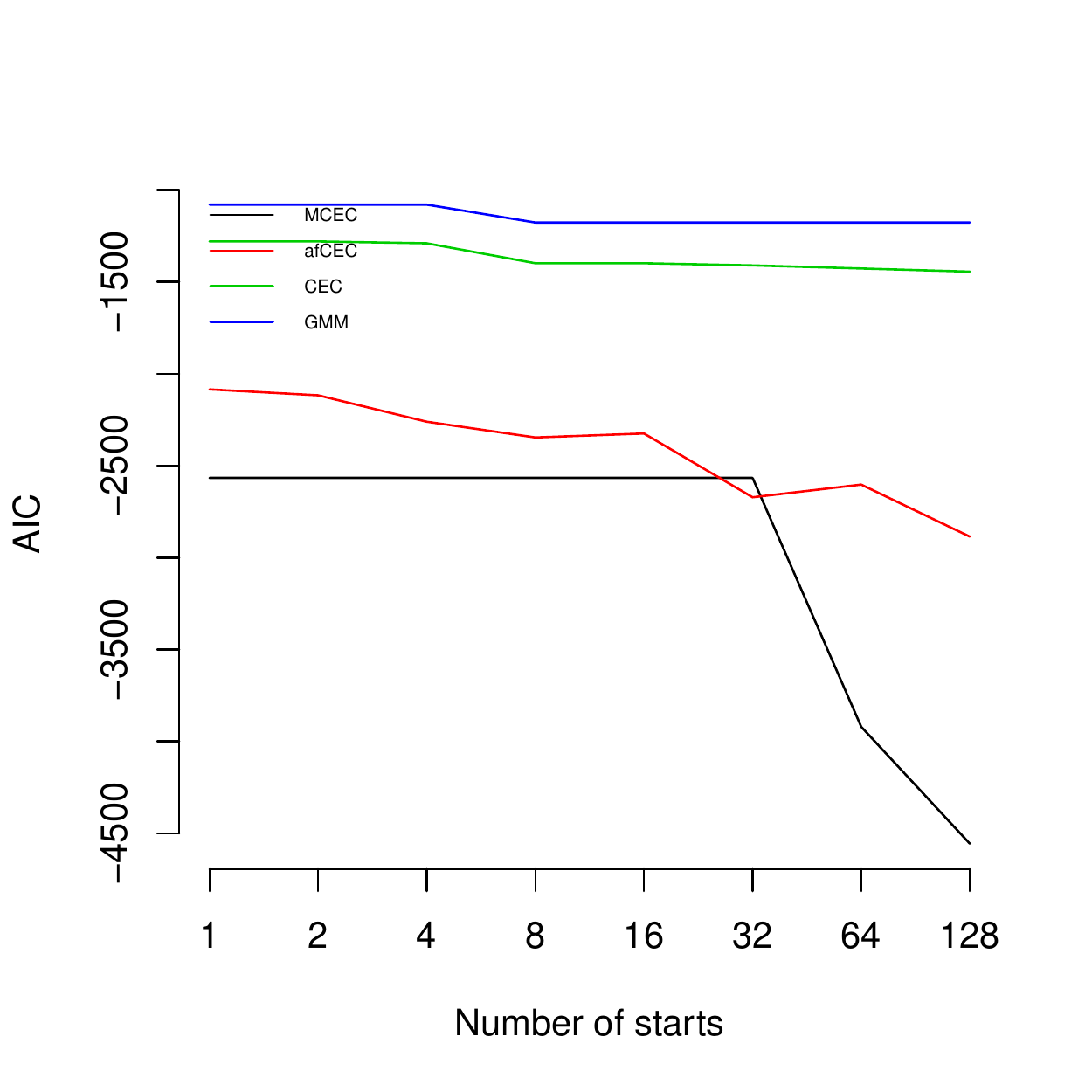} &
    		\includegraphics[width=0.24\textwidth]{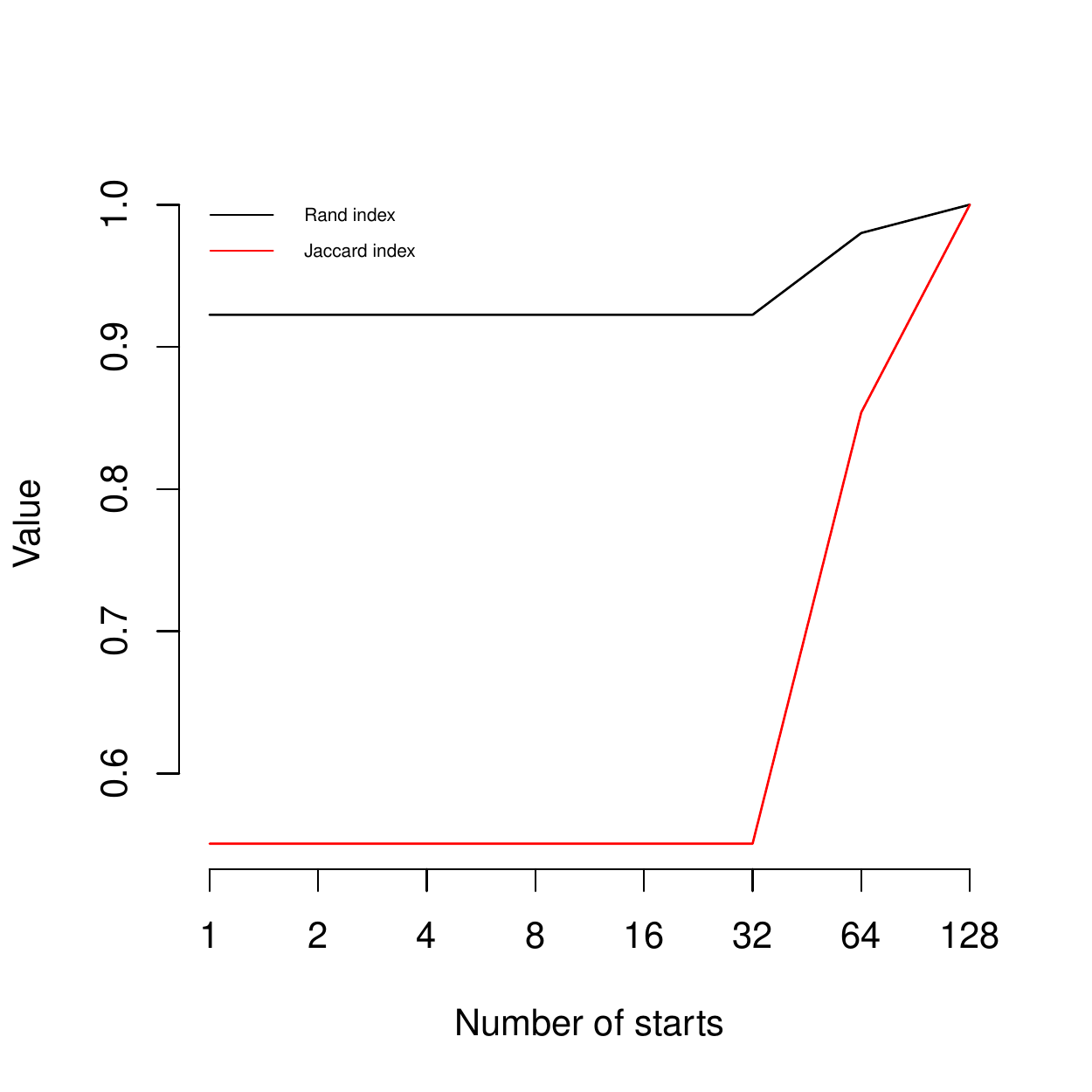} \\
    	\end{tabular}
	\caption{Experiment No. 1: Values of the performance metrics computed for the consecutive numbers of starts being the powers of two and the increasing number of clusters in the data set - Part 2. The charts in the first, second, and third columns show the computed values of MLE, BIC, and AIC, respectively, for the consecutive number of starts being the powers of two for all four clustering methods used in the comparison, ie. MCEC, afCEC, CEC, and GMM, while the last one represents the values of Rand and Jaccard indices computed similarly for the clustering obtained using the MCEC algorithm. The table's first, second, third, and fourth row represents the test case with five, six, seven, and eight clusters. }
\end{figure}


\begin{table}[H]
\caption{Experiment No. 2: Values of the performance metrics computed for the consecutive numbers of starts being the powers of two - Part 1}
\centering
\setlength\tabcolsep{1.5pt}
\resizebox{\columnwidth}{!}{%
\begin{tabular}{c|>{\tiny}c|>{\tiny}c>{\tiny}c>{\tiny}c>{\tiny}c>{\tiny}c|>{\tiny}c>{\tiny}c>{\tiny}c|>{\tiny}c>{\tiny}c>{\tiny}c|>{\tiny}c>{\tiny}c>{\tiny}c|}
\cline{2-16}
& \rule{0pt}{1.5cm} & \multicolumn{5}{c|}{\makecell{MCEC }} & \multicolumn{3}{c|}{\makecell{afCEC Hartigan}} & \multicolumn{3}{c|}{CEC} & \multicolumn{3}{c|}{GMM} \\
\cline{3-16}

& \rotatebox[origin=c]{90}{\rlap{\parbox{1.5cm}{\makecell{Number of \\ starts}}}} &
\multicolumn{1}{c|}{MLE} & \multicolumn{1}{c|}{BIC} & \multicolumn{1}{c|}{AIC} & \multicolumn{1}{c|}{\makecell{Rand \\ index}} & \multicolumn{1}{c|}{\makecell{Jaccard \\ index}} &
\multicolumn{1}{c|}{MLE} & \multicolumn{1}{c|}{BIC} & \multicolumn{1}{c|}{AIC} &
\multicolumn{1}{c|}{MLE} & \multicolumn{1}{c|}{BIC} & \multicolumn{1}{c|}{AIC} &
\multicolumn{1}{c|}{MLE} & \multicolumn{1}{c|}{BIC} & \multicolumn{1}{c|}{AIC} \\

\hline
\hline
\multicolumn{1}{|c|}{\multirow{5}{*}{\includegraphics[scale=0.07]{5-2-1.pdf}}} & 1 & \textbf{542.50} & \textbf{-1031.64} & \textbf{-1063.01} & 1.00 & 1.00 & 216.90 & -361.02 & -403.80 & 195.71 & -338.05 & -369.42 & 185.35 & -317.33 & -348.70 \\ 
\multicolumn{1}{|c|}{} & 2 & \textbf{542.50} & \textbf{-1031.64} & \textbf{-1063.01} & 1.00 & 1.00 & 229.02 & -385.26 & -428.04 & 195.71 & -338.05 & -369.42 & 185.35 & -317.33 & -348.70 \\ 
\multicolumn{1}{|c|}{} & 4 & \textbf{542.49} & \textbf{-1031.62} & \textbf{-1062.99} & 1.00 & 1.00 & 221.10 & -369.43 & -412.21 & 195.71 & -338.05 & -369.42 & 185.35 & -317.33 & -348.70 \\ 
\multicolumn{1}{|c|}{} & 8 & \textbf{542.49} & \textbf{-1031.62} & \textbf{-1062.99} & 1.00 & 1.00 & 229.02 & -385.26 & -428.04 & 195.71 & -338.05 & -369.42 & 185.35 & -317.33 & -348.70 \\ 
\multicolumn{1}{|c|}{} & 16 & \textbf{542.49} & \textbf{-1031.62} & \textbf{-1062.99} & 1.00 & 1.00 & 229.02 & -385.26 & -428.04 & 195.71 & -338.05 & -369.42 & 185.35 & -317.33 & -348.70 \\ 
\hline
\hline
\multicolumn{1}{|c|}{\multirow{9}{*}{\includegraphics[scale=0.07]{5-2-2.pdf}}} & 1 & \textbf{291.52} & \textbf{-455.50} & \textbf{-537.04} & 0.76 & 0.63 & 256.15 & -340.41 & -450.31 & 250.83 & -374.12 & -455.66 & 249.84 & -372.14 & -453.68 \\ 
\multicolumn{1}{|c|}{} & 2 & 291.52 & \textbf{-455.50} & \textbf{-537.04} & 0.76 & 0.63 & \textbf{291.91} & -411.92 & -521.82 & 250.83 & -374.12 & -455.66 & 249.84 & -372.14 & -453.68 \\ 
\multicolumn{1}{|c|}{} & 4 & \textbf{329.32} & \textbf{-531.10} & \textbf{-612.64} & 0.68 & 0.51 & 327.78 & -483.65 & -593.56 & 278.65 & -429.76 & -511.29 & 249.84 & -372.14 & -453.68 \\ 
\multicolumn{1}{|c|}{} & 8 & 329.32 & \textbf{-531.10} & -612.64 & 0.68 & 0.51 & \textbf{347.65} & -523.40 & \textbf{-633.30} & 278.65 & -429.76 & -511.29 & 263.39 & -399.24 & -480.78 \\ 
\multicolumn{1}{|c|}{} & 16 & 329.32 & \textbf{-531.10} & -612.64 & 0.68 & 0.51 & \textbf{338.19} & -504.47 & \textbf{-614.37} & 278.65 & -429.76 & -511.29 & 263.39 & -399.24 & -480.78 \\ 
\multicolumn{1}{|c|}{} & 32 & \textbf{365.79} & \textbf{-604.03} & \textbf{-685.57} & 0.58 & 0.47 & 347.88 & -523.86 & -633.76 & 288.05 & -448.57 & -530.11 & 263.39 & -399.24 & -480.78 \\ 
\multicolumn{1}{|c|}{} & 64 & \textbf{380.02} & \textbf{-632.50} & \textbf{-714.04} & 0.59 & 0.41 & 374.32 & -576.74 & -686.64 & 295.75 & -463.96 & -545.50 & 263.39 & -399.24 & -480.78 \\ 
\multicolumn{1}{|c|}{} & 128 & 254.33 & -381.11 & -462.65 & 0.68 & 0.54 & \textbf{374.86} & \textbf{-577.82} & \textbf{-687.72} & 307.84 & -488.14 & -569.68 & 263.39 & -399.24 & -480.78 \\ 
\multicolumn{1}{|c|}{} & 256 & \textbf{955.79} & \textbf{-1784.03} & \textbf{-1865.57} & 1.00 & 1.00 & 381.74 & -591.58 & -701.48 & 307.84 & -488.14 & -569.68 & 268.87 & -410.20 & -491.74 \\ 
\hline
\hline
\multicolumn{1}{|c|}{\multirow{17}{*}{\includegraphics[scale=0.07]{5-2-3.pdf}}} & 1 & \textbf{308.89} & \textbf{-409.51} & \textbf{-547.78} & 0.70 & 0.40 & 270.37 & -261.07 & -446.75 & 204.44 & -200.60 & -338.87 & 219.98 & -231.69 & -369.97 \\ 
\multicolumn{1}{|c|}{} & 2 & \textbf{308.89} & \textbf{-409.51} & \textbf{-547.78} & 0.70 & 0.40 & 291.48 & -303.28 & -488.96 & 204.44 & -200.60 & -338.87 & 233.11 & -257.95 & -396.23 \\ 
\multicolumn{1}{|c|}{} & 4 & \textbf{335.12} & \textbf{-461.97} & \textbf{-600.25} & 0.74 & 0.44 & 291.48 & -303.28 & -488.96 & 201.25 & -194.24 & -332.51 & 233.11 & -257.95 & -396.23 \\ 
\multicolumn{1}{|c|}{} & 8 & \textbf{335.12} & \textbf{-461.97} & \textbf{-600.25} & 0.74 & 0.44 & 321.50 & -363.32 & -549.00 & 218.26 & -228.24 & -366.52 & 233.11 & -257.95 & -396.23 \\ 
\multicolumn{1}{|c|}{} & 16 & 335.12 & -461.97 & -600.25 & 0.74 & 0.44 & \textbf{374.20} & \textbf{-468.73} & \textbf{-654.41} & 253.84 & -299.40 & -437.68 & 233.11 & -257.95 & -396.23 \\ 
\multicolumn{1}{|c|}{} & 32 & 335.12 & -461.97 & -600.25 & 0.74 & 0.44 & \textbf{374.20} & \textbf{-468.73} & \textbf{-654.41} & 253.84 & -299.40 & -437.68 & 233.11 & -257.95 & -396.23 \\ 
\multicolumn{1}{|c|}{} & 64 & \textbf{602.37} & \textbf{-996.47} & \textbf{-1134.74} & 0.91 & 0.75 & 374.20 & -468.73 & -654.41 & 253.84 & -299.40 & -437.68 & 233.11 & -257.95 & -396.23 \\ 
\multicolumn{1}{|c|}{} & 128 & \textbf{602.37} & \textbf{-996.47} & \textbf{-1134.74} & 0.91 & 0.75 & 377.32 & -474.96 & -660.64 & 253.84 & -299.40 & -437.68 & 233.11 & -257.95 & -396.23 \\ 
\multicolumn{1}{|c|}{} & 256 & \textbf{674.43} & \textbf{-1140.58} & \textbf{-1278.85} & 0.83 & 0.60 & 405.84 & -531.99 & -717.67 & 262.66 & -317.04 & -455.32 & 233.11 & -257.95 & -396.23 \\ 
\multicolumn{1}{|c|}{} & 512 & \textbf{674.43} & \textbf{-1140.58} & \textbf{-1278.85} & 0.83 & 0.60 & 405.84 & -531.99 & -717.67 & 262.66 & -317.04 & -455.32 & 233.12 & -257.96 & -396.23 \\ 
\multicolumn{1}{|c|}{} & 1024 & \textbf{674.43} & \textbf{-1140.58} & \textbf{-1278.85} & 0.83 & 0.60 & 405.84 & -531.99 & -717.67 & 262.66 & -317.04 & -455.32 & 233.12 & -257.96 & -396.23 \\ 
\multicolumn{1}{|c|}{} & 2048 & \textbf{674.43} & \textbf{-1140.58} & \textbf{-1278.85} & 0.83 & 0.60 & 419.20 & -558.72 & -744.40 & 257.16 & -306.04 & -444.31 & 233.12 & -257.96 & -396.23 \\ 
\multicolumn{1}{|c|}{} & 4096 & \textbf{769.91} & \textbf{-1331.56} & \textbf{-1469.83} & 0.88 & 0.70 & 426.38 & -573.07 & -758.75 & 261.29 & -314.32 & -452.59 & 233.12 & -257.96 & -396.23 \\ 
\multicolumn{1}{|c|}{} & 8192 & \textbf{935.62} & \textbf{-1662.96} & \textbf{-1801.24} & 1.00 & 1.00 & 426.38 & -573.07 & -758.75 & 261.29 & -314.32 & -452.59 & 233.12 & -257.96 & -396.23 \\ 
\multicolumn{1}{|c|}{} & 16384 & \textbf{935.62} & \textbf{-1662.96} & \textbf{-1801.24} & 1.00 & 1.00 & 441.16 & -602.64 & -788.32 & 261.29 & -314.32 & -452.59 & 233.12 & -257.96 & -396.23 \\ 
\multicolumn{1}{|c|}{} & 32768 & \textbf{1004.11} & \textbf{-1799.95} & \textbf{-1938.22} & 1.00 & 1.00 & 451.79 & -623.90 & -809.59 & 261.29 & -314.32 & -452.59 & 233.12 & -257.96 & -396.23 \\ 
\multicolumn{1}{|c|}{} & 65536 & \textbf{1261.75} & \textbf{-2315.23} & \textbf{-2453.50} & 1.00 & 1.00 & 451.79 & -623.90 & -809.59 & 261.29 & -314.32 & -452.59 & 233.12 & -257.96 & -396.23 \\ 
\hline
\hline

\end{tabular}
}
\end{table}


\begin{table}[H]
\caption{Experiment No. 2: Values of the performance metrics computed for the consecutive numbers of starts being the powers of two - Part 2}
\centering
\setlength\tabcolsep{1.5pt}
\resizebox{\columnwidth}{!}{%
\begin{tabular}{c|>{\tiny}c|>{\tiny}c>{\tiny}c>{\tiny}c>{\tiny}c>{\tiny}c|>{\tiny}c>{\tiny}c>{\tiny}c|>{\tiny}c>{\tiny}c>{\tiny}c|>{\tiny}c>{\tiny}c>{\tiny}c|}
\cline{2-16}
& \rule{0pt}{1.5cm} & \multicolumn{5}{c|}{\makecell{MCEC}} & \multicolumn{3}{c|}{\makecell{afCEC Hartigan}} & \multicolumn{3}{c|}{CEC} & \multicolumn{3}{c|}{GMM} \\
\cline{3-16}

& \rotatebox[origin=c]{90}{\rlap{\parbox{1.5cm}{\makecell{Number of \\ starts}}}} &
\multicolumn{1}{c|}{MLE} & \multicolumn{1}{c|}{BIC} & \multicolumn{1}{c|}{AIC} & \multicolumn{1}{c|}{\makecell{Rand \\ index}} & \multicolumn{1}{c|}{\makecell{Jaccard \\ index}} &
\multicolumn{1}{c|}{MLE} & \multicolumn{1}{c|}{BIC} & \multicolumn{1}{c|}{AIC} &
\multicolumn{1}{c|}{MLE} & \multicolumn{1}{c|}{BIC} & \multicolumn{1}{c|}{AIC} &
\multicolumn{1}{c|}{MLE} & \multicolumn{1}{c|}{BIC} & \multicolumn{1}{c|}{AIC} \\

\hline
\hline
\multicolumn{1}{|c|}{\multirow{22}{*}{\includegraphics[scale=0.07]{5-2-4.pdf}}} & 1 & 160.31 & -27.42 & -226.62 & 0.74 & 0.32 & 166.87 & 59.28 & -207.74 & \textbf{176.57} & \textbf{-59.94} & \textbf{-259.14} & 90.75 & 111.69 & -87.51 \\ 
\multicolumn{1}{|c|}{} & 2 & -13775.74 & 27844.69 & 27645.49 & 0.75 & 0.40 & \textbf{200.16} & -7.30 & \textbf{-274.32} & 176.57 & \textbf{-59.94} & -259.14 & 124.37 & 44.47 & -154.73 \\ 
\multicolumn{1}{|c|}{} & 4 & 185.27 & \textbf{-77.34} & -276.54 & 0.67 & 0.31 & \textbf{226.52} & -60.02 & \textbf{-327.04} & 176.57 & -59.94 & -259.14 & 151.37 & -9.55 & -208.75 \\ 
\multicolumn{1}{|c|}{} & 8 & 185.27 & -77.34 & -276.54 & 0.67 & 0.31 & \textbf{272.01} & \textbf{-151.01} & \textbf{-418.02} & 176.57 & -59.94 & -259.14 & 151.37 & -9.55 & -208.75 \\ 
\multicolumn{1}{|c|}{} & 16 & 185.27 & -77.34 & -276.54 & 0.67 & 0.31 & \textbf{283.57} & \textbf{-174.13} & \textbf{-441.15} & 185.78 & -78.36 & -277.56 & 151.37 & -9.55 & -208.75 \\ 
\multicolumn{1}{|c|}{} & 32 & $\approx -1.23 \cdot 10^{14}$ & $\approx 2.46 \cdot 10^{14}$ & $\approx 2.46 \cdot 10^{14}$ & 0.79 & 0.43 & \textbf{283.57} & \textbf{-174.13} & \textbf{-441.15} & 192.41 & -91.62 & -290.82 & 151.37 & -9.55 & -208.75 \\ 
\multicolumn{1}{|c|}{} & 64 & \textbf{767.41} & \textbf{-1241.62} & \textbf{-1440.83} & 0.85 & 0.57 & 356.76 & -320.50 & -587.51 & 195.09 & -96.98 & -296.18 & 151.37 & -9.55 & -208.75 \\ 
\multicolumn{1}{|c|}{} & 128 & \textbf{767.41} & \textbf{-1241.62} & \textbf{-1440.83} & 0.85 & 0.57 & 356.76 & -320.50 & -587.51 & 199.42 & -105.65 & -304.85 & 151.37 & -9.55 & -208.75 \\ 
\multicolumn{1}{|c|}{} & 256 & \textbf{767.41} & \textbf{-1241.62} & \textbf{-1440.83} & 0.85 & 0.57 & 356.76 & -320.50 & -587.51 & 199.42 & -105.65 & -304.85 & 151.37 & -9.55 & -208.75 \\ 
\multicolumn{1}{|c|}{} & 512 & \textbf{767.41} & \textbf{-1241.62} & \textbf{-1440.83} & 0.85 & 0.57 & 382.36 & -371.70 & -638.71 & 245.16 & -197.12 & -396.32 & 151.37 & -9.55 & -208.75 \\ 
\multicolumn{1}{|c|}{} & 1024 & \textbf{767.41} & \textbf{-1241.62} & \textbf{-1440.83} & 0.85 & 0.57 & 382.36 & -371.70 & -638.71 & 245.16 & -197.12 & -396.32 & 151.37 & -9.55 & -208.75 \\ 
\multicolumn{1}{|c|}{} & 2048 & \textbf{767.41} & \textbf{-1241.62} & \textbf{-1440.83} & 0.85 & 0.57 & 397.35 & -401.70 & -668.71 & 245.16 & -197.12 & -396.32 & 151.37 & -9.55 & -208.75 \\ 
\multicolumn{1}{|c|}{} & 4096 & \textbf{766.93} & \textbf{-1240.65} & \textbf{-1439.85} & 0.94 & 0.77 & 397.35 & -401.70 & -668.71 & 245.16 & -197.12 & -396.32 & 152.91 & -12.63 & -211.83 \\ 
\multicolumn{1}{|c|}{} & 8192 & \textbf{766.93} & \textbf{-1240.65} & \textbf{-1439.85} & 0.94 & 0.77 & 397.35 & -401.70 & -668.71 & 245.16 & -197.12 & -396.32 & 152.91 & -12.63 & -211.83 \\ 
\multicolumn{1}{|c|}{} & 16384 & \textbf{766.93} & \textbf{-1240.65} & \textbf{-1439.85} & 0.94 & 0.77 & 422.34 & -451.67 & -718.68 & 245.16 & -197.12 & -396.32 & 164.01 & -34.82 & -234.02 \\ 
\multicolumn{1}{|c|}{} & 32768 & \textbf{927.79} & \textbf{-1562.37} & \textbf{-1761.57} & 0.91 & 0.70 & 422.34 & -451.67 & -718.68 & 245.16 & -197.12 & -396.32 & 164.01 & -34.82 & -234.02 \\ 
\multicolumn{1}{|c|}{} & 65536 & \textbf{949.68} & \textbf{-1606.16} & \textbf{-1805.37} & 0.90 & 0.68 & 423.48 & -453.95 & -720.97 & 245.16 & -197.12 & -396.32 & 164.01 & -34.82 & -234.02 \\ 
\multicolumn{1}{|c|}{} & 131072 & \textbf{1012.16} & \textbf{-1731.12} & \textbf{-1930.32} & 0.91 & 0.71 & 423.49 & -453.96 & -720.97 & 245.16 & -197.12 & -396.32 & 164.01 & -34.82 & -234.02 \\ 
\multicolumn{1}{|c|}{} & 262144 & \textbf{1012.16} & \textbf{-1731.12} & \textbf{-1930.32} & 0.91 & 0.71 & 423.49 & -453.96 & -720.97 & 245.16 & -197.12 & -396.32 & 164.01 & -34.82 & -234.02 \\ 
\multicolumn{1}{|c|}{} & 524288 & \textbf{1012.16} & \textbf{-1731.12} & \textbf{-1930.32} & 0.91 & 0.71 & 439.58 & -486.14 & -753.15 & 245.16 & -197.12 & -396.32 & 164.01 & -34.82 & -234.02 \\ 
\multicolumn{1}{|c|}{} & 1048576 & \textbf{1293.90} & \textbf{-2294.60} & \textbf{-2493.80} & 1.00 & 1.00 & 432.14 & -471.27 & -738.28 & 245.16 & -197.12 & -396.32 & 164.01 & -34.82 & -234.02 \\ 
\multicolumn{1}{|c|}{} & 2097152 & \textbf{1551.54} & \textbf{-2809.89} & \textbf{-3009.09} & 1.00 & 1.00 & 432.14 & -471.27 & -738.28 & 245.16 & -197.12 & -396.32 & 164.01 & -34.82 & -234.02 \\ 
\hline
\hline

\end{tabular}
}
\end{table}

\begin{figure}[H]
	\centering
	    \begin{tabular}{c@{}c@{}c@{}c@{}}
    		\includegraphics[width=0.24\textwidth]{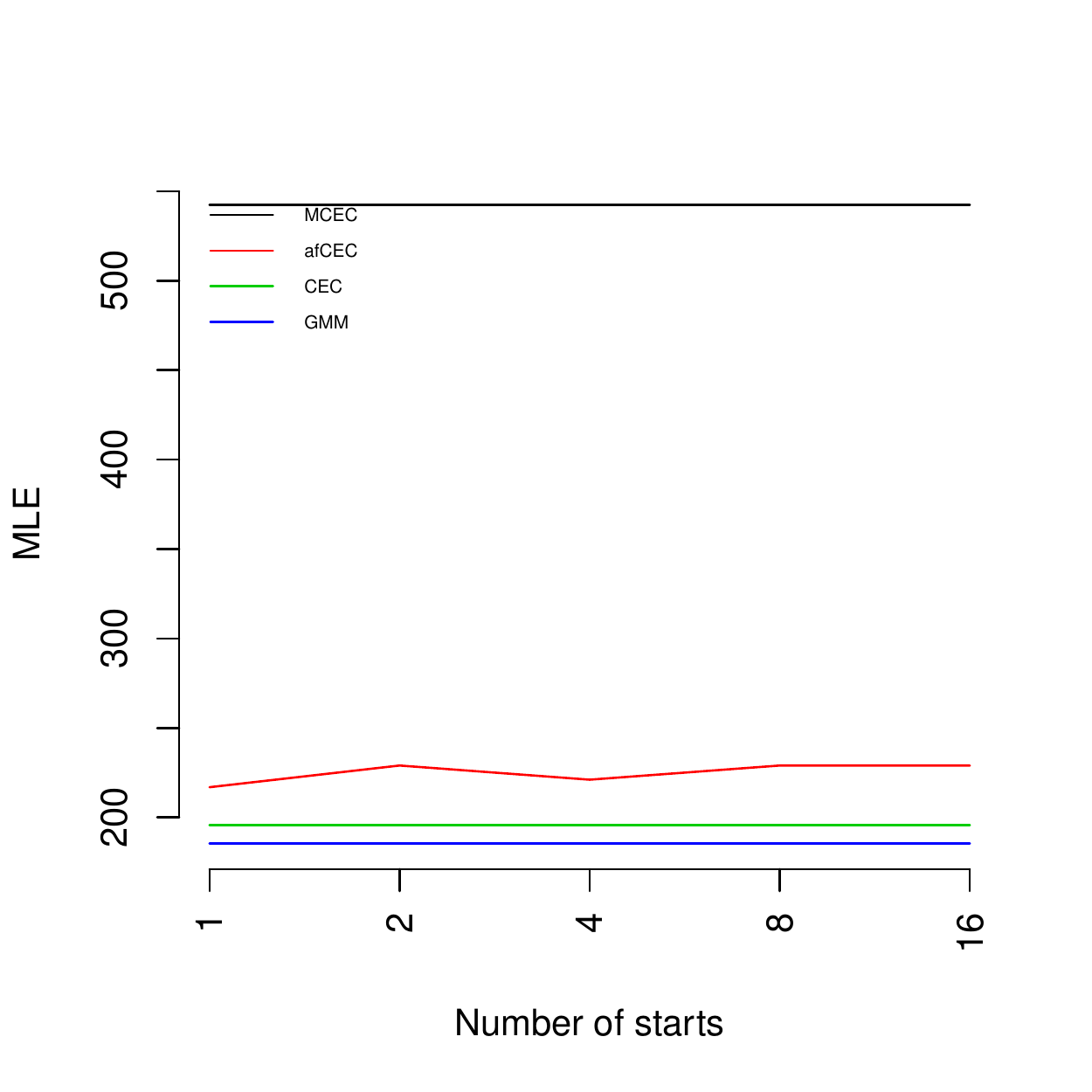} &
    		\includegraphics[width=0.24\textwidth]{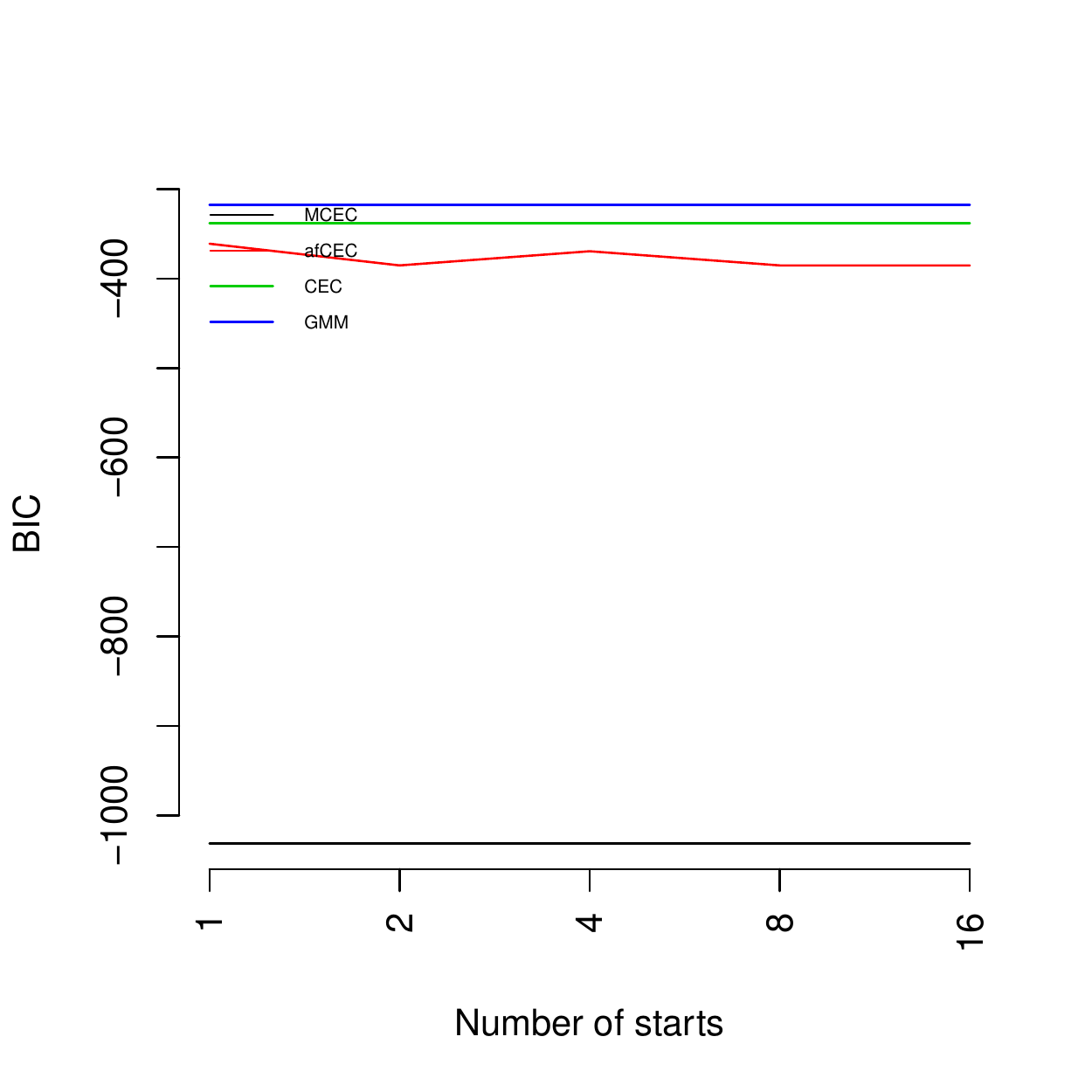} &
    		\includegraphics[width=0.24\textwidth]{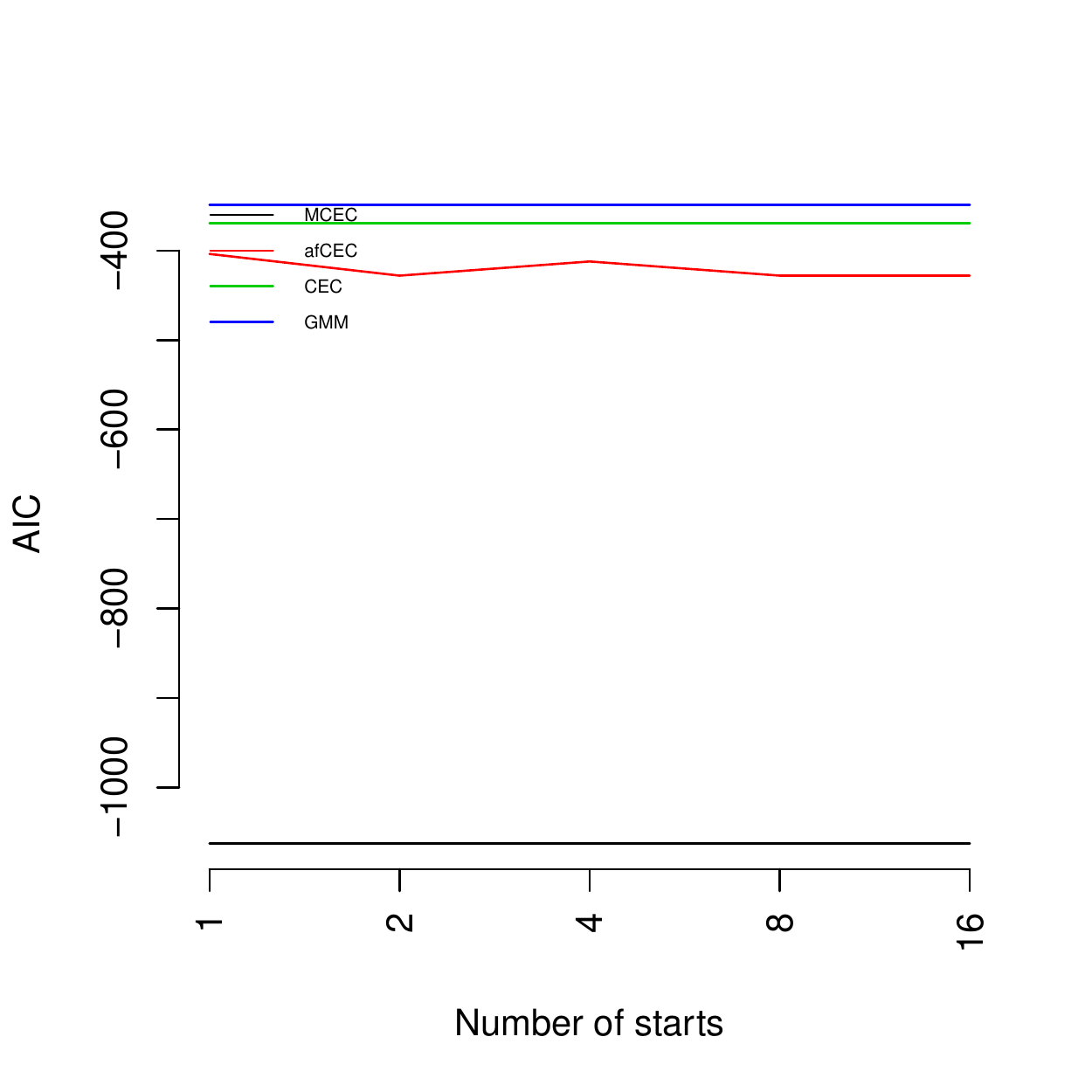} &
    		\includegraphics[width=0.24\textwidth]{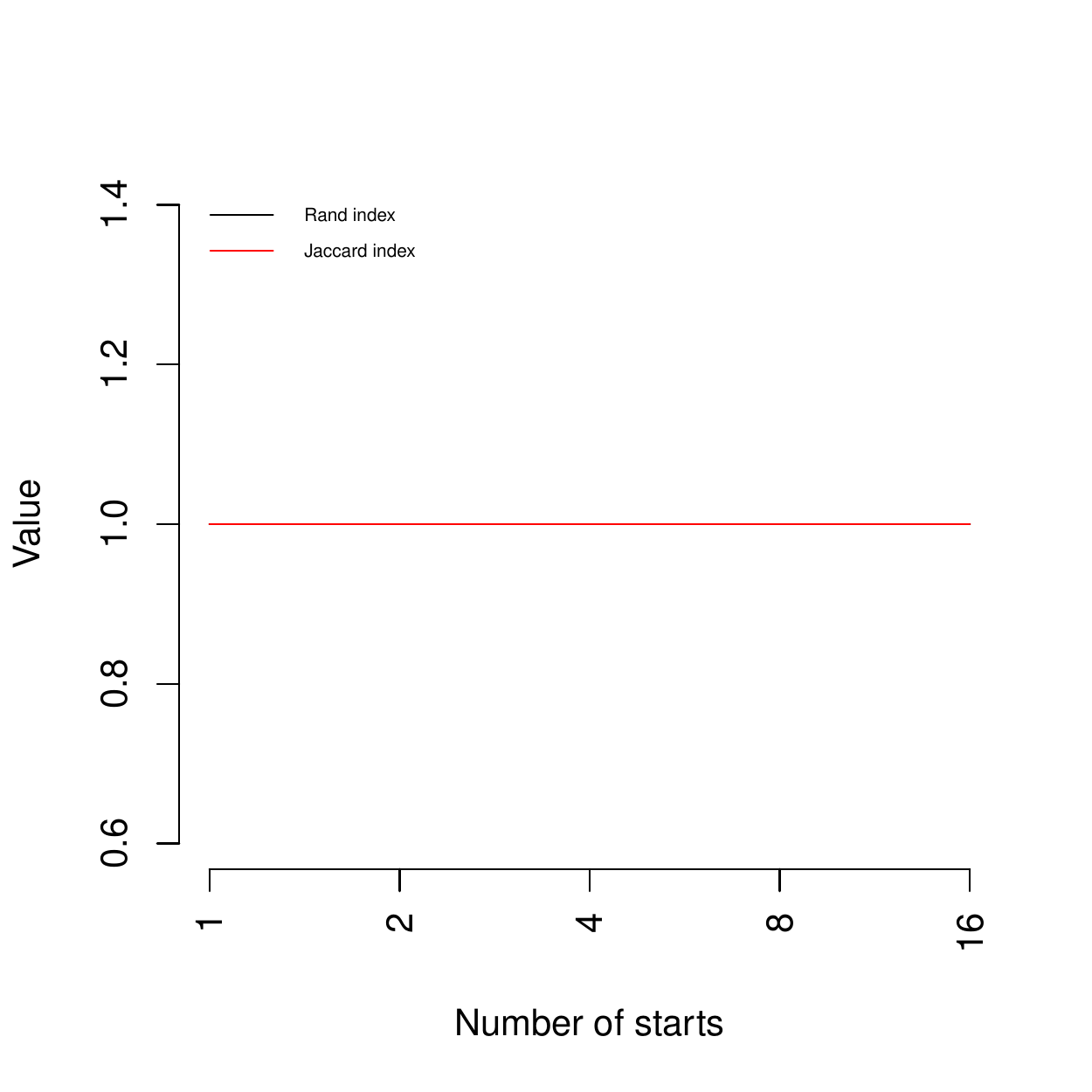} \\
    		\includegraphics[width=0.24\textwidth]{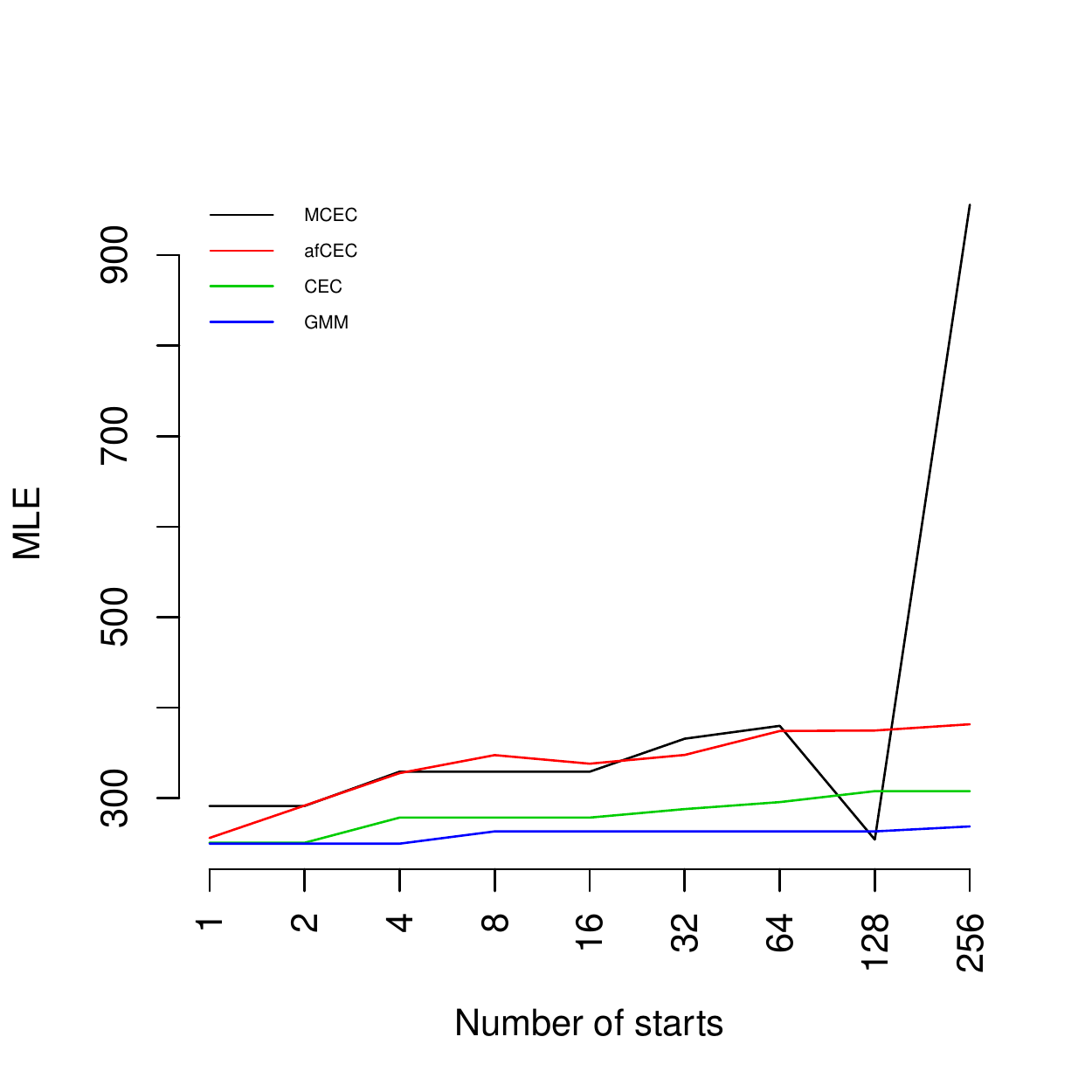} &
    		\includegraphics[width=0.24\textwidth]{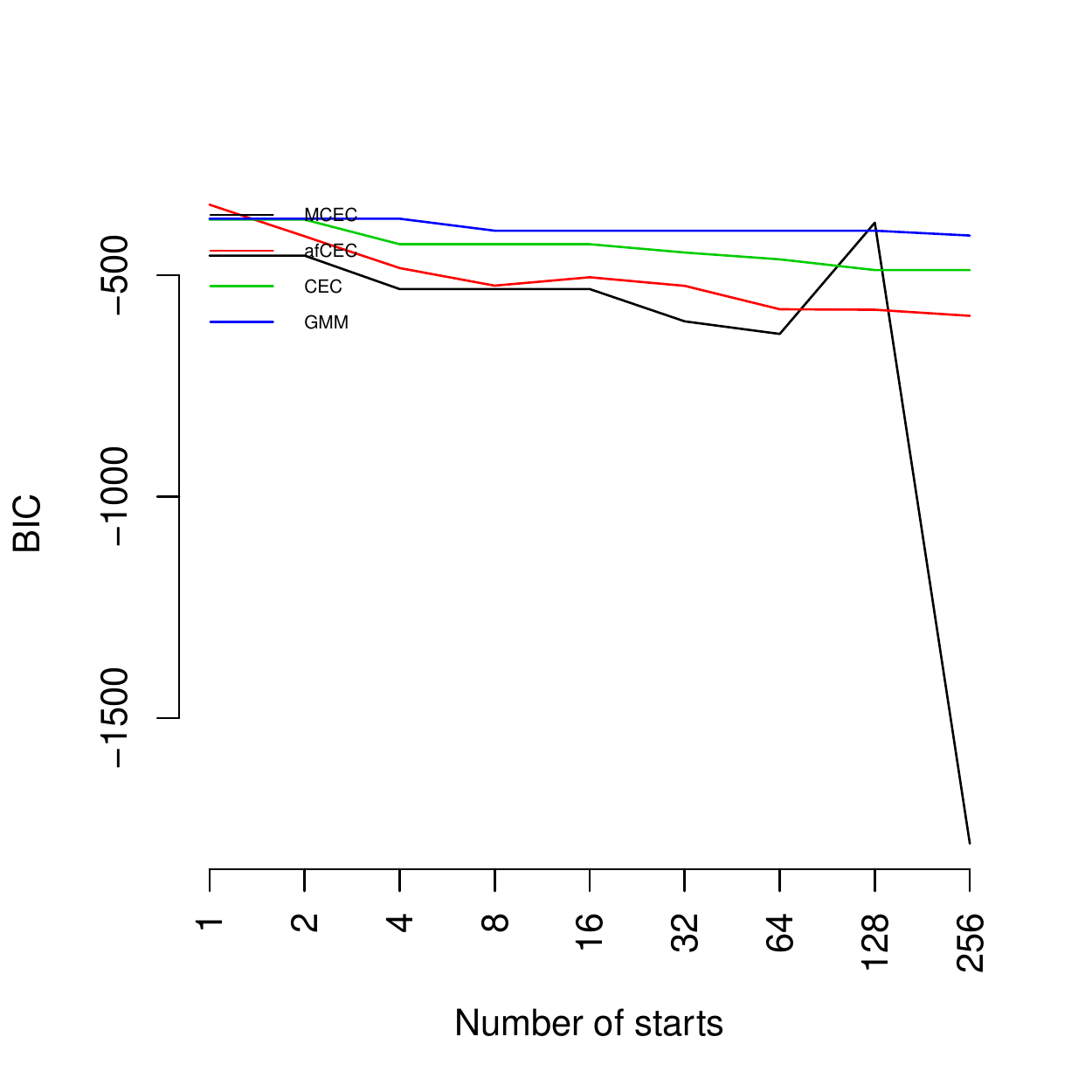} &
    		\includegraphics[width=0.24\textwidth]{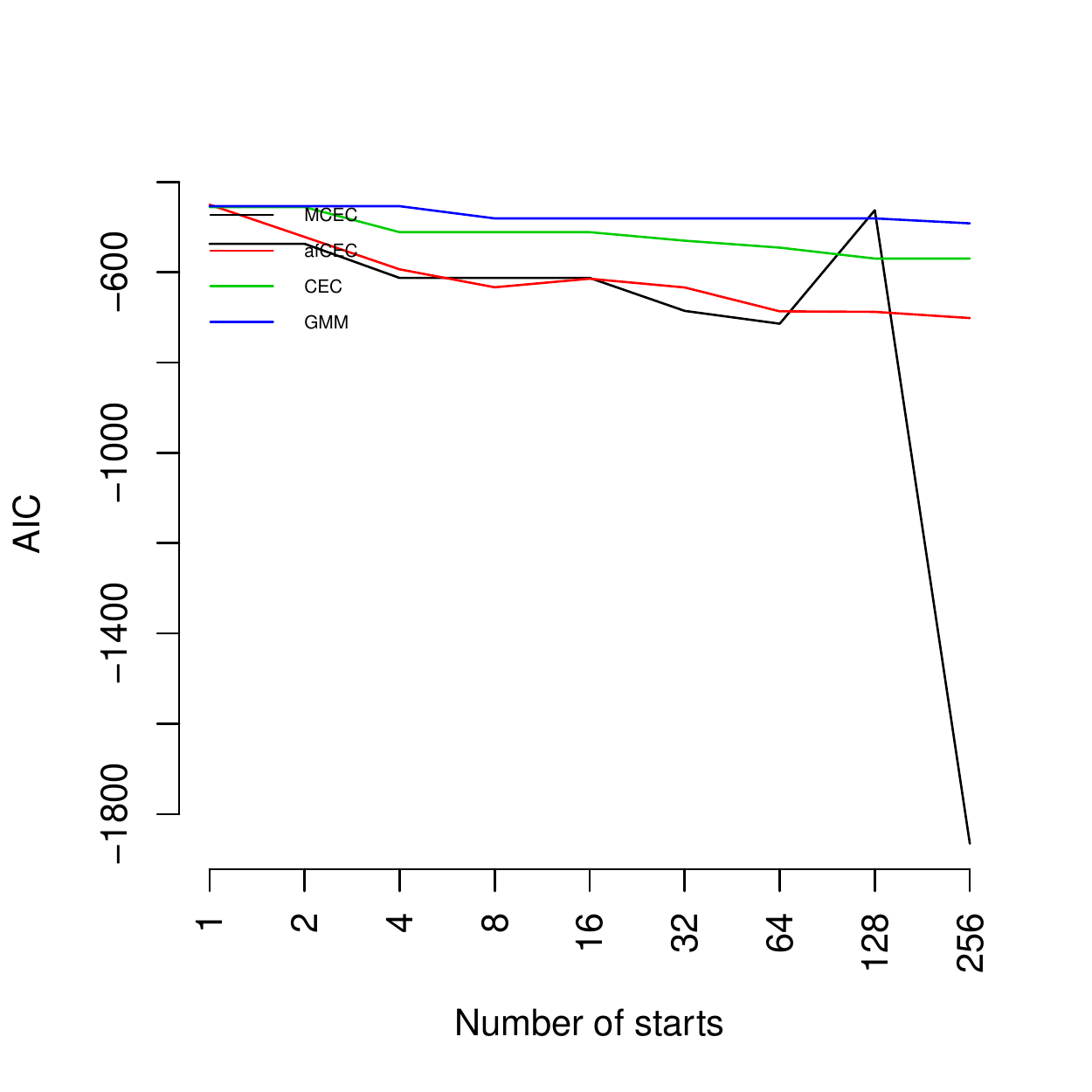} &
    		\includegraphics[width=0.24\textwidth]{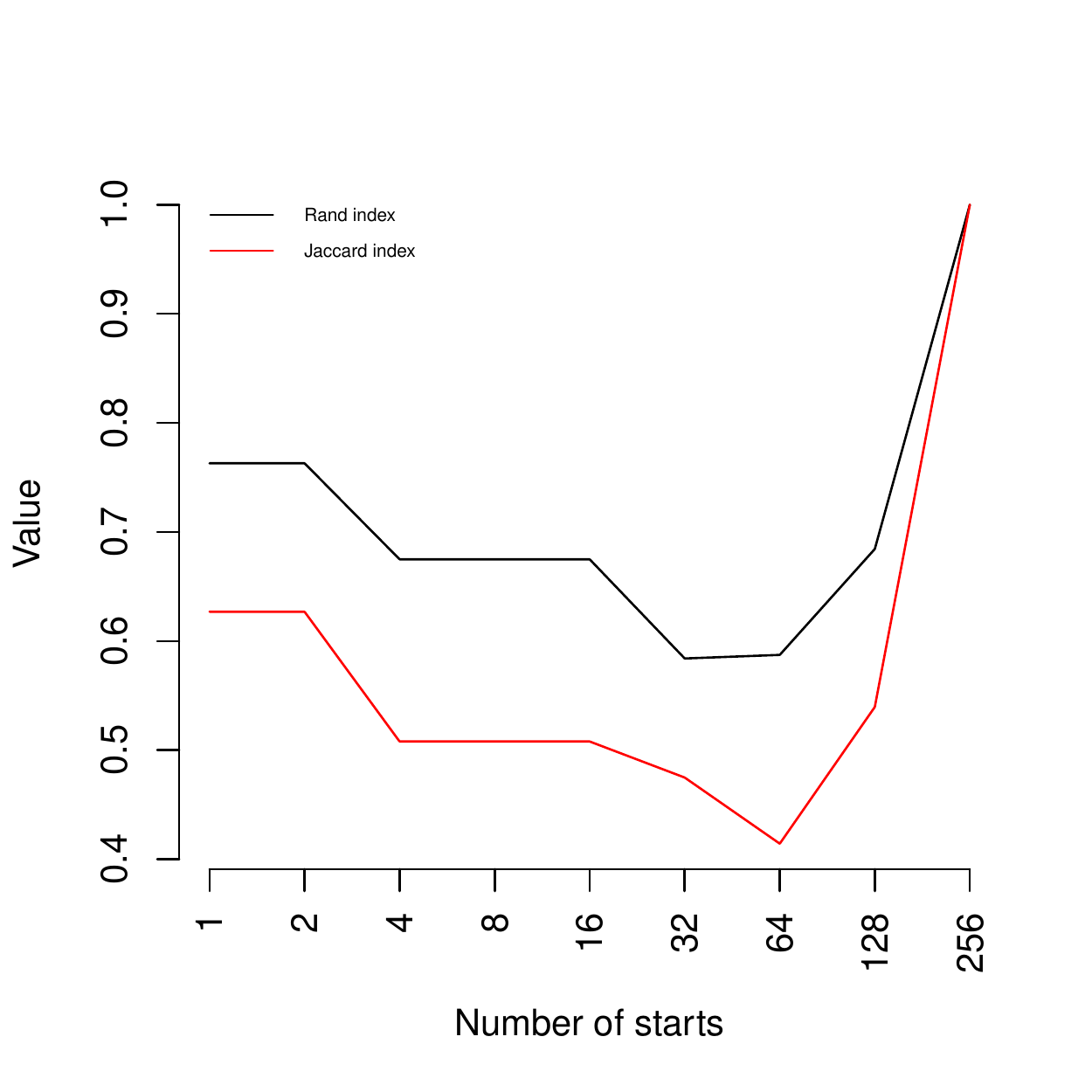} \\
    		\includegraphics[width=0.24\textwidth]{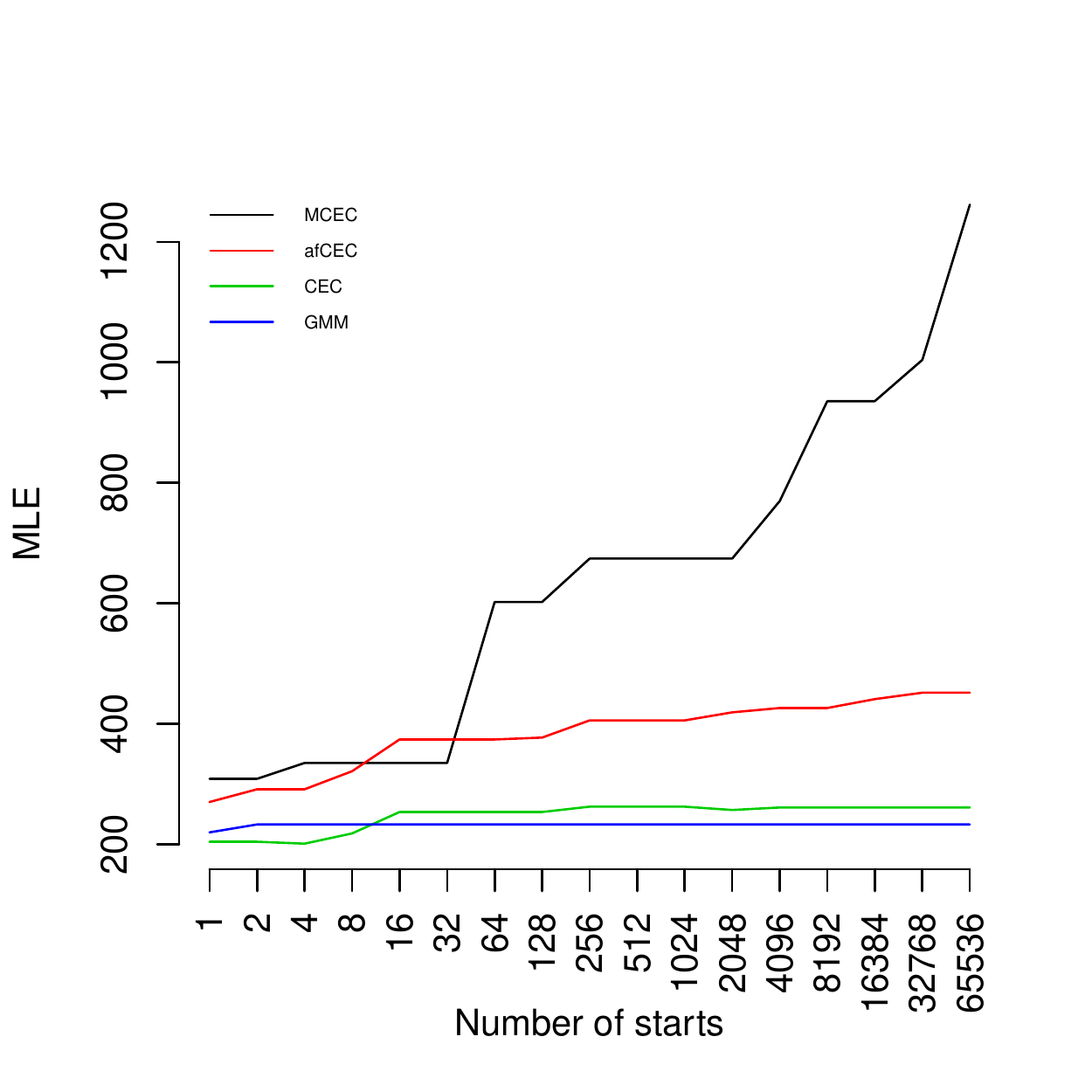} &
    		\includegraphics[width=0.24\textwidth]{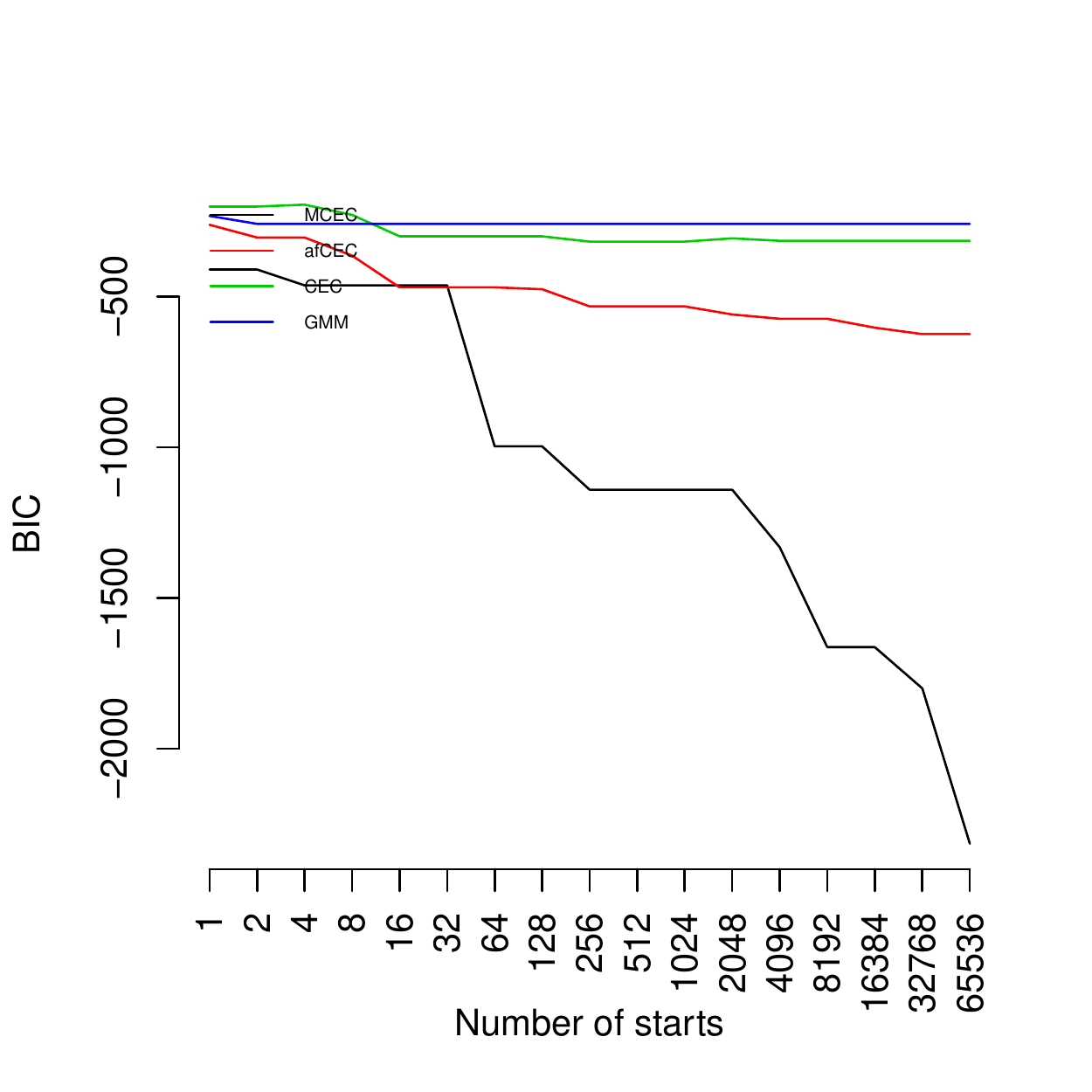} &
    		\includegraphics[width=0.24\textwidth]{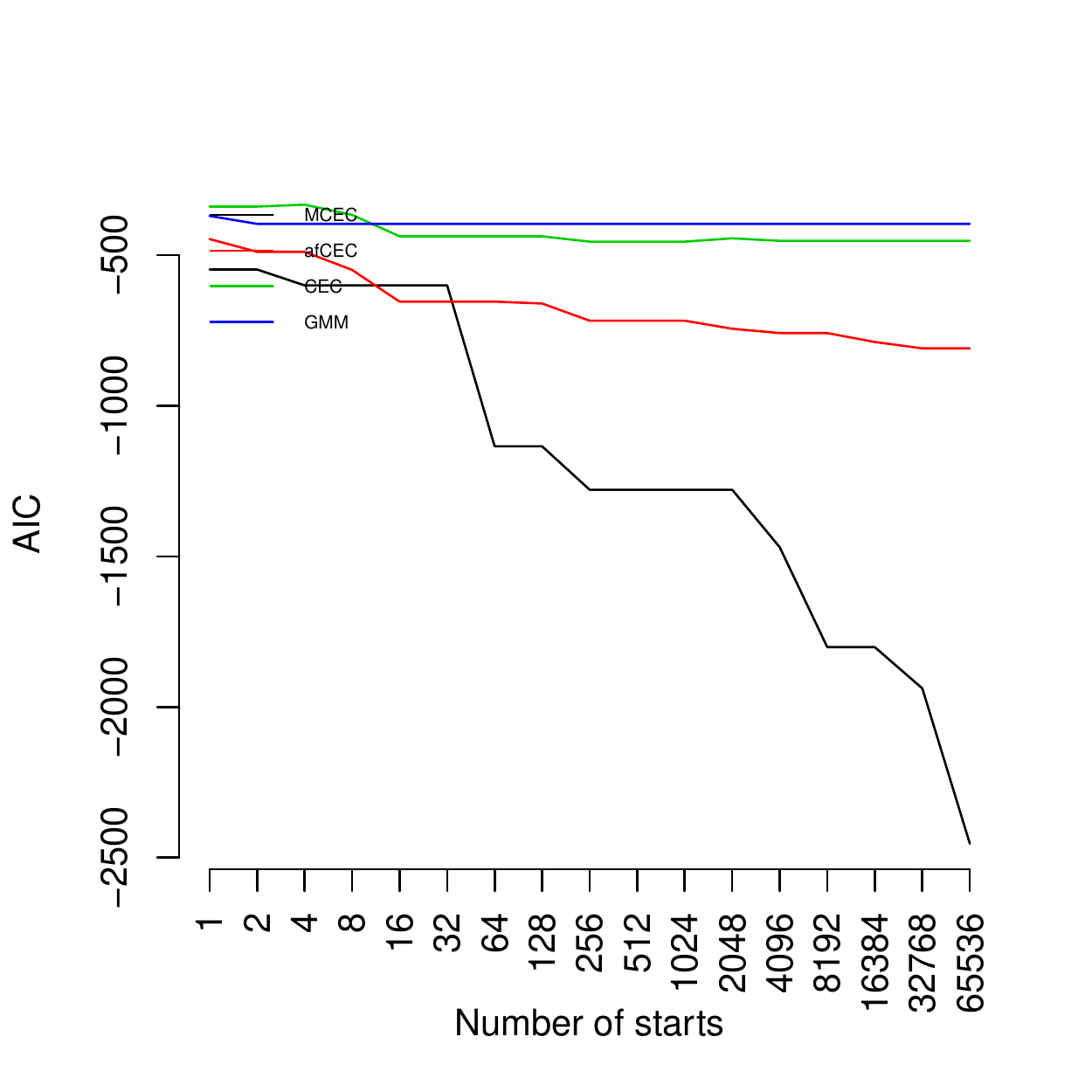} &
    		\includegraphics[width=0.24\textwidth]{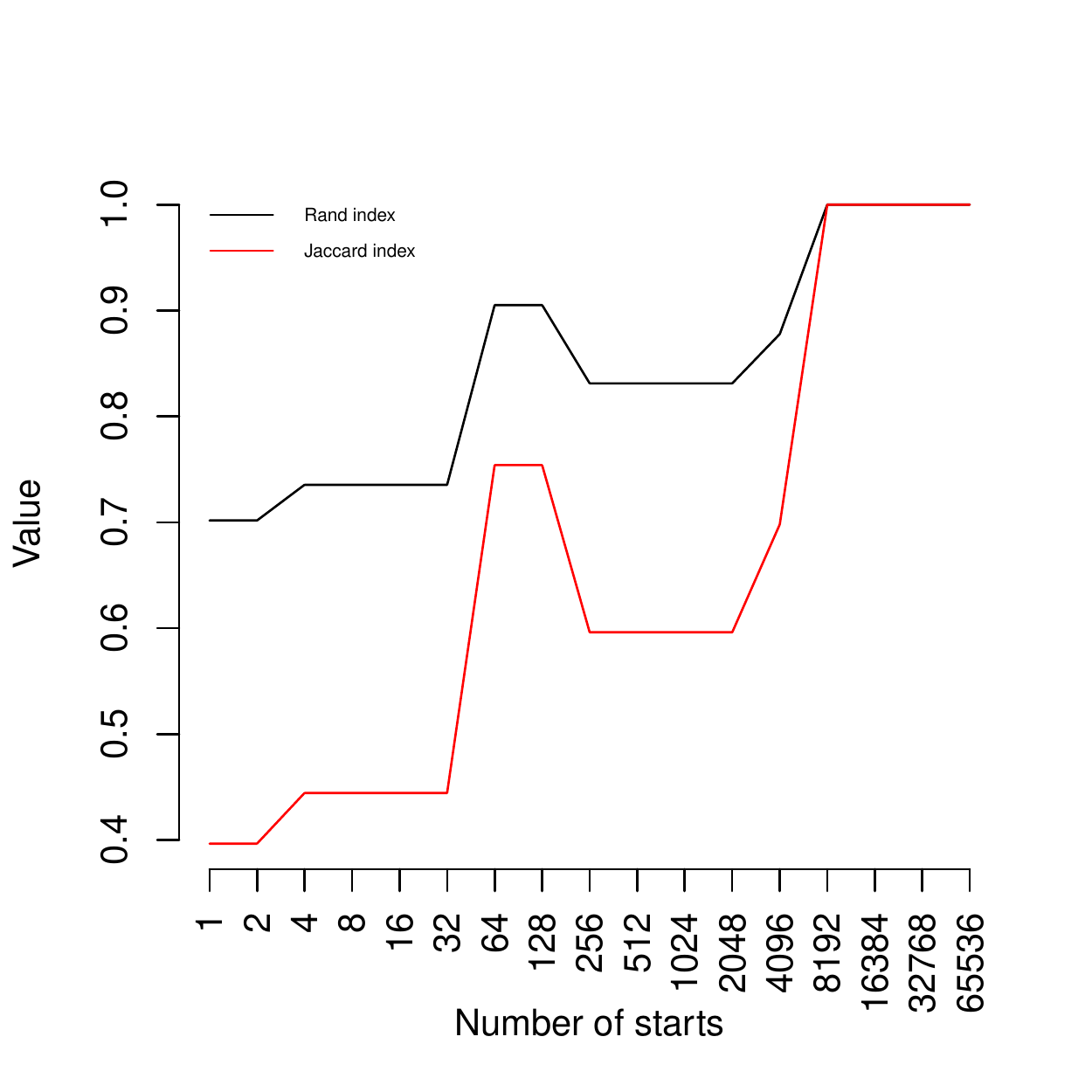} \\
    		\includegraphics[width=0.24\textwidth]{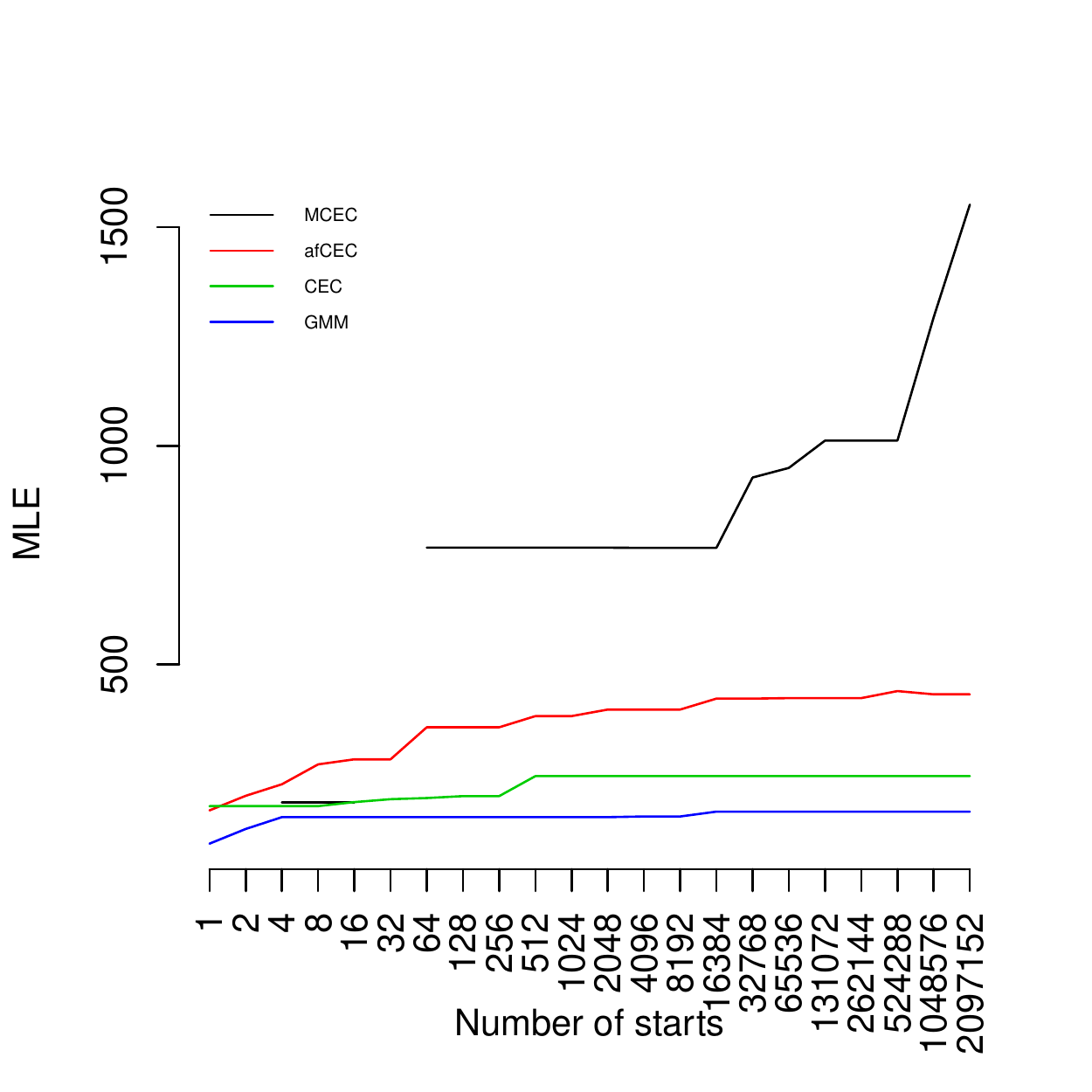} &
    		\includegraphics[width=0.24\textwidth]{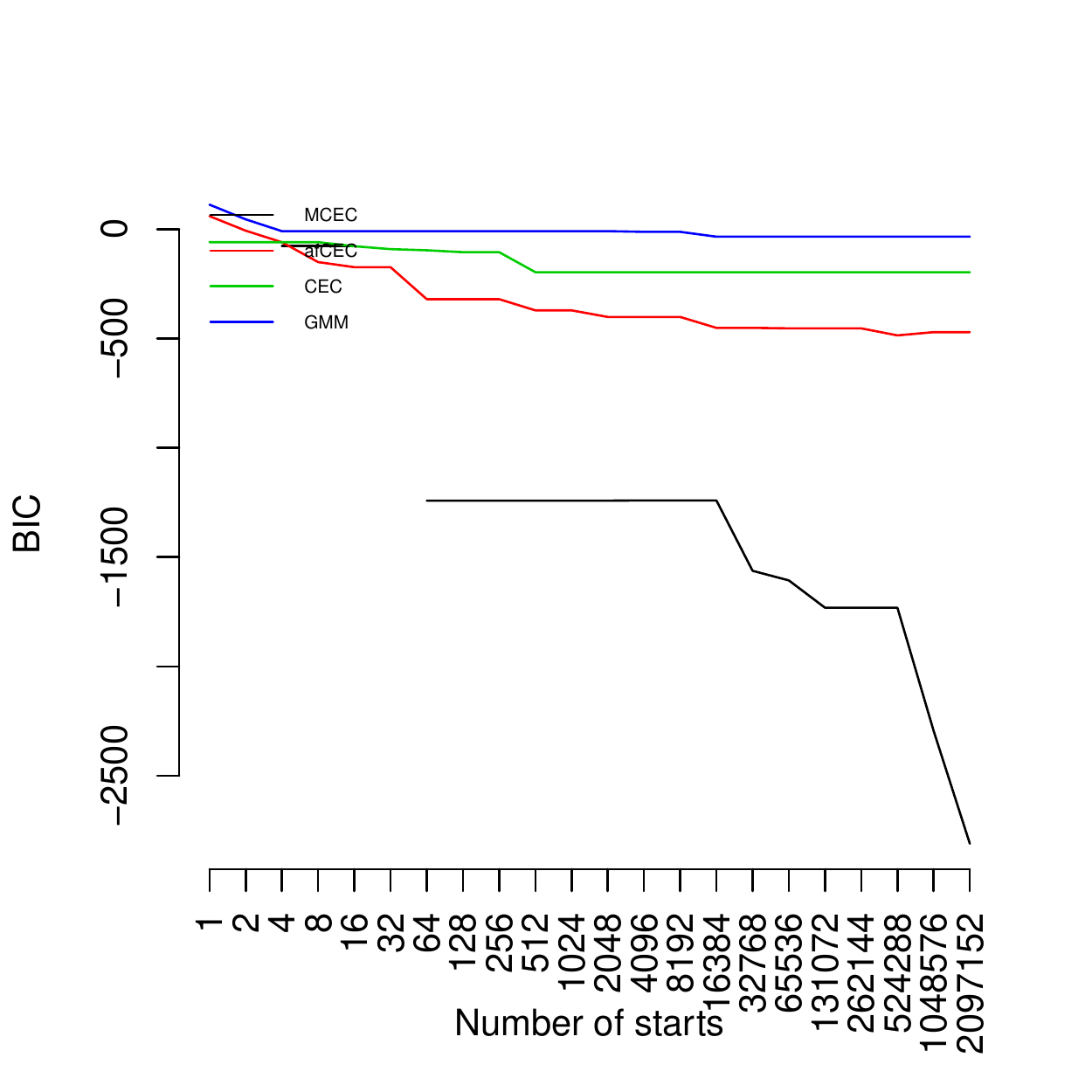} &
    		\includegraphics[width=0.24\textwidth]{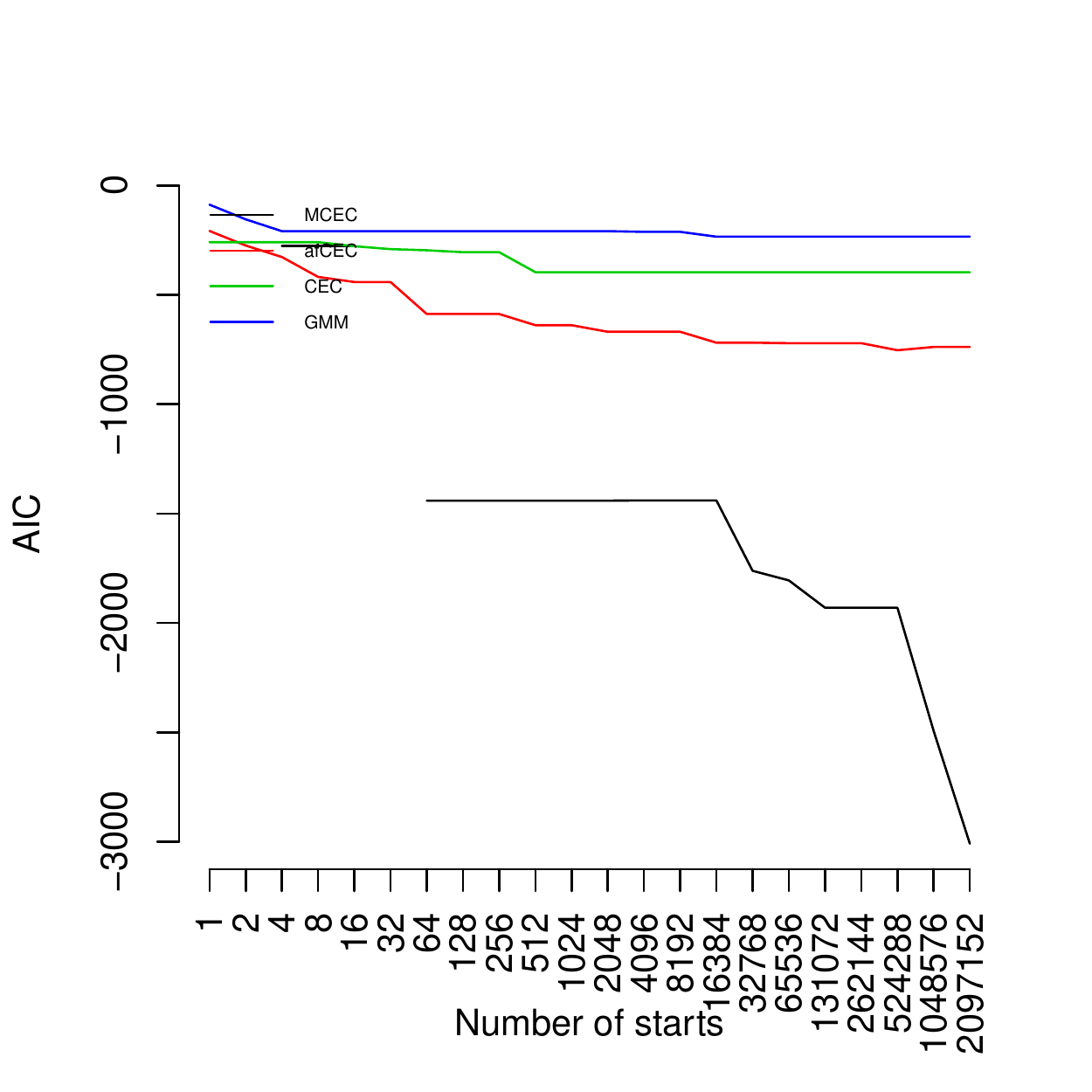} &
    		\includegraphics[width=0.24\textwidth]{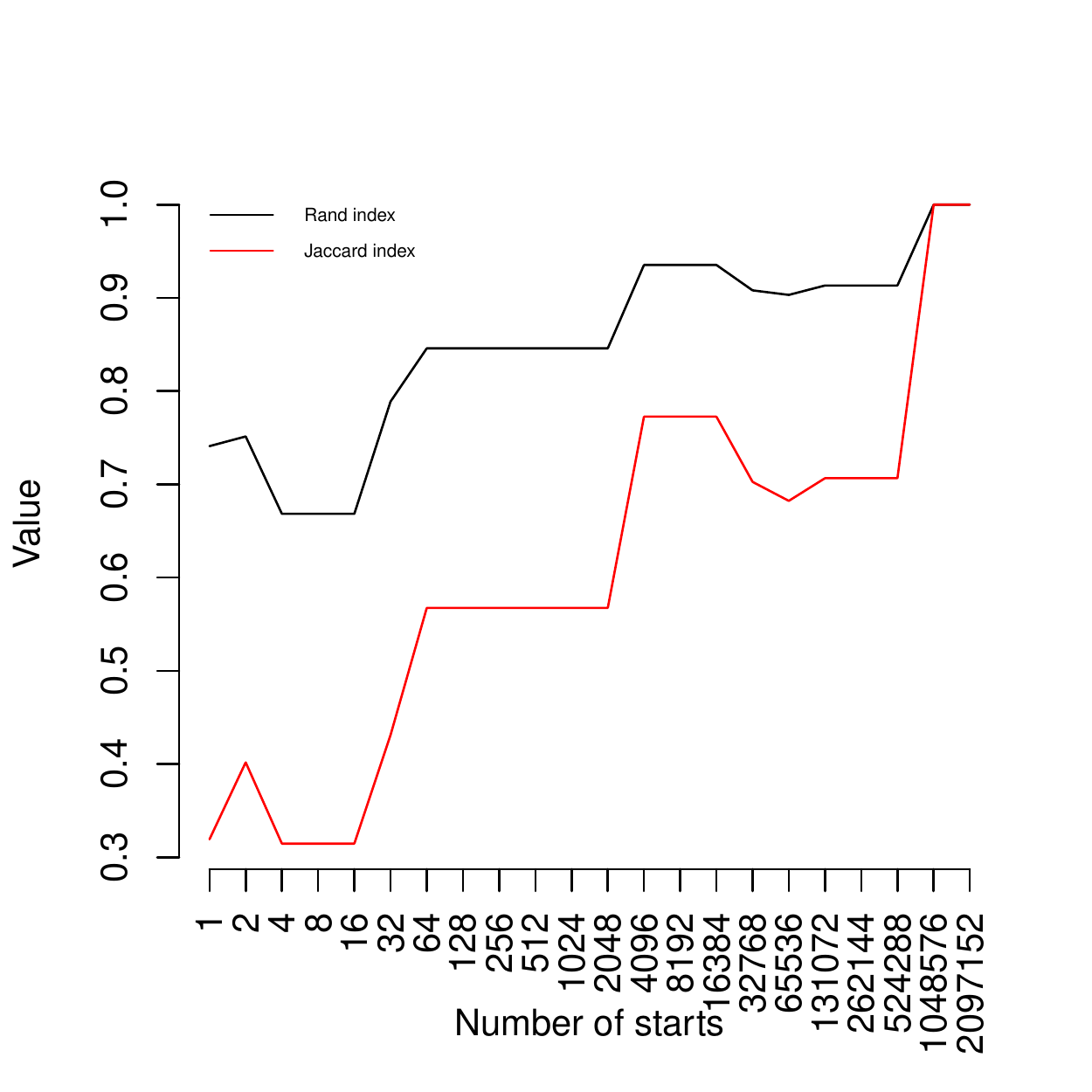} \\
    	\end{tabular}
	\caption{Experiment No. 2: Values of the performance metrics computed for the consecutive numbers of starts being the powers of two and the increasing number of clusters in the data set (the $i$-th row of the table corresponds to the test case with $i$ clusters). The charts in the first, second, and third columns show the computed values of MLE, BIC, and AIC, respectively, for the consecutive number of starts being the powers of two for all four clustering methods used in the comparison, ie. MCEC, afCEC, CEC, and GMM, while the last one represents the values of Rand and Jaccard indices computed similarly for the clustering obtained using the MCEC algorithm.}
\end{figure}

\begin{table}[H]
\caption{Experiment No. 3: Values of the performance metrics computed for the consecutive numbers of starts being the powers of two}
\centering
\setlength\tabcolsep{1.5pt}
\resizebox{\columnwidth}{!}{%
\begin{tabular}{c|>{\tiny}c|>{\tiny}c>{\tiny}c>{\tiny}c>{\tiny}c>{\tiny}c|>{\tiny}c>{\tiny}c>{\tiny}c|>{\tiny}c>{\tiny}c>{\tiny}c|>{\tiny}c>{\tiny}c>{\tiny}c|}
\cline{2-16}
& \rule{0pt}{1.5cm} & \multicolumn{5}{c|}{\makecell{MCEC}} & \multicolumn{3}{c|}{\makecell{afCEC Hartigan}} & \multicolumn{3}{c|}{CEC} & \multicolumn{3}{c|}{GMM} \\
\cline{3-16}

& \rotatebox[origin=c]{90}{\rlap{\parbox{1.5cm}{\makecell{Number of \\ starts}}}} &
\multicolumn{1}{c|}{MLE} & \multicolumn{1}{c|}{BIC} & \multicolumn{1}{c|}{AIC} & \multicolumn{1}{c|}{\makecell{Rand \\ index}} & \multicolumn{1}{c|}{\makecell{Jaccard \\ index}} &
\multicolumn{1}{c|}{MLE} & \multicolumn{1}{c|}{BIC} & \multicolumn{1}{c|}{AIC} &
\multicolumn{1}{c|}{MLE} & \multicolumn{1}{c|}{BIC} & \multicolumn{1}{c|}{AIC} &
\multicolumn{1}{c|}{MLE} & \multicolumn{1}{c|}{BIC} & \multicolumn{1}{c|}{AIC} \\

\hline
\hline
\multicolumn{1}{|c|}{\multirow{7}{*}{\includegraphics[scale=0.07]{5-3-1.pdf}}} & 1 & 522.04 & \textbf{-960.90} & -1014.08 & 1.00 & 1.00 & \textbf{553.23} & -934.57 & \textbf{-1044.47} & 416.93 & -706.33 & -787.87 & 387.71 & -647.88 & -729.42 \\ 
\multicolumn{1}{|c|}{} & 2 & 522.04 & -960.90 & -1014.08 & 1.00 & 1.00 & \textbf{616.01} & \textbf{-1060.13} & \textbf{-1170.03} & 416.93 & -706.33 & -787.87 & 387.71 & -647.88 & -729.42 \\ 
\multicolumn{1}{|c|}{} & 4 & 525.66 & -968.15 & -1021.32 & 1.00 & 1.00 & \textbf{616.01} & \textbf{-1060.13} & \textbf{-1170.03} & 507.09 & -886.63 & -968.17 & 387.71 & -647.89 & -729.42 \\ 
\multicolumn{1}{|c|}{} & 8 & 525.66 & -968.15 & -1021.32 & 1.00 & 1.00 & \textbf{739.76} & \textbf{-1307.62} & \textbf{-1417.53} & 507.09 & -886.63 & -968.17 & 387.71 & -647.89 & -729.43 \\ 
\multicolumn{1}{|c|}{} & 16 & 546.32 & -1009.47 & -1062.65 & 1.00 & 1.00 & \textbf{739.76} & \textbf{-1307.62} & \textbf{-1417.53} & 507.09 & -886.63 & -968.17 & 390.28 & -653.01 & -734.55 \\ 
\multicolumn{1}{|c|}{} & 32 & 546.32 & -1009.47 & -1062.65 & 1.00 & 1.00 & \textbf{739.76} & \textbf{-1307.62} & \textbf{-1417.53} & 507.09 & -886.63 & -968.17 & 390.28 & -653.01 & -734.55 \\ 
\multicolumn{1}{|c|}{} & 64 & \textbf{821.85} & \textbf{-1560.51} & \textbf{-1613.69} & 1.00 & 1.00 & 742.68 & -1313.47 & -1423.37 & 507.09 & -886.63 & -968.17 & 390.28 & -653.01 & -734.55 \\ 
\hline
\hline
\multicolumn{1}{|c|}{\multirow{19}{*}{\includegraphics[scale=0.07]{5-3-2.pdf}}} & 1 & 618.32 & -1043.25 & -1174.64 & 0.53 & 0.39 & \textbf{917.99} & \textbf{-1442.97} & \textbf{-1709.98} & 595.11 & -897.02 & -1096.22 & 469.12 & -645.04 & -844.24 \\ 
\multicolumn{1}{|c|}{} & 2 & 618.32 & -1043.25 & -1174.64 & 0.53 & 0.39 & \textbf{917.99} & \textbf{-1442.97} & \textbf{-1709.98} & 676.31 & -1059.41 & -1258.61 & 482.22 & -671.24 & -870.44 \\ 
\multicolumn{1}{|c|}{} & 4 & 718.30 & -1243.21 & -1374.60 & 0.64 & 0.51 & \textbf{917.99} & \textbf{-1442.97} & \textbf{-1709.98} & 676.31 & -1059.41 & -1258.61 & 482.22 & -671.24 & -870.44 \\ 
\multicolumn{1}{|c|}{} & 8 & 718.30 & -1243.21 & -1374.60 & 0.64 & 0.51 & \textbf{917.99} & \textbf{-1442.97} & \textbf{-1709.98} & 676.31 & -1059.41 & -1258.61 & 486.24 & -679.29 & -878.49 \\ 
\multicolumn{1}{|c|}{} & 16 & 718.30 & -1243.21 & -1374.60 & 0.64 & 0.51 & \textbf{974.27} & \textbf{-1555.52} & \textbf{-1822.54} & 732.40 & -1171.60 & -1370.80 & 496.71 & -700.22 & -899.42 \\ 
\multicolumn{1}{|c|}{} & 32 & 718.30 & -1243.21 & -1374.60 & 0.64 & 0.51 & \textbf{986.60} & \textbf{-1580.19} & \textbf{-1847.20} & 732.40 & -1171.60 & -1370.80 & 502.70 & -712.19 & -911.39 \\ 
\multicolumn{1}{|c|}{} & 64 & 771.02 & -1348.66 & -1480.04 & 0.69 & 0.55 & \textbf{1032.78} & \textbf{-1672.55} & \textbf{-1939.56} & 732.40 & -1171.60 & -1370.80 & 502.70 & -712.19 & -911.39 \\ 
\multicolumn{1}{|c|}{} & 128 & 771.02 & -1348.66 & -1480.04 & 0.69 & 0.55 & \textbf{1067.13} & \textbf{-1741.26} & \textbf{-2008.27} & 733.08 & -1172.95 & -1372.15 & 502.70 & -712.19 & -911.39 \\ 
\multicolumn{1}{|c|}{} & 256 & 801.71 & -1410.03 & -1541.41 & 0.55 & 0.40 & \textbf{1067.13} & \textbf{-1741.26} & \textbf{-2008.27} & 774.33 & -1255.46 & -1454.66 & 502.70 & -712.19 & -911.39 \\ 
\multicolumn{1}{|c|}{} & 512 & 810.94 & -1428.48 & -1559.87 & 0.64 & 0.51 & \textbf{1067.13} & \textbf{-1741.26} & \textbf{-2008.27} & 774.33 & -1255.46 & -1454.66 & 502.70 & -712.19 & -911.39 \\ 
\multicolumn{1}{|c|}{} & 1024 & 810.94 & -1428.48 & -1559.87 & 0.64 & 0.51 & \textbf{1067.13} & \textbf{-1741.26} & \textbf{-2008.27} & 774.33 & -1255.46 & -1454.66 & 502.70 & -712.19 & -911.39 \\ 
\multicolumn{1}{|c|}{} & 2048 & \textbf{1175.35} & \textbf{-2157.31} & \textbf{-2288.69} & 0.97 & 0.93 & 1067.13 & -1741.26 & -2008.27 & 777.94 & -1262.68 & -1461.88 & 502.70 & -712.19 & -911.39 \\ 
\multicolumn{1}{|c|}{} & 4096 & \textbf{1175.35} & \textbf{-2157.31} & \textbf{-2288.69} & 0.97 & 0.93 & 1088.33 & -1783.64 & -2050.66 & 803.08 & -1312.95 & -1512.15 & 502.70 & -712.19 & -911.39 \\ 
\multicolumn{1}{|c|}{} & 8192 & \textbf{1175.35} & \textbf{-2157.31} & \textbf{-2288.69} & 0.97 & 0.93 & 1094.65 & -1796.28 & -2063.29 & 803.08 & -1312.95 & -1512.15 & 502.70 & -712.19 & -911.39 \\ 
\multicolumn{1}{|c|}{} & 16384 & \textbf{1175.35} & \textbf{-2157.31} & \textbf{-2288.69} & 0.97 & 0.93 & 1098.33 & -1803.65 & -2070.66 & 803.08 & -1312.95 & -1512.15 & 502.70 & -712.19 & -911.39 \\ 
\multicolumn{1}{|c|}{} & 32768 & \textbf{1175.35} & \textbf{-2157.31} & \textbf{-2288.69} & 0.97 & 0.93 & 1112.86 & -1832.70 & -2099.72 & 803.08 & -1312.95 & -1512.15 & 503.88 & -714.56 & -913.76 \\ 
\multicolumn{1}{|c|}{} & 65536 & \textbf{1175.35} & \textbf{-2157.31} & \textbf{-2288.69} & 0.97 & 0.93 & 1112.86 & -1832.70 & -2099.72 & 803.08 & -1312.95 & -1512.15 & 503.88 & -714.56 & -913.76 \\ 
\multicolumn{1}{|c|}{} & 131072 & \textbf{1175.35} & \textbf{-2157.31} & \textbf{-2288.69} & 0.97 & 0.93 & 1112.86 & -1832.70 & -2099.72 & 803.08 & -1312.95 & -1512.15 & 503.88 & -714.56 & -913.76 \\ 
\multicolumn{1}{|c|}{} & 262144 & \textbf{1175.35} & \textbf{-2157.31} & \textbf{-2288.69} & 0.97 & 0.93 & 1112.05 & -1831.09 & -2098.11 & 803.08 & -1312.95 & -1512.15 & 503.88 & -714.56 & -913.76 \\ 
\hline
\hline

\end{tabular}
}
\end{table}

\begin{figure}[H]
	\centering
	    \begin{tabular}{c@{}c@{}c@{}c@{}}
    		\includegraphics[width=0.24\textwidth]{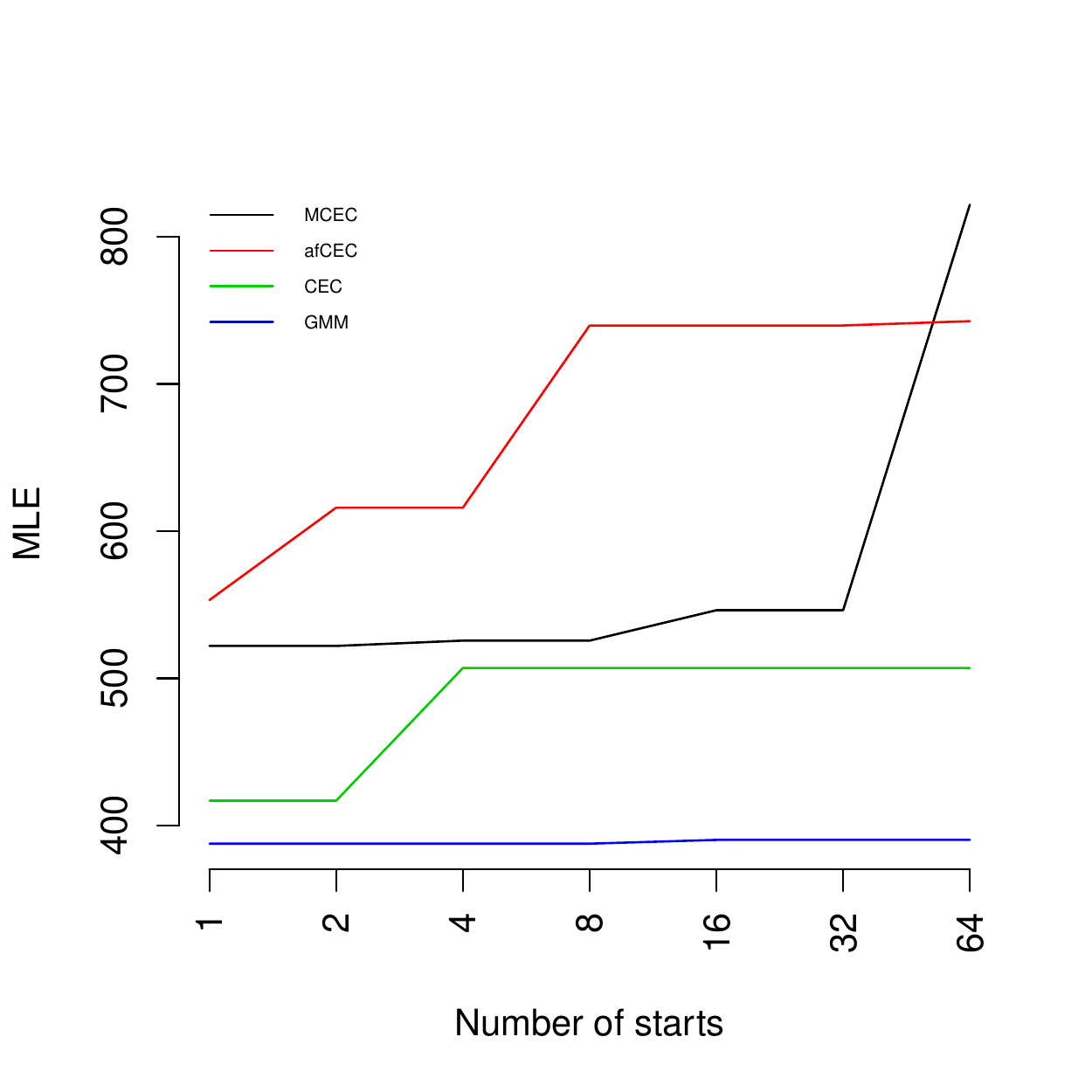} &
    		\includegraphics[width=0.24\textwidth]{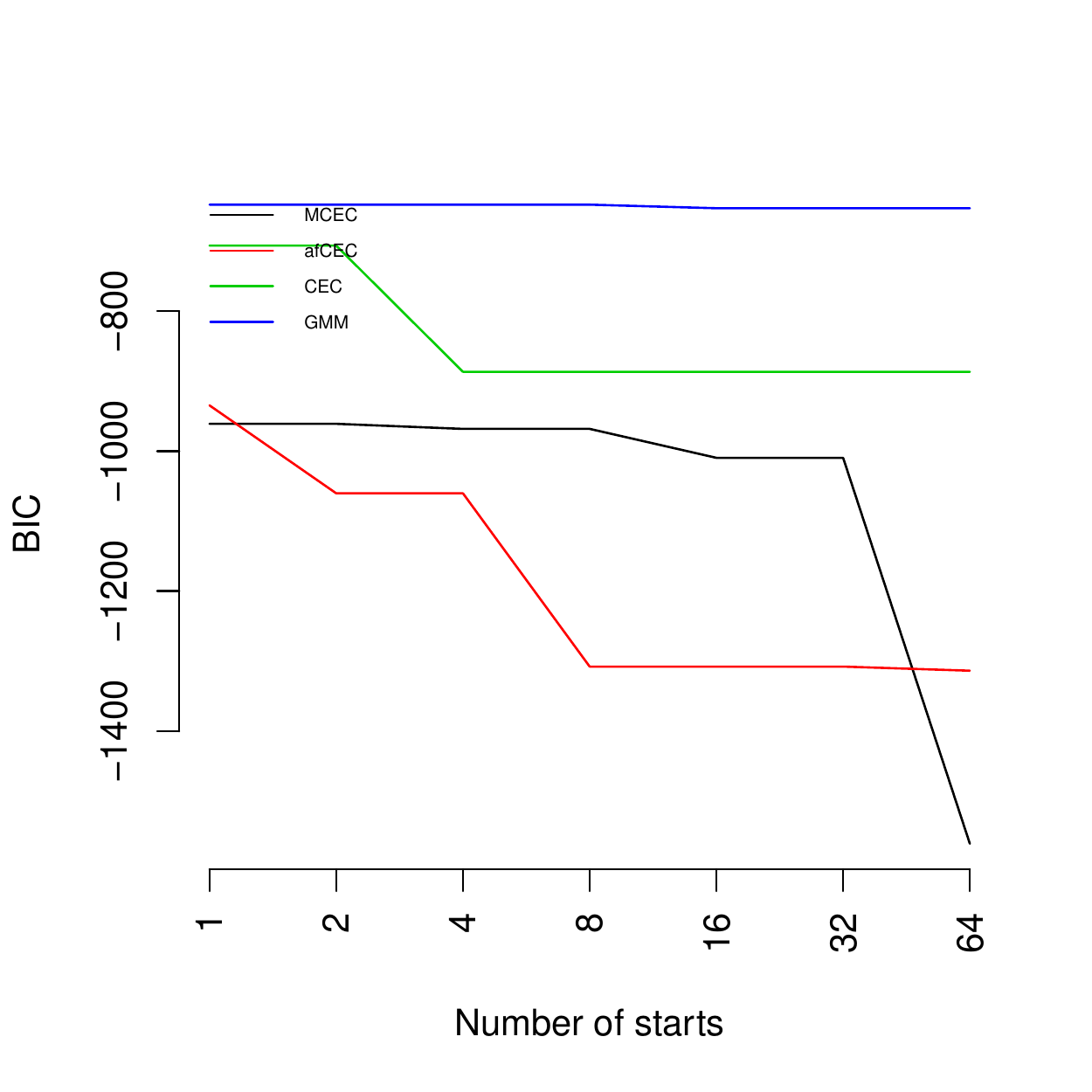} &
    		\includegraphics[width=0.24\textwidth]{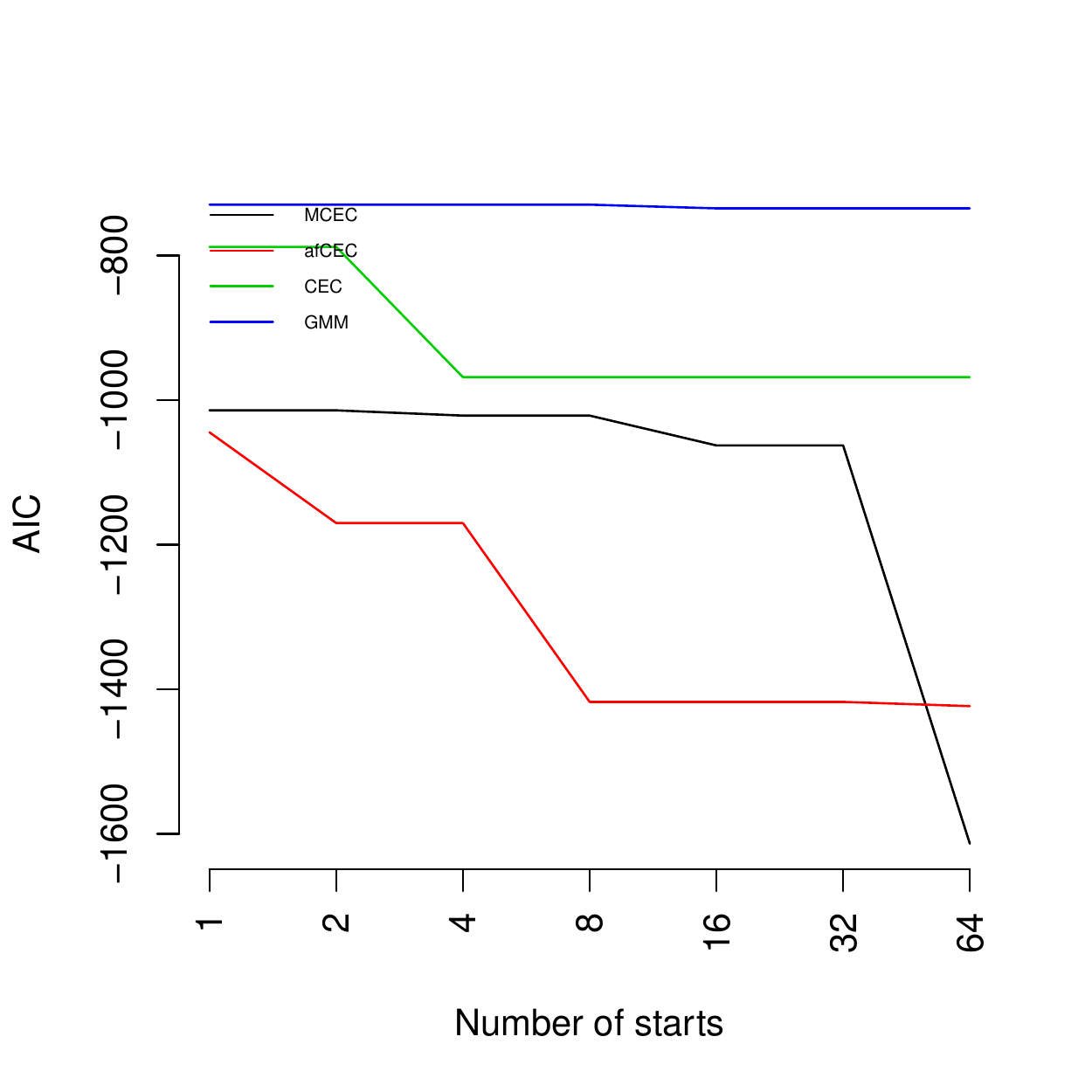} &
    		\includegraphics[width=0.24\textwidth]{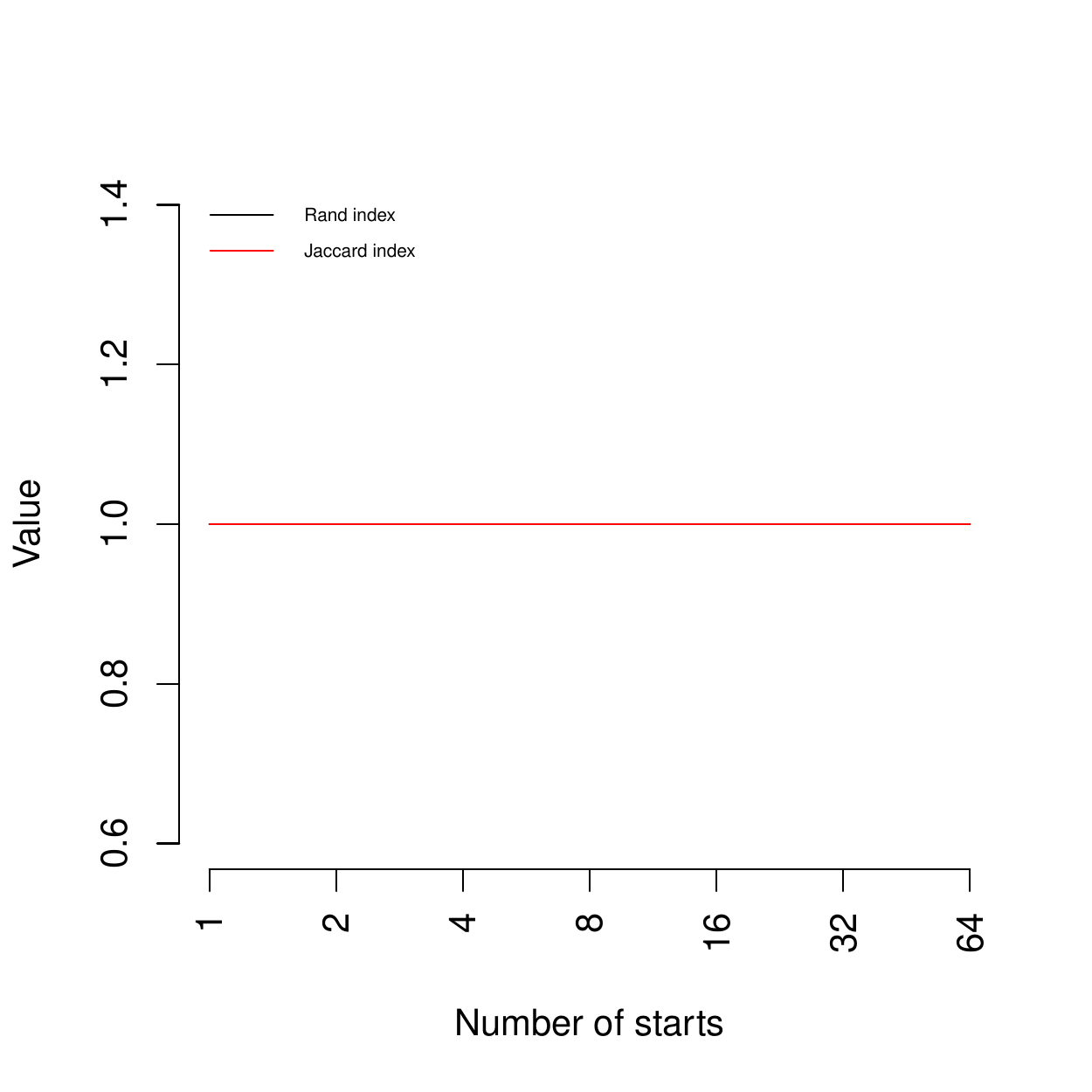} \\
    		\includegraphics[width=0.24\textwidth]{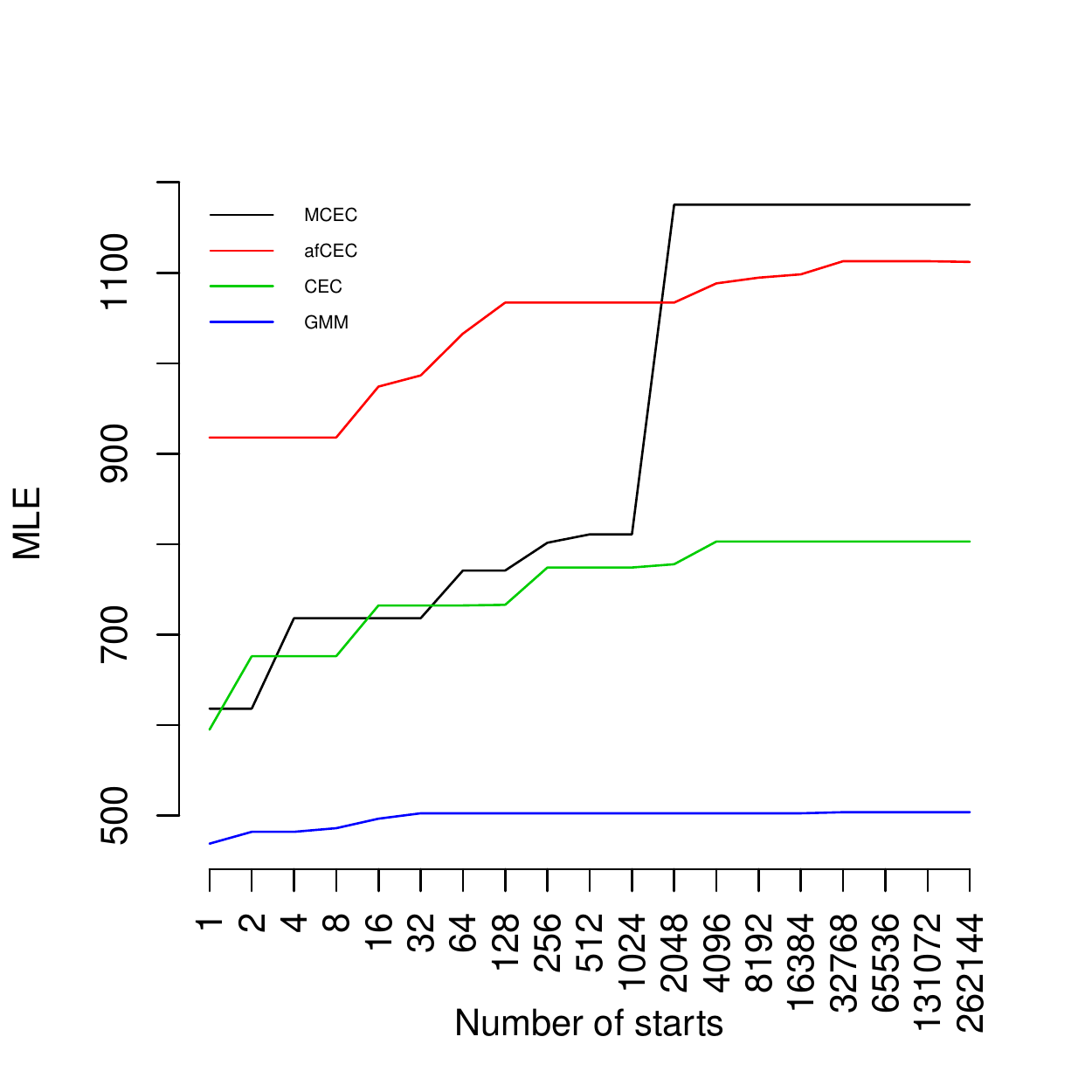} &
    		\includegraphics[width=0.24\textwidth]{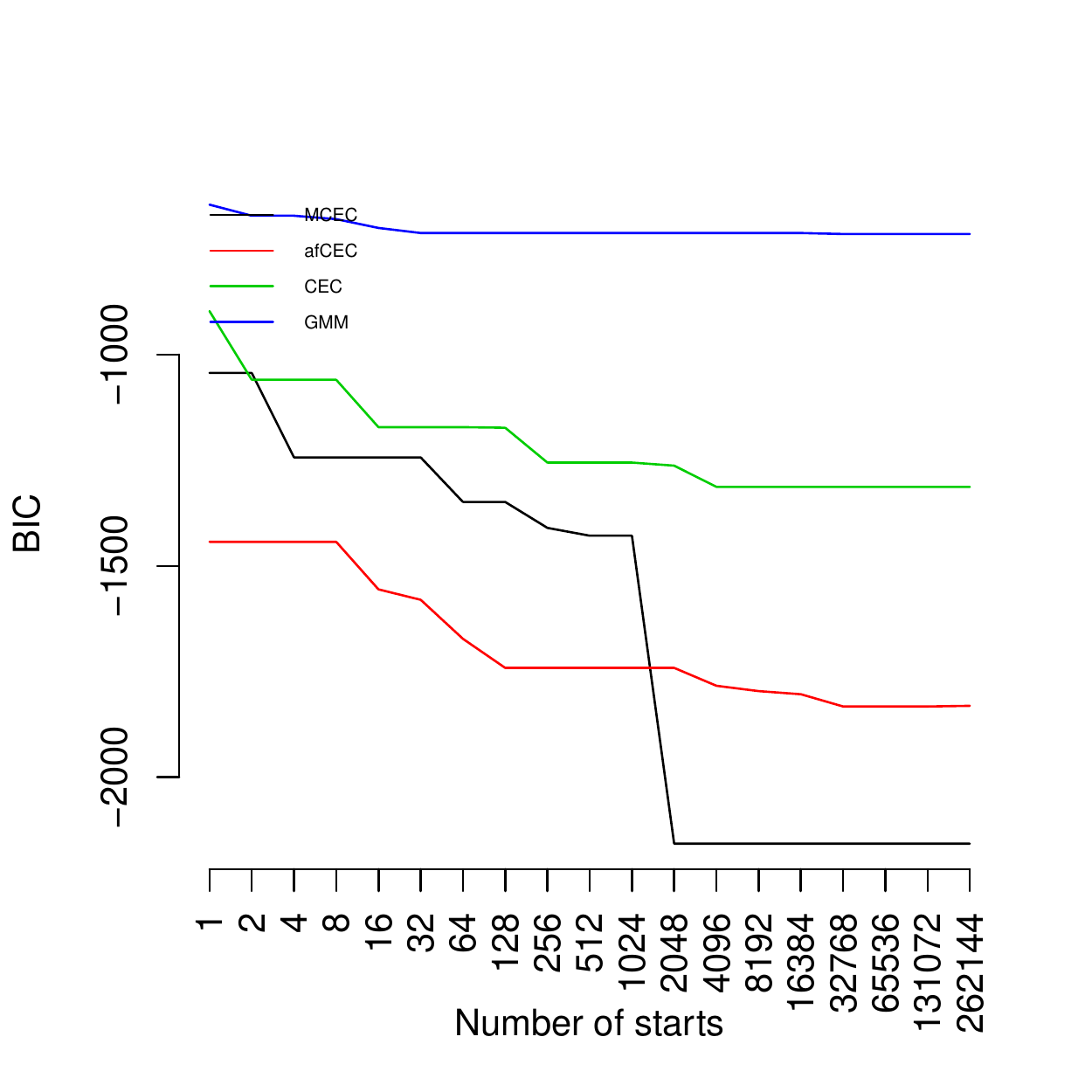} &
    		\includegraphics[width=0.24\textwidth]{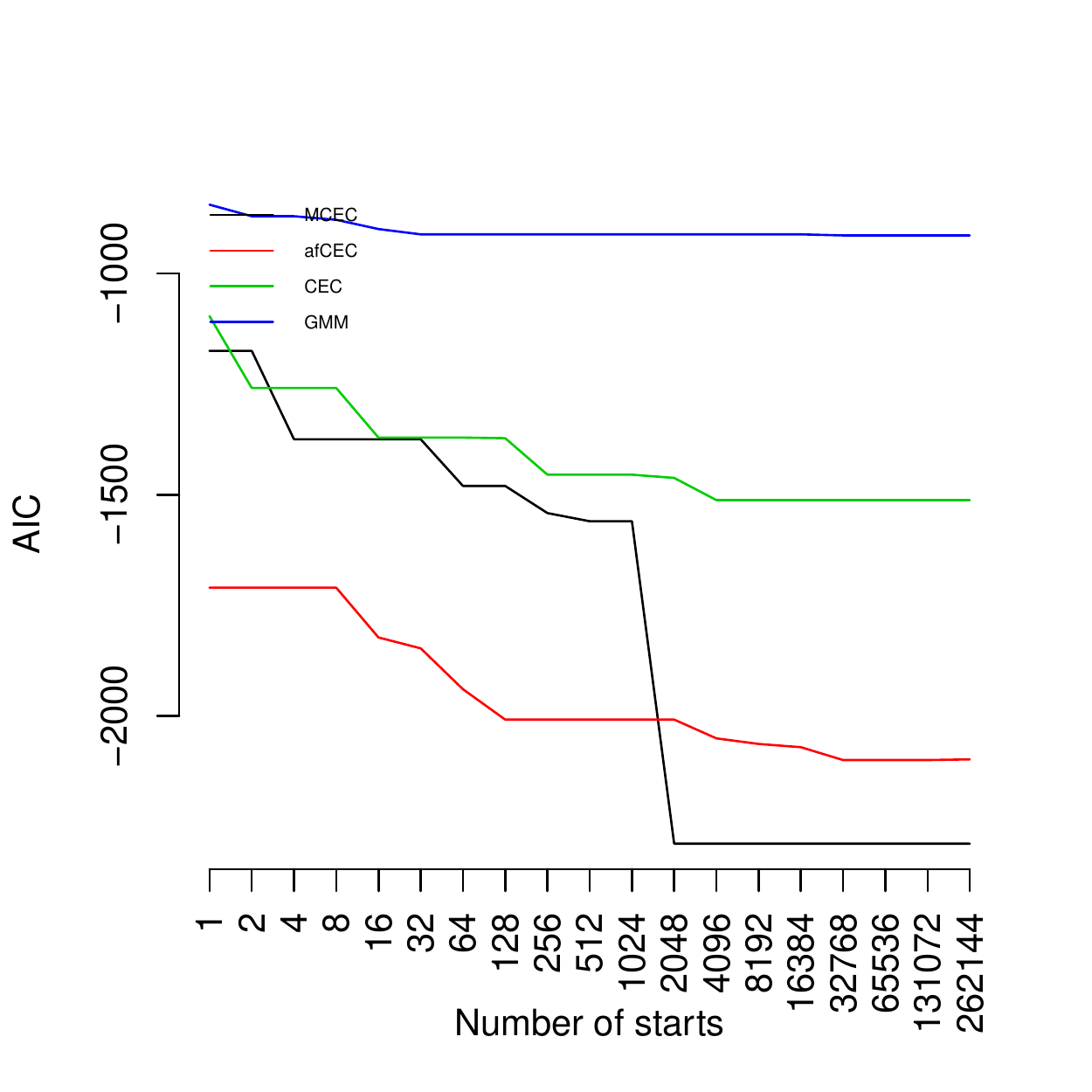} &
    		\includegraphics[width=0.24\textwidth]{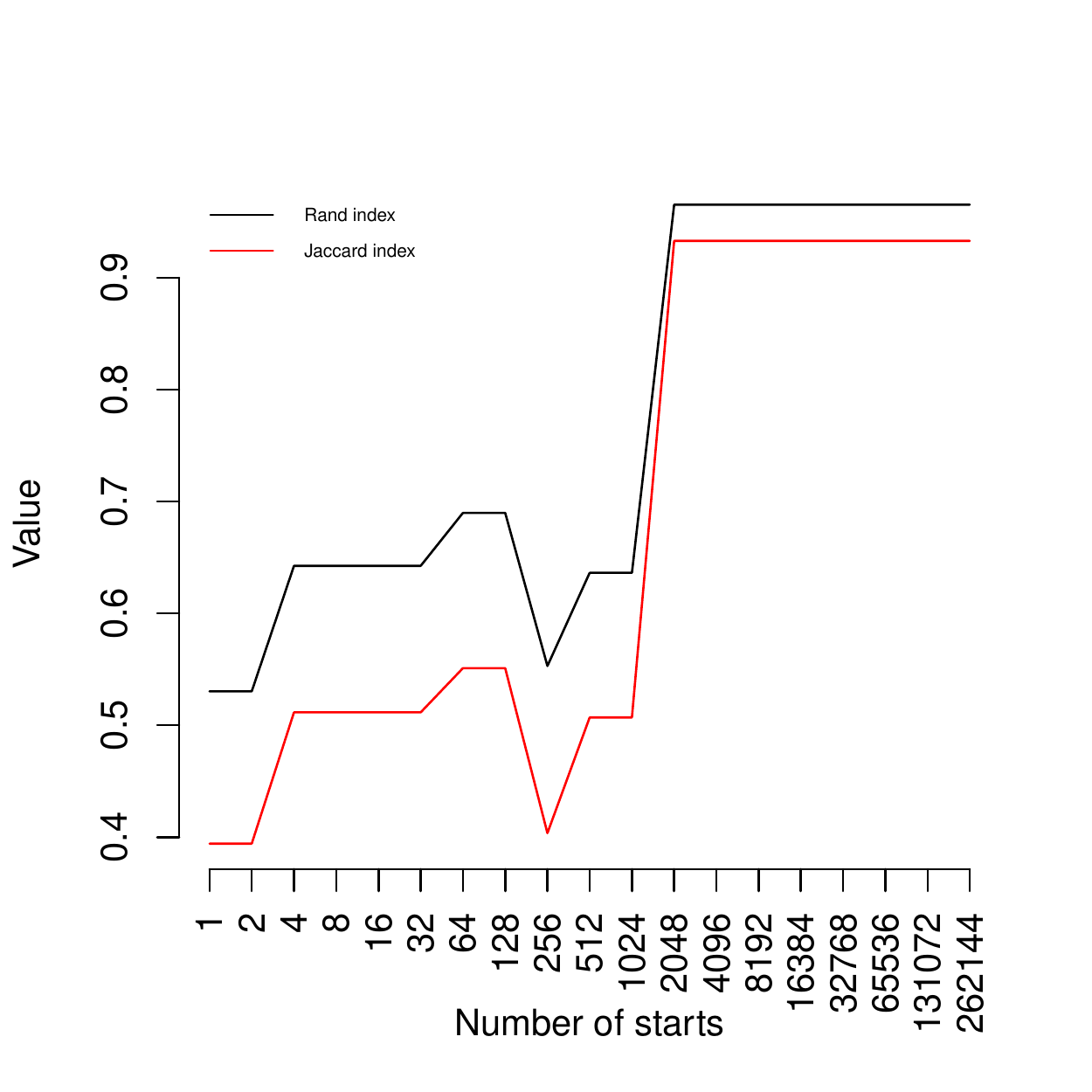} \\
    	\end{tabular}
	\caption{Experiment No. 3: Values of the performance metrics computed for the consecutive numbers of starts being the powers of two and the increasing number of clusters in the data set (the $i$-th row of the table corresponds to the test case with $i$ clusters). The charts in the first, second, and third columns show the computed values of MLE, BIC, and AIC, respectively, for the consecutive number of starts being the powers of two for all four clustering methods used in the comparison, ie. MCEC, afCEC, CEC, and GMM, while the last one represents the values of Rand and Jaccard indices computed similarly for the clustering obtained using the MCEC algorithm.}
\end{figure}

\begin{table}[H]
\caption{Experiment No. 4: Values of the performance metrics computed for the consecutive numbers of starts being the powers of two}
\centering

\setlength\tabcolsep{1.5pt}
\resizebox{\columnwidth}{!}{%
\begin{tabular}{c|>{\tiny}c|>{\tiny}c>{\tiny}c>{\tiny}c>{\tiny}c>{\tiny}c|>{\tiny}c>{\tiny}c>{\tiny}c|>{\tiny}c>{\tiny}c>{\tiny}c|>{\tiny}c>{\tiny}c>{\tiny}c|}
\cline{2-16}
& \rule{0pt}{1.5cm} & \multicolumn{5}{c|}{\makecell{MCEC}} & \multicolumn{3}{c|}{\makecell{afCEC Hartigan}} & \multicolumn{3}{c|}{CEC} & \multicolumn{3}{c|}{GMM} \\
\cline{3-16}

& \rotatebox[origin=c]{90}{\rlap{\parbox{1.5cm}{\makecell{Number of \\ starts}}}} &
\multicolumn{1}{c|}{MLE} & \multicolumn{1}{c|}{BIC} & \multicolumn{1}{c|}{AIC} & \multicolumn{1}{c|}{\makecell{Rand \\ index}} & \multicolumn{1}{c|}{\makecell{Jaccard \\ index}} &
\multicolumn{1}{c|}{MLE} & \multicolumn{1}{c|}{BIC} & \multicolumn{1}{c|}{AIC} &
\multicolumn{1}{c|}{MLE} & \multicolumn{1}{c|}{BIC} & \multicolumn{1}{c|}{AIC} &
\multicolumn{1}{c|}{MLE} & \multicolumn{1}{c|}{BIC} & \multicolumn{1}{c|}{AIC} \\

\hline
\hline
\multicolumn{1}{|c|}{\multirow{17}{*}{\includegraphics[scale=0.07]{5-4-1.pdf}}} & 1 & 229.64 & -353.93 & -421.29 & 1.00 & 1.00 & \textbf{289.12} & \textbf{-406.35} & \textbf{-516.25} & 77.40 & -27.25 & -108.79 & 65.56 & -3.59 & -85.13 \\ 
\multicolumn{1}{|c|}{} & 2 & 229.64 & -353.93 & -421.29 & 1.00 & 1.00 & \textbf{289.12} & \textbf{-406.35} & \textbf{-516.25} & 80.13 & -32.73 & -114.27 & 65.56 & -3.59 & -85.13 \\ 
\multicolumn{1}{|c|}{} & 4 & 229.64 & -353.93 & -421.29 & 1.00 & 1.00 & \textbf{289.12} & \textbf{-406.35} & \textbf{-516.25} & 117.20 & -106.87 & -188.40 & 71.49 & -15.43 & -96.97 \\ 
\multicolumn{1}{|c|}{} & 8 & 229.64 & -353.93 & -421.29 & 1.00 & 1.00 & \textbf{387.79} & \textbf{-603.69} & \textbf{-713.59} & 117.20 & -106.87 & -188.40 & 71.49 & -15.43 & -96.97 \\ 
\multicolumn{1}{|c|}{} & 16 & 229.64 & -353.93 & -421.29 & 1.00 & 1.00 & \textbf{387.79} & \textbf{-603.69} & \textbf{-713.59} & 117.20 & -106.87 & -188.40 & 71.49 & -15.43 & -96.97 \\ 
\multicolumn{1}{|c|}{} & 32 & 229.64 & -353.93 & -421.29 & 1.00 & 1.00 & \textbf{387.79} & \textbf{-603.69} & \textbf{-713.59} & 117.20 & -106.87 & -188.40 & 71.49 & -15.43 & -96.97 \\ 
\multicolumn{1}{|c|}{} & 64 & 289.01 & -472.65 & -540.01 & 1.00 & 1.00 & \textbf{387.79} & \textbf{-603.69} & \textbf{-713.59} & 123.88 & -120.22 & -201.76 & 71.49 & -15.43 & -96.97 \\ 
\multicolumn{1}{|c|}{} & 128 & \textbf{485.06} & \textbf{-864.77} & \textbf{-932.13} & 1.00 & 1.00 & 387.79 & -603.69 & -713.59 & 123.88 & -120.22 & -201.76 & 71.49 & -15.43 & -96.97 \\ 
\multicolumn{1}{|c|}{} & 256 & \textbf{533.12} & \textbf{-960.89} & \textbf{-1028.24} & 1.00 & 1.00 & 387.79 & -603.69 & -713.59 & 123.44 & -119.34 & -200.88 & 71.49 & -15.43 & -96.97 \\ 
\multicolumn{1}{|c|}{} & 512 & \textbf{602.20} & \textbf{-1099.05} & \textbf{-1166.41} & 1.00 & 1.00 & 387.79 & -603.69 & -713.59 & 123.44 & -119.34 & -200.88 & 71.49 & -15.43 & -96.97 \\ 
\multicolumn{1}{|c|}{} & 1024 & \textbf{602.20} & \textbf{-1099.05} & \textbf{-1166.41} & 1.00 & 1.00 & 387.79 & -603.69 & -713.59 & 123.44 & -119.34 & -200.88 & 71.49 & -15.43 & -96.97 \\ 
\multicolumn{1}{|c|}{} & 2048 & \textbf{673.19} & \textbf{-1241.03} & \textbf{-1308.39} & 1.00 & 1.00 & 387.79 & -603.69 & -713.59 & 123.44 & -119.34 & -200.88 & 71.49 & -15.43 & -96.97 \\ 
\multicolumn{1}{|c|}{} & 4096 & \textbf{640.36} & \textbf{-1175.36} & \textbf{-1242.72} & 1.00 & 1.00 & 387.79 & -603.69 & -713.59 & 123.44 & -119.34 & -200.88 & 71.49 & -15.43 & -96.97 \\ 
\multicolumn{1}{|c|}{} & 8192 & \textbf{640.36} & \textbf{-1175.36} & \textbf{-1242.72} & 1.00 & 1.00 & 387.79 & -603.69 & -713.59 & 123.44 & -119.34 & -200.88 & 71.49 & -15.43 & -96.97 \\ 
\multicolumn{1}{|c|}{} & 16384 & \textbf{743.51} & \textbf{-1381.66} & \textbf{-1449.02} & 1.00 & 1.00 & 387.79 & -603.69 & -713.59 & 123.44 & -119.34 & -200.88 & 73.48 & -19.43 & -100.97 \\ 
\multicolumn{1}{|c|}{} & 32768 & \textbf{743.51} & \textbf{-1381.66} & \textbf{-1449.02} & 1.00 & 1.00 & 387.79 & -603.69 & -713.59 & 123.44 & -119.34 & -200.88 & 73.48 & -19.43 & -100.97 \\ 
\multicolumn{1}{|c|}{} & 65536 & \textbf{758.31} & \textbf{-1411.26} & \textbf{-1478.62} & 1.00 & 1.00 & 387.79 & -603.69 & -713.59 & 123.44 & -119.34 & -200.88 & 73.48 & -19.43 & -100.97 \\ 
\hline
\hline

\end{tabular}
}
\end{table}

\begin{figure}[H]
	\centering
	    \begin{tabular}{c@{}c@{}c@{}c@{}}
    		\includegraphics[width=0.24\textwidth]{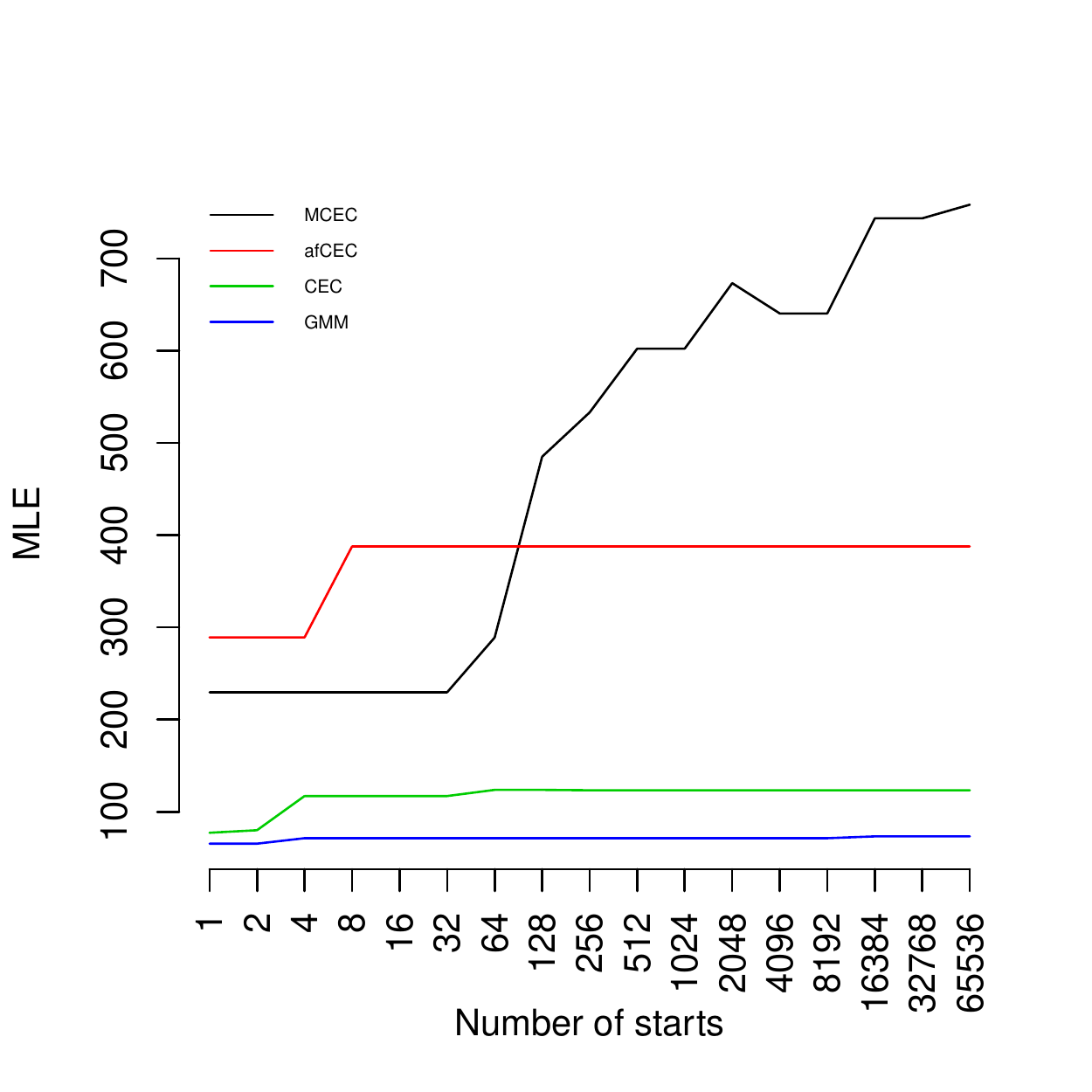} &
    		\includegraphics[width=0.24\textwidth]{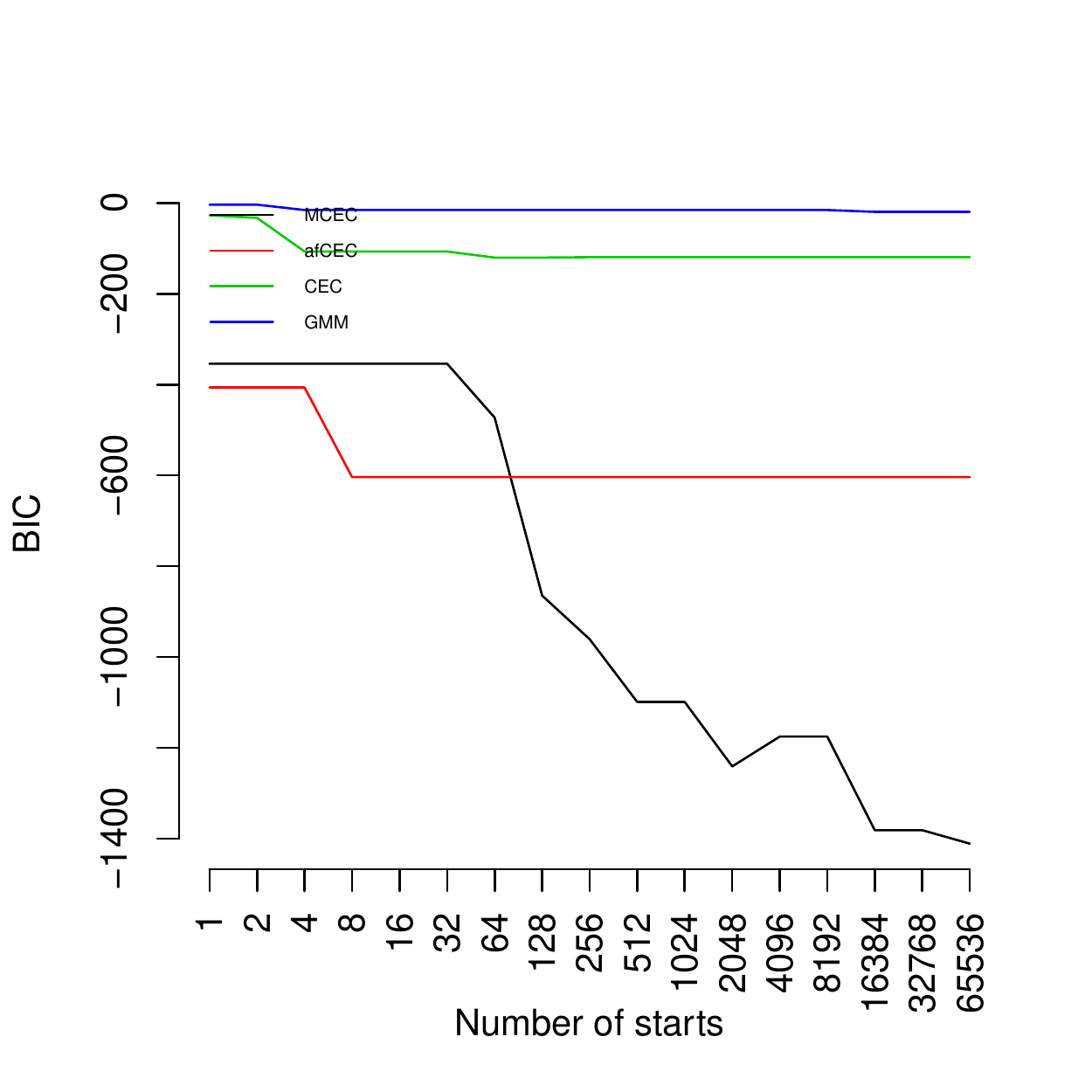} &
    		\includegraphics[width=0.24\textwidth]{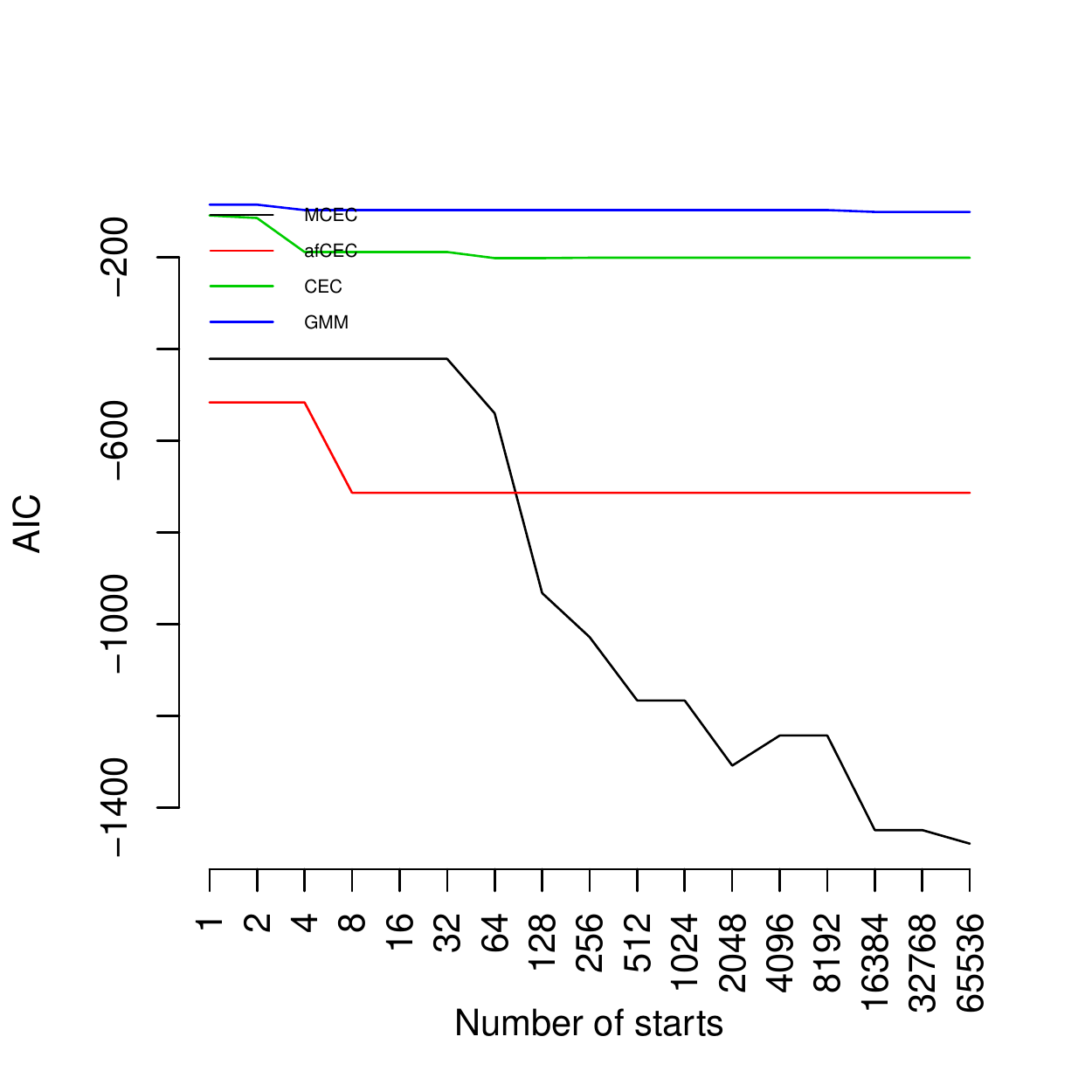} &
    		\includegraphics[width=0.24\textwidth]{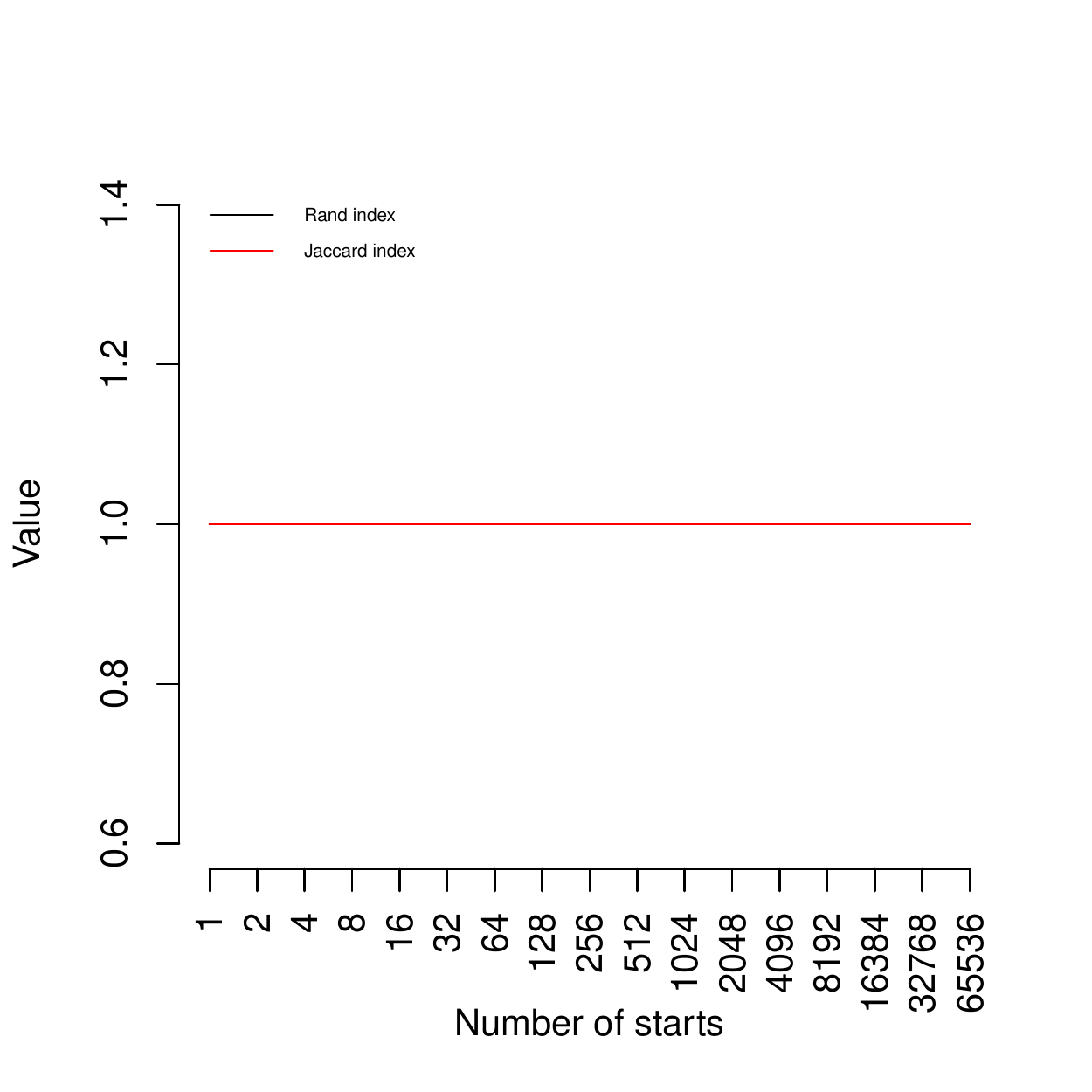} \\
    	\end{tabular}
	\caption{Experiment No. 4: Values of the performance metrics computed for the consecutive numbers of starts being the powers of two and the increasing number of clusters in the data set (the $i$-th row of the table corresponds to the test case with $i$ clusters). The charts in the first, second, and third columns show the computed values of MLE, BIC, and AIC, respectively, for the consecutive number of starts being the powers of two for all four clustering methods used in the comparison, ie. MCEC, afCEC, CEC, and GMM, while the last one represents the values of Rand and Jaccard indices computed similarly for the clustering obtained using the MCEC algorithm.}
\end{figure}


\section*{Acknowledgments}
The work of J. Tabor was supported by the National Centre of Science (Poland) Grant No. 2019/33/B/ST6/00894.
The work of P. Spurek was supported by the National Centre of Science (Poland) Grant No. 2021/43/B/ST6/01456.


\end{document}